\documentclass{article}

\usepackage[preprint]{neurips_2026}

\usepackage[utf8]{inputenc} % allow utf-8 input
\usepackage[T1]{fontenc}    % use 8-bit T1 fonts
\usepackage{url}            % simple URL typesetting
\usepackage{booktabs}       % professional-quality tables
\usepackage{amsfonts}       % blackboard math symbols
\usepackage{nicefrac}       % compact symbols for 1/2, etc.
\usepackage{microtype}      % microtypography
\usepackage{xcolor}         % colors
\definecolor{mydarkred}{rgb}{0.6,0,0}
\definecolor{mydarkgreen}{rgb}{0,0.6,0}
\usepackage[colorlinks,
            linkcolor=mydarkred,
            citecolor=mydarkgreen]{hyperref}
\usepackage{graphicx}
\usepackage{subcaption}
\usepackage{multirow}

\usepackage{algorithm}
\usepackage{algorithmic}
\usepackage{amsmath}
\usepackage{amssymb}
\usepackage{mathtools}
\usepackage{amsthm}
\usepackage{wrapfig}
\usepackage{bm}
\usepackage{tikz}
\usepackage{multicol}
\usepackage{float}
\usepackage{minitoc}
\usepackage{natbib}
\usepackage{nicefrac}       % compact symbols for 1/2, etc.
\usepackage{adjustbox}
\usepackage{colortbl}

\usepackage[capitalize,noabbrev]{cleveref}

\theoremstyle{plain}
\newtheorem{theorem}{Theorem}[section]
\newtheorem{proposition}[theorem]{Proposition}

\theoremstyle{definition}
\newtheorem{definition}[theorem]{Definition}
\newtheorem{assumption}[theorem]{Assumption}
\theoremstyle{remark}

\usepackage[textsize=tiny]{todonotes}

\title{Occupancy-based Quantile Risk Control}

\author{
Zihao Shi\textsuperscript{1}\thanks{Work done while working at SUSTech as a visiting scholar.},\enspace
Huajun Xi\textsuperscript{2},\enspace
Bingyi Jing\textsuperscript{3,4},\enspace
Hongxin Wei\textsuperscript{1}\thanks{Correspondence to: Hongxin Wei <\texttt{weihx@sustech.edu.cn}>}\\
\textsuperscript{1}Southern University of Science and Technology \\
\textsuperscript{2}Mohamed bin Zayed University of Artificial Intelligence \\
\textsuperscript{3}The Chinese University of Hong Kong, Shenzhen
\textsuperscript{4}Shenzhen Loop Area Institute
}

\begin{document}

\maketitle

\begin{abstract}
  Conformal risk control is an emerging framework for the safe deployment of machine learning models with finite-sample guarantees. 
  To accommodate a broader class of risk notions, quantile risk control extends this framework to quantile-based risk measures. 
  However, existing methods either suffer from excessive conservatism or lack rigorous finite-sample guarantees. 
  To address these limitations, we introduce Occupancy-based Quantile Risk Control (OQRC), a novel method that provides tight risk control bounds with finite-sample validity. 
  Our key idea is to formulate risk control as a finite-occupancy problem by partitioning the loss space with the ordered calibration losses. 
  Specifically, we estimate the distribution of test losses across the resulting bins and upper-bound the risk by the maximum loss attained within each bin. 
  We then select the parameter $\lambda$ such that this upper bound does not exceed a predefined threshold $\alpha$ with high probability $1-\delta$. 
  Theoretically, we establish a finite-sample guarantee showing that OQRC yields tight risk control bounds that converge to the optimal bounds at a provable rate of $\mathcal{O}_p(n^{-1/2})$. 
  Extensive experiments demonstrate the effectiveness of our method, reducing the risk gap by up to 78.64\% on common benchmarks. 

  % Our key idea is to reformulate risk control as a finite-occupancy problem by partitioning the loss space using calibration data. 
  % Theoretically, we establish the tightness of OQRC by demonstrating that the resulting bounds converge to the optimal bounds at a rate of $\mathcal{O}_p(n^{-1/2})$, where $n$ is the size of calibration set. 

  % by controlling the expected value of a monotone loss function below a user-specified threshold. 
  % Recent works develop quantile risk control methods that extend the expected loss control to a more general quantile risk setting. 
  % However, existing algorithms often suffer from excessive conservatism or lack finite-sample guarantees. 

  % Theoretically, we demonstrate that OQRC provides rigorous tight and finite-sample guarantees for quantile risk control. 
  % establishes a tight upper bound on risk to satisfy a predefined threshold $\alpha$ with a high probability $1-\delta$,
  % This method establishes a tight upper bound on risk to satisfy a predefined threshold $\alpha$ with a high probability $1-\delta$, providing rigorous finite-sample guarantees. 
  % Extensive experiments demonstrate that OQRC achieves a tighter risk control bounds compared to alternative methods with finite-sample guarantees. 
  % across multiple datasets and risk measures demonstrate that our approach significantly reduces the discrepancy between empirical and target risks, achieving up to a 50\% reduction compared to alternative methods. 
\end{abstract}

\section{Introduction}

Machine learning (ML) models are increasingly deployed in risk-sensitive domains, such as medical diagnostics~\cite{caruana2015intelligible}, finance~\cite{gu2020empirical}, and autonomous driving~\cite{bojarski2016end}. 
Despite their high accuracy, the black-box nature of ML models makes them untrustworthy for users due to the uncontrollable risk of predictions, e.g., imaging-related errors in radiology~\cite{waymel2019impact} and misdiagnoses in healthcare~\cite{esmaeilzadeh2021patients}. 
To this end, Conformal Risk Control (CRC)~\cite{angelopoulos2022conformal} introduces a model-agnostic framework that guarantees the prediction risk remains below a predefined threshold. 
% user-specified
With a monotone loss function, the framework can be applied to control the false-negative rate in classification tasks~\cite{ding2024leveraging} or to improve the factuality of outputs generated by large language models~\cite{jiang2025conformal}. 
By offering such flexible risk guarantees, it presents a rigorous and feasible solution toward the trustworthy deployment of black-box ML models. 
% errors

% makes it obscures the potential risks of prediction mistakes, such as imaging-related errors in radiology~\cite{waymel2019impact} and misdiagnoses in health care~\cite{esmaeilzadeh2021patients}. crisis of trust
% This framework defines the loss as the penalty for the discrepancy between a single prediction and the true target, and formulates the risk as the expected loss over the underlying data distribution. 
% Notably, this approach is training-free and lightweight, which makes it a powerful technique for the safe deployment of machine learning models. 

While elegant and simple, the framework of conformal risk control is limited to controlling the expected loss. 
To address broader risk notions, Quantile Risk Control (QRC)~\cite{snell2022quantile} extends this framework by providing guarantees on quantile-based risk measures, such as value-at-risk~\cite{longerstaey1996riskmetricstm} and conditional value-at-risk~\cite{rockafellar2000optimization}. 
Despite these advancements, this method tends to overestimate prediction risk, leading to conservative results. 
To alleviate this issue, recent work~\cite{chen2025conformal} leverages L-statistics~\cite{mosteller2006some} to derive tight risk control bounds. 
However, their theoretical results rely on the asymptotic normality of L-statistics and offer no guarantees for finite-sample regimes. 
This motivates us to establish a non-conservative risk control method with finite-sample guarantees.

\begin{figure}[t]
    \centering
    \includegraphics[width=\linewidth]{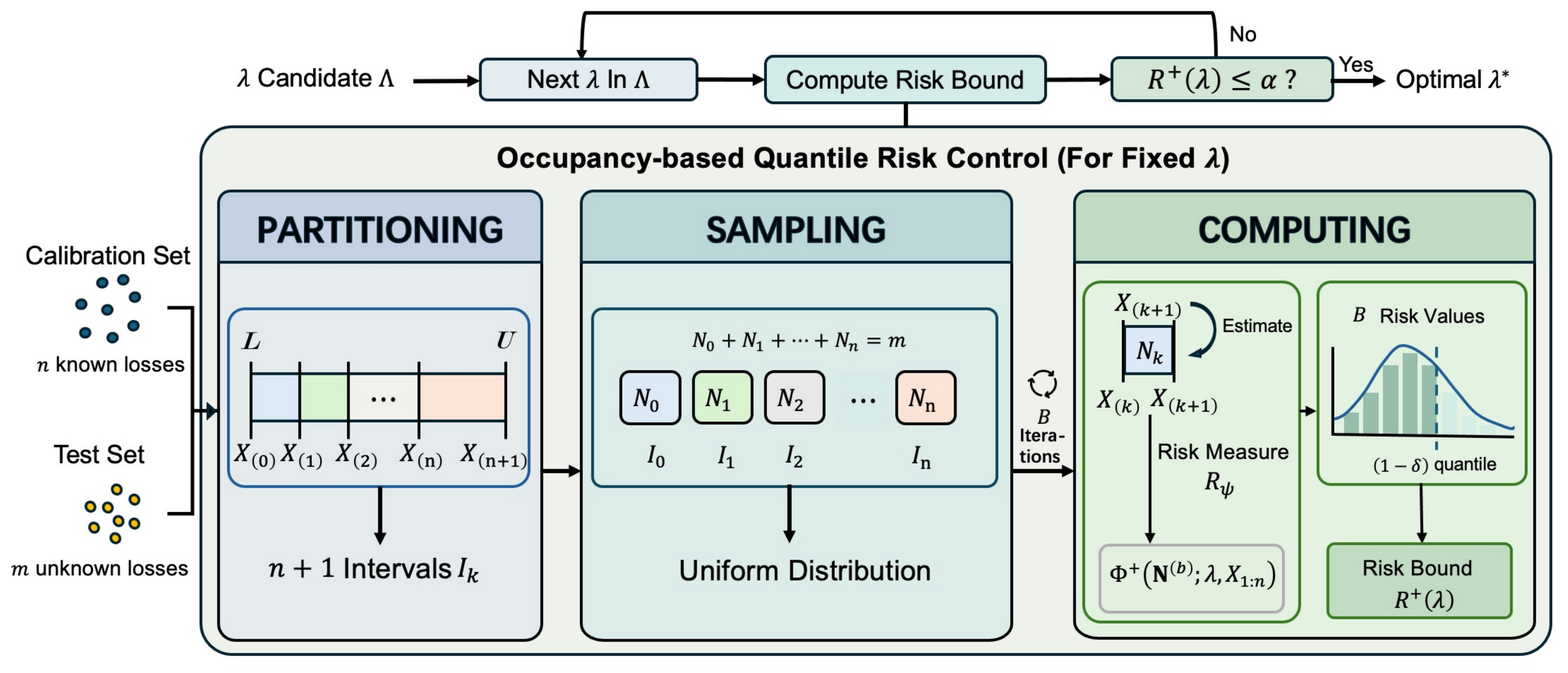}
    \caption{\textbf{Overview of Occupancy-based Quantile Risk Control.} Given a risk level $\alpha$ and confidence $1-\delta$, our OQRC implements a sequential procedure that calibrates a parameter $\lambda^*$ of the ML model such that the quantile risk $R_{\mathrm{test}}$ satisfies $\mathbb{P}(R_{\mathrm{test}}(\lambda^*)\le\alpha)\ge1-\delta$.}
    \label{fig:method}
    \vspace{-5mm}
\end{figure}

In this work, we introduce Occupancy-based Quantile Risk Control (OQRC), a novel risk control method that provides tight risk control bounds with finite-sample guarantees (see Figure~\ref{fig:method}). 
Our key idea is to formulate risk control as a finite-occupancy problem by partitioning the loss space with the ordered calibration losses. 
Specifically, we estimate the distribution of test losses across the resulting intervals and upper-bound the risk by the maximum loss attained within each interval. 
We then select the parameter $\lambda$ such that this upper bound does not exceed a predefined threshold $\alpha$ with high probability $1-\delta$.
Theoretically, we establish a finite-sample guarantee showing that OQRC yields tight risk control bounds that converge to the optimal bounds at a rate of $\mathcal{O}_p(n^{-1/2})$.

To validate the effectiveness of our method, we conduct extensive experiments on seven real-world datasets across three tasks: tumor segmentation~\cite{pogorelov2017kvasir, borgli2020hyperkvasir}, multi-label image classification~\cite{lin2014microsoft}, and text emotion recognition~\cite{demszky2020goemotions}. 
For each task, we encompass three risk measures: VaR, CVaR, and VaR-Interval. 
To assess the conservatism of risk control, we use a new metric -- \textit{RiskGap}, defined as the empirical average of the absolute difference between the empirical and target risks. 
Empirical results show that OQRC significantly reduces the RiskGap compared to prior methods. 
For instance, when using false-negative proportion as the loss function for multi-label image classification, our method reduces the RiskGap by 78.64\% compared to alternative approaches. 
Moreover, OQRC exhibits consistent performance across various hyperparameter configurations and model architectures, highlighting its robustness. 
Finally, we apply OQRC to unconditional control~\cite{angelopoulos2022conformal}, and the results indicate that OQRC successfully bounds the risk below $\alpha$ across different risk measures. 
In summary, these results demonstrate the effectiveness, robustness, and superiority of our method. 
% for quantile risk control. 
% Moreover, through comprehensive analyses, we validate that OQRC is robust to the size of the calibration set, the predefined threshold $\alpha$, the confidence level $1-\delta$ and the model architecture. 

% In particular, a smaller risk gap indicates a tighter risk control. 

% We validate the effectiveness of our method through extensive experiments on tumor segmentation, image multilabel classification, and text emotion recognition. 
% These evaluations encompass three risk measures including VaR, CVaR, and VaR-Interval. 
% We utilize the risk gap, defined as the empirical average of the absolute difference between the empirical risk and the target risk level $\alpha$, to assess the conservatism of risk control methods. 
% A smaller risk gap indicates a tighter control achieved by the method. 
% The results demonstrate that our method achieves a tighter control compared to alternative baselines. 
% For example, in the tumor segmentation task using the false negative proportion as the loss function under the VaR-Interval risk measure, our method reduces the average risk gap by 50\% compared to other approaches. 
% % This property enables the application of our method to more general loss functions. 

We summarize our contributions as follows: 
\begin{itemize}
\item % 方法与优势（两句话）
We introduce Occupancy-based Quantile Risk Control (OQRC), a simple and effective method that reformulates risk control as a finite occupancy problem. 
We demonstrate that OQRC is also applicable to unconditional risk control, highlighting its versatility.
% of our method
% We demonstrate that OQRC is applicable across various model architectures and can accommodate non-monotone loss functions, highlighting the broad applicability of this approach in real-world scenarios. 
% We show that OQRC achieves tight risk control bounds while providing model-agnostic and finite-sample guarantees. 
% a quantile risk control framework that provides finite-sample guarantees while maintaining tight bounds, and we theoretically prove the validity of the framework. 
\item % 理论保证
Theoretically, we establish rigorous guarantees demonstrating that OQRC achieves finite-sample quantile risk control. 
In addition, we prove the tightness of OQRC by analyzing the convergence rate of the corresponding risk control bounds to the optimal bounds. 
% We establish rigorous statistical properties for OQRC, which include finite-sample guarantees of quantile risk control and a comprehensive analysis of the loss distribution. 
% To overcome the computational bottlenecks of exact evaluations within this framework, we subsequently introduce a Monte-Carlo approximation grounded in our theoretical findings. 
% This approach substantially reduces the computational burden, and we establish the validity of the framework as the number of Monte-Carlo iterations approaches infinity. 
\item % 实验（验证实验不算贡献），广泛的实验，验证了方法的优势，分析性实验得出了什么结论，有什么发现
Through extensive empirical validation, we demonstrate that OQRC significantly narrows the risk gap compared to alternative methods. 
Our further evaluations validate the robustness of OQRC across diverse hyperparameter configurations and model architectures. 
% We conduct comprehensive experiments across various datasets and risk measures, demonstrating that OQRC achieves tighter risk control bounds compared to alternative methods. 
% Our further evaluations validate the robustness of OQRC across diverse settings. 
% We empirically validate that OQRC is robustness across various settings. 
% Extensive empirical evaluations demonstrate that the proposed method achieves consistently tighter control than existing baselines. 
\end{itemize}

% 结构问题：第一段最后一句话要落在risk control的重要性上。在第三句话就要提出来risk control。第一句话讲背景，第二句话将challenage。
% risk没说出来是什么东西
% 第一句话逻辑不对，particular写得不对
% 第二句话转折并介绍risk是什么东西（risk没被控制并举个例子）
% 介绍CRC那一句话写得重复了
% 第一段不要写CRC的细节，只需要知道这个的重要性
% 最后一句话与第一句话重复了
% 第一段话重复太多了
% 第一句话引出大背景，第二句话要说risk很重要（转折），第三句话说risk control这个框架，第四句话说这个framework有什么好处，第五句话说risk control很重要

% 第一句话说recent work为了什么目的提出了什么新的目标。第二句话说For example
% 第二段缩写太多了，应该写motivation
% 二三句话稍微介绍QRC
% 第四句话说有什么问题
% 第二段不要超过6句话
% 现在写得太啰嗦了
% 写得时候要清楚每一句话要说什么

% abstract应该是intro每一段对应一两句话

\section{Preliminaries}

\paragraph{Problem Setup.} 
We consider a risk control problem. 
Let $h : \mathcal{Z} \to \widehat{\mathcal{Y}}$ be a black-box predictor that maps an input space $\mathcal{Z}\subset \mathbb{R}^d$ to a prediction space $\widehat{\mathcal{Y}}$. 
We assume a loss function $\ell_\lambda : \widehat{\mathcal{Y}} \times \mathcal{Y} \to [L,U]$ that quantifies the quality of a prediction $\hat{Y}$ with respect to the target output $Y$. 
$\lambda\in\Lambda$ is a tunable parameter within the ML model, used to control the risk level of predictions. 
% The predictions are evaluated using a bounded loss function
% where $\lambda$ is a tunable parameter and $L,U\in\mathbb{R}$. 
Given data $(Z, Y)\in\mathcal{D}$, we define the random variable $X \coloneqq \ell_\lambda(h(Z), Y)$ to be the loss induced by $h$ on $\mathcal{D}$. 
Recall that the cumulative distribution function (CDF) of a random variable $X$ is defined as $F_X(x) \coloneqq \mathbb{P}(X \le x)$. 
The goal of the risk control problem is to choose the parameter $\lambda$ to ensure that 
\begin{equation}\label{eq:risk_control}
    \mathbb{P}\bigl(R(F_X;\lambda) \le \alpha\bigr) \ge 1-\delta. % R(\widehat{F}_m;\lambda
\end{equation}
where $R(\cdot;\lambda)\in\mathcal{R}$ is a risk measure on $F_X$, $\alpha$ is a predefined risk level, and $\delta$ is a confidence level. 

To select a valid $\lambda$, a standard practice is to hold out a calibration set, compute the corresponding losses $X_{1:n}=\{X_1,...,X_n\}$, and subsequently utilize these observed losses to calculate $\lambda$. 
Following the previous work~\cite{angelopoulos2022conformal, Bates2021DistributionFreeRP}, we have the following mild assumptions: 
\begin{assumption}\label{asp:iid}
    The observed losses on the calibration set $X_{1:n}$ and the unknown losses on the test set $X^{\text{test}}_{1:m} = \{X^{\text{test}}_{1}, \dots, X^{\text{test}}_{m}\}$ are independent and identically distributed (i.i.d.). 
\end{assumption}
\begin{assumption}
    Losses in $X_{1:n}\cup X^{\text{test}}_{1:m}$ are non-increasing and left-continuous across $\lambda$. 
\end{assumption}

% We assume the availability of a known calibration set $X_{1:n}=\{X_1,...,X_n\}$ and an unobserved test set , where all loss samples are 

Notably, a strictly minimized $R(F_X;\lambda)$ is not optimal in a risk control problem, since an extremely low risk often incurs a sacrifice of valid information~\cite{angelopoulos2022conformal}. 
Consider a scenario where a model identifies tumor pixels in medical images and the risk represents the proportion of the missed tumor pixels. 
Selecting every pixel in the image minimizes this risk while rendering the output entirely uninformative. 
Therefore, provided the risk control objective is met, a value of $R(F_X;\lambda)$ closer to the target level $\alpha$ is preferred because it leads to predictions containing more useful information. 

\begin{figure}[t]
    \centering
    \begin{minipage}{0.24\linewidth}
        \centering
        \includegraphics[width=\linewidth]{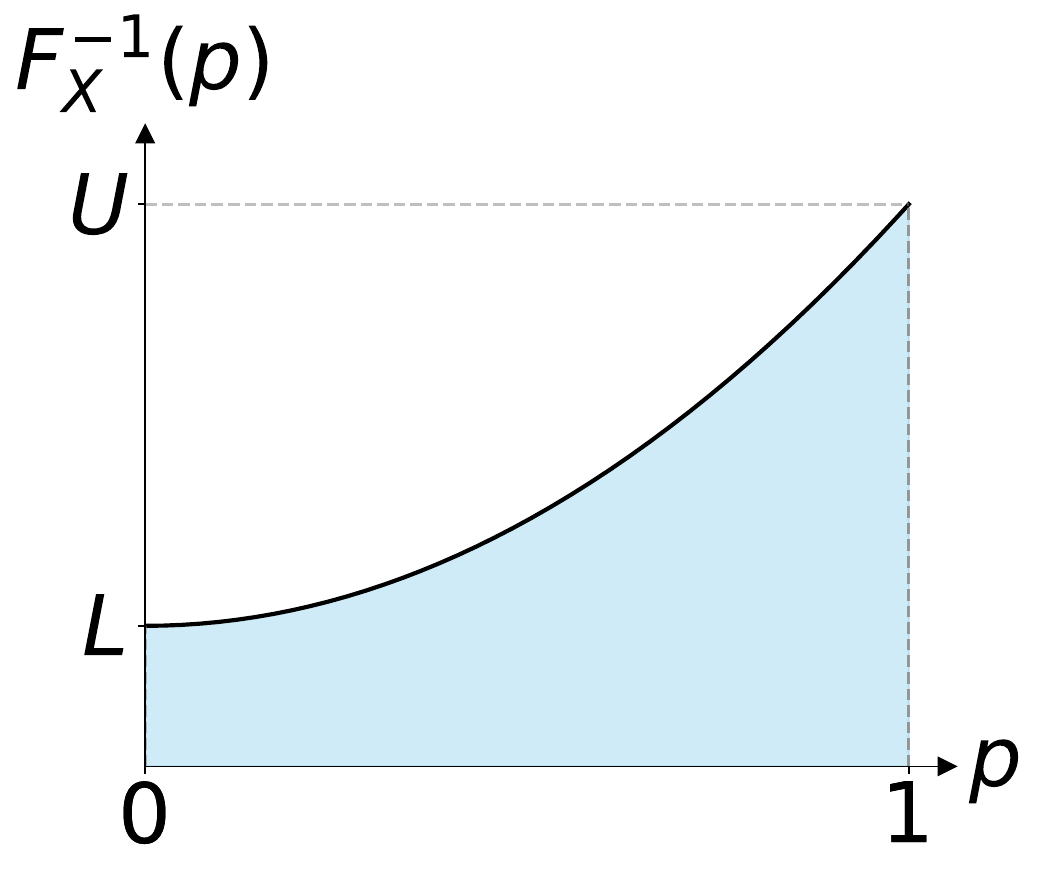}
        \subcaption{Expected loss}
        \label{subfig:Expectation}
    \end{minipage}\hfill
    \begin{minipage}{0.24\linewidth}
        \centering
        \includegraphics[width=\linewidth]{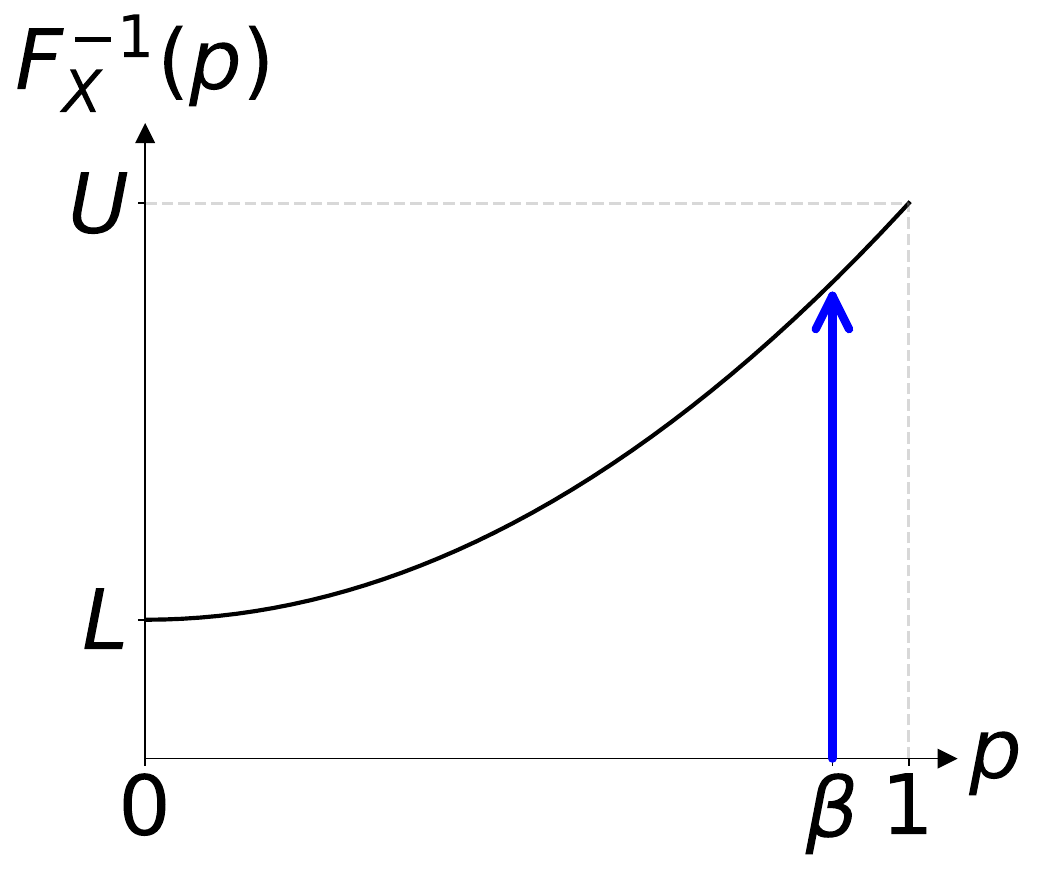}
        \subcaption{$\beta$-VaR}
        \label{subfig:VaR}
    \end{minipage}\hfill
    \begin{minipage}{0.24\linewidth}
        \centering
        \includegraphics[width=\linewidth]{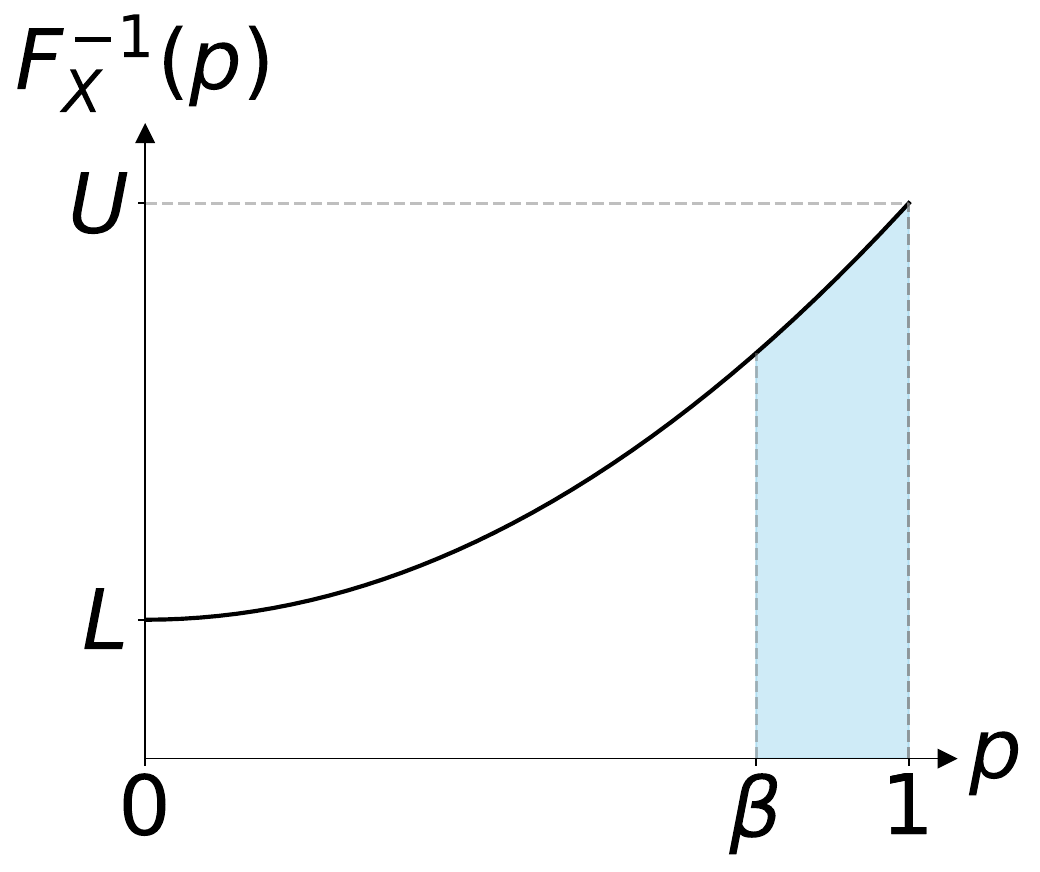}
        \subcaption{$\beta$-CVaR}
        \label{subfig:CVaR}
    \end{minipage}\hfill
    \begin{minipage}{0.24\linewidth}
        \centering
        \includegraphics[width=\linewidth]{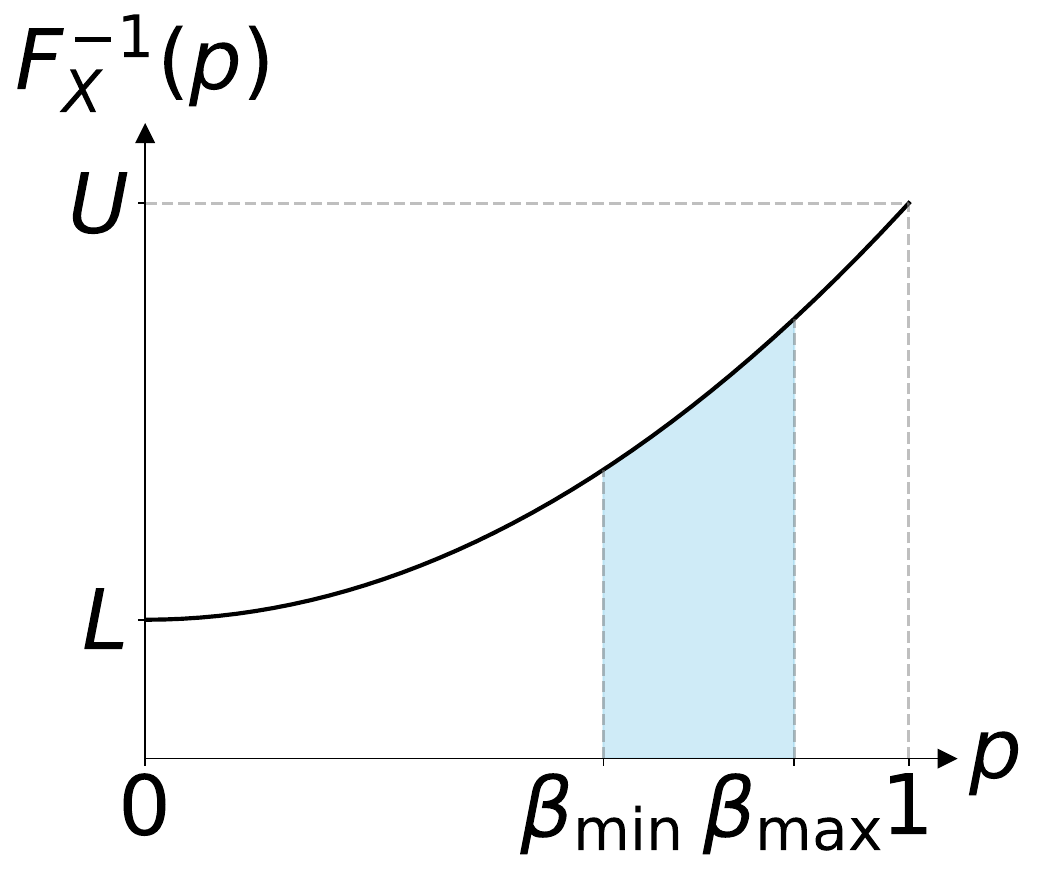}
        \subcaption{VaR-Interval}
        \label{subfig:CVaR-Interval}
    \end{minipage}
    \caption{\textbf{Examples of quantile-based risk measures.} The horizontal axis represents the quantile, and the vertical axis denotes the loss value. Expected loss, $\beta$-CVaR, and $[\beta_{\min}, \beta_{\max}]$-VaR-Interval indicate the average loss within the sky-blue region, while $\beta$-VaR indicates the loss at the $\beta$-quantile.}
    \label{fig:quantile_risk_measure}
\end{figure}

\paragraph{Quantile Risk Control.} 
We focus on a general class of risk measures identified as quantile-based risk measures (QBRMs). The quantile function is defined as
\(
F_X^{-1}(p) \coloneqq \inf\{x : F_X(x) \ge p\}.
\)

\begin{definition}[Quantile-based Risk Measure]
Let $\psi(p)$ denote a weighting function such that
\(\psi(p) \ge 0\) and \(\int_0^1 \psi(p)\,dp = 1\). 
The quantile-based risk measure parameterized by $\psi$ is expressed as
\[
R_{\psi}(F_X;\lambda) \coloneqq \int_0^1 \psi(p) F_X^{-1}(p)\,dp.
\]
\end{definition}

Figure~\ref{fig:quantile_risk_measure} illustrates several representative examples of QBRMs. 
Specifically, the expected loss corresponds to $\psi(p)=1$. 
For the $\beta$-VaR, we define $\psi(p)=\delta_\beta(p)$, where $\delta_\beta$ denotes the Dirac delta function centered at $\beta$. 
For the $\beta$-CVaR, we define $\psi(p)=\mathbf{1}\{p\ge\beta\}\frac{1}{1-\beta}$. 
Finally, the $[\beta_{\min}, \beta_{\max}]$-VaR-Interval is formulated using $\psi(p)=\mathbf{1}\{\beta_{\min}\le p\le\beta_{\max}\}\frac{1}{\beta_{\max}-\beta_{\min}}$. 

% Several representative QBRMs are detailed in Table~\ref{tab:qbrm}. 
% The expected loss captures the overall expectation across the data distribution, whereas the VaR isolates the maximum loss incurred by the majority of the population, thereby excluding a proportion of $(1-\beta)$ as outliers. 
% Conversely, the CVaR focuses on the tail of the distribution that exhibits the highest loss. 
% Finally, the VaR-Interval proves valuable when the precise cutoff for outliers remains unknown. 

% \begin{table}[ht]
% \centering
% \caption{Several quantile-based risk measures and their corresponding weight functions. Here, $\delta_{\beta}$ denotes the Dirac delta function centered at $\beta$. }
% \label{tab:qbrm}
% \begin{tabular}{ll}
% \toprule
% Risk Measure & Weighting Function $\psi(p)$ \\
% \midrule
% Expected loss & $1$ \\[4pt]
% $\beta$-VaR & $\delta_{\beta}(p)$ \\[6pt]
% $\beta$-CVaR &
% $\psi(p)=
% \begin{cases}
% \frac{1}{1-\beta}, & p \ge \beta,\\
% 0, & \text{otherwise}
% \end{cases}$ \\[12pt]
% $[\beta_{\min}, \beta_{\max}]$-VaR-Interval &
% $\psi(p)=
% \begin{cases}
% \frac{1}{\beta_{\max}-\beta_{\min}}, & \beta_{\min} \le p \le \beta_{\max},\\
% 0, & \text{otherwise}
% \end{cases}$ \\
% \bottomrule
% \end{tabular}
% \end{table}

Prior studies have investigated the rigorous, distribution-free control of quantile-based risk measures. 
Snell et al.~\cite{snell2022quantile} achieve quantile risk control by constructing a lower confidence bound $F^-_X$ for $F_X$ through a one-sided test statistic. 
However, this approach often tends to overestimate prediction risk, yielding conservative risk control bounds. 
To alleviate this issue, Chen et al.~\cite{chen2025conformal} control the risk via L-statistics. 
They use L-statistics to estimate the empirical distribution $\widehat{F}_n$ of $F_X$, and then control the risk by adding a correction term to the corresponding risk measures of $\widehat{F}_n$, thereby deriving tighter risk control bounds than before. 
Nevertheless, the theoretical guarantees of this method hold only asymptotically, making the approach potentially fail in finite-sample settings. 
These limitations motivate us to develop a method that provides both tight bounds and finite-sample guarantees. 

% In addition to the post-hoc risk control methods discussed above, some approaches control the risk through model training. 
% We do not compare with these approached, since our method is training-free and lightweight. 
% Because our method is training-free and lightweight, we do not discuss or compare with approaches that control risk via model training, such as conformal risk training~\cite{yeh2025conformal}. % 最该加的地方是往related work里加
% We provide a detailed review of prior quantile risk control methods in Appendix~\ref{sec:othermethod}. 

\section{Method}
\label{sec:method}

We propose a novel method for quantile risk control termed \emph{Occupancy-based Quantile Risk Control} (OQRC). 
The key idea is to reformulate risk control as a finite occupancy problem~\cite{feller1991introduction} by partitioning the loss space using calibration data. 
We use these resulting partitioned bins to construct an upper confidence bound (UCB) $R^+(\lambda)$ such that for every $\lambda$, we have 
% Our method builds upon the framework of risk-controlling prediction sets (RCPS)~\cite{Bates2021DistributionFreeRP}. 
% We need to find an upper confidence bound (UCB) $R^+(\lambda)$ such that for every $\lambda$, we have
\begin{equation}\label{eq:upper_confidence_bound}
    \mathbb{P}(R_\psi(F_X;\lambda)\le R^+(\lambda))\ge 1-\delta.
\end{equation}
Subsequently, we achieve valid risk control (\ref{eq:risk_control}) by identifying $\lambda^*:=\inf\bigl\{\lambda\in\Lambda: R^+(\lambda)\le\alpha\bigr\}$. 

In the following, we present the details of this procedure. 
We define the empirical CDF on test set as $\widehat{F}_m(x)=\frac{1}{m}\sum_{i=1}^{m}\mathbf{1}\{X^{\mathrm{test}}_i \le x\}$. 
The risk on the test set is then given by $R_{\mathrm{test}}(\lambda)=R_{\psi}(\widehat{F}_m;\lambda)$. 

\subsection{Upper Confidence Bound via Partitioned Occupancy}

We first partition the loss space using the calibration losses to construct the bins for the occupancy problem. 
Let $X_{(1)} \le X_{(2)} \le \dots \le X_{(n)}$ denote the order statistics of $X_{1:n}$. 
To simplify the notation in what follows, we let $X_{(0)} = L$ and $X_{(n+1)} = U$. 
Then the order statistics of the loss samples in the calibration set partition the interval $[L, U]$ into $n+1$ bins: 
% \begin{equation}\label{eq:bins}
%     \begin{aligned}
%         &I_j=
%         \begin{cases}
%             \;(X_{(j)}, X_{(j+1)}] & \text{if}\;\;X_{(j)}<X_{(j+1)} \\
%             \;\{X_{(j)}\} & \text{if}\;\;X_{(j)}=X_{(j+1)}
%         \end{cases} \quad \text{for} \;\; j = 1, \dots, n, \\
%         &I_0=
%         \begin{cases}
%             \;[X_{(0)}, X_{(1)}]\quad & \text{if}\;\;X_{(n)}<X_{(n+1)} \\
%             \;\{X_{(n+1)}\} & \text{if}\;\;X_{(n)}=X_{(n+1)}
%         \end{cases}
%     \end{aligned}
% \end{equation}
\begin{equation}\label{eq:bins}
    \begin{cases}
        I_j=\;(X_{(j)}, X_{(j+1)}] \quad\text{for} \;\; j = 1, \dots, n, \\
        I_0=\;[X_{(0)}, X_{(1)}]
    \end{cases}
\end{equation}
Since each test loss sample lies in \([L,U]\), it must fall into exactly one of the \(n+1\) bins in \eqref{eq:bins}. 
For clarity of presentation in the main body of the paper, we assume that $X$ contains NoTies, meaning the probability of any two random variables $X_i$ and $X_j$ taking the same value is zero. 
This assumption is not strictly necessary. 
Appendix~\ref{pf:Distribution} provides the theoretical guarantees without this restriction, which ensures our algorithm remains applicable even when $X$ contains Ties. 

For each bin \(I_j\), we define the occupancy count $N_j = \sum_{t=1}^m \mathbf{1}\!\left\{ X_t^{\mathrm{test}} \in I_j \right\}$, $j=0,1,\dots,n$. 
By construction, \(\sum_{j=0}^n N_j = m.\)
Hence, the vector \(\mathbf{N} = (N_0, N_1, \dots, N_n)\) fully characterizes how the \(m\) test samples are distributed across the bins induced by the \(n\) observed calibration losses. 
We now compute the joint distribution of $\mathbf{N}$. 
A surprising result is that it follows a discrete uniform distribution. 
This theoretical finding serves as the foundation of our method. 
\begin{theorem}\label{thm:Distribution}
    Under the Assumption~\ref{asp:iid}, $\mathbf{N}$ follows a discrete uniform distribution, which is
    $$\mathbb{P}\bigl(N_0=x_0, N_1=x_1, ..., N_n=x_n\bigr)=\frac{m!\cdot n!}{(m+n)!}.$$
    The randomness here arises from the sampling of different calibration and test losses. 
\end{theorem}

The proofs for all theoretical results are provided in Appendix~\ref{sec:proofs}. 
This is an elegant result, as the joint distribution of $N_0, N_1, \dots, N_n$ is entirely independent of the actual values of these variables. 
Leveraging this occupancy count, we next construct the UCB $R^+(\lambda)$. 
Let
\[
\mathcal{C}_{m,n}
=
\left\{
\mathbf{x}=(x_0,\dots,x_n)\in \mathbb{N}_0^{n+1}
:
\sum_{i=0}^{n} x_i = m
\right\}.
\]
Let $S_{i}(\mathbf{x}) = \sum_{k=0}^{i} x_k$ denote the cumulative sums, where $S_{-1}$ is defined as $0$. 
By Theorem~\ref{thm:Distribution}, the random vector $\mathbf{N}$ is uniformly distributed over $\mathcal{C}_{m,n}$, and \(
|\mathcal{C}_{m,n}|=\binom{m+n}{n}. 
\)
We further define the upper risk functionals associated with $\mathbf{x}$ by 
\begin{equation*}
% \label{eq:func-upper}
\Phi^+(\mathbf{x}; \lambda,X_{1:n})
=
\sum_{i=0}^{n}
\int_{\frac{S_{i-1}(\mathbf{x})}{m}}^{\frac{S_{i}(\mathbf{x})}{m}}
\psi(t)X_{(i+1)}\,dt. 
% =
% \sum_{i=0}^{n}\sum_{j=1}^{x_i}
% \int_{\frac{S_{i-1}(\mathbf{x})+j-1}{m}}^{\frac{S_{i-1}(\mathbf{x})+j}{m}}
% \psi(t)X_{(i+1)}\,dt. 
\end{equation*}
Intuitively, these upper risk functionals are defined to scale each test loss sample falling into $I_j$ to the maximum value within that bin, and then compute the overall risk across the test samples. 
Utilizing this upper risk functional, we define a CDF of the upper risk: 
\[
F^+(r;\lambda,X_{1:n})
=
\frac{1}{|\mathcal{C}_{m,n}|}
\sum_{\mathbf{x}\in\mathcal{C}_{m,n}}
\mathbf{1}\!\left\{\Phi^+(\mathbf{x}; \lambda,X_{1:n})\le r\right\}. 
\]
Thus, defining the UCB as $R^+(\lambda)=(F^+)^{-1}(1-\delta)$, we have 
\begin{theorem}\label{thm:QRC}
    Given a confidence level $1-\delta \in (0, 1)$, $\forall \lambda$, $R^+(\lambda)$ is a valid UCB that satisfies (\ref{eq:upper_confidence_bound}). 
    % \begin{equation}\label{eq:quantile-OQRC}
    %     \mathbb{P}\Bigl(R_{\mathrm{test}}(\lambda)\le R^+(\lambda)\Bigr)\ge1-\delta. 
    % \end{equation}
    % The randomness here arises from the sampling of different calibration and test losses. 
\end{theorem}
% Suppose $F^m_{\mathrm{test}}$ is the ECDF of the test losses, defined over $m$ data points in the test set. 
% Consequently, the risk measure of this ECDF can be expressed as $R_{\mathrm{test}} := R_{\psi}(F^m_{\mathrm{test}})$. 
% In the following, we provide an estimation for the CDF of $R_{\mathrm{test}}$. 
% For CDFs $F$ and $G$, let $F\succeq G$ denote $F(x)\ge G(x)$ for all $x\in \mathbb{R}$. 
% Let $\mathbf{x} \coloneqq (x_0, \dots, x_n)$ and let $S_{i}(\mathbf{x}) = \sum_{k=0}^{i} x_k$ denote the cumulative sums, where $S_{-1}$ is defined as $0$. 
% By utilizing the joint distribution of $N_0, N_1, \dots, N_n$ from Theorem~\ref{thm:Distribution}, we can derive the upper bounds for the CDF of $R_{\mathrm{test}}$. 
% \begin{theorem}\label{thm:upperbound}
%     Given the calibration losses $X_{1:n}$, we define $R_{\mathrm{test}}^{\mathrm{upper}}$ as:
%     $$\mathbb{P}\Bigl(R_{\mathrm{test}}^{\mathrm{upper}}=\sum_{i=0}^{n}\sum_{j=1}^{x_i}\int_{\frac{S_{i-1}(\mathbf{x})+j-1}{m}}^{\frac{S_{i-1}(\mathbf{x})+j}{m}}\psi(t)X_{(i+1)}dt\Bigr)=\frac{m!\cdot n!}{(m+n)!}.$$
%     Let $F_{\mathrm{test}}$ denote the CDF of $R_{\mathrm{test}}$. We define $F^{\mathrm{upper}}_{\mathrm{test}}$ as the CDF of $R_{\mathrm{test}}^{\mathrm{upper}}$. Then we have:
%     $$F_{\mathrm{test}}\succeq F^{\mathrm{upper}}_{\mathrm{test}}.$$
% \end{theorem}
% Then Theorem~\ref{thm:upperbound} allows us to establish a bound for the QRBMs of the test loss samples. 

% This property allows OQRC to maintain its validity across a broader class of general loss functions. 

However, as $m$ and $n$ increase, calculating the CDF of the upper risk $F^+(r;\lambda,X_{1:n})$ becomes computationally demanding. 
To address this limitation for large values of $m$ and $n$, we introduce a Monte-Carlo approach to approximate $F^+(r;\lambda,X_{1:n})$ in the following subsection. 

\subsection{Monte-Carlo Approach for OQRC}

We utilize a Monte Carlo method to mitigate the computational bottlenecks that arise when $m$ and $n$ are excessively large. 
We first introduce the proposed sampling approach, which exploits a stars-and-bars construction to efficiently sample from $\mathcal{C}_{m,n}$. 

\begin{proposition}
\label{prop:stars-bars}
Let $T_1<\cdots<T_n$ be drawn uniformly from all $n$-subsets of $\{1,\dots,m+n\}$. 
Define $T_0=0$ and $T_{n+1}=m+n+1$, and set
\[
x_i = T_{i+1}-T_i-1,
\qquad i=0,1,\dots,n.
\]
Then $\mathbf{x}=(x_0,\dots,x_n)$ is uniformly distributed on $\mathcal{C}_{m,n}$.
\end{proposition}

Based on the sampler yielded by Proposition~\ref{prop:stars-bars}, we generate \ Monte-Carlo replicates
\[
\mathbf{x}^{(1)},\dots,\mathbf{x}^{(B)}
\overset{\mathrm{i.i.d.}}{\sim}
\mathrm{Unif}(\mathcal{C}_{m,n}),
\]
where $B$ is the number of Monte-Carlo samples. 
For each replicate $\mathbf{x}^{(b)}$, we yield an empirical estimator 
$\widehat{F}^+_{B}(r;\lambda,X_{1:n})
=
\frac{1}{B}\sum_{b=1}^{B}
\mathbf{1}\!\left\{
\Phi^{+}(\mathbf{x}^{(b)}; \lambda,X_{1:n})\le r
\right\}.$
The next result demonstrates that the Monte-Carlo approximation converges uniformly to $F^+(r;\lambda,X_{1:n})$. 

\begin{theorem}
\label{thm:mc-concentration}
For any $\varepsilon>0$,
\[
\mathbb{P}_{\mathrm{MC}}
\left(
\sup_{r\in\mathbb{R}}
\left|
\widehat{F}^+_{B}(r;\lambda,X_{1:n})-F^+(r;\lambda,X_{1:n})
\right|
>\varepsilon
\right)
\le 2e^{-2B\varepsilon^2}. 
\]
\end{theorem}

Theorem~\ref{thm:mc-concentration} is a direct consequence of the DKW inequality~\cite{dvoretzky1956asymptotic, massart1990tight} applied to the i.i.d.\ Monte-Carlo samples $\{\Phi^{+}(\mathbf{x}^{(b)}; \lambda, X_{1:n})\}_{b=1}^{B}$. 
% In the following, we present the algorithmic procedure for the Monte-Carlo version of OQRC (Algorithm~\ref{alg:porc_mc} in Appendix~\ref{sec:algorithm}). 
% Dvoretzky--Kiefer--Wolfowitz
We summarize this procedure to Algorithm~\ref{alg:porc_mc} in Appendix~\ref{sec:algorithm}. 

\subsection{Other theoretical results on OQRC}\label{subsec:tight_theory}
% \label{subsec:choose_lambda}

\paragraph{Calibrate the $\lambda$.} % In this subsection,
We next select a valid $\lambda$ to achieve quantile risk control. 
% In practical applications, we typically control the risk by tuning a parameter $\lambda$. 
Adopting the choice of Bates et al.~\cite{Bates2021DistributionFreeRP}, we define $\lambda^*:=\inf\bigl\{\lambda\in\Lambda: R^+(\lambda)\le\alpha\bigr\}.$
% , \forall \lambda'\ge\lambda
Then it leads to the following theorem. 
\begin{theorem}\label{thm:calibrate_lambda}
    Given a risk level $\alpha$ and a confidence level $1-\delta$, we have
    \begin{equation*}
        \mathbb{P}\Bigl(R_{\mathrm{test}}(\lambda^*)\le \alpha\Bigr)\ge1-\delta. 
    \end{equation*}
\end{theorem}
% We take the tumor segmentation task as an example. 
% The model input $Z$ is a $d\times d$ image, and the output $Y$ denotes the probability of each pixel is a tumor, corresponding to an output space of $[0,1]^{d\times d}$. 
% We utilize $\lambda$ to extract a pixel subset $\mathcal{C}_\lambda(Z)=\bigl\{h(Z)_y\ge 1-\lambda\bigr\}$, which represents the tumor region predicted by the model. 
% We define the loss function as the false negative proportion (FNP) of this pixel subset, expressed as
% \begin{equation}\label{eq:false_negative_proportion}
%     l_\lambda^{\mathrm{FNP}}(Z)=1-\frac{|Y \cap \mathcal{C}_\lambda(Z)|}{|Y|}.
% \end{equation}
% Given a target risk level $\alpha$ and a confidence level $1-\delta$, we aim to identify an appropriate parameter $\lambda$ such that
% \begin{equation}\label{eq:risk_control}
%     \mathbb{ P} \Bigl(R_{\mathrm{test}}(\lambda)\le \alpha \Bigr)\ge 1-\delta.
% \end{equation}
% We typically assume that the set $\mathcal{C}_\lambda$ is nested with respect to $\lambda$, meaning
% $$\lambda_1\le \lambda_2 \rightarrow \mathcal{C}_{\lambda_1}\subset \mathcal{C}_{\lambda_2}.$$
% This objective is formulated as
We formalize this selection process for $\lambda$ as Algorithm~\ref{alg:lambda_calibration} in Appendix~\ref{sec:algorithm}. 

\paragraph{Tightness of OQRC.}
We then show that OQRC achieves tight control by proving that the risk control bounds obtained from OQRC converge to the optimal bounds at a rate of $\mathcal{O}_p(n^{-1/2})$. 
For any $\lambda$, we define the optimal $1-\delta$ bound for risk control as $M^*(\lambda) := \inf \{M: \mathbb{P}(R_{\mathrm{test}}(\lambda) \le M) \ge 1-\delta\}$. 
Under this optimal bound, selecting $\lambda' = \inf\{\lambda \in \Lambda: M^*(\lambda) \le \alpha\}$ yields a value that coincides with $\lambda'' = \inf\{\lambda \in \Lambda: \mathbb{P}(R_{\mathrm{test}}(\lambda) \le \alpha)\ge1-\delta\}$, indicating the optimality of $M^*$. 
% , which justifies the optimality of $M^*$ as a bound. 
Let $f_X$ denote the probability density function of $X$. 
% with its support defined as $\mathrm{supp}(f_X) = \overline{\{x : f_X(x) \neq 0\}}$. 
% In what follows, 
We establish the following asymptotic property. 
% that $R^+(\lambda)$ converges to $M^*(\lambda)$. 

\begin{theorem}\label{thm:tightness}
    For a fixed $\lambda\in\Lambda$ and a fixed $m$, assume $0<c\le f_X\le C<\infty$,
    for $x\in[L,U]$ and some constants $c, C>0$. Then we have
    % If $\;\forall \lambda\in\Lambda$, $f_X$ does not vanish on $[L,U],$ then we have 
    % Assume the probability distribution function of $X$ does not vanish on $[L, U]$ across $\lambda$, then $\forall \lambda\in\Lambda$ we have 
    $$|R^+(\lambda) - M^*(\lambda)|=\mathcal{O}_p(n^{-1/2}).$$
\end{theorem}

% The condition of Theorem~\ref{thm:tightness} is a mild assumption, which we empirically verify in Appendix~\ref{sec:tightness_theory}. 
This theorem demonstrates that as the calibration set size $n$ increases, the risk bounds obtained by OQRC asymptotically converge to the optimal bounds, thereby proving the tightness of our method. 
To the best of our knowledge, we are the first to develop a tight quantile risk control method. % a method that achieves

\begin{table}[t]
\centering
\caption{\textbf{Performance of quantile risk control methods on 7 datasets across 3 tasks for 2 quantile-based risk measures.} Since CDRC-L fails to achieve valid quantile risk control, the values of RiskGap and AvgSize for this method are not reported. The best results are shown in \textbf{bold}.}
\renewcommand\arraystretch{1.1} 
\setlength{\tabcolsep}{5mm}     
\resizebox{1\textwidth}{!}{   
\begin{adjustbox}{max width=\textwidth}
\begin{tabular}{*{6}{c}}
    \toprule
    \multirow{2}*{Dataset} & \multirow{2}*{Risk Measure} & \multirow{2}*{Methods} & \multicolumn{3}{c}{Metrics} \\
    \cmidrule(lr){4-6}
    & &                     & Cov   & RiskGap $(\downarrow)$   & AvgSize $(\downarrow)$   \\ \hline
    \multirow{8}*{Polyp}            & \multirow{4}*{CVaR}
      & OrderStats          & 1.000 & 0.366$\pm$0.018          & 47.021$\pm$2.637         \\
    & & One-sided BJ        & 1.000 & 0.324$\pm$0.034          & 41.907$\pm$3.775         \\
    & & CDRC-L              & 0.844 & -                        & -                        \\
    & & \textbf{OQRC(Ours)} & 0.935 & \textbf{0.120$\pm$0.062} & \textbf{20.061$\pm$3.441}\\
    \cmidrule(lr){2-6}
                                    & \multirow{4}*{VaR-Interval}
      & OrderStats          & 1.000 & 0.362$\pm$0.022          & 41.299$\pm$4.188         \\
    & & One-sided BJ        & 1.000 & 0.332$\pm$0.027          & 34.926$\pm$4.162         \\
    & & Two-sided BJ        & 1.000 & 0.331$\pm$0.026          & 34.729$\pm$4.364         \\
    & & \textbf{OQRC(Ours)} & 0.928 & \textbf{0.168$\pm$0.074} & \textbf{17.216$\pm$3.410}\\
    \midrule
    \multirow{8}*{MS COCO}          & \multirow{4}*{CVaR}
      & OrderStats          & 1.000 & 0.339$\pm$0.012          & 7.612$\pm$0.234          \\
    & & One-sided BJ        & 1.000 & 0.325$\pm$0.024          & 6.845$\pm$0.866          \\
    & & CDRC-L              & 0.841 & -                        & -                        \\
    & & \textbf{OQRC(Ours)} & 0.959 & \textbf{0.082$\pm$0.050} & \textbf{3.345$\pm$0.255} \\
    \cmidrule(lr){2-6}
                                    & \multirow{4}*{VaR-Interval}
      & OrderStats          & 1.000 & 0.390$\pm$0.020          & 7.024$\pm$0.851          \\
    & & One-sided BJ        & 1.000 & 0.310$\pm$0.074          & 5.127$\pm$0.982          \\
    & & Two-sided BJ        & 1.000 & 0.309$\pm$0.075          & 5.090$\pm$0.935          \\
    & & \textbf{OQRC(Ours)} & 0.904 & \textbf{0.066$\pm$0.041} & \textbf{3.115$\pm$0.237} \\
    \midrule
    \multirow{8}*{Go Emotions}      & \multirow{4}*{CVaR}
      & OrderStats          & 1.000 & 0.372$\pm$0.030          & 15.712$\pm$4.184         \\
    & & One-sided BJ        & 1.000 & 0.336$\pm$0.042          & 11.092$\pm$2.586         \\
    & & CDRC-L              & 0.844 & -                        & -                        \\
    & & \textbf{OQRC(Ours)} & 0.925 & \textbf{0.139$\pm$0.074} & \textbf{5.439$\pm$0.838} \\
    \cmidrule(lr){2-6}
                                    & \multirow{4}*{VaR-Interval}
      & OrderStats          & 1.000 & 0.399$\pm$0.001          & 12.035$\pm$2.975         \\
    & & One-sided BJ        & 1.000 & 0.398$\pm$0.002          & 9.397$\pm$2.028          \\
    & & Two-sided BJ        & 1.000 & 0.398$\pm$0.012          & 9.316$\pm$2.001          \\
    & & \textbf{OQRC(Ours)} & 0.937 & \textbf{0.254$\pm$0.107} & \textbf{5.018$\pm$0.743} \\
  \bottomrule
\end{tabular}
\end{adjustbox}
}
\label{tab:main}
\vspace{-3mm}
\end{table}

\section{Experiments}\label{sec:EXP}

\subsection{Experimental setups}

\paragraph{Datasets and setup.} 
We evaluate the performance of OQRC across three tasks: tumor segmentation, image multi-label classification, and text emotion recognition. 
The tumor segmentation evaluation pools data from several open-source gut polyp segmentation datasets, specifically Kvasir~\cite{pogorelov2017kvasir}, Hyper-Kvasir~\cite{borgli2020hyperkvasir}, CVC-ColonDB, CVC-ClinicDB, and ETIS-Larib. 
This extracted collection comprises 798 data points consisting of images and their corresponding tumor masks. 
In the multi-label classification setting, we utilize the MS COCO dataset~\cite{lin2014microsoft}, which includes 80 object classes and allows multiple classes to appear within a single instance. 
We sample 2000 data points from this source. 
The text emotion recognition evaluation employs the Go Emotions dataset~\cite{demszky2020goemotions}, which contains Reddit comments annotated with binary labels for 28 fine-grained emotion categories. 
Each individual comment can possess multiple positive labels. 
This subset includes 5427 examples. 
% The detailed setups for these tasks are provided in Appendix~\ref{subsec:task}. 
Across all evaluations, we uniformly configure the size of the calibration set to 200 and the size of the test set to 500. 
We use the false negative proportion (FNP) as our loss function, with a detailed explanation in Appendix~\ref{subsec:loss}. 
We repeat each experiment 1000 times to report the average OQRC performance. 
% which is non-increasing with respect to $\lambda$. 

\paragraph{Model.} 
We employ PraNet~\cite{fan2020pranet} for tumor segmentation, ResNet50~\cite{he2016deep} for image multi-label classification, and BERT~\cite{devlin2019bert} for text emotion recognition. 
To validate the robustness of our method across various model architectures, we evaluate four different architectures on the image multi-label classification task, which include ResNet~\cite{he2016deep}, EfficientNet~\cite{tan2019efficientnet}, TResNet~\cite{ridnik2021tresnet}, and ConvNeXt~\cite{liu2022convnet}. 
All models are fine-tuned on the corresponding datasets, with training details in Appendix~\ref{subsec:training}. 

\paragraph{Compared methods.} 
We present five methods in total: Order Statistics (\textbf{OrderStats}), truncated Berk-Jones (\textbf{One-sided BJ}, \textbf{Two-sided BJ}), conformal distortion risk control via L-statistics (\textbf{CDRC-L}), and our method. 
The Two-sided BJ is specifically designed for VaR-Interval, while CDRC-L is tailored for CVaR and VaR. 
Therefore, these two methods are only evaluated in experiments with specific risk measures. 
Detailed explanations of these methods are provided in Appendix~\ref{sec:othermethod}. 

% \paragraph{Loss function.} 
% % We use the false negative proportion (FNP) as the loss function for all experimental results 
% We use the false negative proportion (FNP) as our loss function, which is non-increasing with respect to $\lambda$. 
% We provide a detailed explanation of FNP in Appendix~\ref{subsec:loss}. 
% The expression for the FNP is given as follows: 
% \begin{equation*}
%     \ell^{\mathrm{FNP}}_{\lambda}(Z)=1-\frac{|Y \cap \mathcal{C}_\lambda(Z)|}{|Y|}. 
% \end{equation*}
% \mathrm{max}\{|\mathcal{C}_\lambda(Z)|, 1\}
% except for Table 5, where the false discovery proportion (FDP) is employed instead. 
% The expression for FNP is given in Equation~\ref{eq:false_negative_proportion}, while the expression for FDP is defined as follows: 

\paragraph{Evaluation metrics.} 
We evaluate the performance of quantile risk control methods using three specific metrics. 
The first metric is the coverage (\textbf{Cov}), defined as the empirical probability that the test risk remains under $\alpha$ across experimental trials. 
A method is considered valid if the coverage exceeds $1-\delta$, and invalid otherwise. 
The second metric is the average risk gap (\textbf{RiskGap}), defined as the empirical average of the absolute difference between the empirical and target risks. 
This metric assesses the conservatism of the method. 
Assuming that the method is valid, a larger RiskGap indicates a more conservative approach, which implies that a smaller gap corresponds to tighter control. 
Finally, we utilize the average prediction set size (\textbf{AvgSize}), defined as the average length of the prediction sets. 
% to measure the magnitude of the generated prediction sets. 
Similar to the RiskGap, when the method is valid, a smaller AvgSize implies tighter control of the risk. 
We provide a detailed explanation of these metrics in Appendix~\ref{subsec:metrics}. 

\subsection{Main results}

\begin{figure}[t]
    \centering
    \begin{minipage}{0.33\linewidth}
        \centering
        \includegraphics[width=\linewidth]{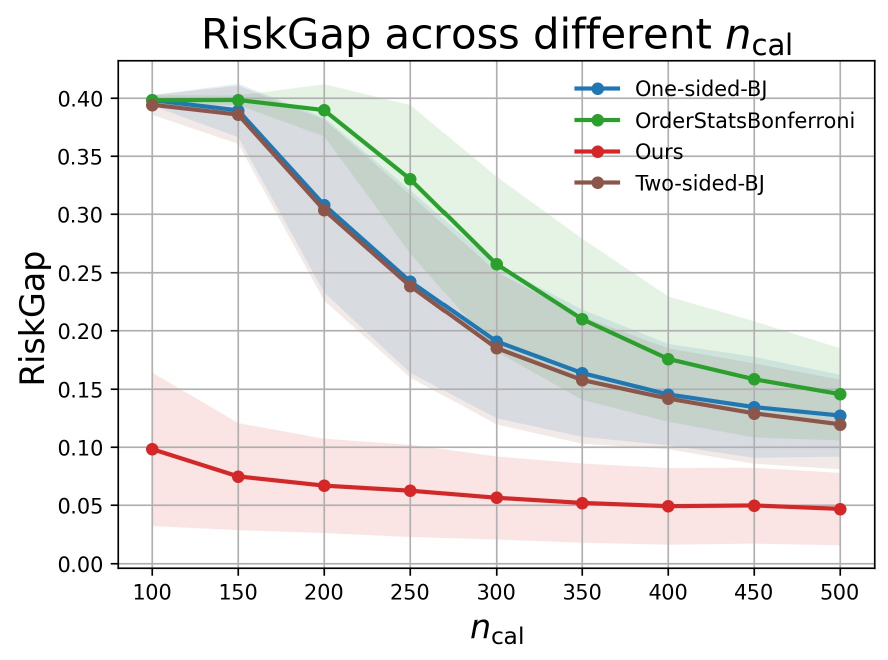}
        % \subcaption{RiskGap across different $n_{\mathrm{val}}$}
        \label{subfig:alpha_riskgap}
    \end{minipage}\hfill
    \begin{minipage}{0.33\linewidth}
        \centering
        \includegraphics[width=\linewidth]{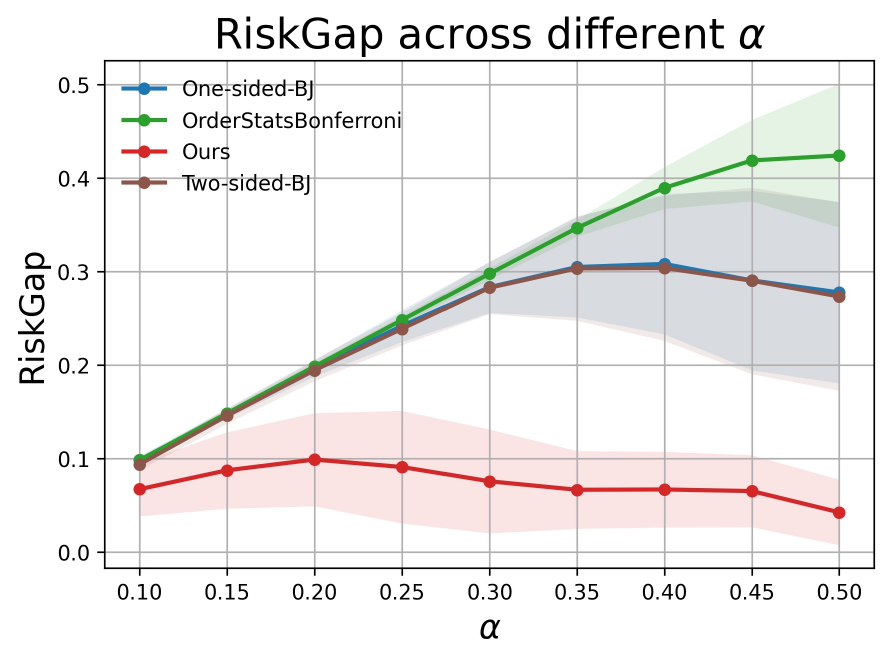}
        % \subcaption{RiskGap across different $\alpha$.}
        \label{subfig:delta_riskgap}
    \end{minipage}\hfill
    \begin{minipage}{0.33\linewidth}
        \centering
        \includegraphics[width=\linewidth]{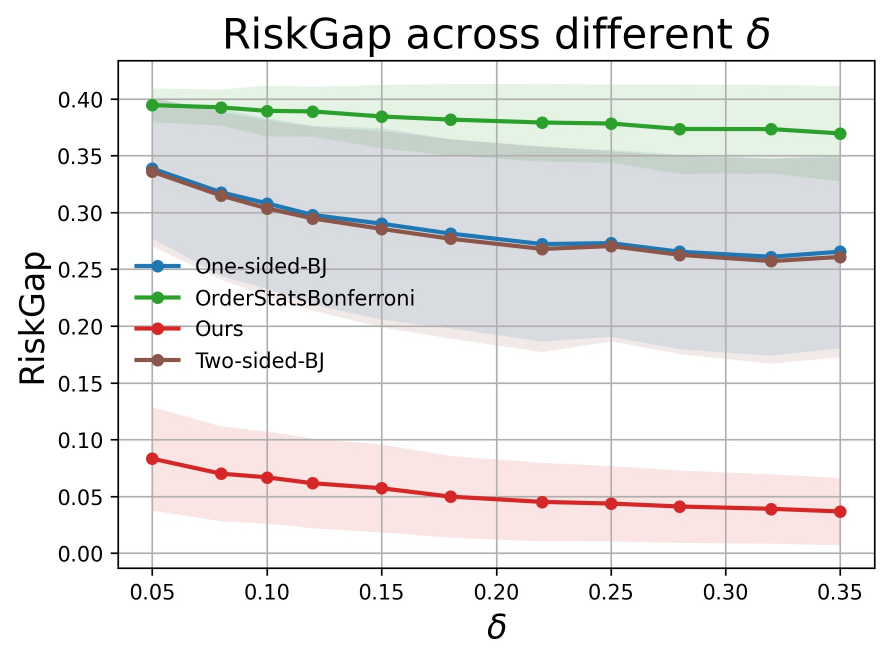}
        % \subcaption{}
        \label{subfig:n_val_riskgap}
    \end{minipage}\\
    \begin{minipage}{0.33\linewidth}
        \centering
        \includegraphics[width=\linewidth]{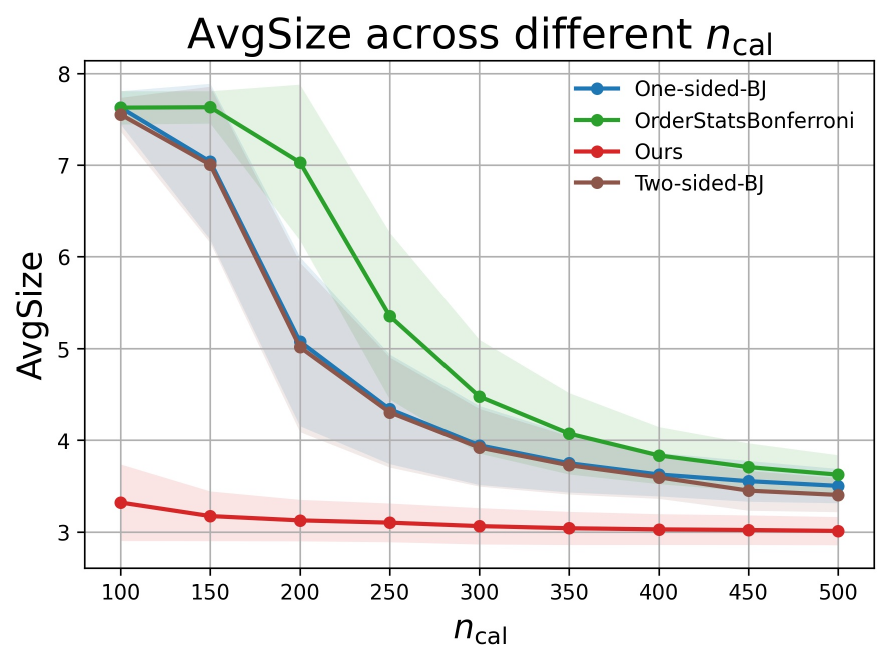}
        % \subcaption{}
        \label{subfig:alpha_setsize}
    \end{minipage}\hfill
    \begin{minipage}{0.33\linewidth}
        \centering
        \includegraphics[width=\linewidth]{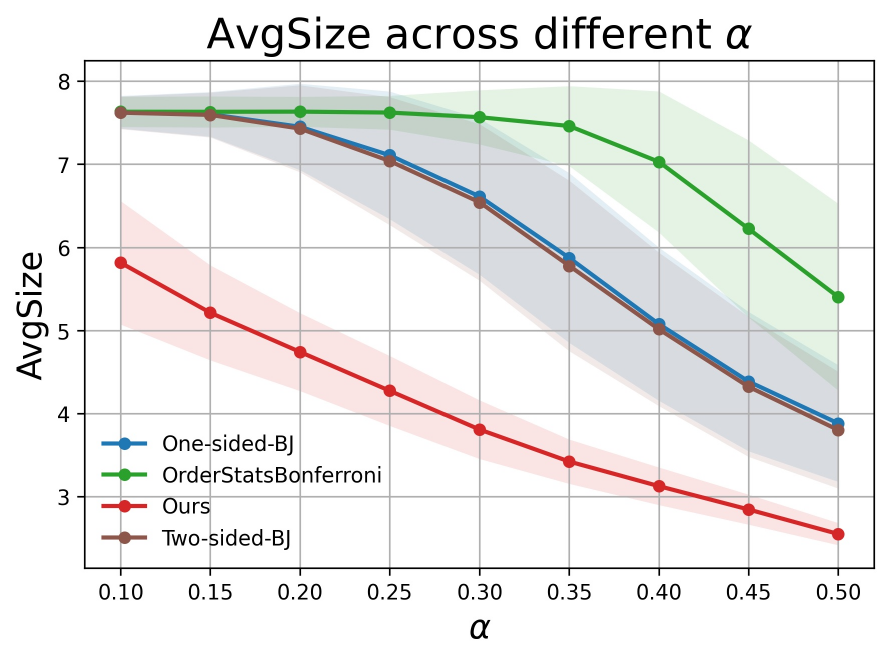}
        % \subcaption{}
        \label{subfig:delta_setsize}
    \end{minipage}\hfill
    \begin{minipage}{0.33\linewidth}
        \centering
        \includegraphics[width=\linewidth]{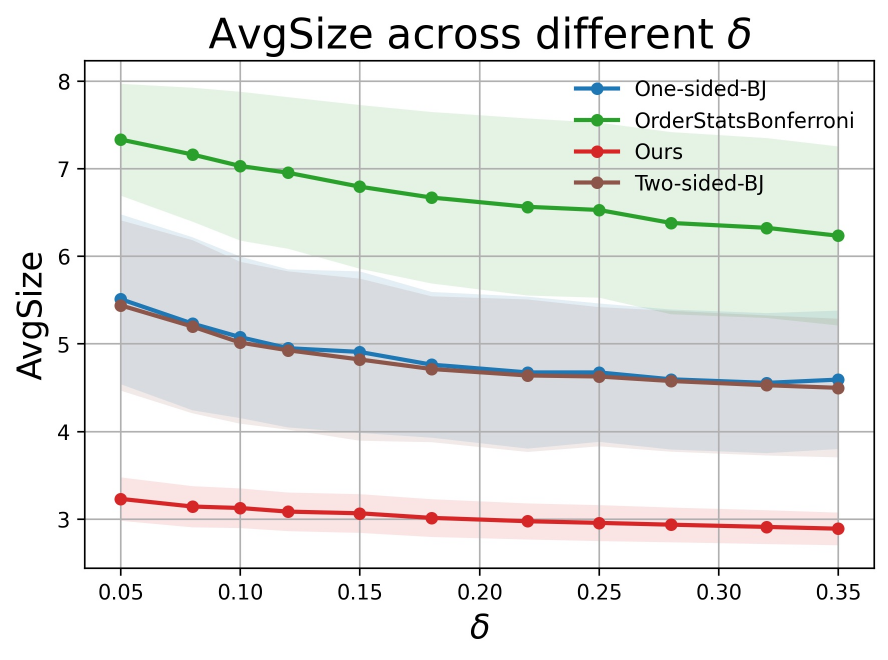}
        % \subcaption{}
        \label{subfig:n_val_setsize}
    \end{minipage}
    \caption{\textbf{Performance comparison across various hyperparameter configurations.} The top three figures compare the RiskGap, while the bottom three figures compare the AvgSize. Due to the highly similar performance of One-sided BJ and Two-sided BJ, their curves nearly overlap in the figures.}
    \label{fig:parameter_analysis}
    \vspace{-3mm}
\end{figure}

\paragraph{OQRC achieves tighter risk control while maintaining valid coverage.} 
Table~\ref{tab:main} compares our proposed method with several baseline methods across coverage, risk gap, and average prediction set size by using the CVaR and the VaR-Interval as risk measures. 
In the experimental configuration, $\beta$ is set to $0.8$ for the CVaR, while $\beta_{\text{min}}$ and $\beta_{\text{max}}$ for the VaR-Interval are $0.85$ and $0.95$, respectively. 
The target risk level $\alpha$ is fixed at $0.4$, and the confidence level $1-\delta$ is $0.9$. 
The results demonstrate that our method yields tighter control than OrderStats, One-sided BJ, and Two-sided BJ, which translates to smaller values for the RiskGap and the AvgSize. 
For instance, on the MS COCO dataset under the VaR-Interval measure, our approach reduces the RiskGap by $78.64\%$ and the AvgSize by $38.80\%$ compared to Two-sided BJ. 
Meanwhile, CDRC-L fails to satisfy the $1-\delta$ coverage requirement in any of the three tasks, suggesting that the method lacks valid risk control when the size of the calibration set is small. 
We also observe that the empirical coverage of OQRC modestly exceeds the predefined confidence level $0.9$, an experimental result that appears inconsistent with the theoretical intuition presented in Subsection~\ref{subsec:tight_theory}. 
Our detailed analysis in Appendix~\ref{sec:tightness_theory} demonstrates that these empirical findings remain consistent with our theoretical results. 
% This discrepancy arises because the theoretical framework of this method relies on asymptotic results that hold only as the size of the sample grows sufficiently large. 
% In contrast, OQRC provides finite-sample guarantees, ensuring robust performance even with limited calibration data. 

\paragraph{OQRC maintains its effectiveness under various hyperparameter configurations.} 
In Figure~\ref{fig:parameter_analysis}, we compare the quantile risk control performance across different sizes of the calibration set, target risk levels $\alpha$, and confidence levels $1-\delta$. 
We conduct experiments on the MS COCO dataset under the [0.85, 0.95]-VaR-Interval risk measure. 
The results show that OQRC achieves tighter risk control than the other baseline methods. 
For instance, compared to alternative approaches, our approach reduces the RiskGap by 82.54\% and the AvgSize by 44.66\% on average across varying values of $\delta$. 
Appendix~\ref{subsec:results_hyperparameter} provides additional experiments that involve other risk measures and datasets. 

\begin{figure}[t]
    \centering
    \includegraphics[width=\linewidth]{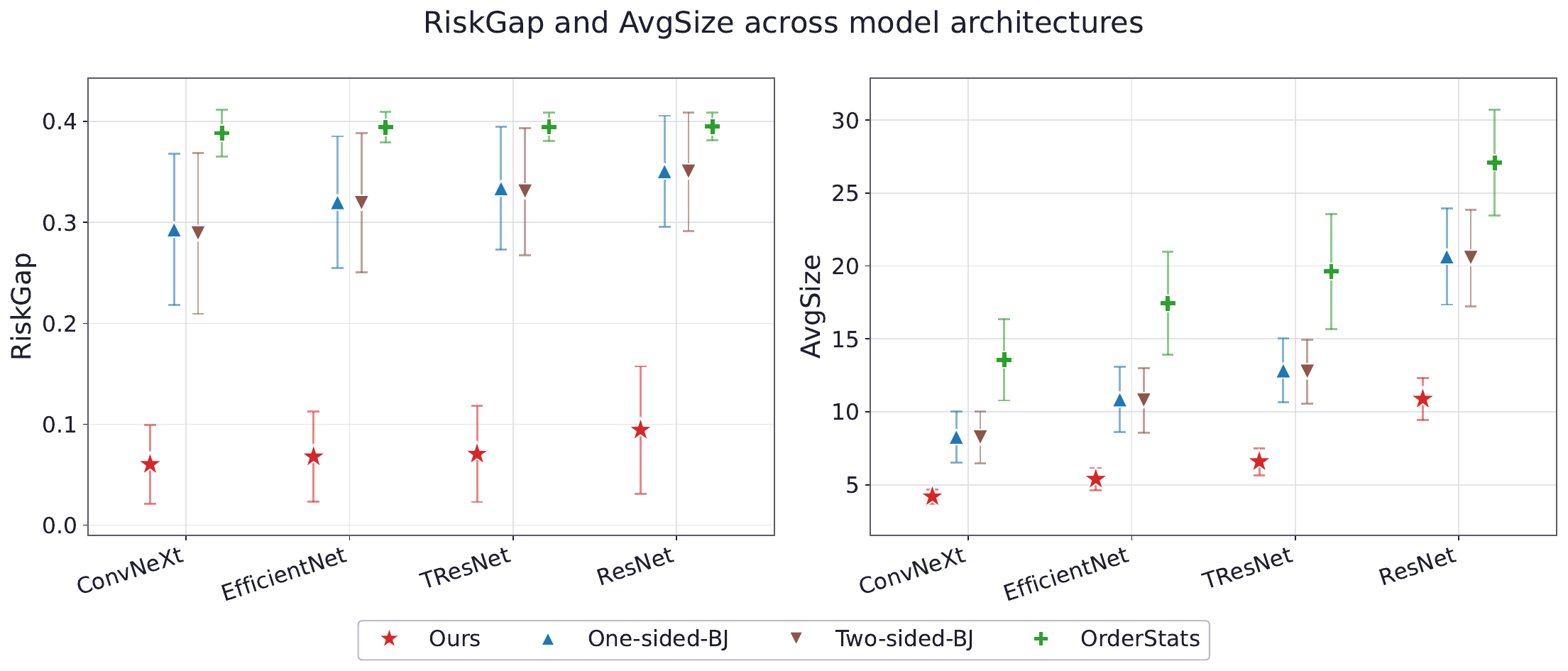}
    \caption{Performance comparison across different model architectures.}
    \label{fig:different_model}
    \vspace{-5mm}
\end{figure}

\paragraph{OQRC performs effectively across different model architectures.} 
Figure~\ref{fig:different_model} compares the performance of our approach against baseline methods under various model architectures. 
We conduct experiments on the MS COCO dataset under the [0.85, 0.95]-VaR-Interval risk measure. 
Across all these models, OQRC decreases the average RiskGap from 0.347 to 0.075 and reduces the average AvgSize from 15.226 to 6.849, consistently outperforming other baseline methods. 
These results highlight the model-agnostic nature of our method. 
Appendix~\ref{subsec:results_model} provides further performance comparisons between our approach and the baseline methods by using alternative risk measures. 

% \paragraph{OQRC is still valid under general loss functions.} 
% Table 3 evaluates our method by using the FDP as the loss function. 

\section{Discussion}\label{sec:discussion}

\paragraph{Unconditional quantile risk control.} 
Achieving unconditional quantile risk control, which requires ensuring that $R_\psi(F_X) \le \alpha$, is a challenging problem~\cite{angelopoulos2022conformal}. 
Yeh et al.~\cite{yeh2025conformal} propose a training-based solution to achieve unconditional control over the optimized certainty equivalent (OCE) risks. 
However, training-free methods that achieve unconditional control over QBRMs are still lacking. 
Our method offers a solution to this problem, leading to the following theorem. 
\begin{theorem}\label{thm:expected_result}
    For any $\lambda$, the following inequality holds: 
    $$\mathbb{E}\bigl[R_{\mathrm{test}}(\lambda)\bigr] \le \mathbb{E}\bigl[\Phi^+(\mathbf{x}; \lambda,X_{1:n})\bigr],$$
    where the expectation is taken over the randomness of calibration and test losses. 
\end{theorem}
% \begin{theorem}\label{thm:expected_result}
%     For any $\lambda$, the following inequality holds:
%     $$\mathbb{E}_{m,n}\bigl[R_{\mathrm{test}}(\lambda)\bigr] \le \mathbb{E}_{m,n}\bigl[\mathbb{E}_{\mathbf{x}}[\Phi^+(\mathbf{x}; \lambda,X_{1:n})]\bigr],$$
%     where the expectation $\mathbb{E}_{m,n}$ is taken over the randomness of calibration and test losses. 
% \end{theorem}
Consequently, we select $\lambda$ as $\lambda^*:=\mathrm{inf}\{\lambda\in\Lambda:\mathbb{E}\bigl[\Phi^+(\mathbf{x}; \lambda,X_{1:n})]\le\alpha\}$, then we achieve the unconditional quantile risk $\mathbb{E}[R_{\mathrm{test}}(\lambda^*)]\le\alpha$. 
%selecting an appropriate parameter $\lambda$ such that $\mathbb{E}\bigl[\Phi^+(\mathbf{x}; \lambda)\bigr] \le \alpha$ is sufficient to obtain the unconditional control of QBRMs. 
The procedure for selecting $\lambda$ is detailed in Subsection~\ref{subsec:tight_theory}. 
As demonstrated in Table~\ref{tab:unconditional_quantile_risk}, our method successfully bounds $\mathbb{E}\bigl[R_{\mathrm{test}}\bigr]$ strictly below the threshold $\alpha$ across three distinct tasks under various predefined values of $\alpha$. 
% We conduct all experiments using the FNP loss function, with each result averaged over $1000$ independent trials. 
% and two different risk measures

\paragraph{Relation to conformal risk control.} 
We next demonstrate that the conformal risk control method~\cite{angelopoulos2022conformal} is a special case of Theorem~\ref{thm:expected_result}. 
By setting $m = 1$ and selecting the expected loss as the risk measure, the space $\mathcal{C}$ is formed by $(n+1)$-dimensional one-hot vectors. 
Let $\mathbf{x}$ be a one-hot vector where the $r$-th component is 1 and all other components are 0. 
Under this formulation, we obtain $\Phi^+(\mathbf{x}; \lambda,X_{1:n}) = X_{(r+1)}$. 
It follows that $\mathbb{E}\bigl[\Phi^+(\mathbf{x}; \lambda,X_{1:n})\bigr] = \sum_{i=0}^{n}\frac{1}{n+1}X_{(i+1)}$. 
Given that $R_{\mathrm{test}} = X^{\mathrm{test}}$, we substitute these relations into Theorem~\ref{thm:expected_result} and choose $\lambda^* := \inf\bigl\{\lambda \in \Lambda : \sum_{i=0}^{n}\frac{1}{n+1}X_{(i+1)}(\lambda) \le \alpha\bigr\}$. 
Then we recover the CRC procedure: 
$$\mathbb{E}\bigl[X^{\mathrm{test}}(\lambda^*)\bigr] \le \alpha.$$

\begin{table}[t] % !ht
    \centering
    \begin{small}
    \begin{tabular}{llccllcc}
    \toprule
        & Dataset & Risk Level & Empirical Risk &  & Dataset & Risk Level & Empirical Risk \\ 
        \midrule
        \multirow{9}{*}{\rotatebox[origin=c]{90}{CVaR}} & \multirow{3}{*}{Polyp} & $\alpha=0.4$ & 0.379 (+0.021)  & \multirow{9}{*}{\rotatebox[origin=c]{90}{Interval}} & \multirow{3}{*}{Polyp} & $\alpha=0.4$ & 0.382 (+0.018) \\ % ± 0.086 ± 0.141
        ~ & & $\alpha=0.2$ & 0.181 (+0.019) & & & $\alpha=0.2$ & 0.185 (+0.015) \\ % ± 0.053 ± 0.078
        ~ & & $\alpha=0.1$ & 0.077 (+0.023) & & & $\alpha=0.1$ & 0.091 (+0.009) \\ % ± 0.027 ± 0.038
        \cmidrule(lr){2-4}\cmidrule(lr){6-8} 
        ~ & \multirow{3}{*}{MS COCO} & $\alpha=0.4$ & 0.385 (+0.015) & & \multirow{3}{*}{MS COCO} & $\alpha=0.4$ & 0.393 (+0.007) \\ % ± 0.040 ± 0.040
        ~ & & $\alpha=0.2$ & 0.178 (+0.022) & & & $\alpha=0.2$ & 0.189 (+0.011) \\ % ± 0.046 ± 0.060
        ~ & & $\alpha=0.1$ & 0.081 (+0.019) & & & $\alpha=0.1$ & 0.091 (+0.009) \\ % ± 0.028 ± 0.054
        \cmidrule(lr){2-4}\cmidrule(lr){6-8} 
        ~ & \multirow{3}{*}{Go Emotions} & $\alpha=0.4$ & 0.387 (+0.013) & & \multirow{3}{*}{Go Emotions} & $\alpha=0.4$ & 0.369 (+0.031) \\ % ± 0.101 ± 0.186
        ~ & & $\alpha=0.2$ & 0.186 (+0.014) & & & $\alpha=0.2$ & 0.179 (+0.021) \\ % ± 0.072 ± 0.144
        ~ & & $\alpha=0.1$ & 0.088 (+0.012) & & & $\alpha=0.1$ & 0.078 (+0.022) \\ % ± 0.051 ± 0.096
        \bottomrule
    \end{tabular}
    \end{small}
    \vspace{3mm}
    \caption{\textbf{Unconditional quantile risk control performance via OQRC.} The values in $(\cdot)$ denote the difference between Risk Level and Empirical Risk, where a positive value indicates valid control.} % unconditional risk
    % The results demonstrate that OQRC controls risk below $\alpha$ across various $\alpha$ levels.
    % If the empirical risk falls below the risk level, the unconditional control is considered achieved. 
    % We report the experimental results as Empirical Risk. 
    \label{tab:unconditional_quantile_risk}
    \vspace{-3mm}
\end{table}

\paragraph{The limiting case as $m \to \infty$.} 
We further discuss the OQRC method when $m \to \infty$. 
This condition arises when controlling the risk over the entire loss distribution. 
As $m \to \infty$, we obtain unconditional quantile risk control $R_\psi(F_X;\lambda)\le\alpha$. 
A detailed analysis is presented in Appendix~\ref{sec:asymptotic}. 
% A natural and practically important extension of OQRC is to achieve quantile risk control across the entire loss distribution, which corresponds to taking the test size $m \to \infty$. 
% A detailed analysis is presented in Appendix~\ref{sec:asymptotic}. 

\section{Related Work} 

\paragraph{Distribution-free uncertainty quantification.} 
This line of research originates from conformal prediction~\cite{papadopoulos2002inductive, vovk2005algorithmic, shafer2008tutorial}, which generates prediction sets that contain ground-truth labels with a desired coverage guarantee. 
Subsequent research~\cite{angelopoulos2022conformal, lei2018distribution, angelopoulos2025learn} extends conformal prediction to accommodate more general loss functions. 
These methods formalize the expected loss associated with the output of the model as the risk, which quantifies the potential unreliability of the predictions. 
Such a risk control framework is widely applied in domains including large language models~\cite{jazbec2024fast, wang2025sconu, wang2026coin}, computer vision~\cite{belhasin2023principal, teneggi2023trust, mossina2024conformal}, and medicine~\cite{lysenkova2025error, dai2025statistical, boger2025functional}. 
In this work, we propose a novel method for distribution-free uncertainty quantification to control a general notion of risk beyond expected loss. 

\paragraph{Quantile risk control.} 
In some real-world scenarios, it is necessary to control the tail of the loss distribution beyond the expected loss~\cite{krokhmal2002portfolio, dowd2010using, ghasemi2021cvar, farzaneh2024quantile, wang2025voltage, zhong2026proxy}. 
To control this general notion of risk, Snell et al.~\cite{snell2022quantile} propose a quantile risk control framework. 
Their method employs a one-sided test statistic to control the quantile risk, but such a procedure tends to overestimate the prediction risk and thus produce conservative results. 
Chen et al.~\cite{chen2025conformal} achieve tighter risk control via $L$-statistics, yet their guarantees rely on asymptotic properties as the calibration set size grows to infinity and do not provide finite-sample guarantees. 
To the best of our knowledge, we are the first to propose a quantile risk control method that provides both finite-sample guarantees and theoretically tight control. 
% Snell et al. address this limitation by extending the previous risk control framework to incorporate quantile risk, which establishes a foundation for numerous subsequent studies~\cite{}. 
% However, this method constructs an overly conservative upper bound, which potentially overestimates the risk associated with the predictions of the model. 
% The theoretical validity of this work relies on the assumption of an asymptotically infinite calibration set, a condition that renders the guarantees of risk control potentially ineffective when the number of calibration points is limited. 
% Consequently, we propose a framework for quantile-based risk control that ensures both tight control and finite-sample guarantees. 

\paragraph{Other works on risk control.} 
Some works control the risk through model training methods~\cite{yeh2025conformal, lee2020learning, duchi2021learning, leqi2022supervised, pmlr-v259-he25a}. 
These approaches formulate risk as a learning objective, training the model to either minimize the prediction risk or constrain it below a user-specified threshold. 
OQRC is a post-hoc method without model training, which is computationally more efficient than these approaches. 

% \paragraph{Relaxing Assumptions on Loss Functions.} 
% The original risk control framework requires the loss function to be monotonically non-increasing with respect to the parameter $\lambda$~\cite{angelopoulos2022conformal, chen2025conformal, lei2018distribution}. 
% Recent studies~\cite{angelopoulos2025learn, angelopoulos2026conformal, aldirawi2026non} attempt to relax this monotonicity condition. 
% However, these efforts remain restricted to the expected loss as the sole risk measure. 
% Our framework naturally accommodates non-monotonic loss functions and applies to a more generalized risk measure. 

\section{Conclusion}\label{sec:conclusion}
In this work, we propose Occupancy-based Quantile Risk Control (OQRC), a novel training-free method that achieves finite-sample quantile risk control.  
The key idea of OQRC is to formulate risk control as a finite-occupancy problem by partitioning the loss space using ordered calibration losses. 
Theoretically, we establish the tightness of OQRC by proving that the risk control bounds yielded by OQRC converge to the optimal bounds. 
Extensive empirical evaluations validate the effectiveness of OQRC. 
Our approach also resolves the problem posed by Angelopoulos et al.~\cite{angelopoulos2022conformal}, enabling users to control the prediction risk without configuring a $\delta$. 
We hope that our approach provides valuable insights and methodologies for controlling quantile risk across a broader range of scenarios. 
% We theoretically establish the tightness of OQRC by proving that the risk control bounds yielded by OQRC converge to the optimal bounds at a rate of $\mathcal{O}_p(n^{-1/2})$. 
% Extensive empirical evaluations demonstrate that OQRC yields tighter risk control bounds than existing approaches across various hyperparameters and model architectures. 
% Subsequently, by evaluating the joint distribution of the counts of the $m$ test points that fall into the partitioned bins, our method constructs an upper confidence bound for quantile risk. 

\paragraph{Limitation.} 
Our method is established under the i.i.d. assumption, whereas practical applications may encounter distribution shifts. 
We leave the investigation of this extended setting to future work.

\newpage
{\small
\bibliographystyle{unsrt}
\bibliography{bibliography}
}

\newpage
\appendix

\onecolumn

\section{Proofs}\label{sec:proofs}

\subsection{Proof of Theorem~\ref{thm:Distribution}}\label{pf:Distribution}

\begin{proof}
    Our proof draws inspiration from the work of Gazin et al.~\cite{gazin2024transductive} 
    The theoretical novelty of this proof lies in applying the methodology of conformal p-values~\cite{saunders1999transduction, bates2023testing, jin2023selection} to the risk control framework and relaxing the previous `NoTies' assumption, thereby making it more applicable to real-world scenarios. 
    The first two steps are proven under the assumption that $X$ has `NoTies', while the third step relaxes this `NoTies' assumption. 
    
    \textbf{Step 1: Prove that $\mathbb{P}(\cap_{i=1}^{m} \{X_i^{\mathrm{test}}\in I_{j_i}\})=\frac{n!\Pi_{k=0}^{n}N_k!}{(m+n)!}$, where $N_k=\sum_{i=1}^m \mathbf{1}\{j_i=k\}$.} % \mid I_{j_{1:m}}
    
    Let $X_{1:n+m}=X_{1:n}\cup X^{\text{test}}_{1:m}$, where the first $n$ values are loss samples from the calibration set, and the remaining $m$ values are loss samples from the test set. 
    Let $\sigma$ be a permutation such that 
    $$X_{\sigma(1)} < X_{\sigma(2)} < \dots < X_{\sigma(n+m)}.$$ 
    Thus, on the event $\cap_{i=1}^{m} \{X_i^{\mathrm{test}}\in I_{j_i}\}$, the permuted sequence $X_{1:n+m}$ under $\sigma$ can be uniquely represented as a partition of the following form: 
    \[
    \begin{aligned}
    &\underbrace{
     X_{\sigma(1)} < \cdots < X_{\sigma(N_0)}
    }_{N_0 \; \text{test loss}}
    <
    \underbrace{
    X_{\sigma(N_0+1)} 
    }_{\text{calibration loss}}
    <\\
    &\underbrace{
    X_{\sigma(N_0+2)} < \cdots < X_{\sigma(N_0+N_1+1)}
    }_{N_1 \; \text{test loss}}
    <
    \underbrace{
    X_{\sigma(N_0+N_1+2)}
    }_{\text{calibration loss}}
    <
    \cdots
    <\\
    &\underbrace{
    X_{\sigma(\sum_{i=0}^{n-1}N_i+n)} < \cdots < X_{\sigma(\sum_{i=0}^{n}N_i+n)}
    }_{N_{n} \; \text{test loss}}
    <
    U.
    \end{aligned}
    \]
    % The above probability is calculated conditional on knowing which interval each test loss sample falls into. 
    
    Since the permutations of $\{1, 2, \dots, n+m\}$ are uniformly distributed, the probability of the event $\cap_{i=1}^{m} \{X_i^{\mathrm{test}}\in I_{j_i}\}$ follows the classical probability model, yielding: 
    $$\mathbb{P}(\cap_{i=1}^{m} \{X_i^{\mathrm{test}}\in I_{j_i}\})=\frac{n!N_0!\cdot N_1!\cdot...\cdot N_n!}{(m+n)!}=\frac{n!\Pi_{k=0}^{n}N_k!}{(m+n)!}.$$
    
    \textbf{Step2: Prove that $\mathbb{P}\bigl(N_0=x_0, N_1=x_1, .., N_n=x_n\bigr)=\frac{m!\cdot n!}{(m+n)!}.$} 
    
    \begin{align*}
        \mathbb{P}\bigl(N_0=x_0, N_1=x_1, .., N_n=x_n\bigr) &= \sum_{\forall k, \; (\sum_{i=1}^m \mathbf{1}\{j_i=k\})=x_k}\mathbb{P}(\cap_{i=1}^{m} \{X_i^{\mathrm{test}}\in I_{j_i}\}) \\
        &= \sum_{\forall k, \; (\sum_{i=1}^m \mathbf{1}\{j_i=k\})=x_k}\frac{n!\Pi_{k=0}^{n}x_k!}{(m+n)!} \\
        &= \frac{m!}{\Pi_{k=0}^{n}x_k!} \cdot \frac{n!\Pi_{k=0}^{n}x_k!}{(m+n)!} \\
        &= \frac{m!\cdot n!}{(m+n)!}. 
    \end{align*}
    
    The third equality holds because we assign a label from $0$ to $n$ to each test loss. The label $0$ appears $x_0$ times, $\dots$, and the label $n$ appears $x_n$ times. 
    Since there are $m$ test losses in total, there are $m!$ ways to arrange these labels. 
    However, because swapping identical labels does not produce a new sequence, we must divide by $\prod_{k=0}^{n}x_k!$. 
    
    \textbf{Step 3: Relax the `NoTies' Assumption on $X$.} 
    
    % We define an injective function $\phi: \mathbb{R}^d \rightarrow \mathbb{R}^d$. 
    Let $S_i := S(Z_i)=(\ell(h(Z_i), Y_i), Z_i)\in\mathbb{R}^{d+1}$, and $X_i=\ell(h(Z_i), Y_i)$. 
    % Thus, the output of $S$ is a $(d+1)$-dimensional real vector. 
    We define an order relation on the outputs of $S$ such that $S_1 < S_2$ if and only if there exists an integer $q \in [0, d+1)$ where the first $q$ components of $S_1$ and $S_2$ are equal, and the $(q+1)$-th component of $S_1$ is strictly less than that of $S_2$. 
    Furthermore, $S_1 = S_2$ is defined as all components being identical. 
    It is easy to verify that this order relation $<$ is a strict total order. 
    It is straightforward to show that the $S_i$ satisfy the `NoTies' property, even if the $X_i$ may have Ties. 
    % The role of this procedure is to dispense with the necessity of the 'NoTies' property for $X$. 
    Following this, we define $n+1$ intervals as: 
    \begin{equation}\label{eq:bins_noties}
        \begin{cases}
        I_j = (S_{(j)}, S_{(j+1)}] & \text{for } j = 1, \dots, n, \\
        I_0 = [S_{(0)}, S_{(1)}]
        \end{cases}
    \end{equation}
    Here, $S \in (S_{(j)}, S_{(j+1)}]$ is equivalent to $S_{(j)} < S$ and $S \le S_{(j+1)}$. 
    Note the property that $S_1 < S_2$ implies $X_1 \le X_2$. 
    Substituting (\ref{eq:bins_noties}) into the previous conclusion with some minor modifications (replacing $S_{(i)}$ with $X_{(i)}$ when calculating $R^+$), we obtain the OQRC method under the setting where $X$ contains ties. 
    Thus, we completes the proof of Theorem~\ref{thm:Distribution}. 
\end{proof}

\subsection{Proof of Theorem~\ref{thm:QRC}}

\begin{proof}
    Fix an arbitrary $\lambda \in \Lambda$. 
    We suppress the dependence on $\lambda$ in the notation. 
    Let $K:=|\mathcal{C}_{m,n}|=\binom{m+n}{n}$. 
    Let $\Phi^{+,i} = \Phi^+(\mathbf{N}^i; X_{1:n})$ denote the upper risk for the $i$-th occupancy vector, and let $\Phi^{+,(i)}$ represent the corresponding order statistics. 
    We assume without loss of generality that there are no ties among $\Phi^{+,1}, \Phi^{+,2}, \dots, \Phi^{+,K}$. 
    If ties exist, they can be broken using the auxiliary pair $(\Phi^{+,i}, U_i)$, where $U_i \sim \mathrm{Unif}(0,1)$ is drawn independently of $\Phi^{+,i}$. 
    Let $\mathrm{Rank}(\Phi^{+,i})$ be the rank of $\Phi^{+,i}$ in ascending order. 
    Theorem~\ref{thm:Distribution} indicates that $\mathrm{Rank}(\Phi^{+,i}) \sim \mathrm{Unif}\{1,2,\dots,K\}$. 
    For the occupancy vector $\mathbf{N}$ associated with each $R_{\mathrm{test}}$, the inequality $R_{\mathrm{test}} \le \Phi^+(\mathbf{N}; X_{1:n})$ holds. 
    Let $k := \lceil (1-\delta)K \rceil$. 
    This implies the inclusion relation
    $$\{\Phi^+(\mathbf{N}; X_{1:n}) \le \Phi^{+,(k)}\} \subseteq \{R_{\mathrm{test}} \le \Phi^{+,(k)}\}.$$
    Recall that $R^+ = (F^+)^{-1}(1-\delta)$. 
    Then $R^+\ge\Phi^{+,(k)}$. 
    We thus obtain
    $$\mathbb{P}\!\left( R_{\mathrm{test}} \le R^+ \right) \ge \mathbb{P}\!\left( \Phi^+(\mathbf{N}; X_{1:n}) \le R^+ \right) \ge 1 - \delta.$$
    Hence, $R^+$ serves as a valid $(1-\delta)$ upper confidence bound for $R_{\mathrm{test}}$. 
\end{proof}

\subsection{Proof of Proposition~\ref{prop:stars-bars}}

\begin{proof}
    Let
    \[
    \mathcal T_{m,n}:=\Bigl\{(t_1,\ldots,t_n)\in\{1,\ldots,m+n\}^n:\ 1\le t_1<\cdots<t_n\le m+n\Bigr\}.
    \]
    By assumption, $(T_1,\ldots,T_n)$ is uniformly distributed on $\mathcal T_{m,n}$.
    Define the map
    \[
    \Gamma:\mathcal T_{m,n}\to \mathcal{C}_{m,n},
    \qquad
    \Gamma(t_1,\ldots,t_n)=(x_0,\ldots,x_n),
    \]
    where
    \[
    x_i:=t_{i+1}-t_i-1,\qquad i=0,\ldots,n,
    \]
    with the conventions $t_0:=0$ and $t_{n+1}:=m+n+1$. This is exactly the stars-and-bars map.

    The inverse map is explicit. Given $\mathbf{x}=(x_0,\ldots,x_n)\in \mathcal{C}_{m,n}$, define
    \[
    t_i := i + \sum_{k=0}^{i-1} x_k,\qquad i=1,\ldots,n.
    \]
    Then $(t_1,\ldots,t_n)\in \mathcal T_{m,n}$ and $\Gamma(t_1,\ldots,t_n)=\mathbf{x}$. 
    Hence $\Gamma$ is a bijection.
    
    Since $(T_1,\ldots,T_n)$ is uniform on $\mathcal T_{m,n}$, its image
    \[
    x_i = T_{i+1}-T_i-1,\qquad i=0,\ldots,n,
    \]
    is uniform on $\mathcal{C}_{m,n}$. That is,
    \[
    \mathbf{x}=(x_0,\ldots,x_n)\sim \mathrm{Unif}(\mathcal{C}_{m,n}).
    \]
\end{proof}

\subsection{Proof of Theorem~\ref{thm:calibrate_lambda}}

\begin{proof}
    The proof of this theorem is identical to that of Theorem A.1 in Bates et al.~\cite{Bates2021DistributionFreeRP}. 
\end{proof}

\subsection{Proof of Theorem~\ref{thm:tightness}}

\begin{proof}
    We redefine the optimal bound as $M^*=\inf\bigl\{M:\mathbb P\bigl(R_\psi\!(\frac1m\sum_{j=1}^m\delta_{Y_j})\le M\bigr)\ge 1-\delta\bigr\}$. 
    This formulation is equivalent to the definition of $M^*$ in Subsection \ref{subsec:tight_theory}, but it is more convenient for our subsequent proof. 
    We divide our proof into four steps. 
    
    \textbf{Step 1: Lipschitz property of the quantile risk.}
    
    Let $a=(a_1,\ldots,a_m)$ and $b=(b_1,\ldots,b_m)$ be two vectors in $[L,U]^m$.
    Let $a_{(1)}\le\cdots\le a_{(m)}, b_{(1)}\le\cdots\le b_{(m)}$ be their order statistics. Define the empirical quantile functions $Q_a(t)=a_{(j)}$ for $t\in\left(\frac{j-1}{m},\frac jm\right]$, and similarly for $Q_b$. 
    If $\max_{1\le j\le m}|a_{(j)}-b_{(j)}|\le \varepsilon,$ then $\sup_{t\in[0,1]}|Q_a(t)-Q_b(t)|\le \varepsilon.$ 
    Since $\psi$ is a probability measure,
    \[
    \left|
    \int_0^1 Q_a(t)\,\psi(dt)
    -
    \int_0^1 Q_b(t)\,\psi(dt)
    \right|
    \le
    \int_0^1 |Q_a(t)-Q_b(t)|\,\psi(dt)
    \le
    \varepsilon.
    \]
    Therefore, the quantile risk is $1$-Lipschitz with respect to the
    supremum norm of the sorted losses. 
    
    \textbf{Step 2: Uniform convergence of the calibration quantile function.}
    
    Let $\widehat F_n$ be the empirical CDF of $X_1,\ldots,X_n$. By the
    Dvoretzky--Kiefer--Wolfowitz inequality,
    \[
    \sup_{x\in[L,U]}|\widehat F_n(x)-F(x)|
    =
    \mathcal{O}_p(n^{-1/2}).
    \]
    Because $f(x)\ge c>0$ on $[L,U]$, the population quantile function $Q=F^{-1}$
    is $1/c$-Lipschitz:
    \[
    |Q(u)-Q(v)|\le \frac1c |u-v|,
    \qquad u,v\in[0,1].
    \]
    The standard inverse-CDF perturbation bound gives
    \[
    \sup_{p\in[0,1]}
    \left|
    \widehat Q_n^+(p)-Q(p)
    \right|
    \le
    \frac{1}{c}
    \left(
    \sup_{x\in[L,U]}|\widehat F_n(x)-F(x)|
    +
    \frac1n
    \right).
    \]
    where $\widehat Q_n^+(p) = X_{\left(\min\{n+1,\lceil (n+1)p\rceil\}\right)}.$
    Hence
    \[
    D_n
    :=
    \sup_{p\in[0,1]}
    \left|
    \widehat Q_n^+(p)-Q(p)
    \right|
    =
    \mathcal{O}_p(n^{-1/2}).
    \]
    
    \textbf{Step 3: Coupling the occupancy grid with the test order statistics.}
    
    For $\mathbf{x}\in \mathcal{C}_{m,n}$, form the nondecreasing index sequence $0\le J_1(\mathbf{x})\le\cdots\le J_m(\mathbf{x})\le n$ by listing the index $i$ exactly $x_i$ times. 
    Thus, $x_i$ is the number of indices $j$ for which $J_j(\mathbf{x})=i$. 
    As $\mathbf{x}$ is uniformly distributed on $\mathcal \mathcal{C}_{m,n}$, then $V_j=\frac{J_j(x)+1}{n+1}$ is uniformly distributed on the ordered grid
    \[
    \mathcal G_{m,n}
    =
    \left\{
    (v_1,\ldots,v_m):
    v_j=\frac{k_j+1}{n+1},
    \ 0\le k_1\le\cdots\le k_m\le n
    \right\}.
    \]
    
    Let $0\le U_{(1)}\le\cdots\le U_{(m)}\le 1$ be the order statistics of $m$ i.i.d. $\mathrm{Unif}(0,1)$ random variables. 
    The ordered-grid distribution on $\mathcal G_{m,n}$ is a regular
    $\mathcal{O}(n^{-1})$ discretization of the uniform distribution on the simplex $\Delta_m = \{0\le u_1\le\cdots\le u_m\le 1\}.$ 
    Therefore, for fixed $m$, there exists a coupling of $(V_1,\ldots,V_m)$ and $(U_{(1)},\ldots,U_{(m)})$ such that 
    \[
    \max_{1\le j\le m}|V_j-U_{(j)}|
    \le
    \frac{A_m}{n}
    \]
    for a constant $A_m<\infty$ depending only on $m$. 
    Under this coupling, since $Q$ is $1/c$-Lipschitz, then
    $$\max_{1\le j\le m}
    |Q(V_j)-Q(U_{(j)})|
    \le
    \frac{A_m}{c\,n}.$$
    
    \textbf{Step 4: Comparing the exact OQRC risk with the optimal test risk.}
    
    For $x\in\mathcal \mathcal{C}_{m,n}$, the vector used by the exact occupancy upper
    bound is $(X_{(J_1(x)+1)},\ldots,X_{(J_m(x)+1)}).$
    Equivalently,
    $X_{(J_j(x)+1)}
    =
    \widehat Q_n^+(V_j).$
    Therefore,
    \[
    \max_{1\le j\le m}
    \left|
    X_{(J_j(x)+1)}-Q(U_{(j)})
    \right|
    \le
    \max_{1\le j\le m}
    \left|
    \widehat Q_n^+(V_j)-Q(V_j)
    \right|
    +
    \max_{1\le j\le m}
    \left|
    Q(V_j)-Q(U_{(j)})
    \right|.
    \]
    Using the bounds from Steps 2 and 3, 
    \[
    \max_{1\le j\le m}
    \left|
    X_{(J_j(x)+1)}-Q(U_{(j)})
    \right|
    \le
    D_n+\frac{A_m}{c\,n}.
    \]
    By the Lipschitz property from Step 1, 
    \[
    \left|
    \Phi_n^+(x)
    -
    R_\psi\!\left(
    \frac1m\sum_{j=1}^m \delta_{Q(U_{(j)})}
    \right)
    \right|
    \le
    D_n+\frac{A_m}{c\,n}.
    \]
    Since $Q(U_{(1)}),\ldots,Q(U_{(m)})$ are the order statistics of $m$ i.i.d. samples from $X$, the random variable 
    $R_\psi\!(
    \frac1m\sum_{j=1}^m \delta_{Q(U_{(j)})}
    )$
    has the same distribution as
    $R_\psi\!(
    \frac1m\sum_{j=1}^m \delta_{Y_j}
    ).$
    Hence its $(1-\delta)$ quantile is exactly $M^*$. 
    Then, if two random variables $A$ and $B$ can be coupled so that
    $|A-B|\le \varepsilon$ almost surely, their $(1-\delta)$ quantiles differ
    by at most $\varepsilon$. 
    Applying this fact with $A=\Phi_n^+(x)$, $B=R_\psi\!(\frac1m\sum_{j=1}^m \delta_{Y_j})$, $\varepsilon=D_n+\frac{A_m}{c\,n}$, we obtain
    \[
    \left|R^+-M^*\right|
    \le
    D_n+\frac{A_m}{c\,n}. 
    \]
    Since $D_n=\mathcal{O}_p(n^{-1/2})$, it follows that
    \[
    \left|R^+-M^*\right|
    =
    \mathcal{O}_p(n^{-1/2}).
    \]
    This proves the theorem. 
\end{proof}

\subsection{Proof of Theorem~\ref{thm:expected_result}}

\begin{proof}
    Taking expectation in $R_{\mathrm{test}}\le \Phi^+(\mathbf{N};X_{1:n})$ over the joint randomness of the calibration and test losses, we have
    \[
    \mathbb E\!\left[R_{\mathrm{test}}\right]
    \le
    \mathbb E\!\left[\Phi^+(\mathbf{N};X_{1:n})\right].
    \]
    Equivalently, writing $\mathbf{x}$ for the random occupancy vector induced by the
    test samples, we have
    \[
    \mathbb E\!\left[R_{\mathrm{test}}(\lambda)\right]
    \le
    \mathbb E\!\left[\Phi^+(\mathbf{x};\lambda, X_{1:n})\right].
    \]
    This proves the Theorem~\ref{thm:expected_result}. Consequently, if $\lambda^\ast$ is chosen such that
    \[
    \lambda^* = \mathrm{min}\bigl\{\lambda\in\Lambda:\mathbb E\!\left[\Phi^+(x;\lambda,X_{1:n})\right]\le \alpha\bigr\},
    \]
    then the above inequality immediately implies
    \[
    \mathbb E\!\left[R_{\mathrm{test}}(\lambda^\ast)\right]
    \le
    \mathbb E\!\left[\Phi^+(x;\lambda^\ast,X_{1:n})\right]
    \le \alpha .
    \]
    Thus $\lambda^\ast$ achieves unconditional quantile risk control. 
\end{proof}

\section{Other methods}\label{sec:othermethod}

\subsection{Methods in Quantile Risk Control} 

We first introduce the general framework of quantile risk control~\cite{snell2022quantile}. 
Given a calibration set $X_{1:n}=\{X_1, X_2, \ldots, X_n\}$, we denote its ascending order statistics as $X_{(1)}\le \ldots \le X_{(n)}$. 
Let $U_1, U_2, \ldots, U_n \overset{iid}{\sim}\mathrm{Unif}(0,1)$, where $U_{(i)}$ denotes their corresponding order statistics. 
The procedure begins by selecting a one-sided test statistic of the form $S \triangleq \min_{1 \le i \le n} s_i(U_{(i)})$. 
Next, we compute the critical value $s_\delta=\mathrm{inf}\bigl\{r:\mathbb{P}(S\ge r)\ge 1-\delta\bigr\}$. 
We then construct an empirical bounding function $\hat{F}_n$ defined as
$$\hat{F}n(x) \triangleq 
\begin{cases}
    0, &\text{for } x < X{(1)} \\
    s_1^{-1}(s_\delta), &\text{for } X_{(1)} \le x < X_{(2)} \\
    \vdots \\
    s_n^{-1}(s_\delta), &\text{for } X_{(n)} \le x < x^+ \\
    1, &\text{for } x^+ \le x, 
\end{cases}$$
where $x^+ \in \mathbb{R} \cup \{\infty\}$ is an upper bound such that $F(x^+) = 1$. 
This construction guarantees the probabilistic bound
$$\mathbb{P}(R_\psi(F) \le R_\psi(\hat{F}_n)) \ge 1-\delta.$$
Within this framework, different choices of the one-sided test statistic yield various risk control methods. 
The one-sided Kolmogorov-Smirnov (\textbf{KS}) statistic~\cite{massey1951kolmogorov} sets $s_i(u)=u-\frac{i}{n}$. 
The Berk-Jones (\textbf{BJ}) statistic~\cite{10.1214/16-EJS1172} employs $s_i(u)=I_u(i, n-i+1)$, where $I_u(i, n-i+1)$ denotes the cumulative distribution function of the $\mathrm{Beta}(i, n-i+1)$ distribution evaluated at $u$. 
Furthermore, the order statistics (\textbf{OrderStats}) approach defines $s_i(u)=\mathbf{1}\{i=k\}\,I_u(k,n-k+1)+\mathbf{1}\{i\neq k\}$. 
Extensions include the one-sided truncated Berk-Jones (\textbf{One-sided BJ}) statistic, which corresponds to $s_i(u)=\mathbf{1}\{i\ge k\}\,I_u(i,n-i+1)+\mathbf{1}\{i<k\}$, and the two-sided truncated Berk-Jones (\textbf{Two-sided BJ}) statistic, which specifies $s_i(u)=\mathbf{1}\{k\le i\le l\}\,I_u(i, n-i+1)+\mathbf{1}\{i<k\ \text{or}\ l<i\}$. 

The selection of $k$ and $l$ in the OrderStats, One-sided BJ, and Two-sided BJ methods depends on the $\beta$ parameter in $\beta$-CVaR and the $\beta_{\min}$ and $\beta_{\max}$ parameters in the VaR-Interval measure. 
We define the statistics $M^+_{n,k} \triangleq \min_{k \le i \le n} I_{U_{(i)}}(i, n - i + 1)$ and $M^+_{n,k,l} \triangleq \min_{k \le i \le l} I_{U_{(i)}}(i, n - i + 1)$. 
Denoting $s^k_\delta$ as the critical value of $M^+_{n,k}$, we compute $k^*(\beta_{\min}) \triangleq \min\{k : s_k^{-1}(s^k_\delta) \ge \beta_{\min}\}$. 
Letting $s^{k^*, l}_\delta$ be the critical value of $M^+_{n,k^*,l}$, we then calculate $l^* \triangleq \min\{l : s_l^{-1}(s^{k^*, l}_\delta) \ge \beta_{\max}\}$. 
For the $\beta$-CVaR setting, the index $k$ is determined by substituting $\beta_{\min}$ with $\beta$ in this procedure. 
This establishes the complete parameter selection mechanism for these methods. 

% The selection of $k$ and $l$ is determined by the $\beta$ parameter in $\beta$-CVaR and the $\beta_{\min}$ and $\beta_{\max}$ parameters within the VaR-Interval framework. 
% In the context of $\beta$-CVaR, the index $k$ for the \textbf{OrderStats} and \textbf{One-sided BJ} methods is defined as $k = \min\{j : s_j^{-1}(s_\delta^j) \ge \beta\}$. 
% Similarly, for the VaR-Interval setting, the indices $k$ and $l$ for the \textbf{OrderStats} and \textbf{Two-sided BJ} methods are given by $k = \min\{j : s_j^{-1}(s_\delta^j) \ge \beta_{\min}\}$ and $l = \min\{j : s_j^{-1}(s_\delta^j) \ge \beta_{\max}\}$, respectively. 

\subsection{Conformal Distortion Risk Control via L-statistics} 

We next introduce conformal distortion risk control via L-statistics (\textbf{CDRC-L})~\cite{chen2025conformal}. 
The distortion risk measure serves as an alternative terminology for the quantile-based risk measure. 
Given a calibration set $X_{1:n}=\{X_1, \ldots, X_n\}$, we arrange the samples into order statistics $X_{(1)}\le \ldots \le X_{(n)}$. 
Subsequently, we construct the empirical risk $\hat{R}_\psi=\sum_{i=1}^n\bigl[\psi(\frac{i}{n})-\psi(\frac{i-1}{n})\bigr]X_{(i)}$. 
The principal conclusion by Chen et al.~\cite{chen2025conformal} demonstrates that, by characterizing the sampling uncertainty of the empirical risk as $\hat{\sigma}^2=\frac{1}{n^2}\sum_{i,j=1}^{n}\psi'\!\left(\frac{i}{n}\right)\psi'\!\left(\frac{j}{n}\right)\left(\frac{i\wedge j}{n}-\frac{ij}{n^2}\right)$, the following holds
$$\frac{\sqrt{n}(\hat{R}_\psi-R_\psi)}{ \hat{\sigma}^2}\overset{d}{\to}N(0,1).$$
We then define the upper bound $\hat{R}^+_\psi=\hat{R}_\psi+z_{1-\delta}\frac{\hat{\sigma}}{\sqrt{n}}$, which guarantees the asymptotic result
$$\liminf_{n\to\infty}P_D\!\bigl(R_{\psi} \le \hat{R}^+_\psi \bigr)\ge 1-\delta.$$
Here, $z_{1-\delta}$ denotes the $(1-\delta)$ quantile of the standard normal distribution $N(0,1)$. 
Note that this theoretical property fundamentally requires the size of the calibration set to approach infinity. 
Consequently, the risk control behavior remains indeterminate for finite sample sizes, implying that the guarantee might fail when the number of samples is insufficiently large. 

\section{Algorithms of OQRC}\label{sec:algorithm}

In this section, we present two algorithms: one for deriving the risk bounds of OQRC via Monte-Carlo sampling (Algorithm~\ref{alg:porc_mc}), and another for selecting $\lambda$ to achieve QRC (Algorithm~\ref{alg:lambda_calibration}). 

\begin{algorithm}[htbp]
\caption{Monte-Carlo Occupancy-based Quantile Risk Control (OQRC)}
\label{alg:porc_mc}
\begin{algorithmic}[1]
\REQUIRE Calibration loss samples $X_{1:n}$, test set size $m$, confidence level $\delta$, bounds $U$, weighting function $\psi$, Monte-Carlo replicates $B$.
\ENSURE Upper bound $r^*$ for the test risk measure at confidence level $1-\delta$.
\STATE $(X_{(1)}, X_{(2)}, \dots, X_{(n)}) \leftarrow \operatorname{Sort}(X_{1:n})$
\STATE $X_{(n+1)} \leftarrow U$
\FOR{$b = 1$ \TO $B$}
    \STATE $\mathcal{T} \leftarrow \operatorname{SampleWithoutReplacement}(\{1, 2, \dots, m+n\}, n)$
    \STATE $(T_1, T_2, \dots, T_n) \leftarrow \operatorname{Sort}(\mathcal{T})$ 
    \STATE $T_0 \leftarrow 0, \quad T_{n+1} \leftarrow m+n+1$
    \STATE $S_{-1} \leftarrow 0$
    \FOR{$i = 0$ \TO $n$}
        \STATE $N_i^{(b)} \leftarrow T_{i+1} - T_i - 1$ 
        \STATE $S_i \leftarrow S_{i-1} + N_i^{(b)}$ 
    \ENDFOR
    \STATE $\Phi^{+}(\mathbf{x}^{(b)}; \lambda) \leftarrow \sum_{i=0}^{n} \int_{\frac{S_{i-1}(\mathbf{x}^{(b)})}{m}}^{\frac{S_{i}(\mathbf{x}^{(b)})}{m}} \psi(t)X_{(i+1)} \, dt$
\ENDFOR
\STATE $r^* \leftarrow \operatorname{EmpiricalQuantile}\!\left( \left\{ \Phi^{+}(\mathbf{x}^{(b)}; \lambda) \right\}_{b=1}^B, 1-\delta \right)$
\RETURN $r^*$
\end{algorithmic}
\end{algorithm} 

\begin{algorithm}[htbp]
\caption{Calibrating $\lambda$ for Guaranteed Risk Control via OQRC}
\label{alg:lambda_calibration}
\begin{algorithmic}[1]
\REQUIRE Calibration samples $X_{1:n}$, test size $m$, target risk level $\alpha$, confidence level $\delta$, bounds $U$, weighting function $\psi$, candidate set $\Lambda=\{\lambda_1,\dots,\lambda_K\}$, Monte-Carlo replicates $B$
\ENSURE Calibrated parameter $\lambda^*$ satisfying the OQRC risk guarantee.
\STATE $\Lambda \gets \operatorname{sort}_{\uparrow}(\Lambda)$
\FOR{$k=1$ \TO $K$}
    \STATE $X^{\lambda_k}_{1:n} \leftarrow \{X_1^{\lambda_k},\dots,X_n^{\lambda_k}\}$
    \STATE $r_k \leftarrow \textsc{OQRC-MC}(X^{\lambda_k}_{1:n},\, m,\, \delta,\, U,\, \psi,\, B)$ \COMMENT{using Algorithm~\ref{alg:porc_mc}}
    \IF{$r_k \le \alpha$}
        \STATE $\lambda^* \leftarrow \lambda_k$
        \RETURN $\lambda^*$
    \ENDIF
\ENDFOR
\RETURN FAIL
\end{algorithmic}
\end{algorithm}

\section{Tightness of OQRC}\label{sec:tightness_theory}

In this appendix, we discuss why the coverage in Table~\ref{tab:main} deviates from $1-\delta$. 
This deviation occurs because FNP is not a continuous loss function and fails to satisfy the condition that $f_X$ is bounded from below. 
Our tightness theory relies on the assumption that the loss function possesses such ideal properties. 
To empirically validate our theory, we construct a synthetic dataset containing 5000 points. 
The loss value for each point is sampled randomly between 0 and 1 when $\lambda=0$. 
It then follows a linear function with respect to $\lambda$ and decreases to zero at $\lambda=1$. 
Figure~\ref{fig:tight_analysis} demonstrates that the coverage converges to $1-\delta$ under both risk measures as the size of the calibration set increases. 
We leave the establishment of the OQRC tightness theory for more general loss functions to future work. 

\begin{figure}[htbp]
    \centering
    \begin{minipage}{0.33\linewidth}
        \centering
        \includegraphics[width=\linewidth]{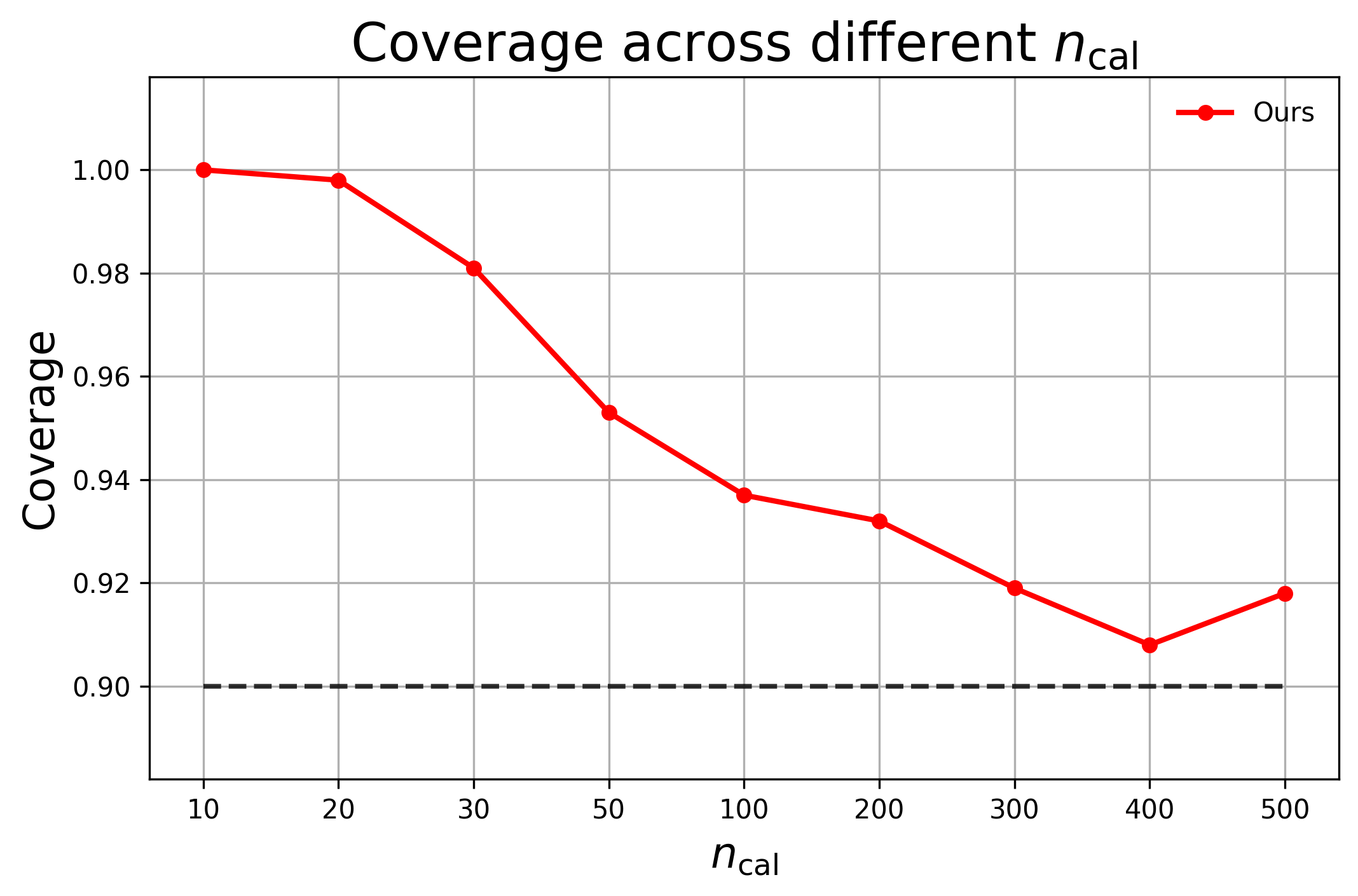}
        \subcaption{$[0.85,0.95]$-VaR-Interval}
    \end{minipage}\hfill
    \begin{minipage}{0.33\linewidth}
        \centering
        \includegraphics[width=\linewidth]{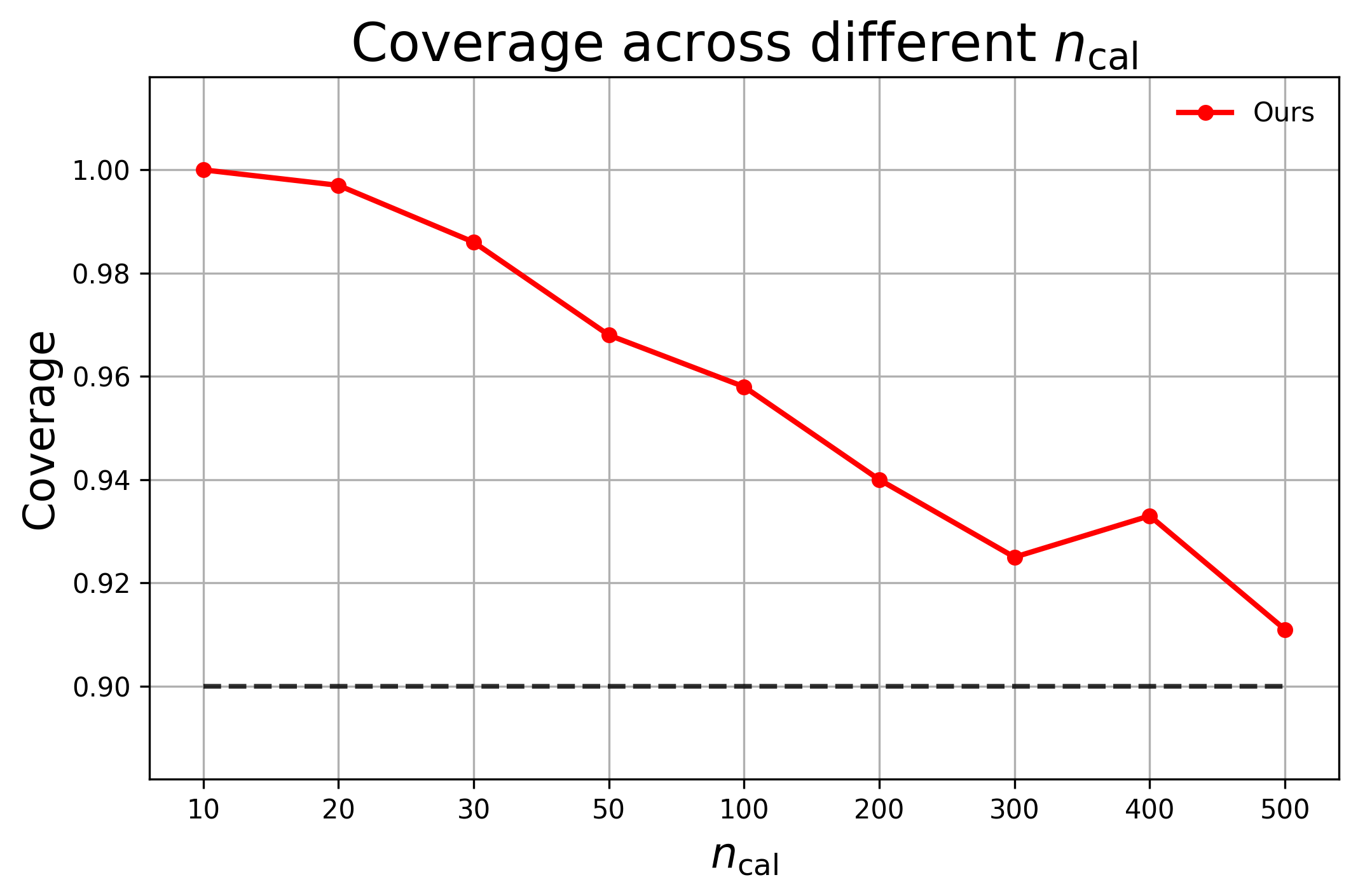}
        \subcaption{$0.9$-CVaR}
    \end{minipage}\hfill
    \begin{minipage}{0.33\linewidth}
        \centering
        \includegraphics[width=\linewidth]{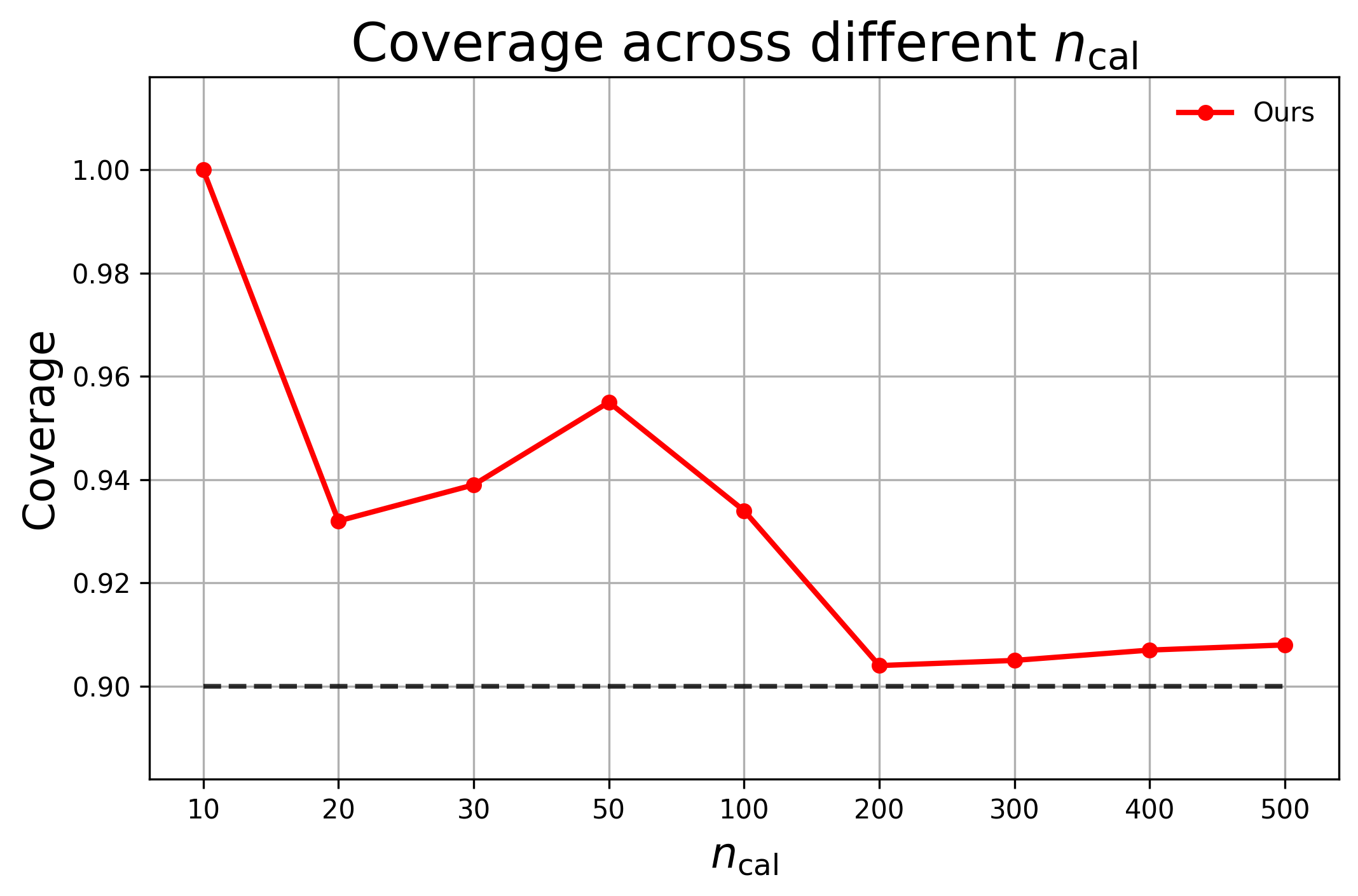}
        \subcaption{$0.8$-VaR}
    \end{minipage}
    \caption{Coverage in synthetic dataset across different calibration set size.}
    \label{fig:tight_analysis}
\end{figure}

\section{Implementation details}\label{sec:details}

% \subsection{Task details}\label{subsec:task}

% \paragraph{Tumor segmentation}

% \paragraph{Multi-label classification problem}

% \paragraph{Text emotion recognition}

\subsection{Loss function details}\label{subsec:loss} 

We use the false negative proportion (FNP) as our loss function, and we introduce FNP below. 
In the context of a multi-label classification problem, consider a model $h:\mathcal{Z}\rightarrow [0,1]^{k}$, where $k$ denotes the number of classes. 
Let the prediction set generated by the model be defined as $\mathcal{C}_\lambda(Z)=\bigl\{h(Z)_y\ge 1-\lambda\bigr\}$, with $\lambda$ serving as the tuning parameter. 
Letting $Y$ denote the set of true labels associated with $Z$, the FNP can be formulated as 
$$l_\lambda^{\mathrm{FNP}}(Z)=1-\frac{|Y \cap \mathcal{C}_\lambda(Z)|}{|Y|}.$$
Tasks such as tumor segmentation and text emotion recognition can also be formulated in a similar manner. 
Specifically, the output set for tumor segmentation comprises the pixels within a medical image, whereas the output set for text emotion recognition constitutes a set of emotion labels. 
It can be shown that the FNP defined in this way is non-increasing with respect to $\lambda$. 

\subsection{Training details}\label{subsec:training} 

All models are trained on an NVIDIA GeForce RTX 4090 GPU. 
For the tumor segmentation task, we utilized a training set comprising 1,450 images strictly disjoint from the validation set. 
We employed a combination of binary cross-entropy and weighted Intersection over Union (IoU) loss as the loss function. 
The training process spanned 20 epochs with a batch size of 16, optimized via the Adam optimizer using a learning rate of $\mathrm{1e}-4$. 

For the image multi-label classification task, we use the resnet50 model for ResNet, tf\_efficientnet\_b3 for EfficientNet, tresnet\_m for TResNet, and convnext\_base for ConvNeXt. 
We train the models on 118,287 images. 
The loss function for all these models is uniformly set to BCEWithLogitsLoss, and the optimizer is uniformly AdamW with a weight decay of 1e-4. 
The learning rate is set to 1e-4 with a CosineAnnealingLR scheduler. All models are trained for 5 epochs with a batch size of 64. 

In the text emotion recognition task, we utilized the pre-fine-tuned bert-base-go-emotion model directly from the Hugging Face repository\footnote{https://huggingface.co/bhadresh-savani/bert-base-go-emotion}. 

\subsection{Evaluation metrics}\label{subsec:metrics} 

We evaluate the performance of risk control framework using three specific metrics. 
The first metric is the coverage (\textbf{Cov}), which represents the empirical probability that the risk on the test set remains under a predefined risk level $\alpha$. 
We define this as 
$$\text{Cov}=\frac{1}{n_{\mathrm{trials}}}\sum_{i=1}^{n_{\mathrm{trials}}}\mathbf{1}_{R_{\mathrm{test}, i}\le \alpha}.$$
A method is considered valid if the coverage exceeds $1-\delta$, and invalid otherwise. 
The second metric is the average risk gap (\textbf{RiskGap}), which assesses the conservatism of the method and is defined as 
$$\text{RiskGap}=\frac{1}{n_{\mathrm{trials}}}\sum_{i=1}^{n_{\mathrm{trials}}}|\alpha-R_{\mathrm{test}, i}|.$$
Assuming that the method is valid, a larger risk gap indicates a more conservative approach, which implies that a smaller gap corresponds to tighter control. 
Finally, we utilize the average prediction set size (\textbf{AvgSize}) to measure the magnitude of the generated prediction sets, which is defined as 
$$\text{AvgSize}=\frac{1}{n_{\mathrm{trials}}}\sum_{i=1}^{n_{\mathrm{trials}}}\bigl(\frac{1}{m}\sum_{Z_{\text{test}}\in \mathcal{D}_{\mathrm{test},i}}|\mathcal{C}_{\lambda_i}(Z_{\text{test}})|\bigr).$$
Similar to the risk gap, when the method is valid, a larger average size denotes a higher level of conservatism, whereas a smaller size signifies tight control of the risk. 

\section{Asymptotic Behavior in the Large Test-Set Regime}\label{sec:asymptotic}

We now discuss the behavior of OQRC in the large test-set regime, where the number of test losses $m$ tends to infinity while the calibration set size $n$ remains fixed. 
In this setting, the empirical distribution $\widehat F_m$ of the test losses converges to the loss distribution $F_X$, and the corresponding empirical risk $R_\psi(\widehat F_m;\lambda)$ approaches the risk $R_\psi(F_X;\lambda)$. 
Therefore, the objective of OQRC changes from controlling risk on finite test set to controlling risk over the entire loss distribution. 

In this context, we draw $n$ independent uniform random variables on $[0,1]$, and sort them as $0<U_{(1)}<\cdots<U_{(n)}<1$, and we define the spacings
\[
q_0=U_{(1)},\qquad
q_i=U_{(i+1)}-U_{(i)},\quad 1\le i\le n-1,
\qquad
q_n=1-U_{(n)}.
\]
The resulting vector $q=(q_0,\ldots,q_n)$ follows the uniform distribution over a simplex, that is
\[
q\sim \mathrm{Dirichlet}(1,\ldots,1).
\]
We then define $\Phi_\infty^+(q;\lambda,X_{1:n})=\sum_{i=0}^n\int_{U_{(i)}}^{U_{(i+1)}}\psi(t)X_{(i)}dt$, where we set $U_{(n+1)}=1$. 
Consequently, in the large-test-set regime, the distribution $F^+$ of the upper risk is replaced by
\[
F_\infty^+(r;\lambda,X_{1:n})
=
\mathbb P_{q\sim \mathrm{Dirichlet}(1,\ldots,1)}
\left(
\Phi_\infty^+(q;\lambda,X_{1:n})\le r
\right).
\]
Subsequently, we define $R_\infty^+ = (F_\infty^+)^{-1}(1-\delta)$. 
As a result, $R_\infty^+$ is a UCB of $R_\psi(\widehat F_m;\lambda)$. 

For unconditional risk control, this limiting representation can be summarized as
\[
\mathbb{E}\left[R_\psi(\widehat F_m;\lambda)\right]
=
\mathbb E_{q\sim \mathrm{Dirichlet}(1,\ldots,1)}
\left[
\Phi_\infty^+(q;\lambda,X_{1:n})
\right].
\]

The Monte-Carlo implementation in this regime is correspondingly simplified. 
When $m\to\infty$, each Monte-Carlo replicate samples a continuous simplex vector
\[
q^{(b)}\sim \mathrm{Dirichlet}(1,\ldots,1),
\qquad b=1,\ldots,B,
\]
and computes $\Phi_\infty^+(q^{(b)};\lambda,X_{1:n}).$ 
The resulting empirical distribution
\[
\widehat F_{\infty,B}^+(r;\lambda,X_{1:n})
=
\frac{1}{B}
\sum_{b=1}^B
\mathbf 1
\left\{
\Phi_\infty^+(q^{(b)};\lambda,X_{1:n})\le r
\right\}
\]
approximates $F_\infty^+(r;\lambda,X_{1:n})$. The empirical $(1-\delta)$-quantile of these Monte-Carlo values gives the UCB of the quantile risk, while their empirical average,
\[
\widehat{\overline R}_{\infty,B}^+(\lambda)
=
\frac{1}{B}
\sum_{b=1}^B
\Phi_\infty^+(q^{(b)};\lambda,X_{1:n}),
\]
estimates the $\mathbb{E}\left[R_\psi(\widehat F_m;\lambda)\right]$.

\section{Additional experimental results}\label{sec:AddExp}

% \subsection{Empirical verification of the assumption in Theorem~\ref{thm:tightness}}\label{subsec:assup_empirical}

\subsection{Extended results across diverse hyperparameter configurations}\label{subsec:results_hyperparameter}

Figures~\ref{fig:parameter_analysis_polyp_0.95} to~\ref{fig:parameter_analysis_go_0.801} compare our method with the baseline approaches in terms of Cov, RiskGap, and AvgSize across three tasks under three risk measures. 
Within each figure, the top three subplots display the coverage, the middle three present the RiskGap, and the bottom three illustrate the AvgSize. 
The experimental results demonstrate that our method consistently maintains the required coverage while significantly reducing both the RiskGap and the AvgSize compared to the baselines across diverse datasets, risk measures, and hyperparameter configurations. 

\subsection{Extended results across diverse model architectures}\label{subsec:results_model}

Figure~\ref{fig:different_model_0.95} compares the performance of our method with the baselines across different model architectures under $[0.85, 0.95]$-VaR-Interval. 
Figure~\ref{fig:different_model_1} presents the comparison under $0.9$-CVaR. 
Figure~\ref{fig:different_model_0.9} presents the comparison under $0.8$-VaR. 
The hyperparameters match the configurations used for the performance evaluation experiment across model architectures in Section~\ref{sec:EXP}.  
The experimental results demonstrate that our approach significantly narrows both the RiskGap and the AvgSize compared to the baselines while achieving valid coverage. 

\begin{figure}[htbp]
    \centering
    \begin{minipage}{0.29\linewidth}
        \centering
        \includegraphics[width=\linewidth]{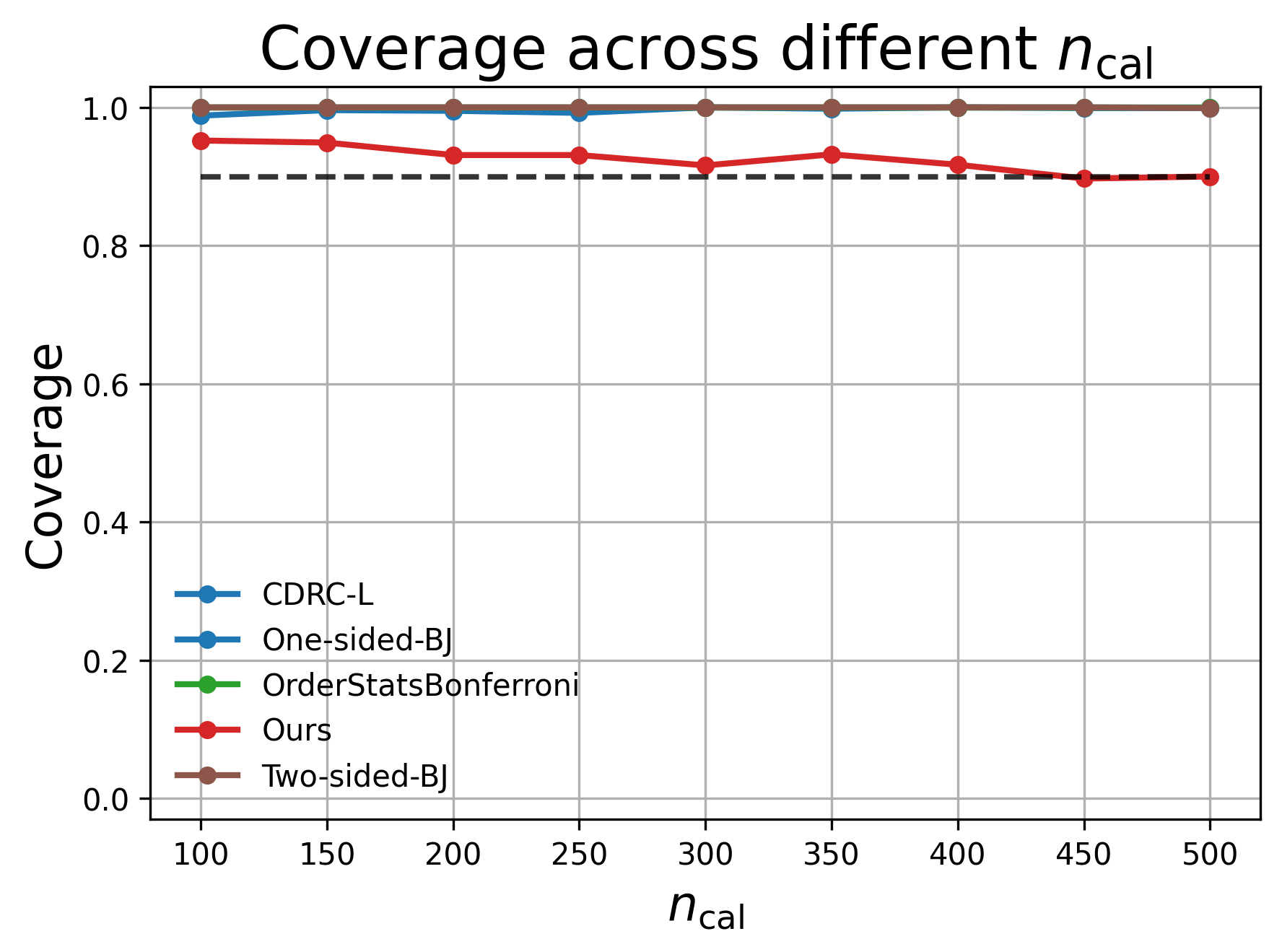}
    \end{minipage}\hfill
    \begin{minipage}{0.29\linewidth}
        \centering
        \includegraphics[width=\linewidth]{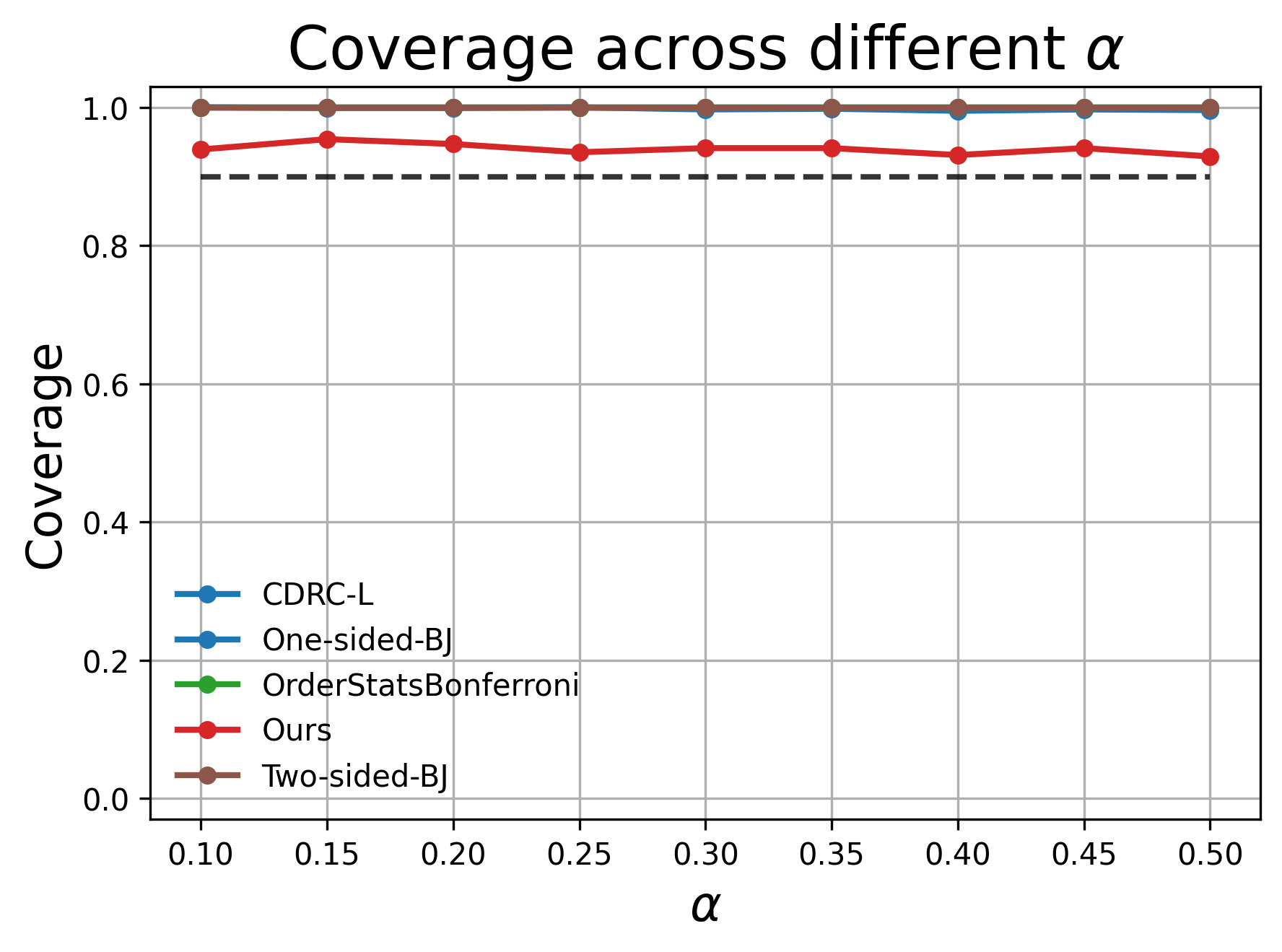}
    \end{minipage}\hfill
    \begin{minipage}{0.29\linewidth}
        \centering
        \includegraphics[width=\linewidth]{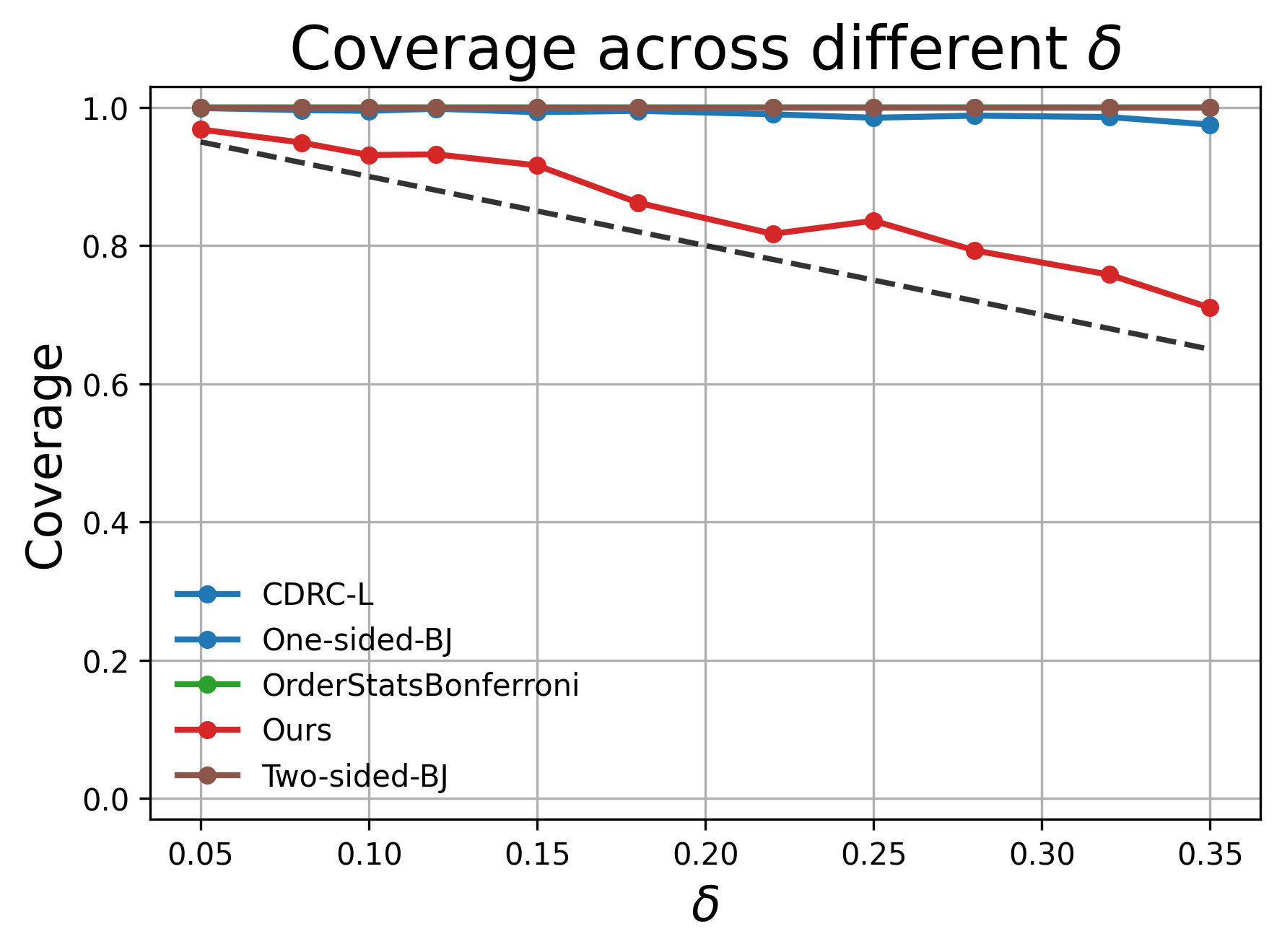}
    \end{minipage}\\
    \begin{minipage}{0.29\linewidth}
        \centering
        \includegraphics[width=\linewidth]{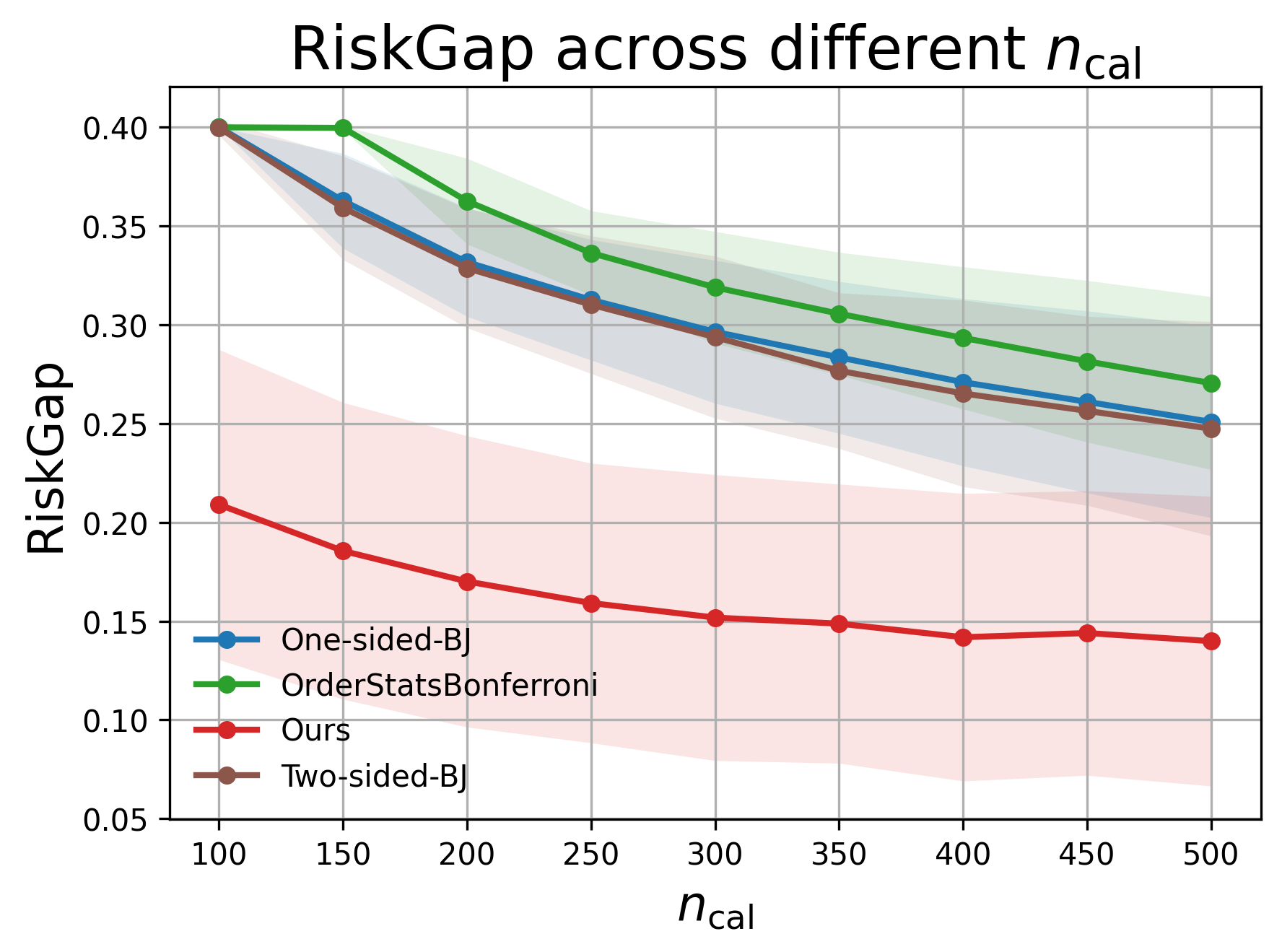}
    \end{minipage}\hfill
    \begin{minipage}{0.29\linewidth}
        \centering
        \includegraphics[width=\linewidth]{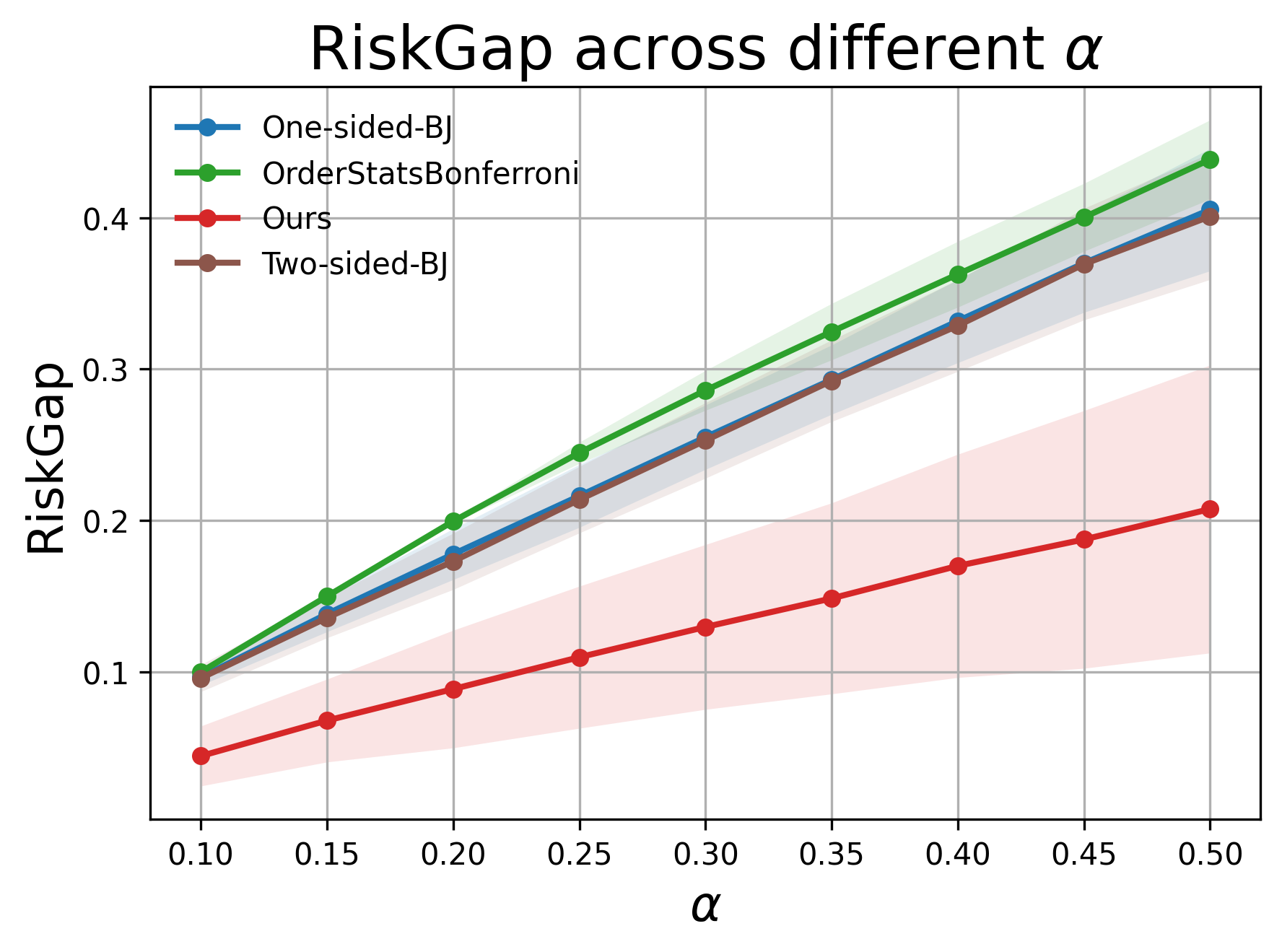}
    \end{minipage}\hfill
    \begin{minipage}{0.29\linewidth}
        \centering
        \includegraphics[width=\linewidth]{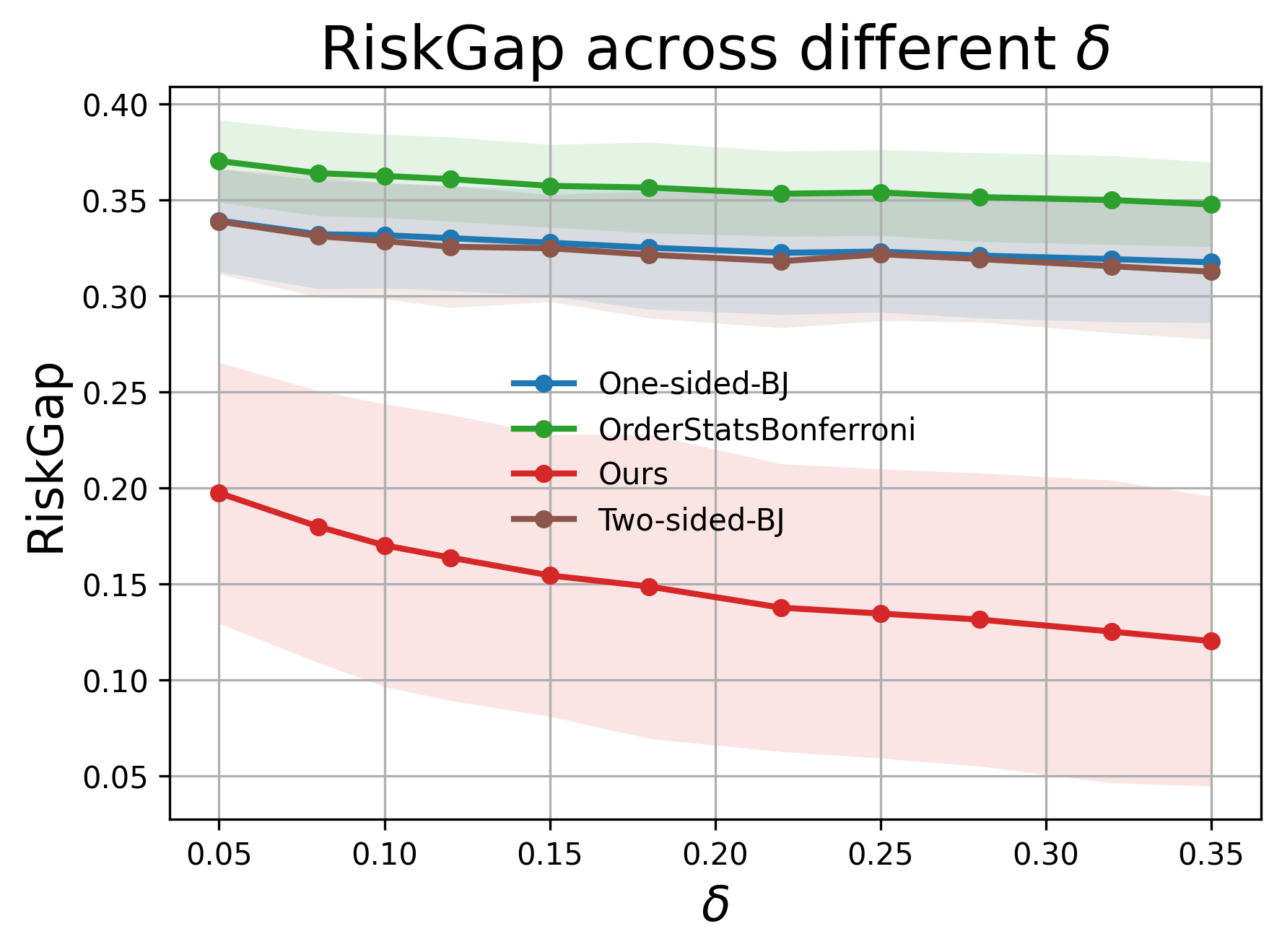}
    \end{minipage}\\
    \begin{minipage}{0.29\linewidth}
        \centering
        \includegraphics[width=\linewidth]{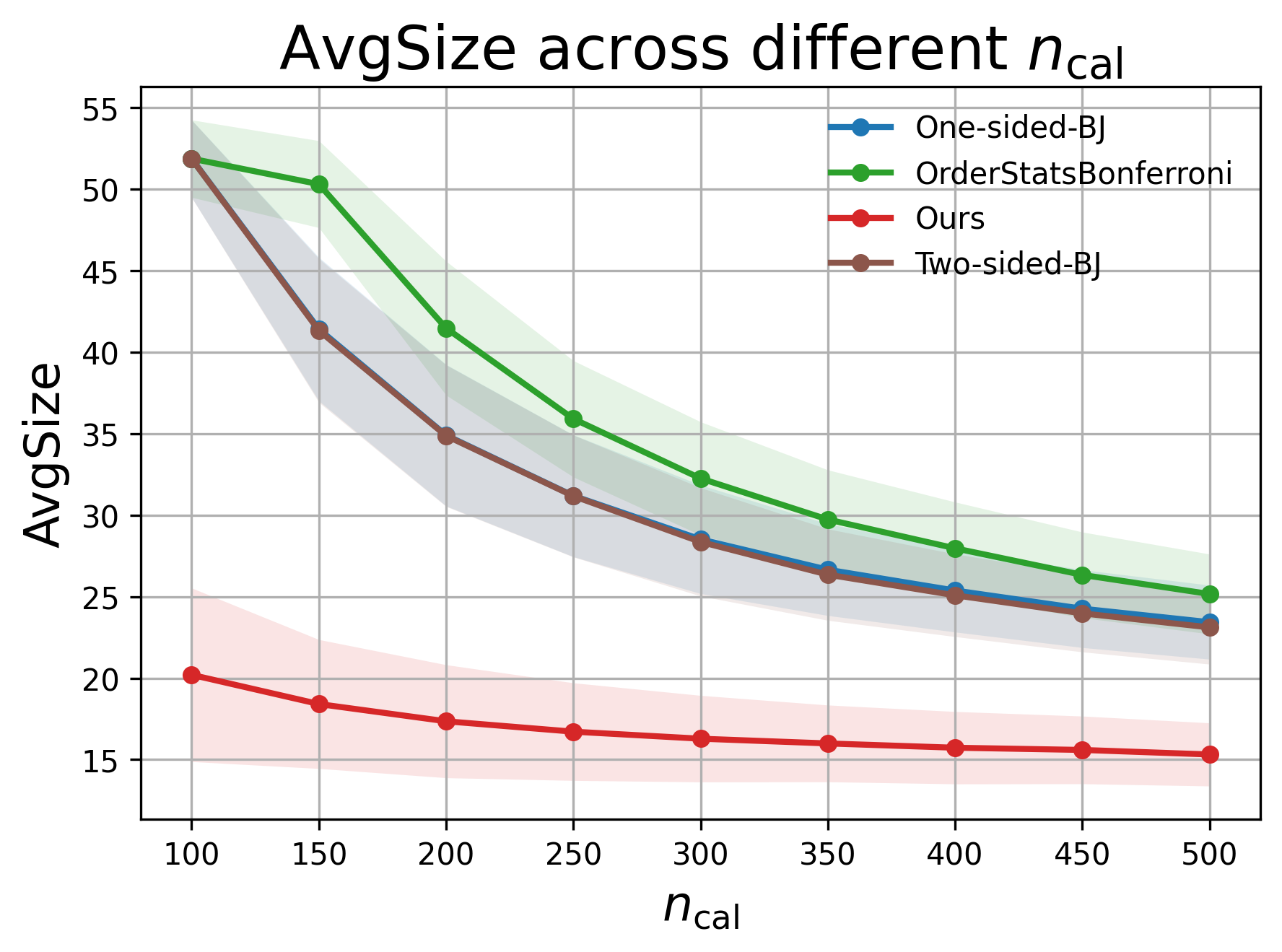}
    \end{minipage}\hfill
    \begin{minipage}{0.29\linewidth}
        \centering
        \includegraphics[width=\linewidth]{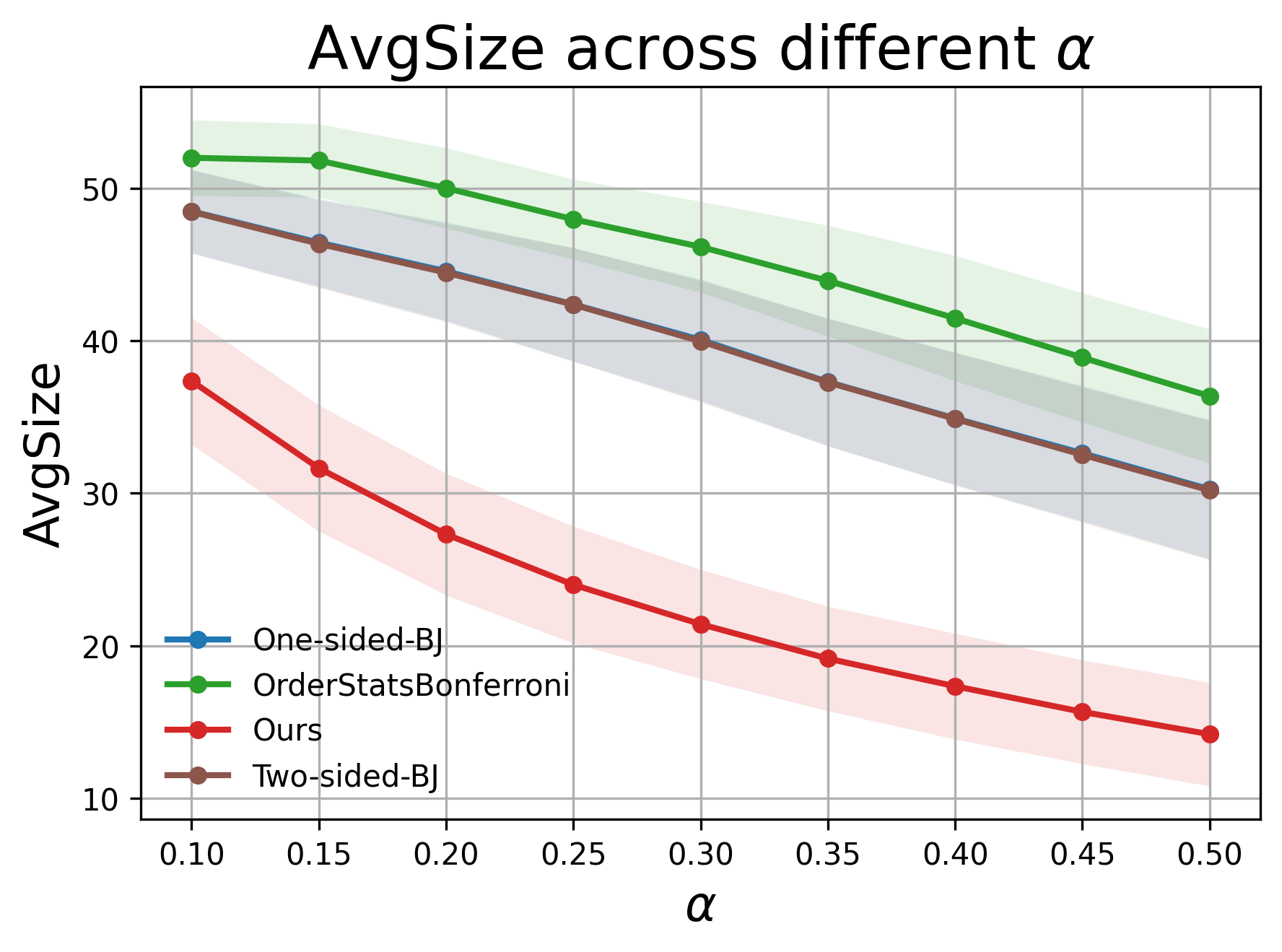}
    \end{minipage}\hfill
    \begin{minipage}{0.29\linewidth}
        \centering
        \includegraphics[width=\linewidth]{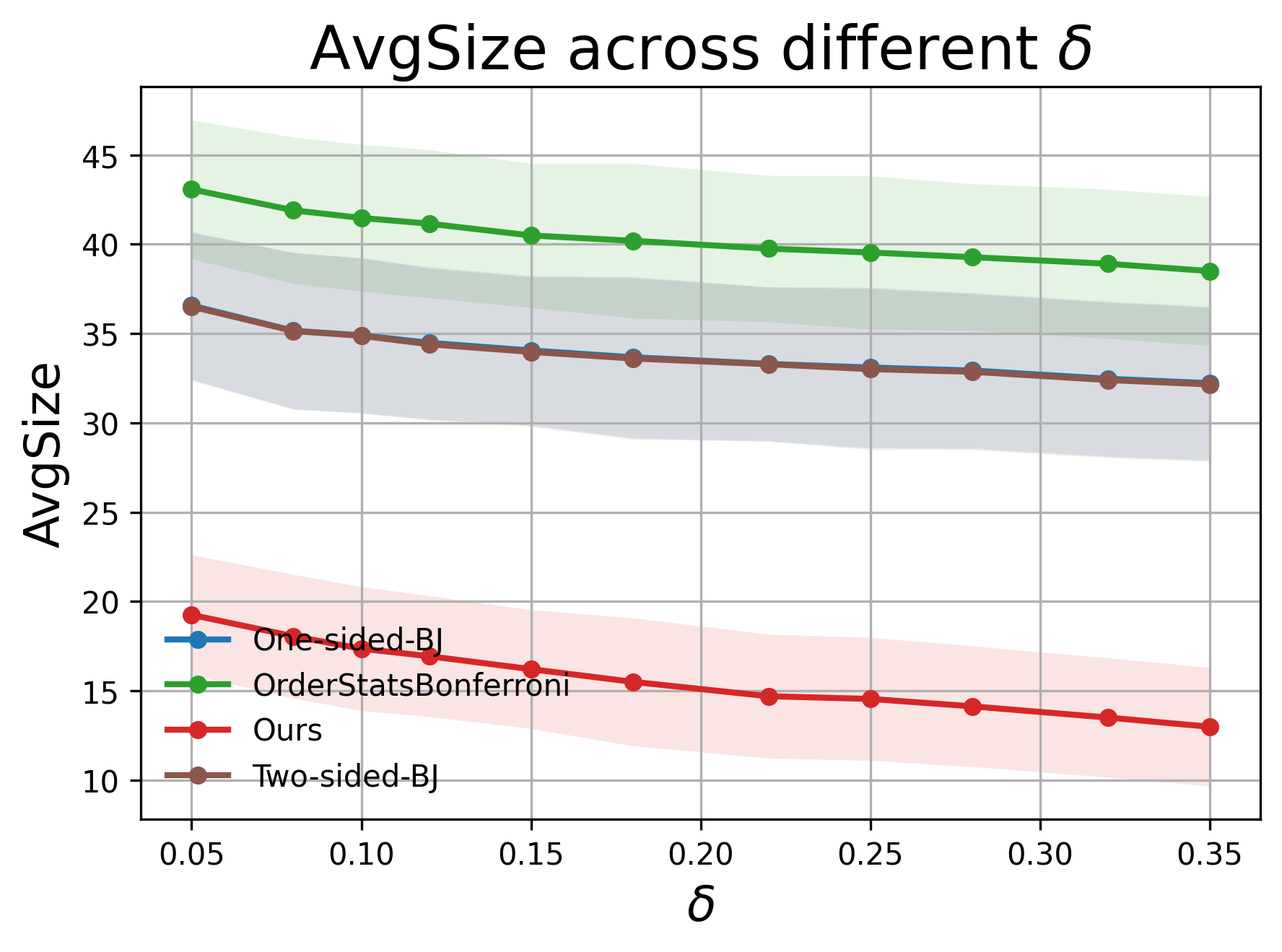}
    \end{minipage}
    \vspace{-3mm}
    \caption{\textbf{Performance comparison across various hyperparameter configurations.} The task is tumor segmentation, and the risk measure is $[0.85, 0.95]$-VaR-Interval.}
    \label{fig:parameter_analysis_polyp_0.95}
    \vspace{-5mm}
\end{figure}

\begin{figure}[htbp]
    \centering
    \begin{minipage}{0.29\linewidth}
        \centering
        \includegraphics[width=\linewidth]{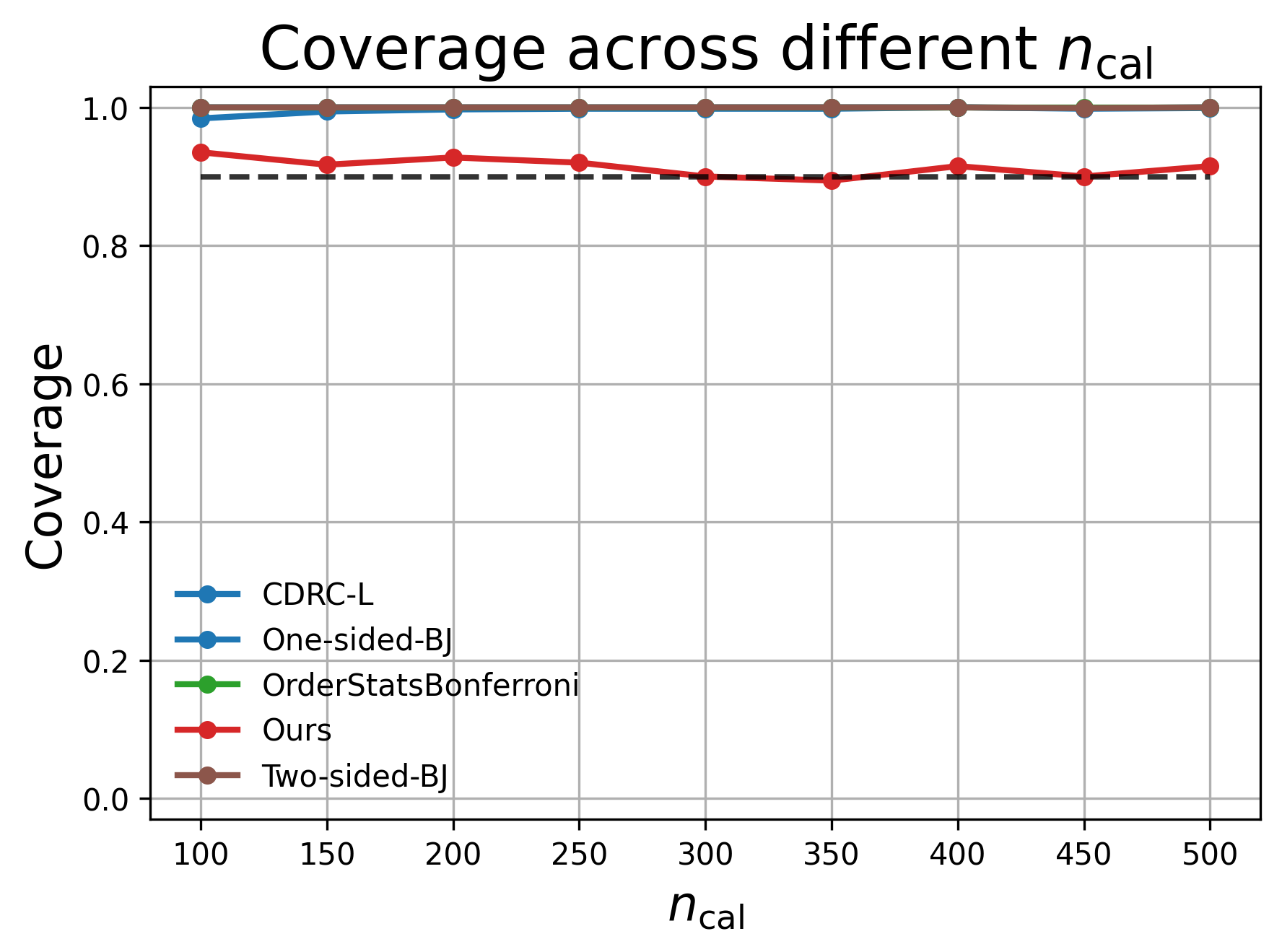}
    \end{minipage}\hfill
    \begin{minipage}{0.29\linewidth}
        \centering
        \includegraphics[width=\linewidth]{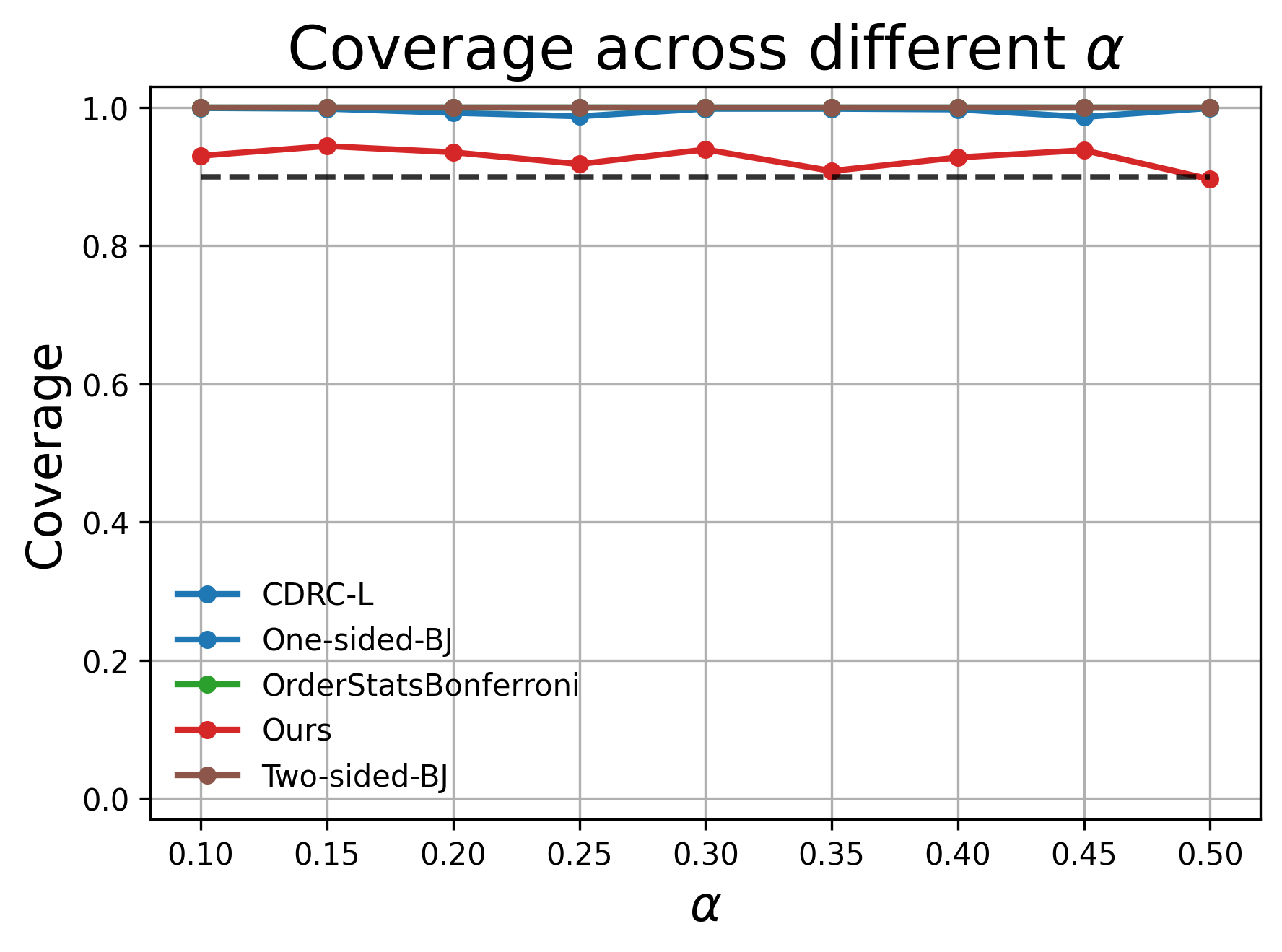}
    \end{minipage}\hfill
    \begin{minipage}{0.29\linewidth}
        \centering
        \includegraphics[width=\linewidth]{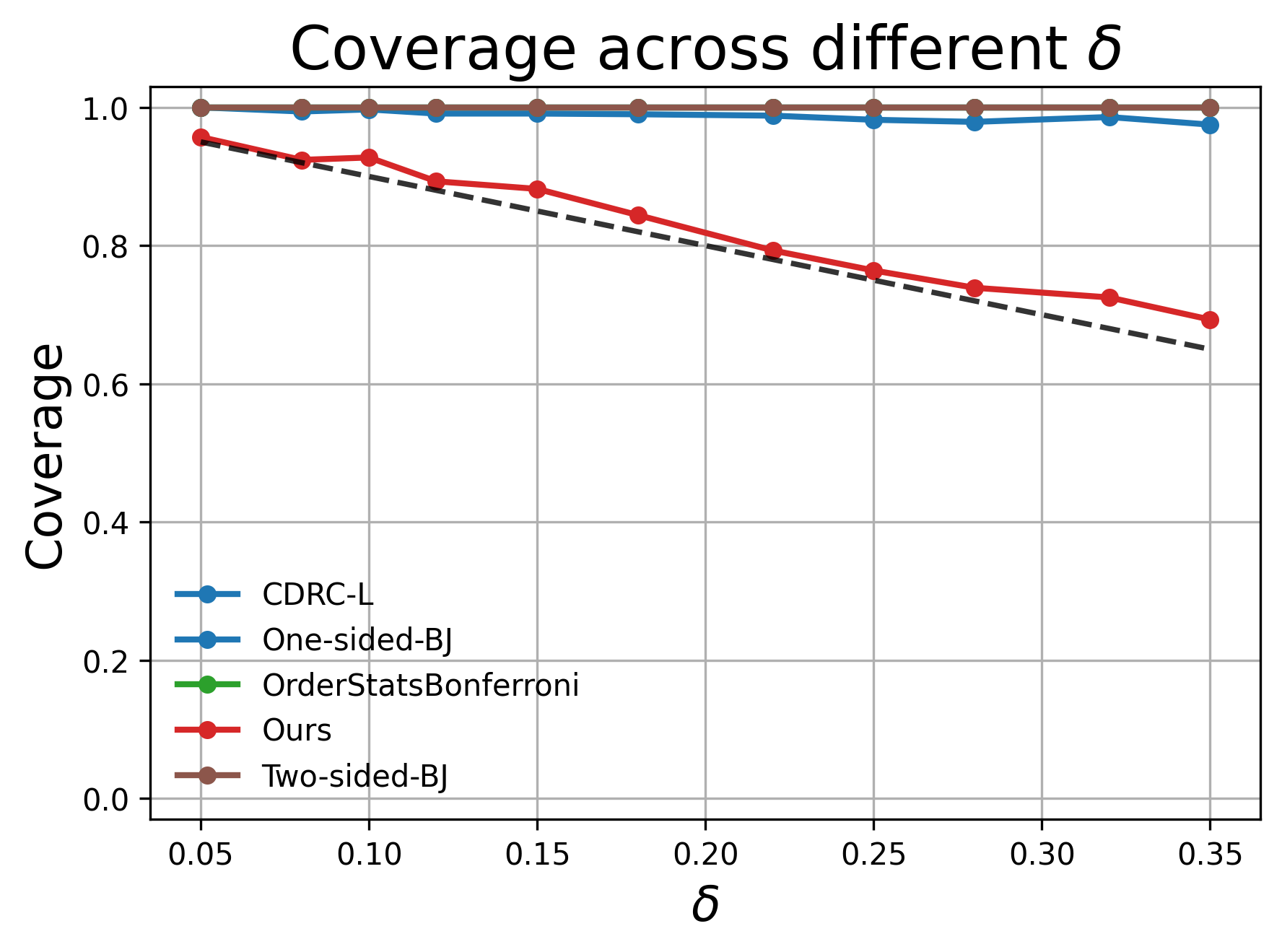}
    \end{minipage}\\
    \begin{minipage}{0.29\linewidth}
        \centering
        \includegraphics[width=\linewidth]{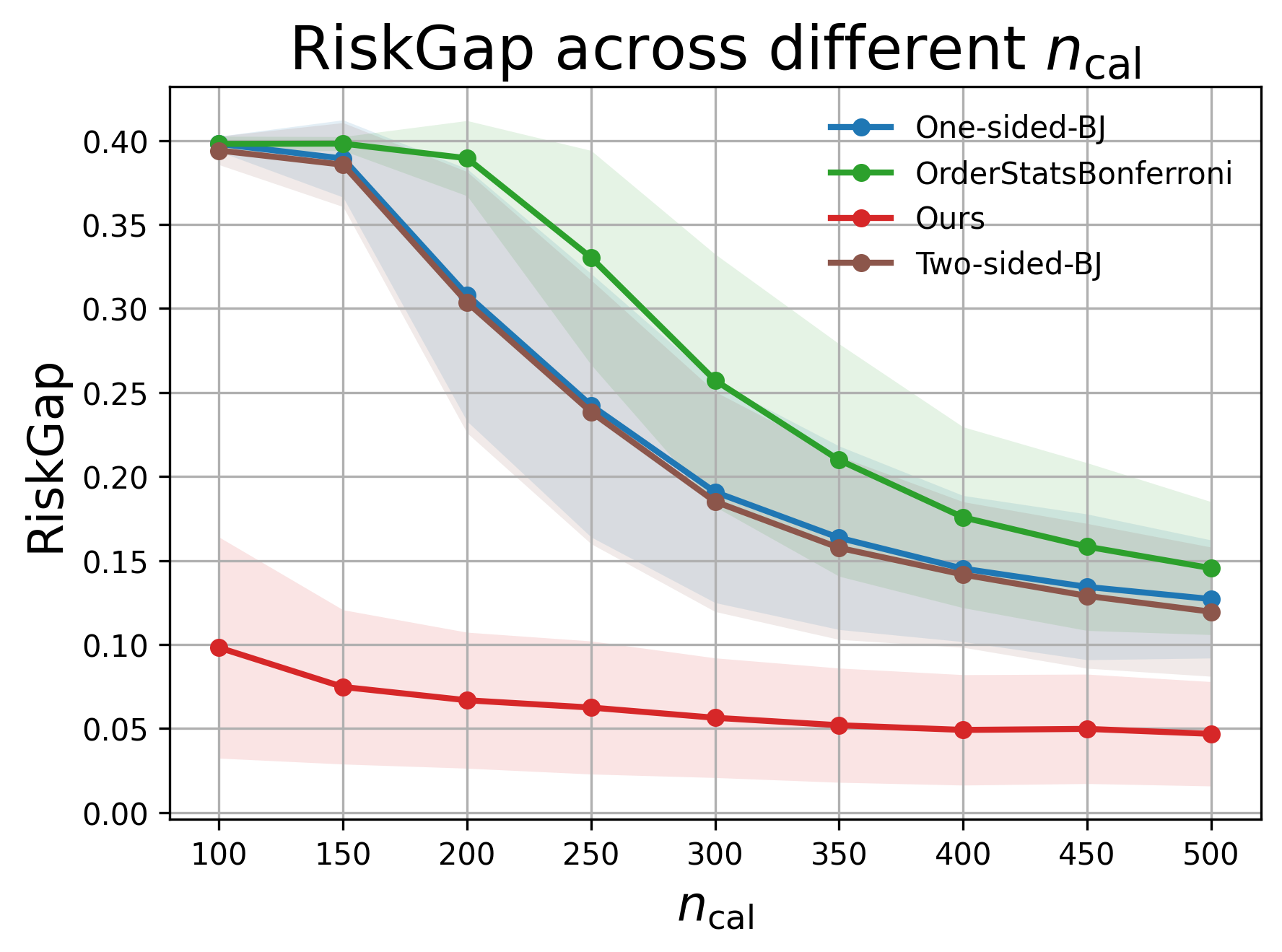}
    \end{minipage}\hfill
    \begin{minipage}{0.29\linewidth}
        \centering
        \includegraphics[width=\linewidth]{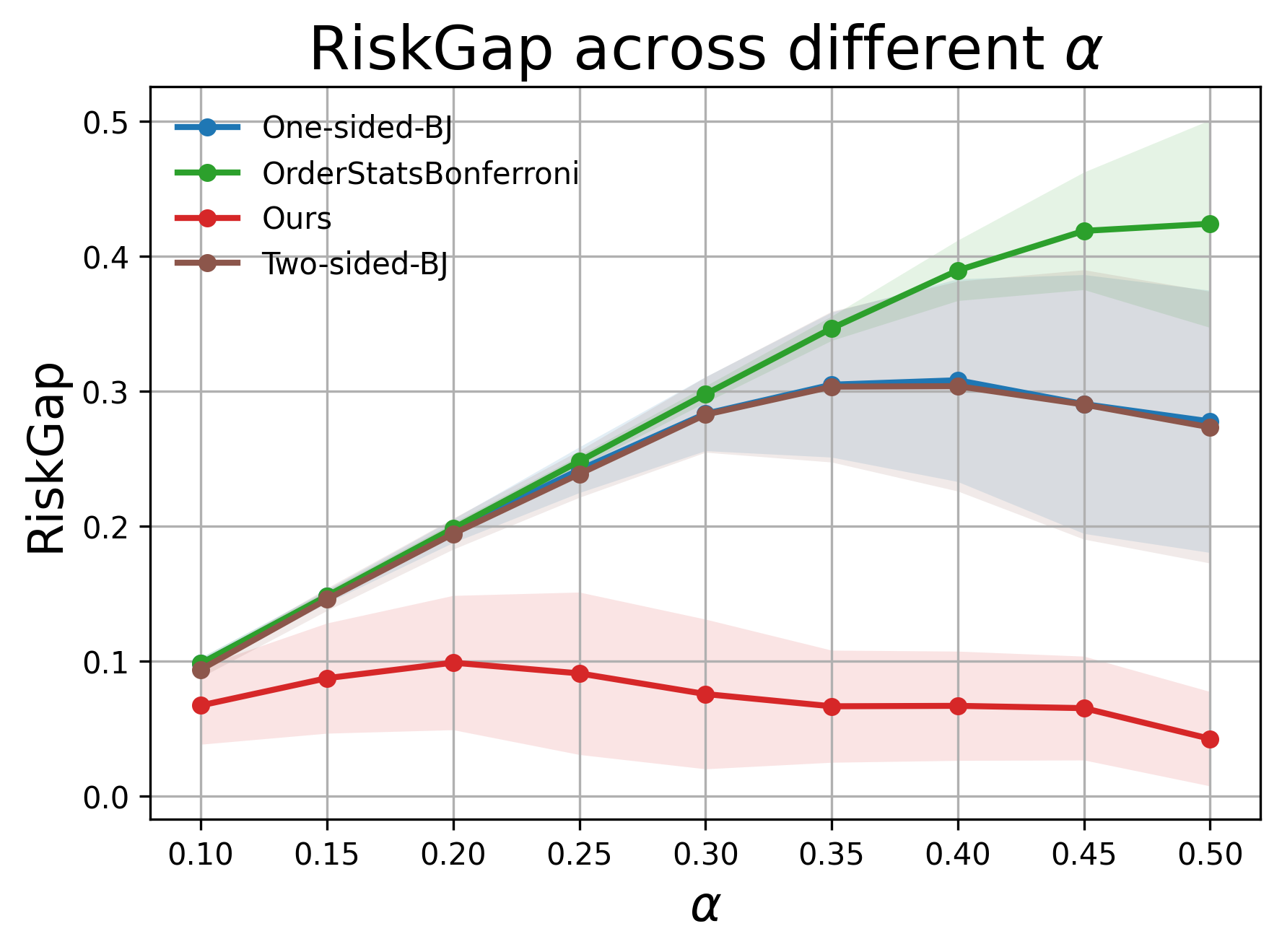}
    \end{minipage}\hfill
    \begin{minipage}{0.29\linewidth}
        \centering
        \includegraphics[width=\linewidth]{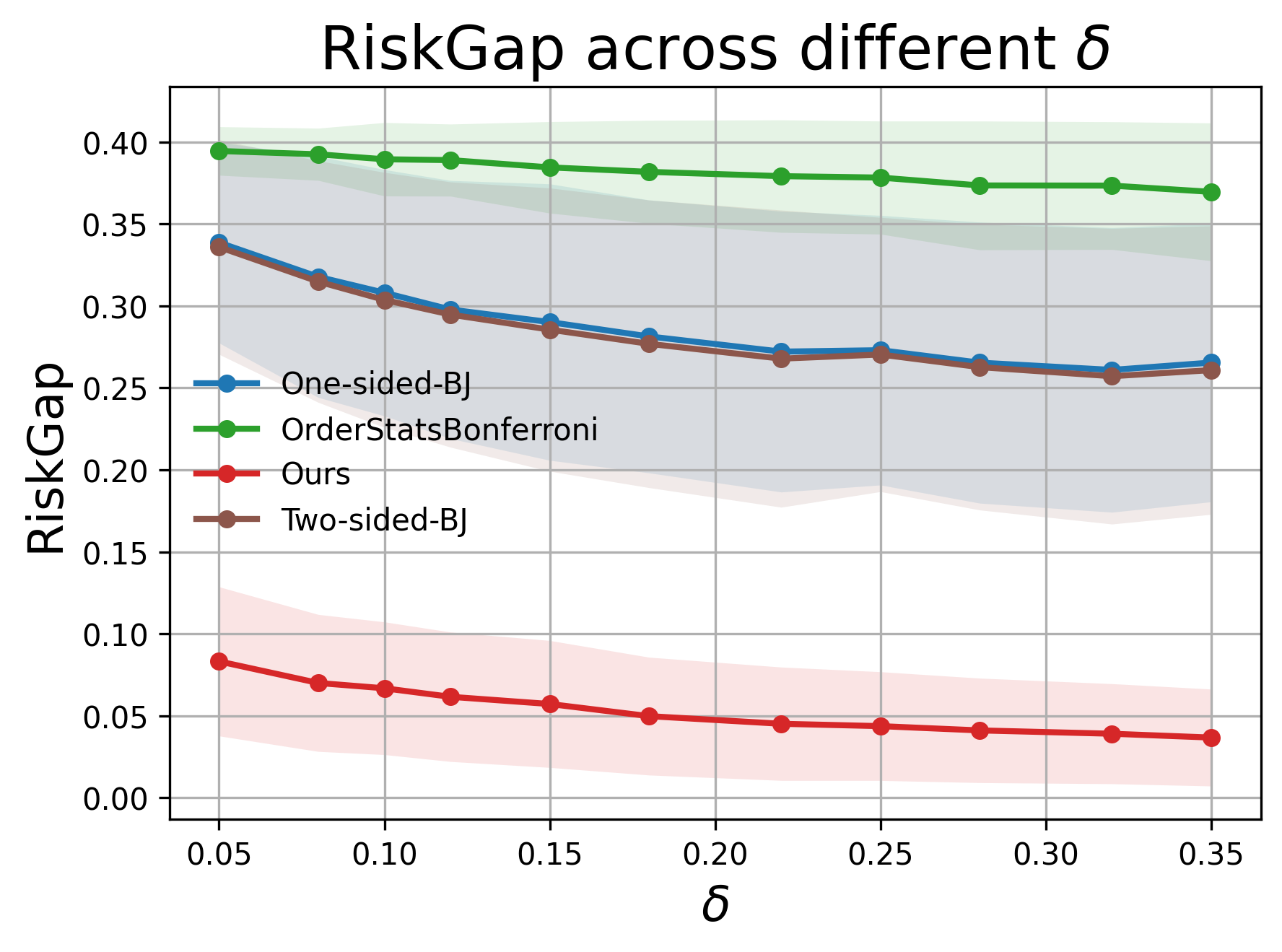}
    \end{minipage}\\
    \begin{minipage}{0.29\linewidth}
        \centering
        \includegraphics[width=\linewidth]{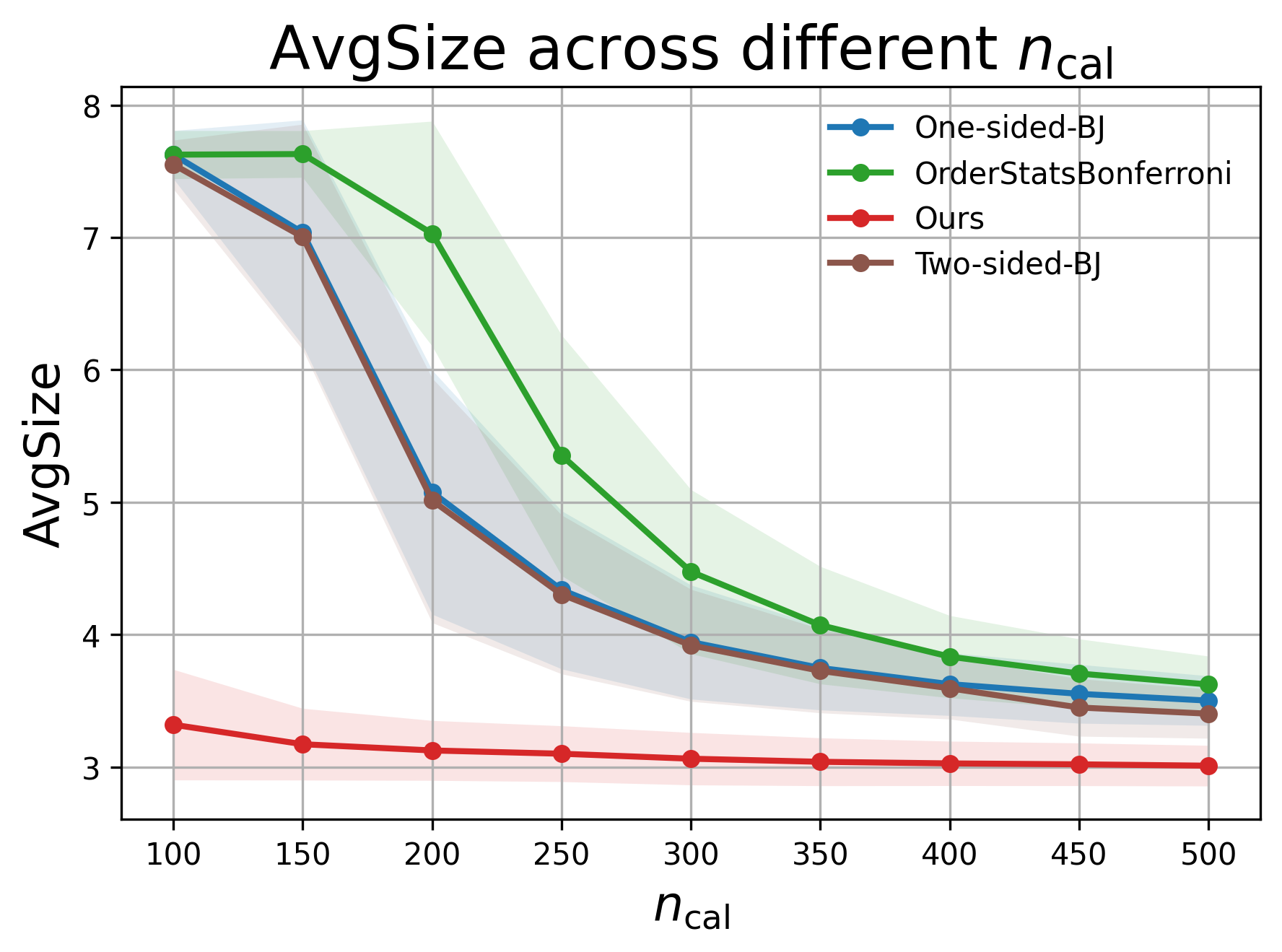}
    \end{minipage}\hfill
    \begin{minipage}{0.29\linewidth}
        \centering
        \includegraphics[width=\linewidth]{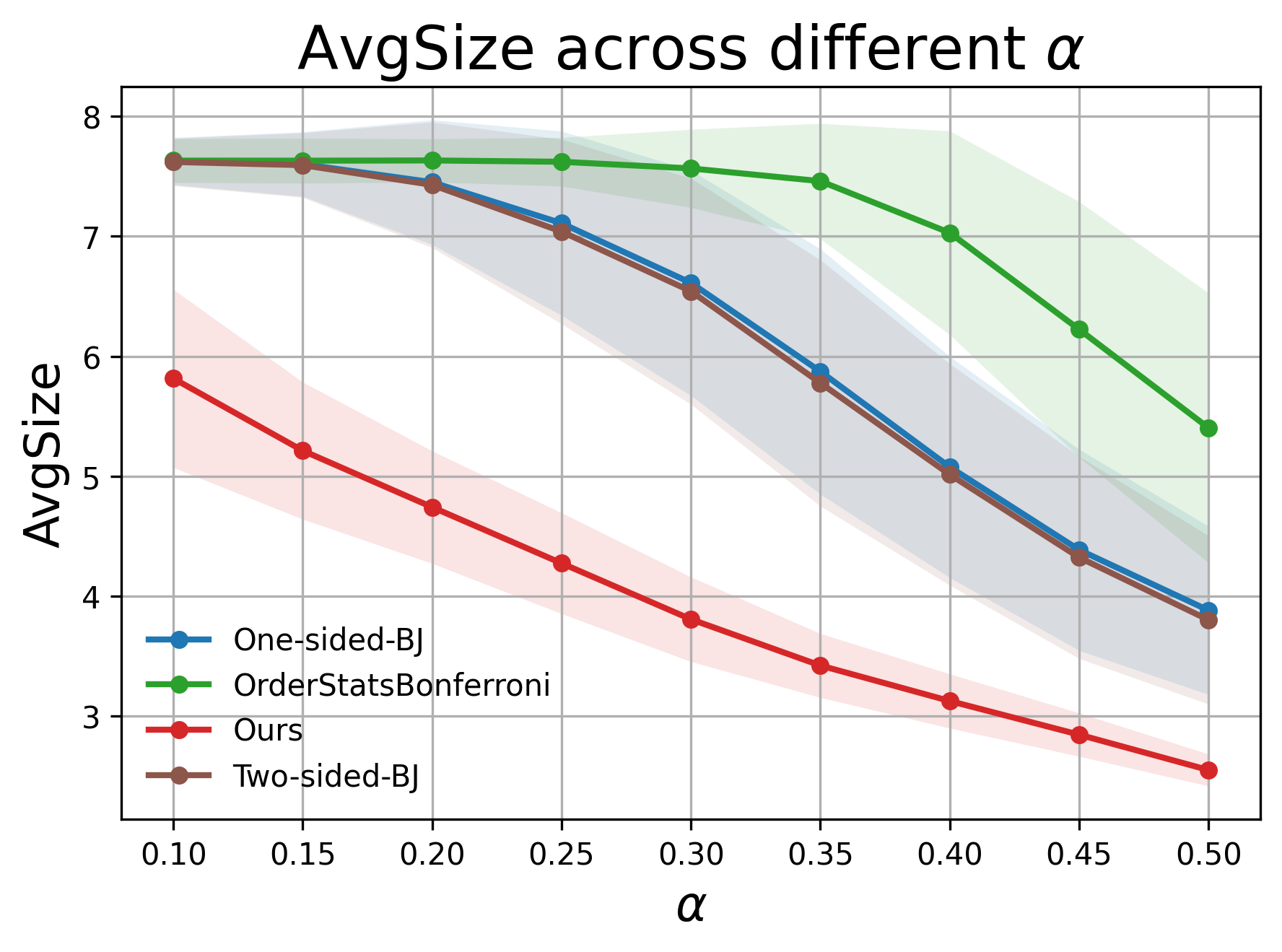}
    \end{minipage}\hfill
    \begin{minipage}{0.29\linewidth}
        \centering
        \includegraphics[width=\linewidth]{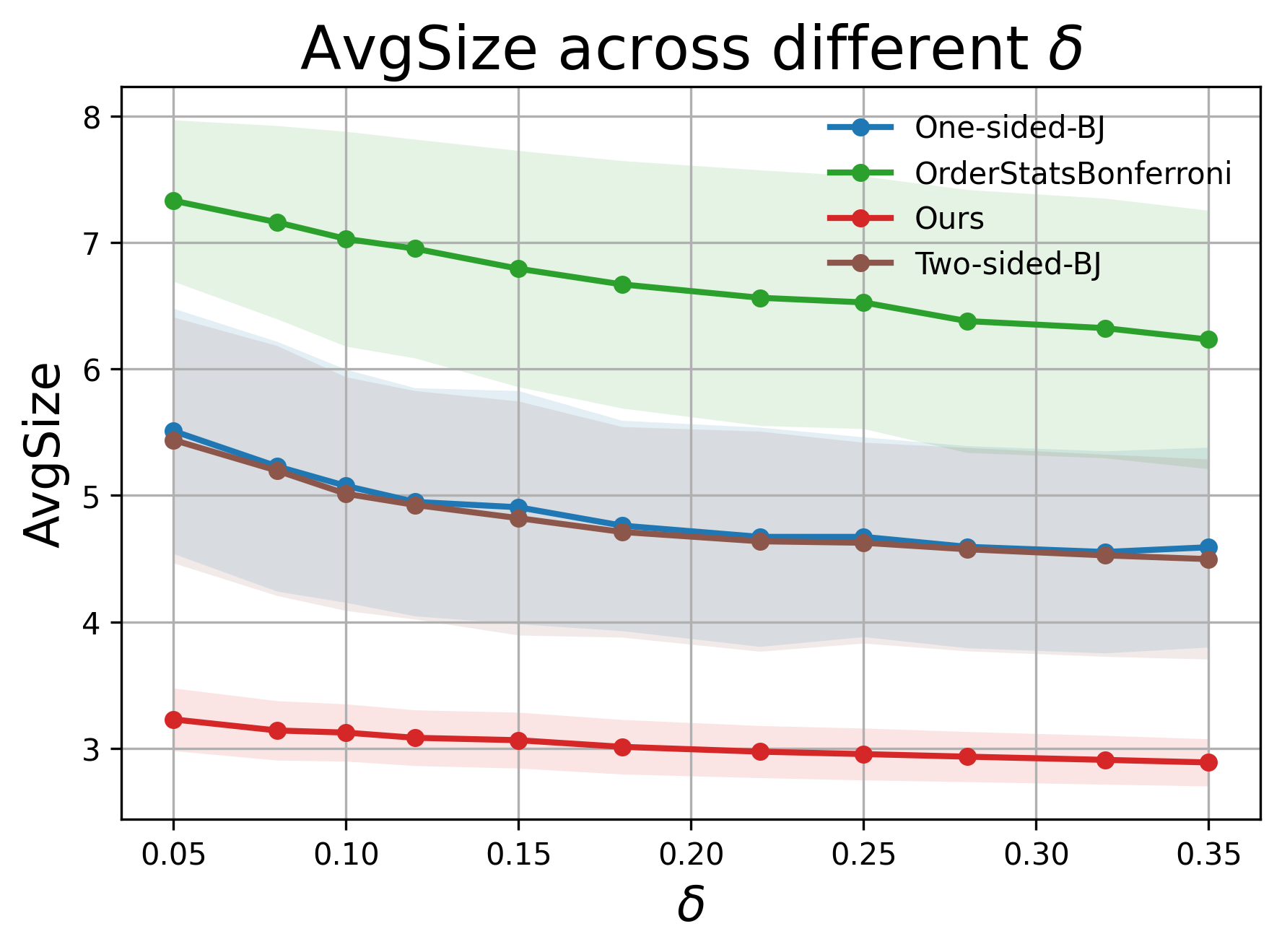}
    \end{minipage}
    \vspace{-3mm}
    \caption{\textbf{Performance comparison across various hyperparameter configurations.} The task is image multi-label classification, and the risk measure is $[0.85, 0.95]$-VaR-Interval.}
    \label{fig:parameter_analysis_coco_0.95}
    \vspace{-5mm}
\end{figure}

\begin{figure}[htbp]
    \centering
    \begin{minipage}{0.29\linewidth}
        \centering
        \includegraphics[width=\linewidth]{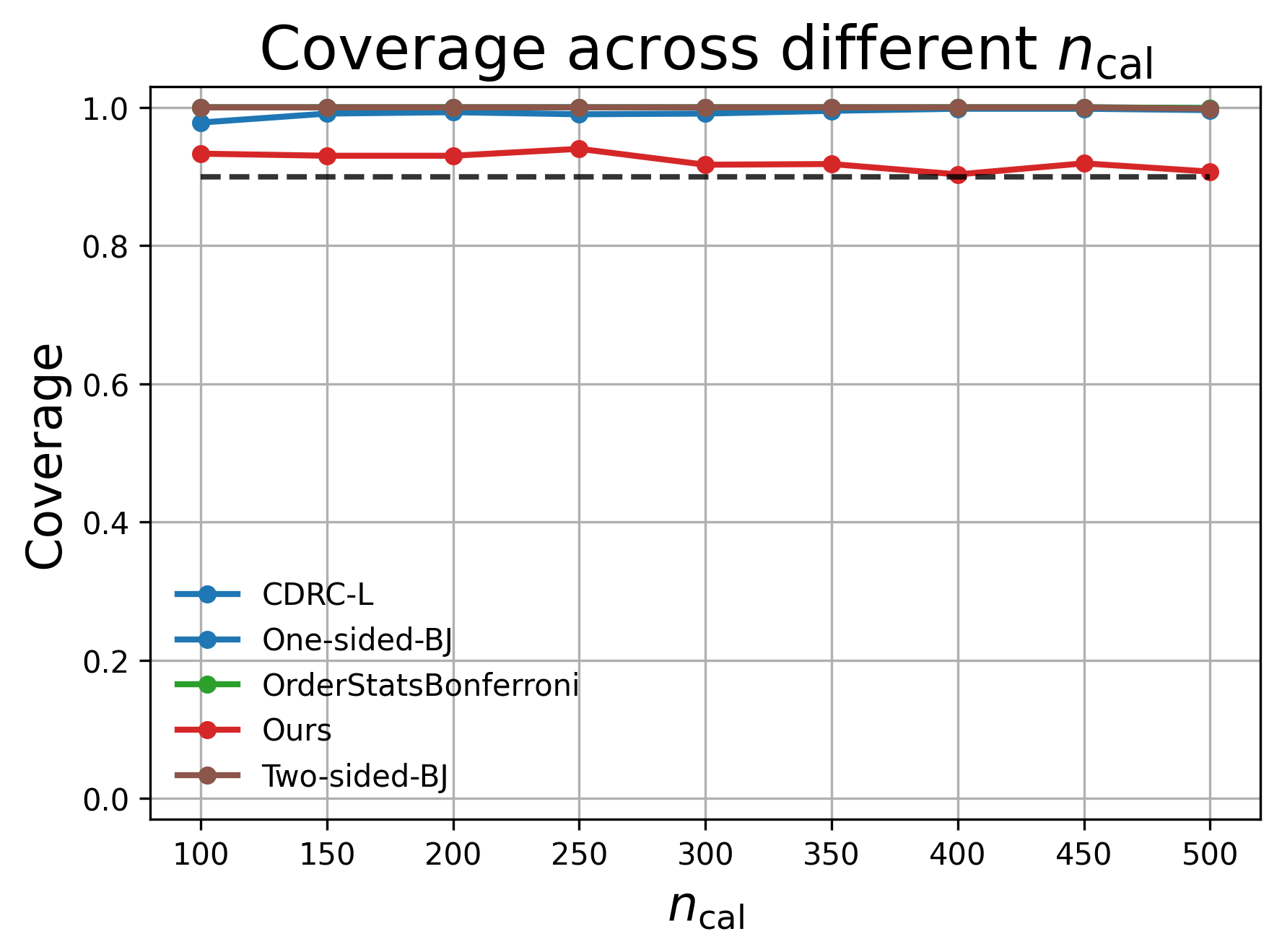}
    \end{minipage}\hfill
    \begin{minipage}{0.29\linewidth}
        \centering
        \includegraphics[width=\linewidth]{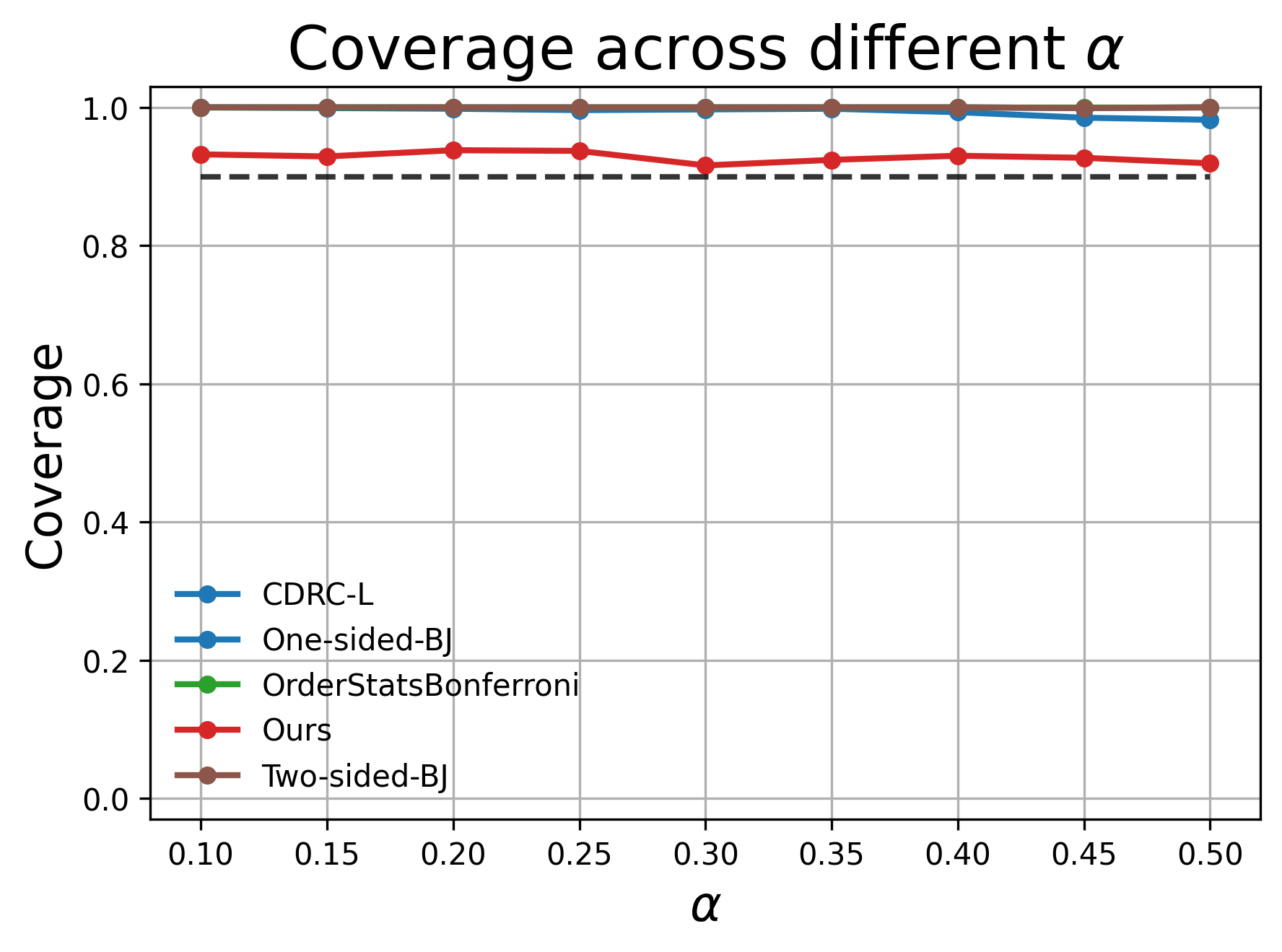}
    \end{minipage}\hfill
    \begin{minipage}{0.29\linewidth}
        \centering
        \includegraphics[width=\linewidth]{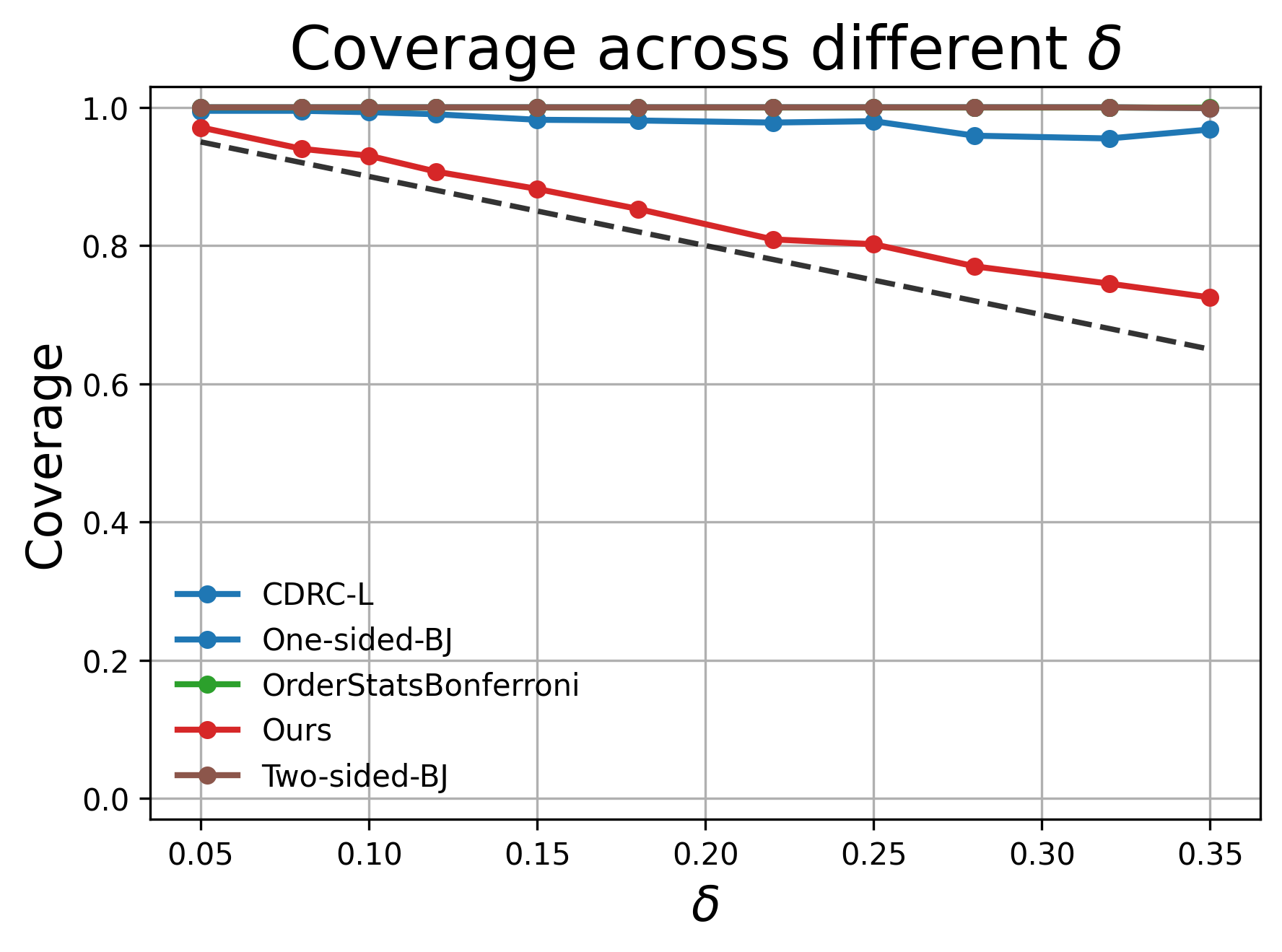}
    \end{minipage}\\
    \begin{minipage}{0.29\linewidth}
        \centering
        \includegraphics[width=\linewidth]{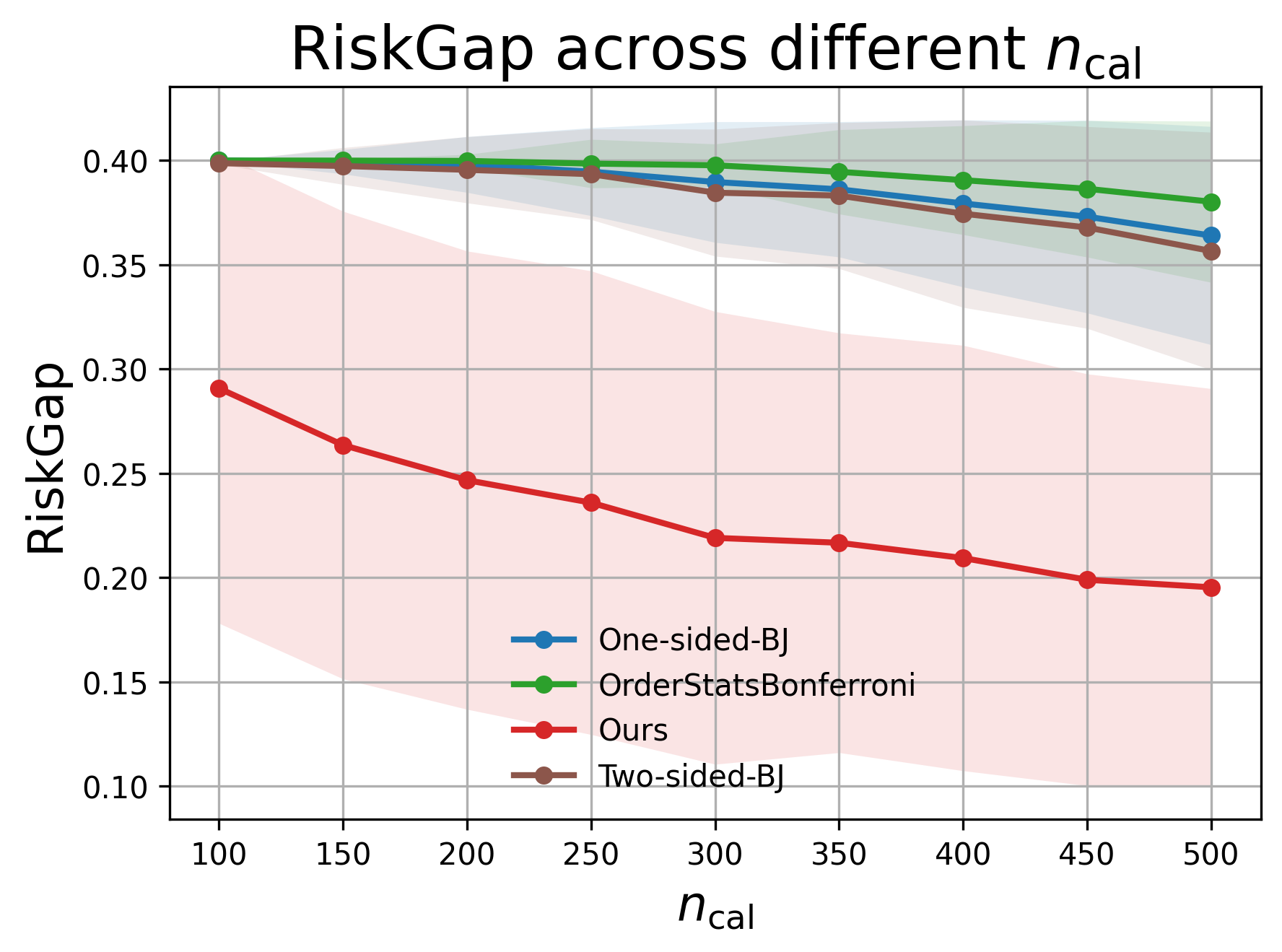}
    \end{minipage}\hfill
    \begin{minipage}{0.29\linewidth}
        \centering
        \includegraphics[width=\linewidth]{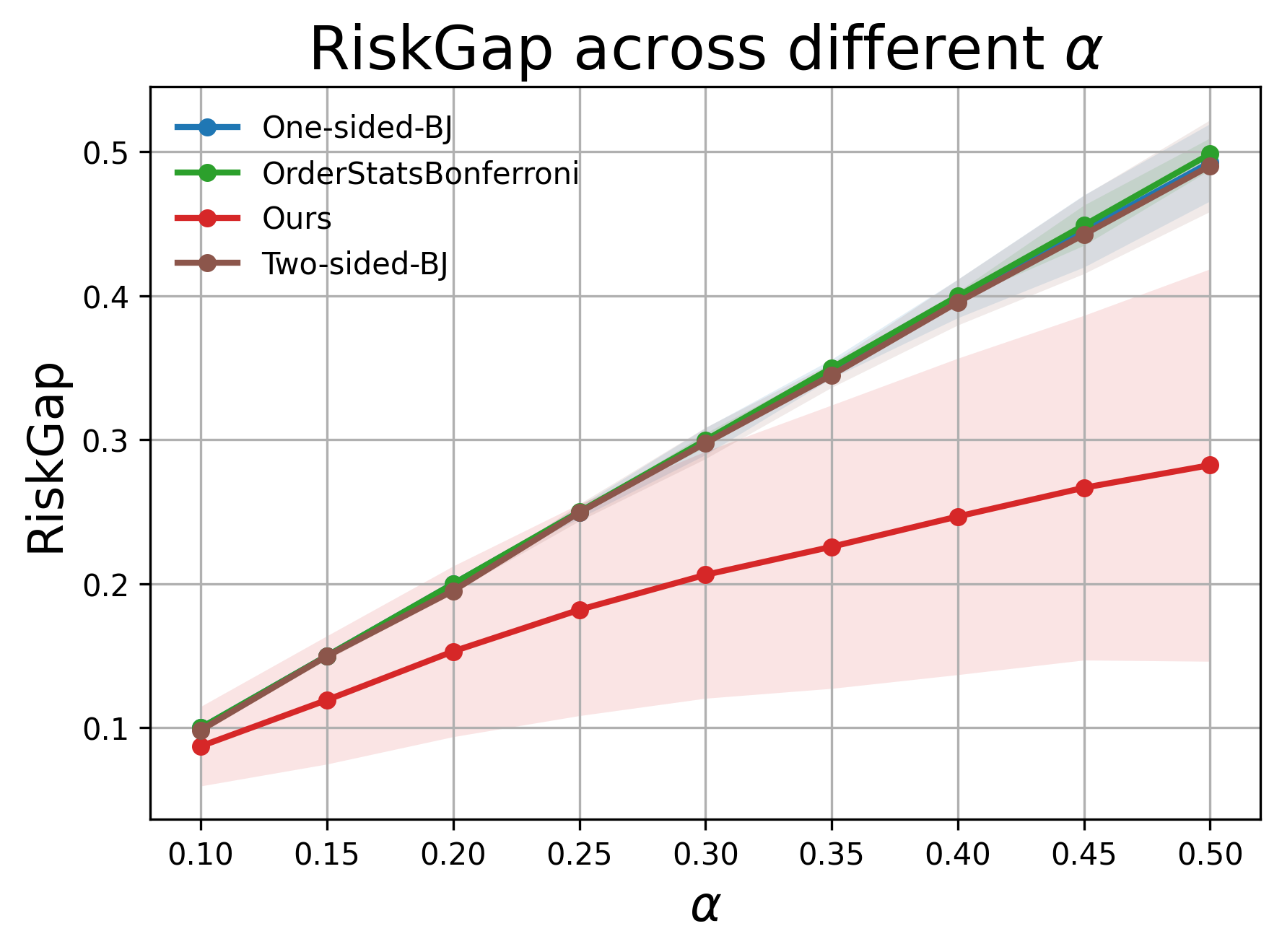}
    \end{minipage}\hfill
    \begin{minipage}{0.29\linewidth}
        \centering
        \includegraphics[width=\linewidth]{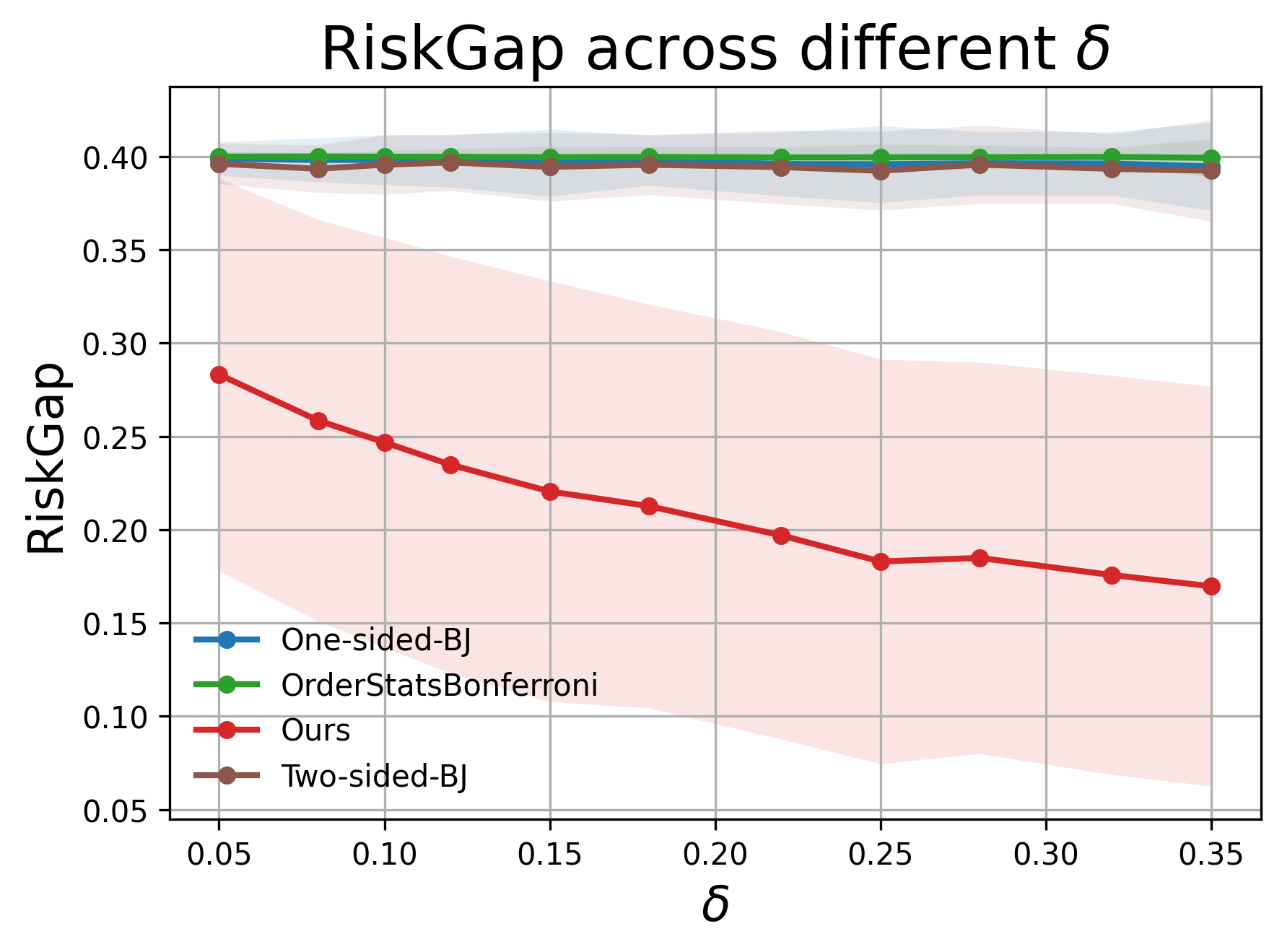}
    \end{minipage}\\
    \begin{minipage}{0.29\linewidth}
        \centering
        \includegraphics[width=\linewidth]{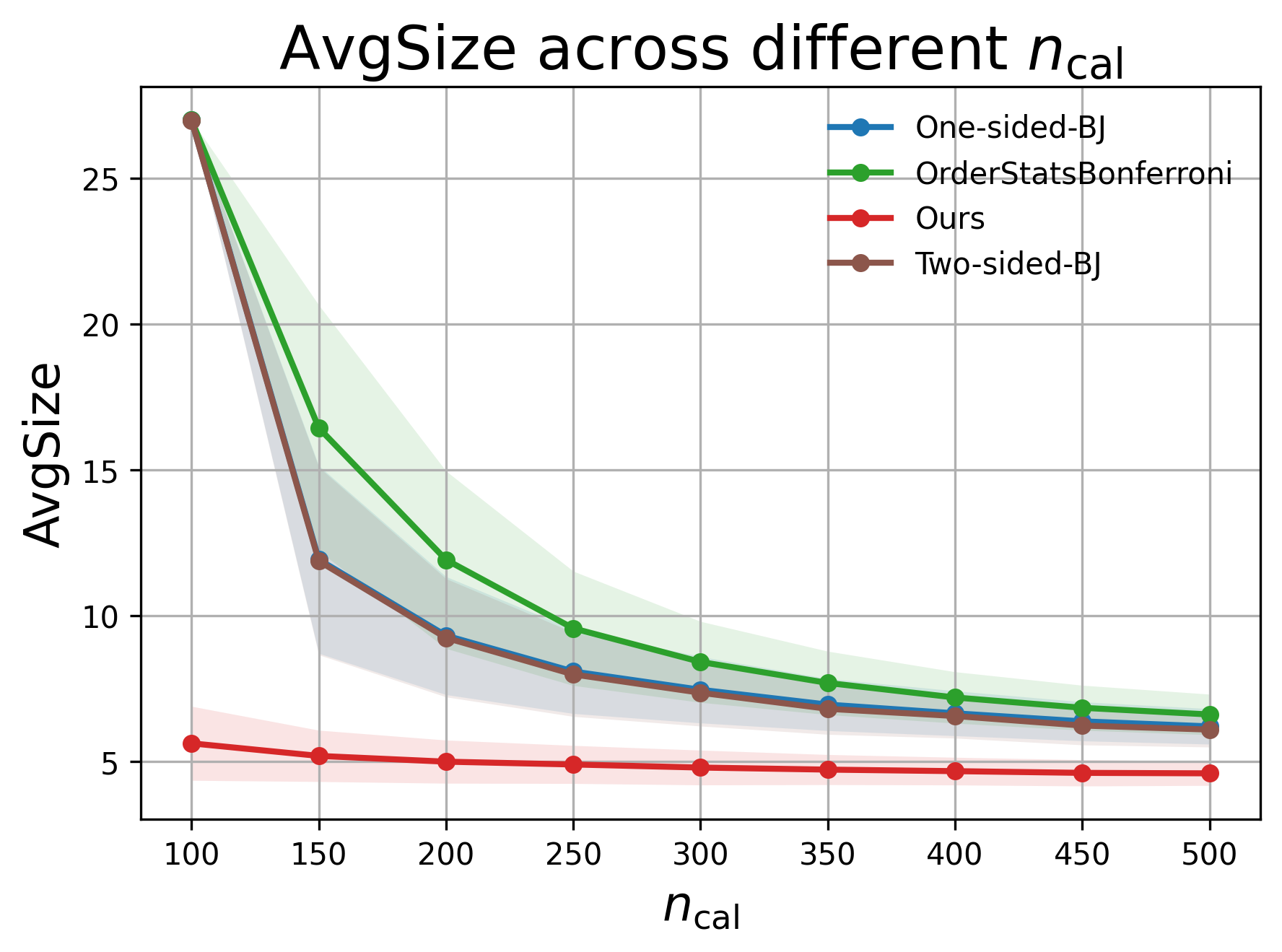}
    \end{minipage}\hfill
    \begin{minipage}{0.29\linewidth}
        \centering
        \includegraphics[width=\linewidth]{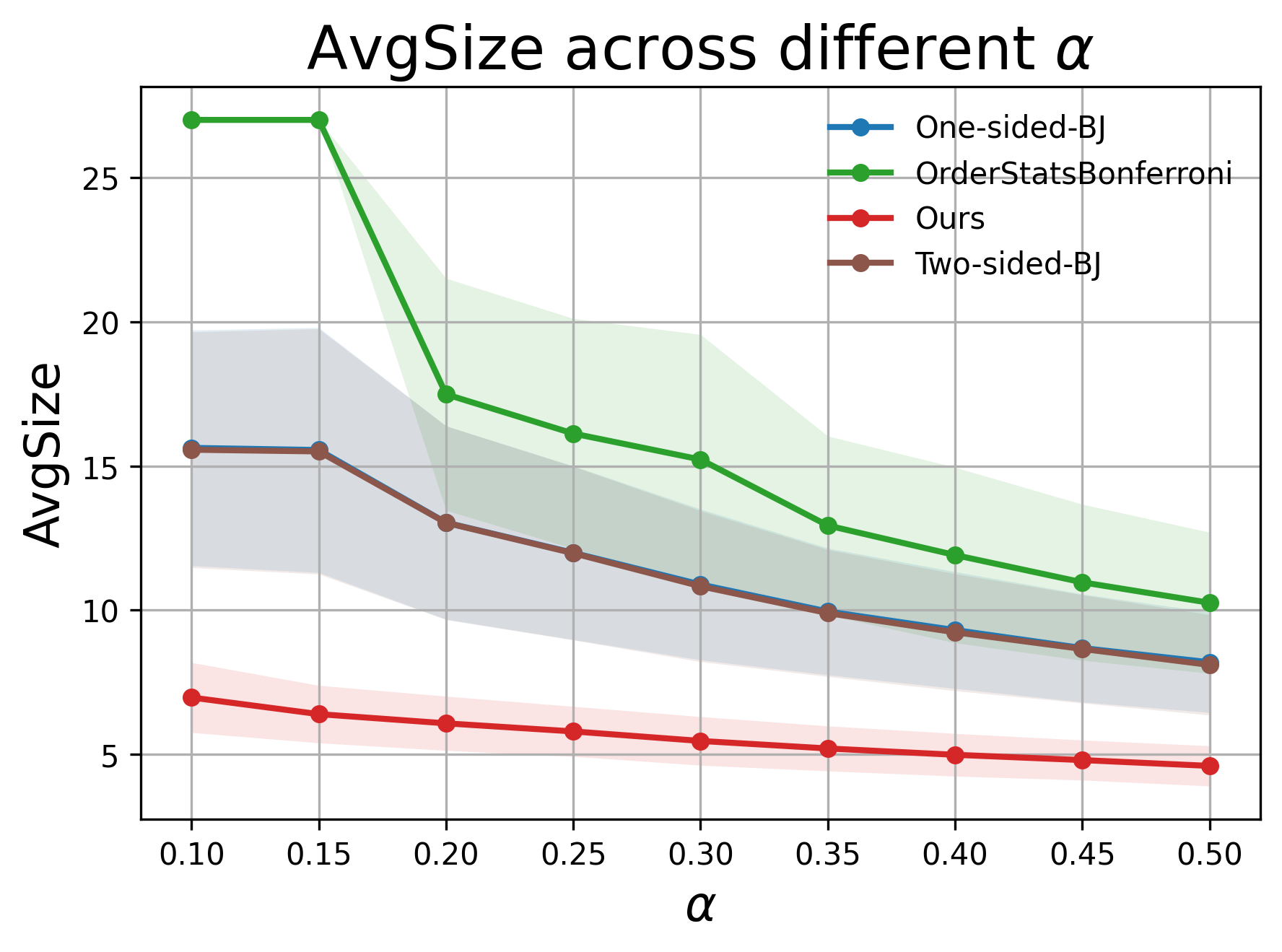}
    \end{minipage}\hfill
    \begin{minipage}{0.29\linewidth}
        \centering
        \includegraphics[width=\linewidth]{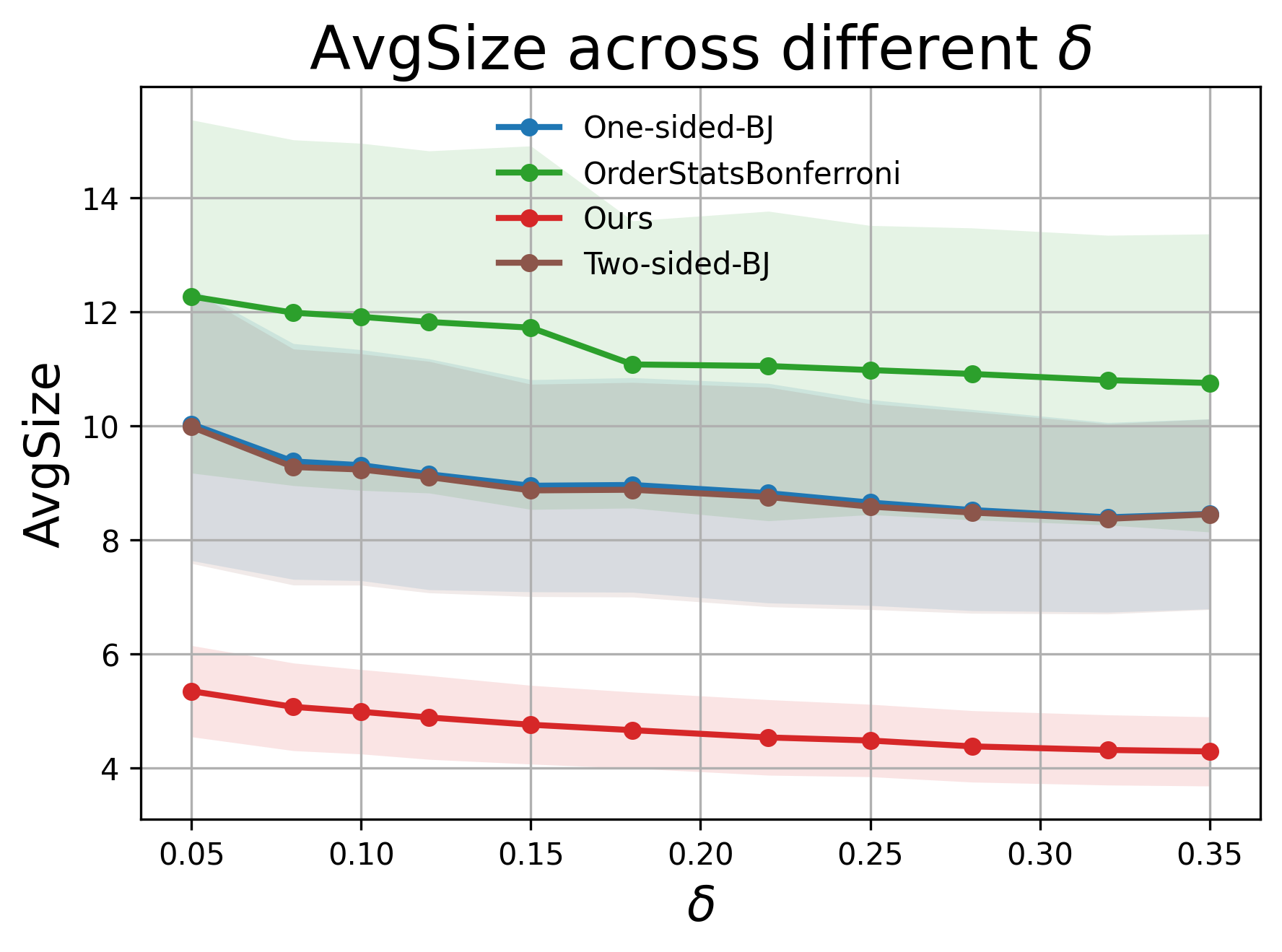}
    \end{minipage}
    \vspace{-3mm}
    \caption{\textbf{Performance comparison across various hyperparameter configurations.} The task is text emotion recognition, and the risk measure is $[0.85, 0.95]$-VaR-Interval.}
    \label{fig:parameter_analysis_go_0.95}
    \vspace{-5mm}
\end{figure}

\begin{figure}[htbp]
    \centering
    \begin{minipage}{0.29\linewidth}
        \centering
        \includegraphics[width=\linewidth]{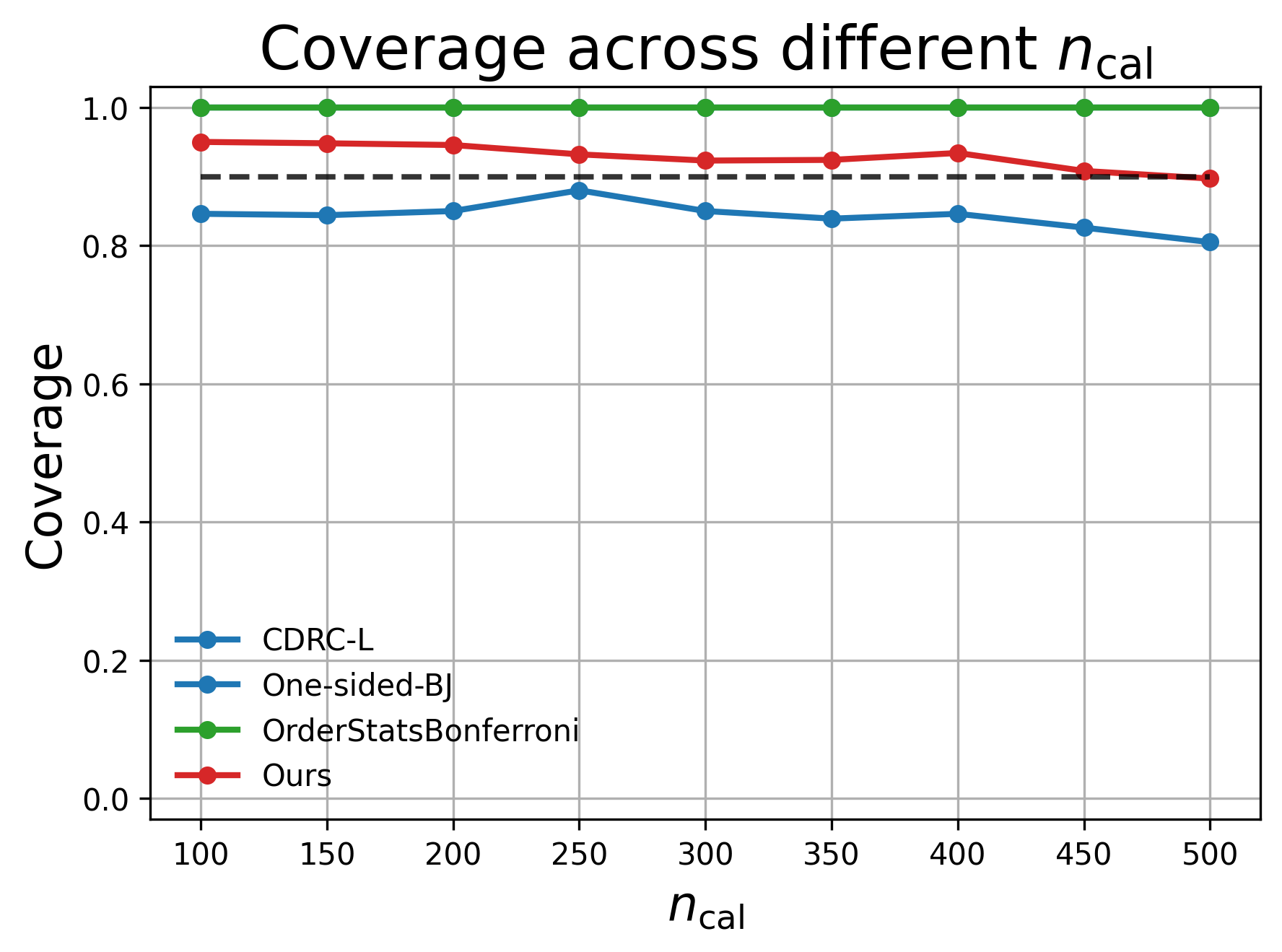}
    \end{minipage}\hfill
    \begin{minipage}{0.29\linewidth}
        \centering
        \includegraphics[width=\linewidth]{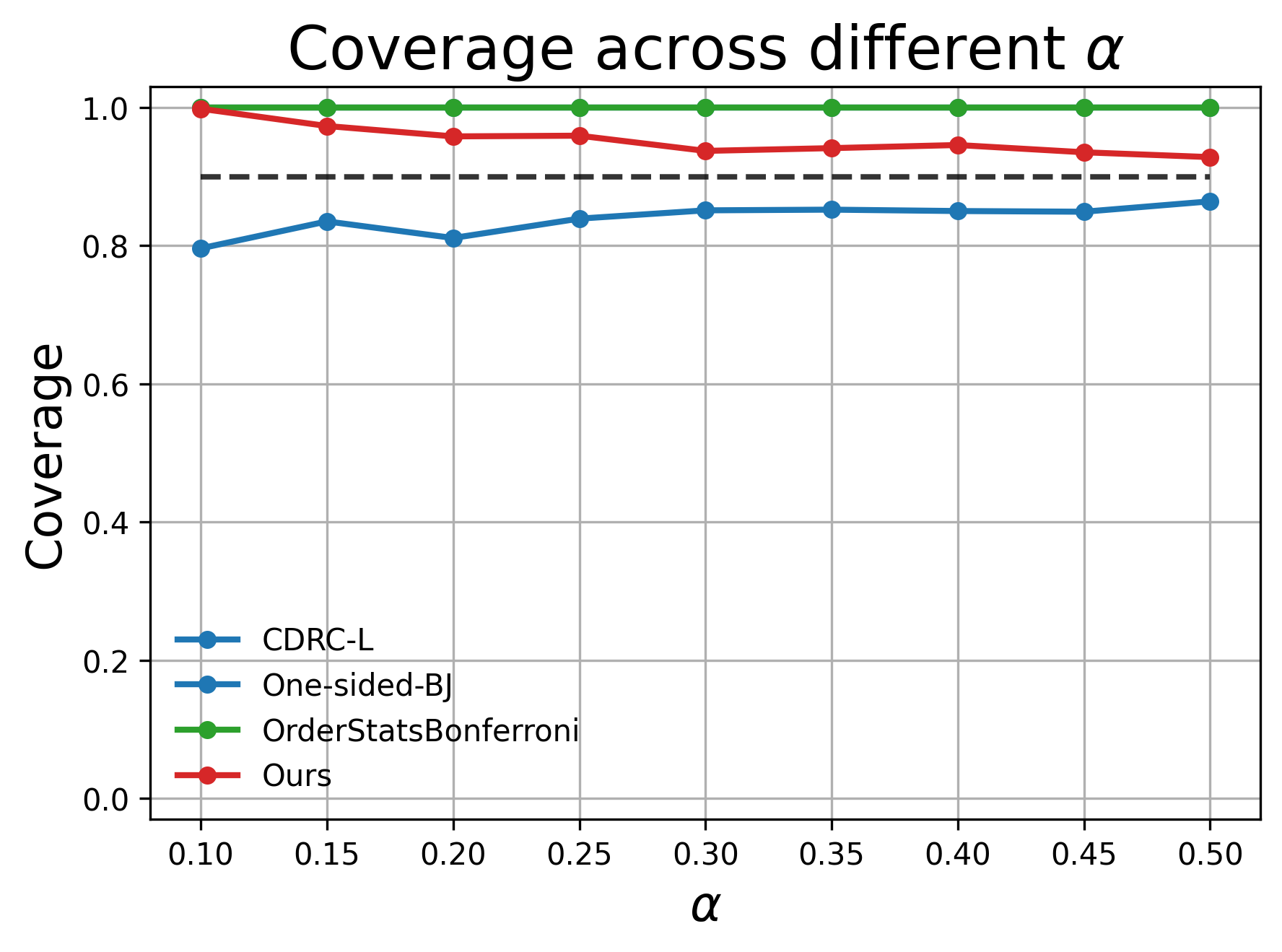}
    \end{minipage}\hfill
    \begin{minipage}{0.29\linewidth}
        \centering
        \includegraphics[width=\linewidth]{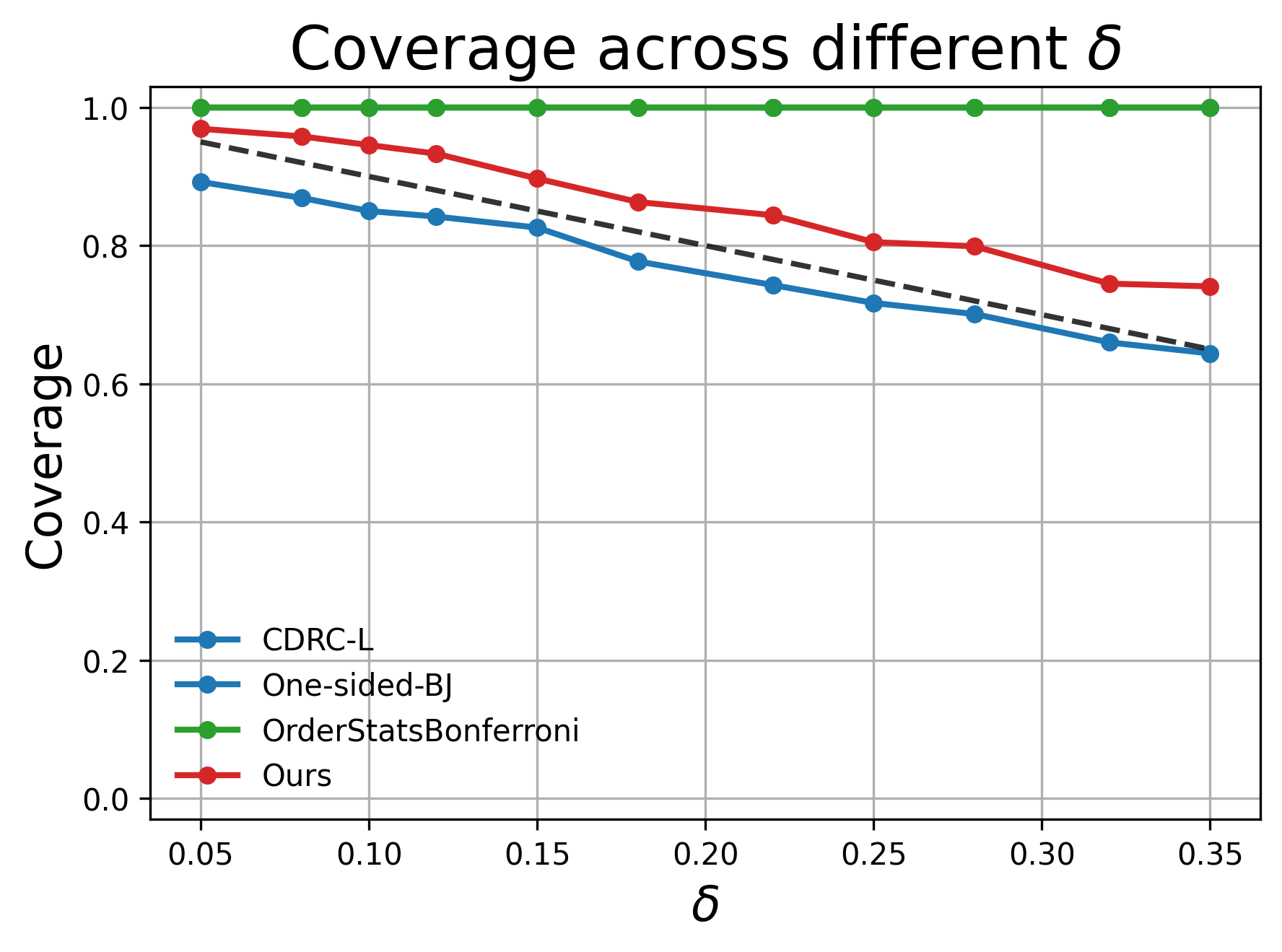}
    \end{minipage}\\
    \begin{minipage}{0.29\linewidth}
        \centering
        \includegraphics[width=\linewidth]{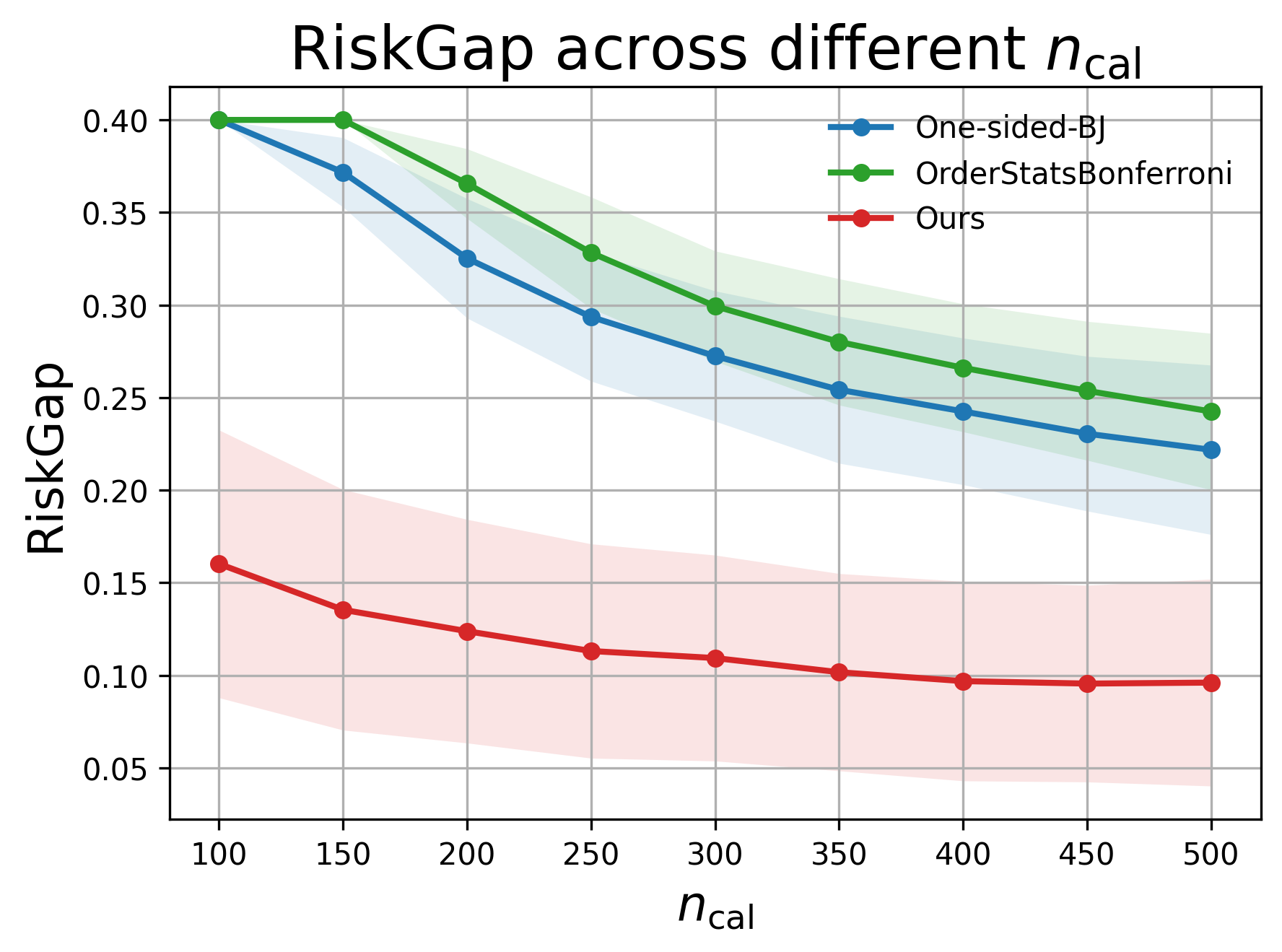}
    \end{minipage}\hfill
    \begin{minipage}{0.29\linewidth}
        \centering
        \includegraphics[width=\linewidth]{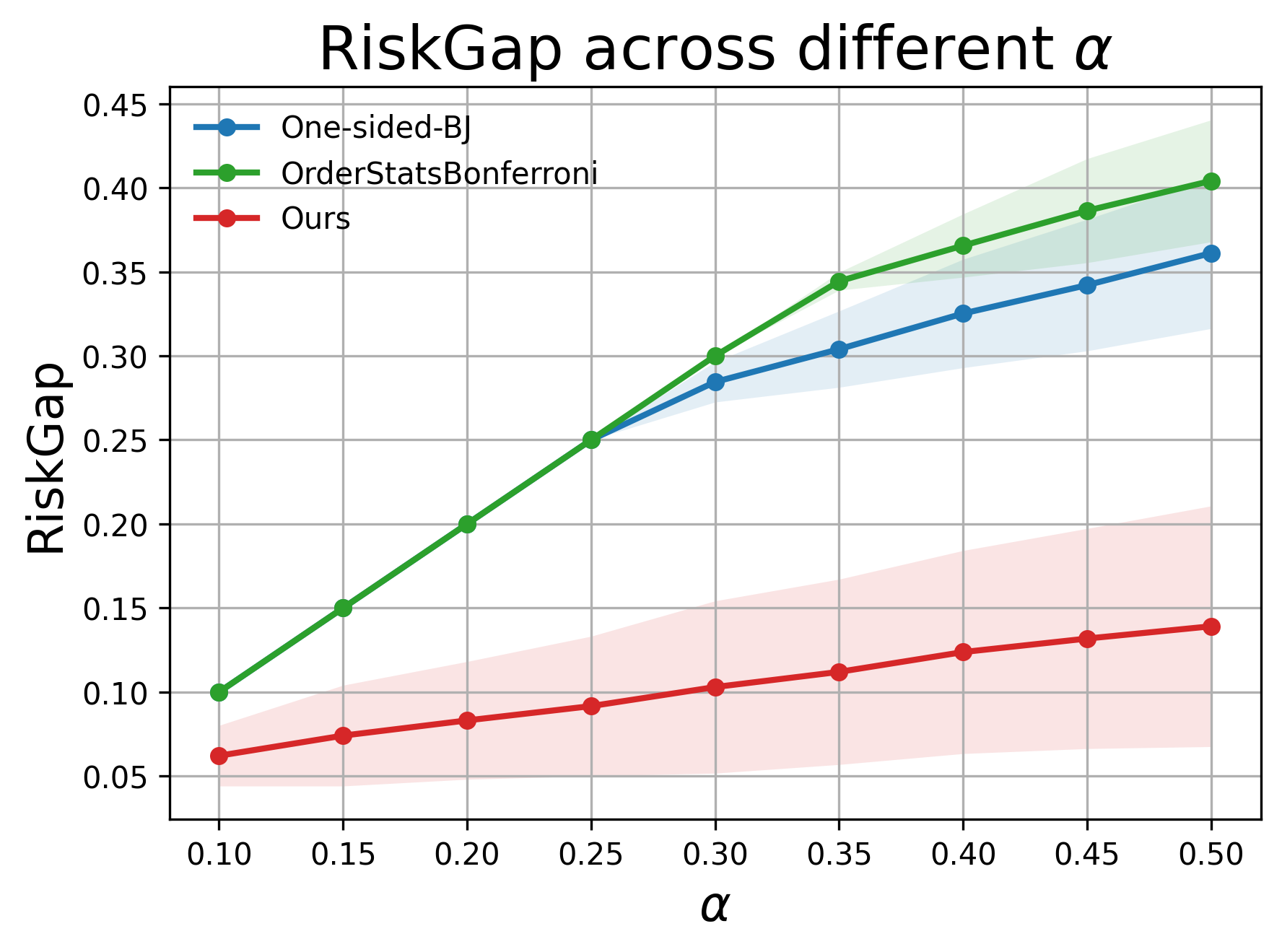}
    \end{minipage}\hfill
    \begin{minipage}{0.29\linewidth}
        \centering
        \includegraphics[width=\linewidth]{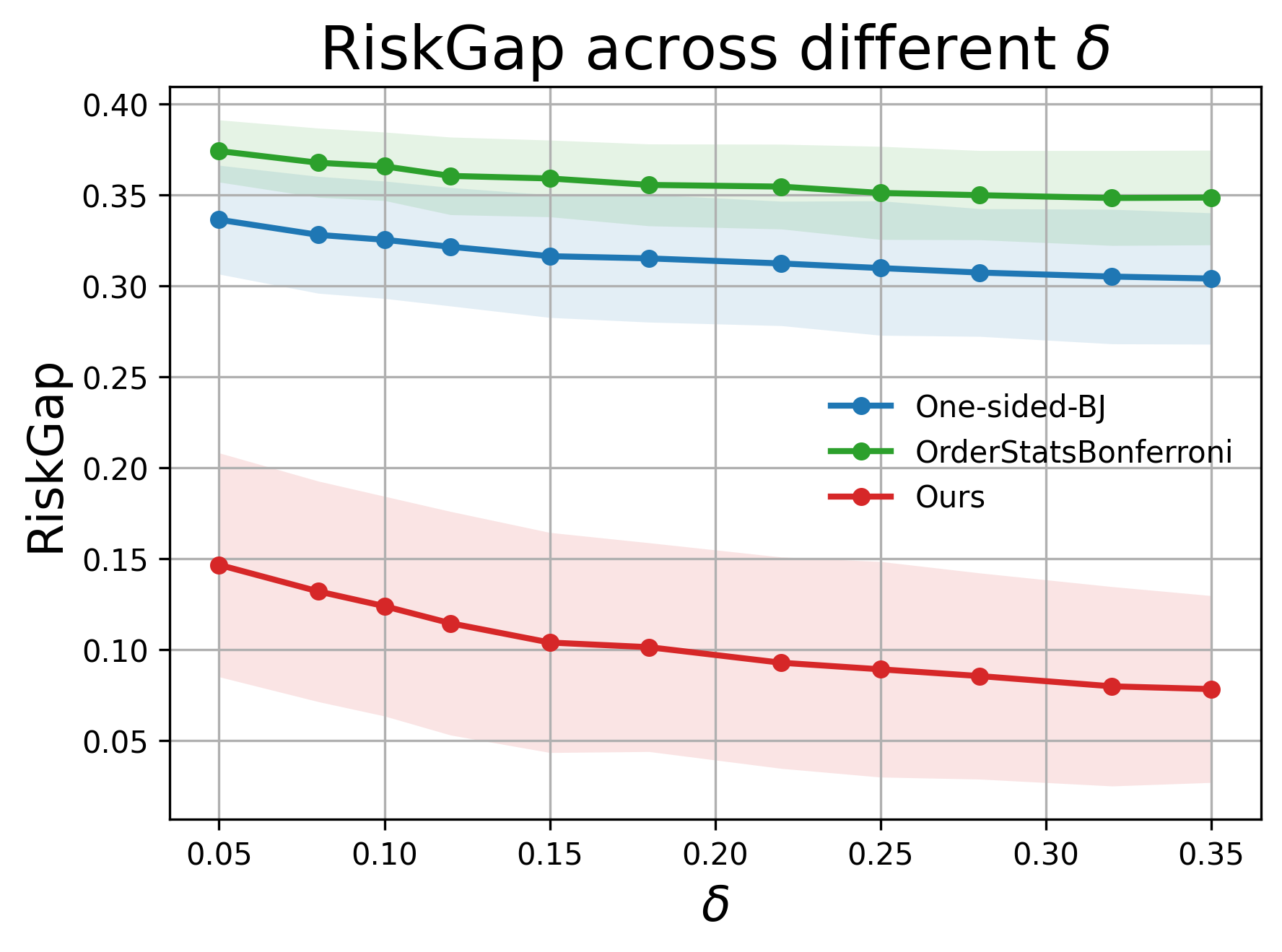}
    \end{minipage}\\
    \begin{minipage}{0.29\linewidth}
        \centering
        \includegraphics[width=\linewidth]{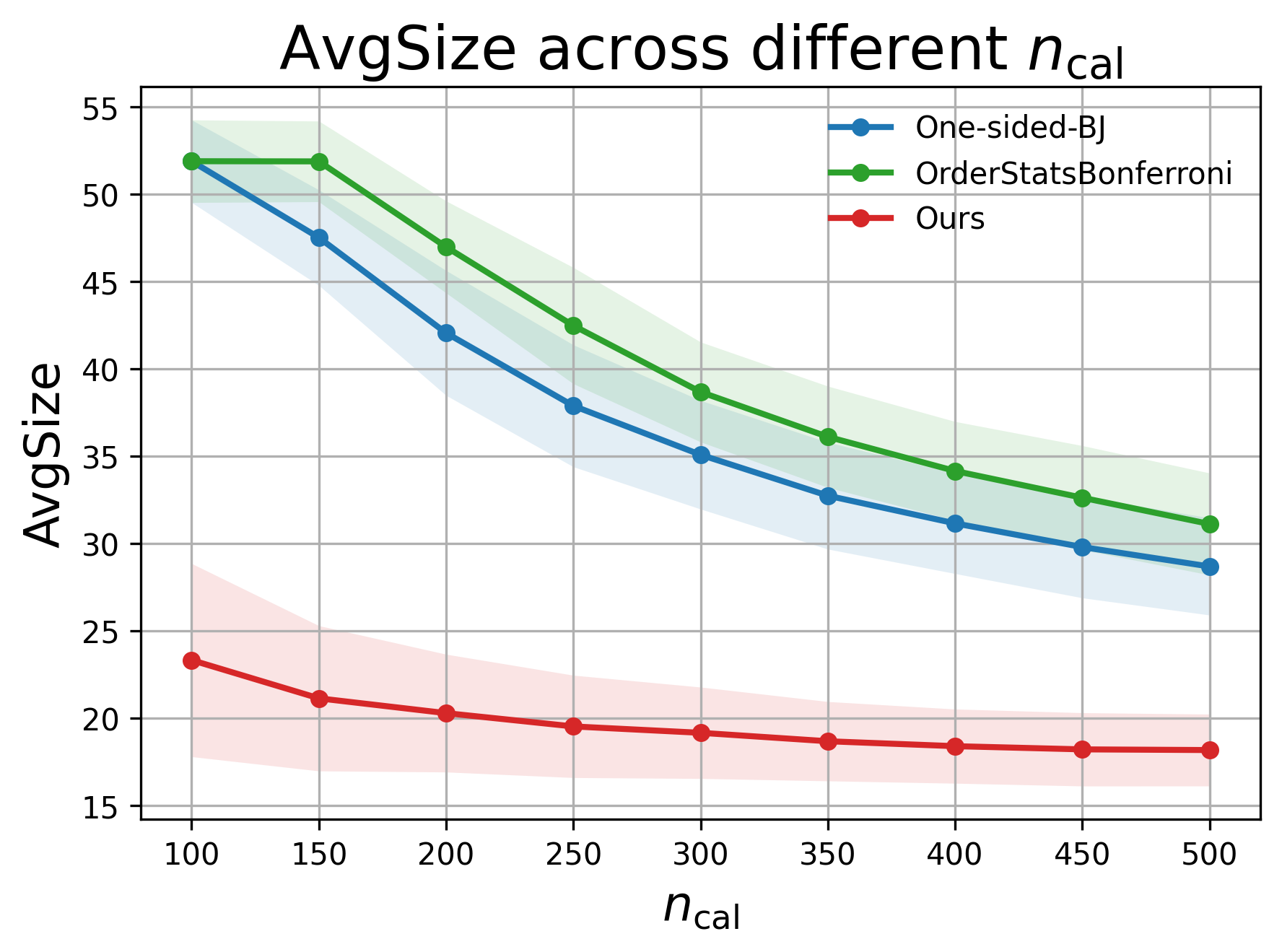}
    \end{minipage}\hfill
    \begin{minipage}{0.29\linewidth}
        \centering
        \includegraphics[width=\linewidth]{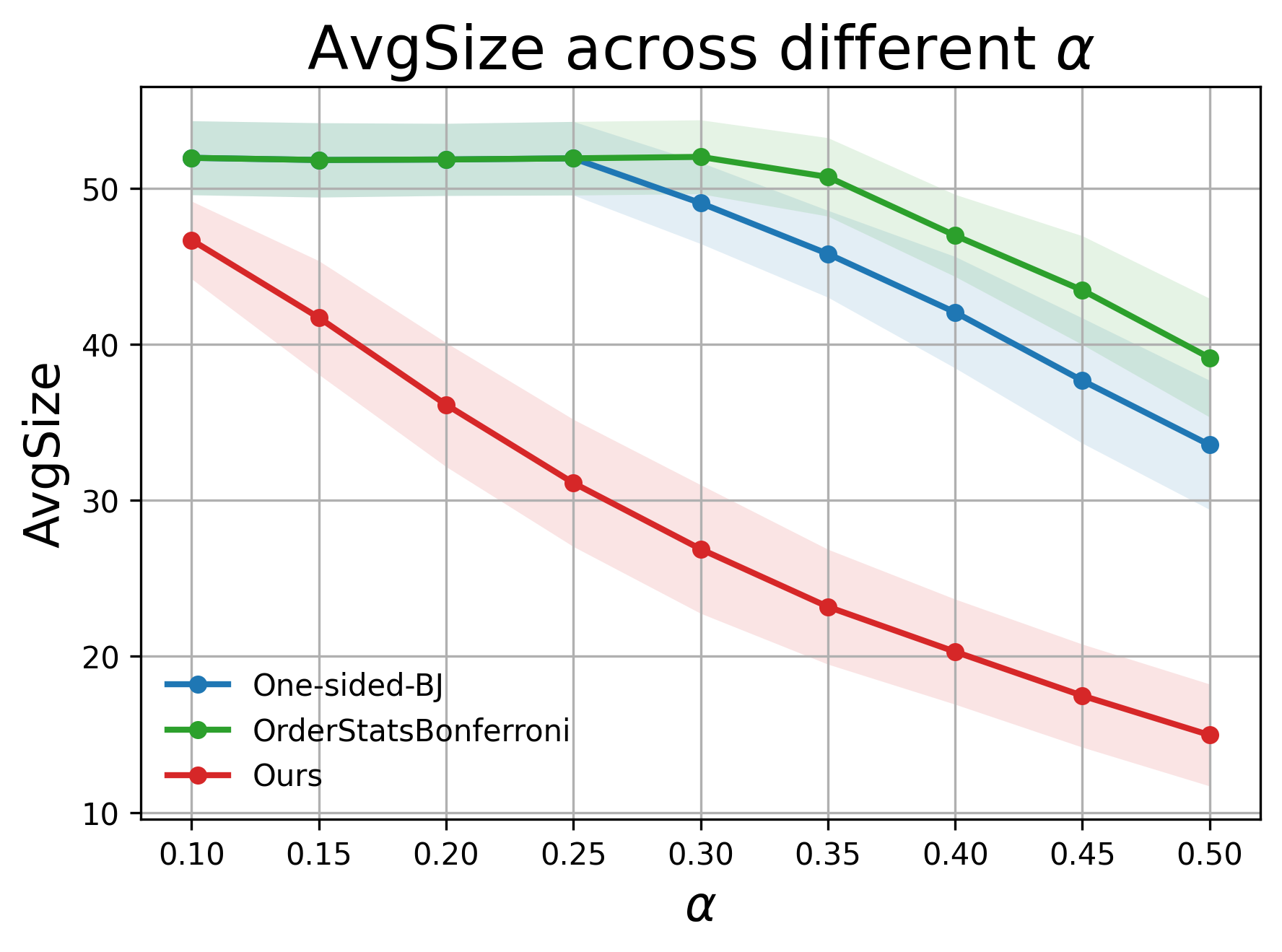}
    \end{minipage}\hfill
    \begin{minipage}{0.29\linewidth}
        \centering
        \includegraphics[width=\linewidth]{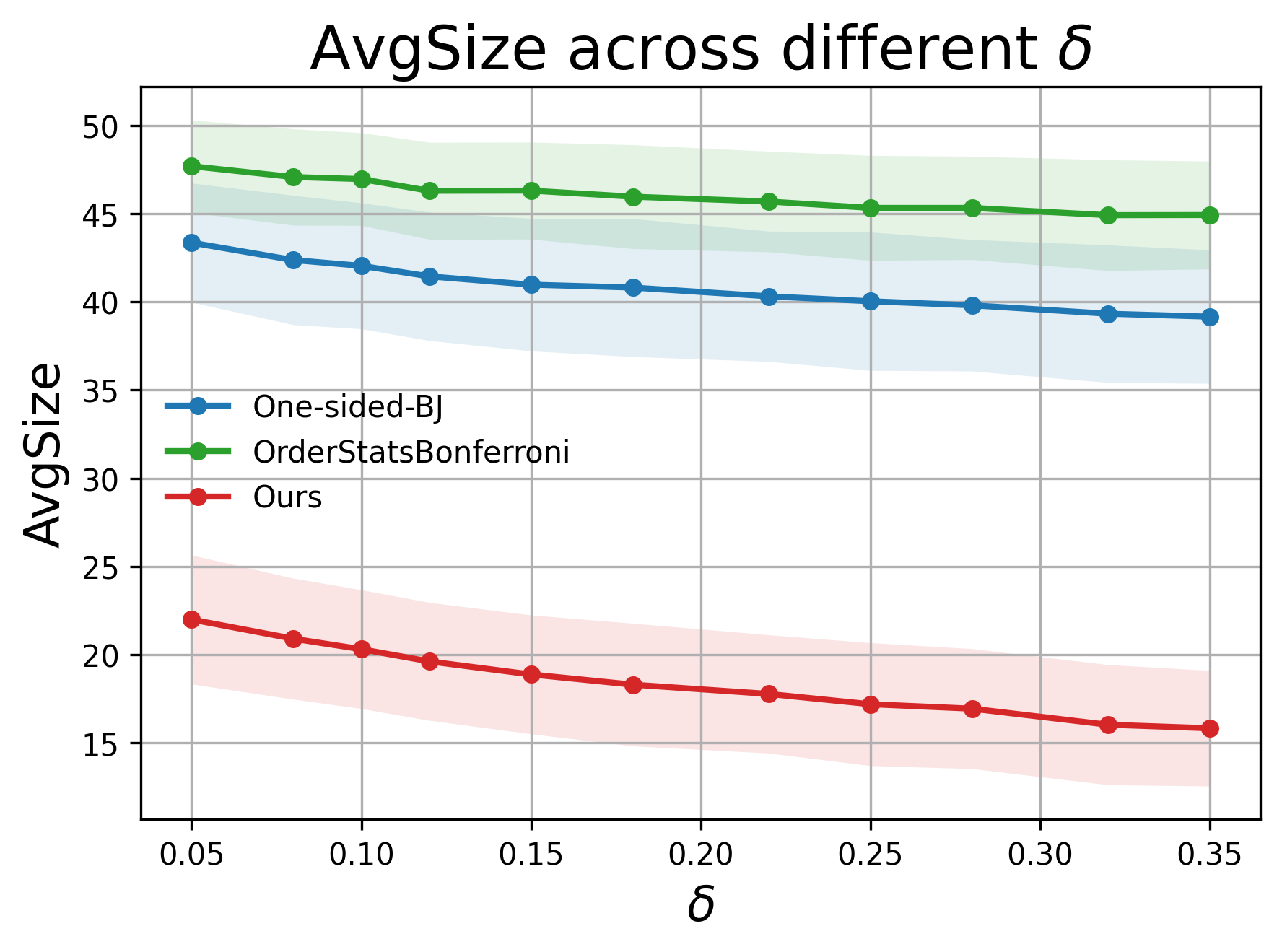}
    \end{minipage}
    \vspace{-3mm}
    \caption{\textbf{Performance comparison across various hyperparameter configurations.} The task is tumor segmentation, and the risk measure is $0.9$-CVaR.}
    \label{fig:parameter_analysis_polyp_1}
    \vspace{-5mm}
\end{figure}

\begin{figure}[htbp]
    \centering
    \begin{minipage}{0.29\linewidth}
        \centering
        \includegraphics[width=\linewidth]{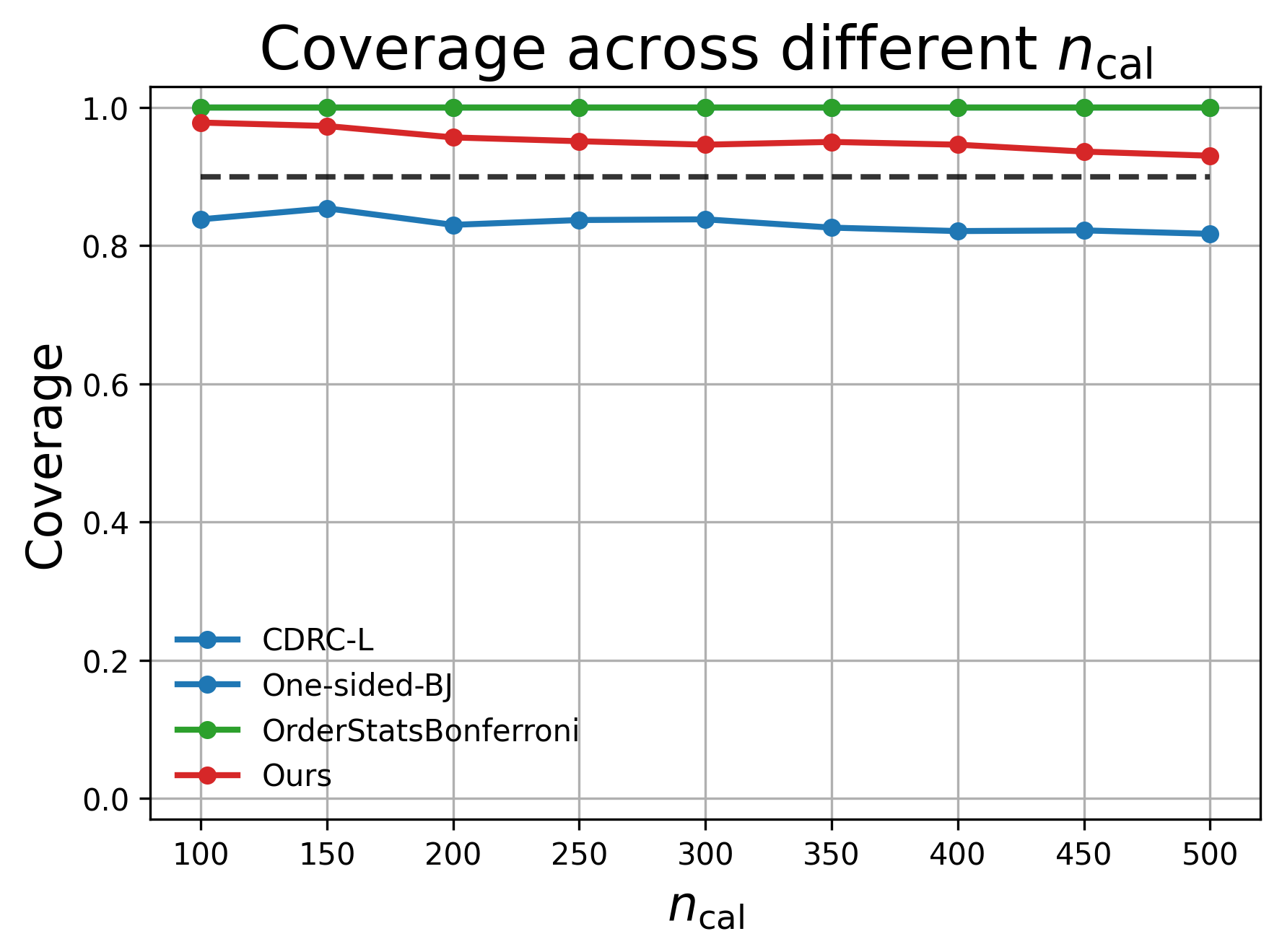}
    \end{minipage}\hfill
    \begin{minipage}{0.29\linewidth}
        \centering
        \includegraphics[width=\linewidth]{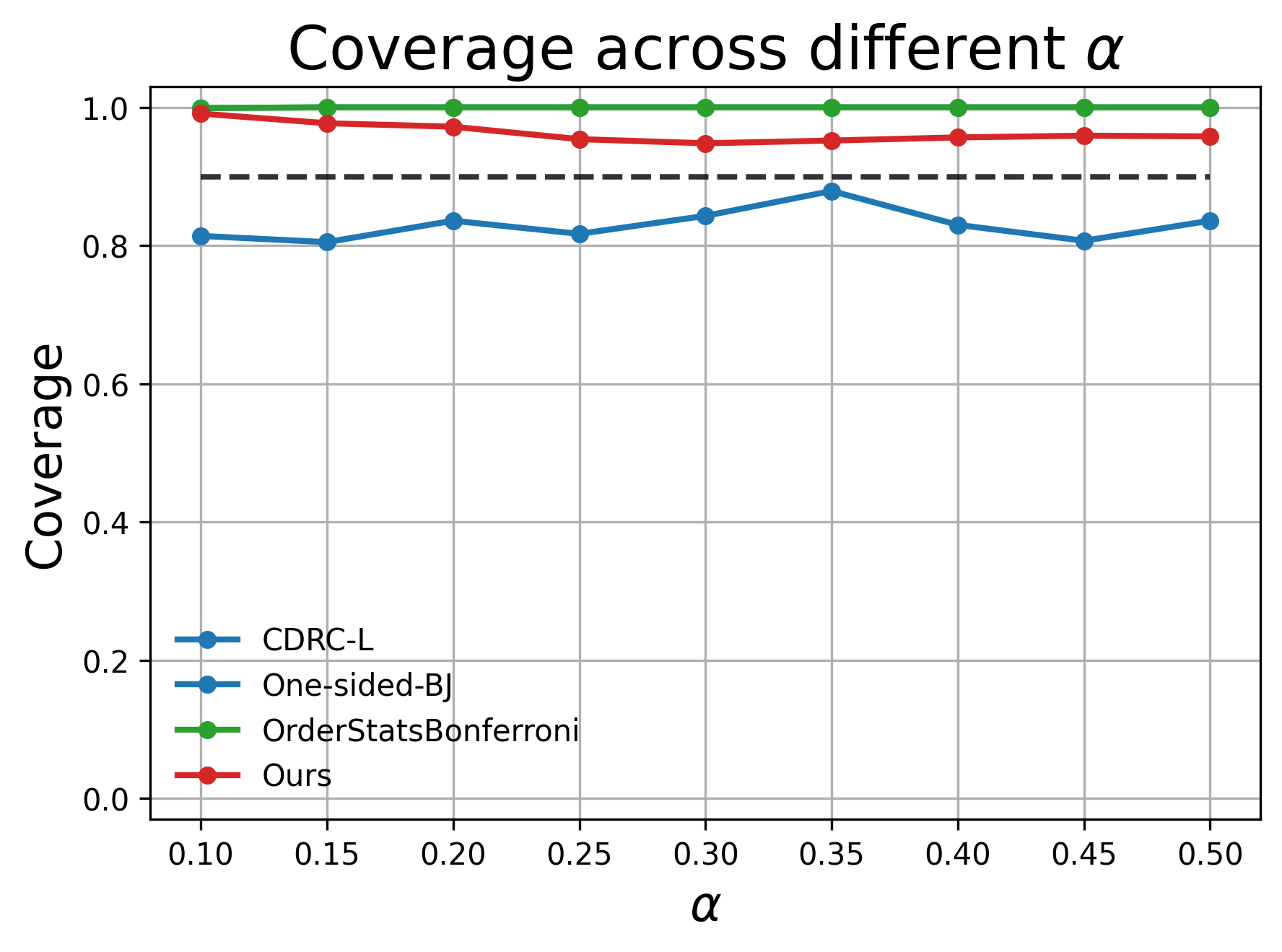}
    \end{minipage}\hfill
    \begin{minipage}{0.29\linewidth}
        \centering
        \includegraphics[width=\linewidth]{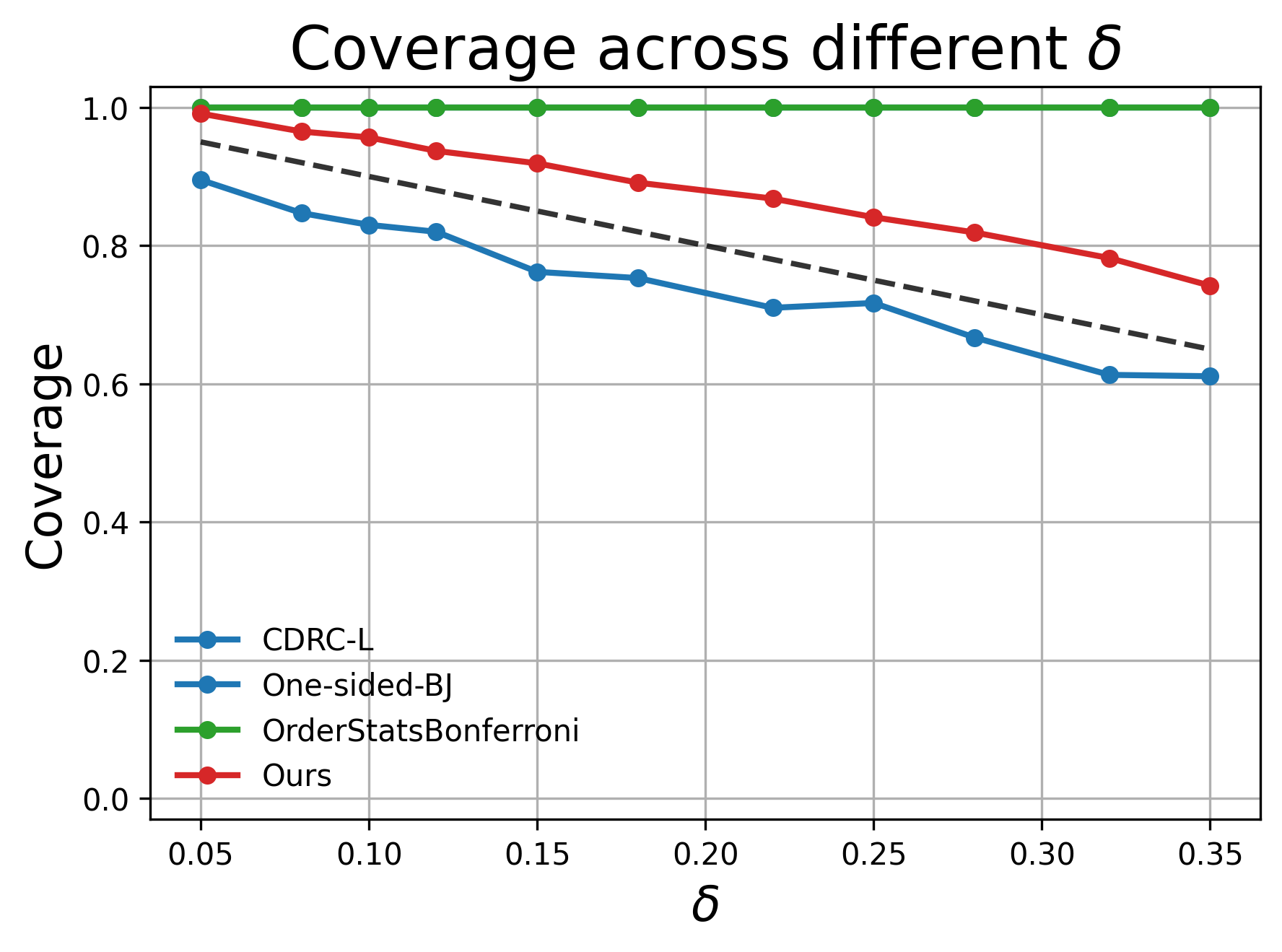}
    \end{minipage}\\
    \begin{minipage}{0.29\linewidth}
        \centering
        \includegraphics[width=\linewidth]{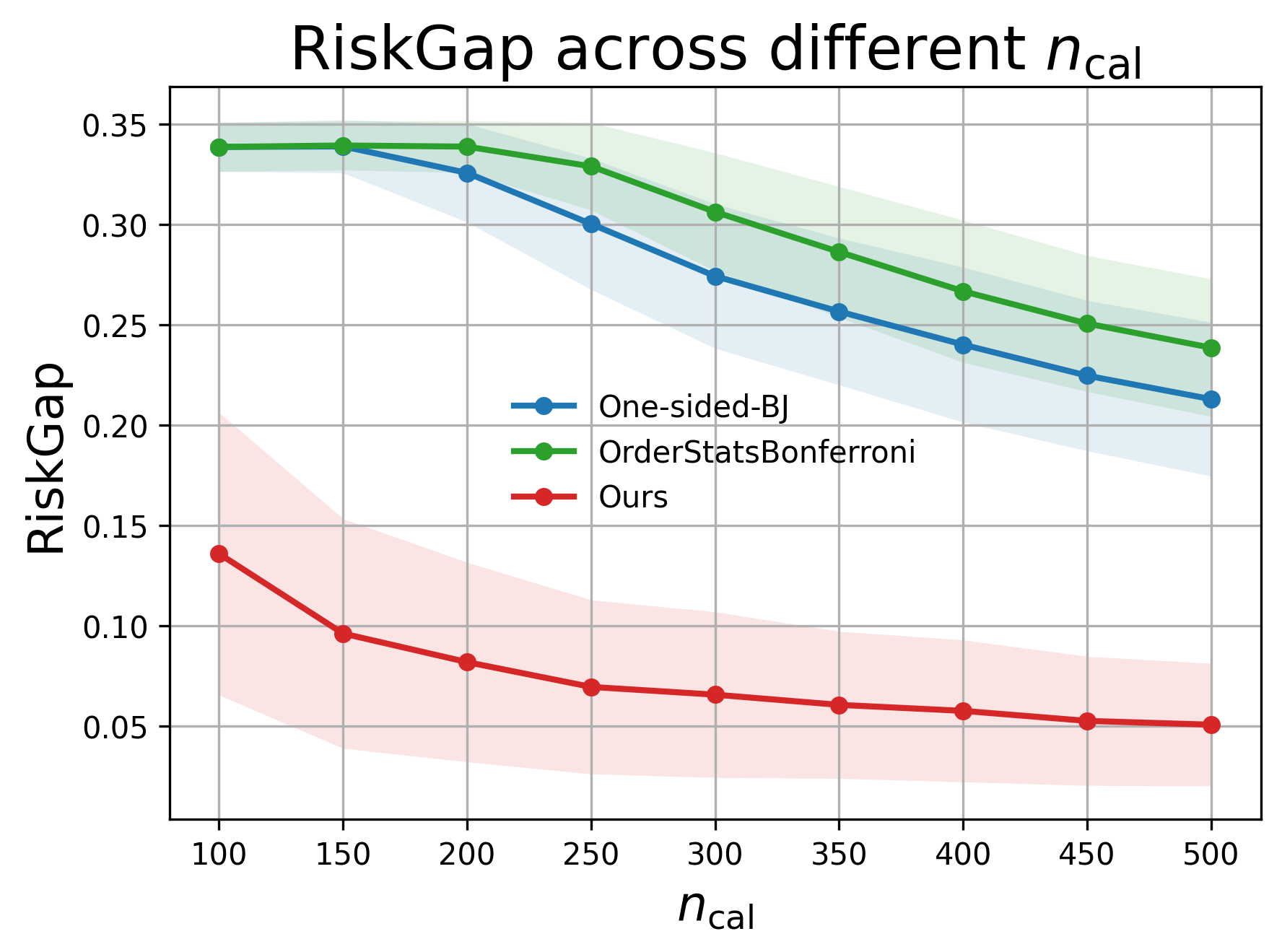}
    \end{minipage}\hfill
    \begin{minipage}{0.29\linewidth}
        \centering
        \includegraphics[width=\linewidth]{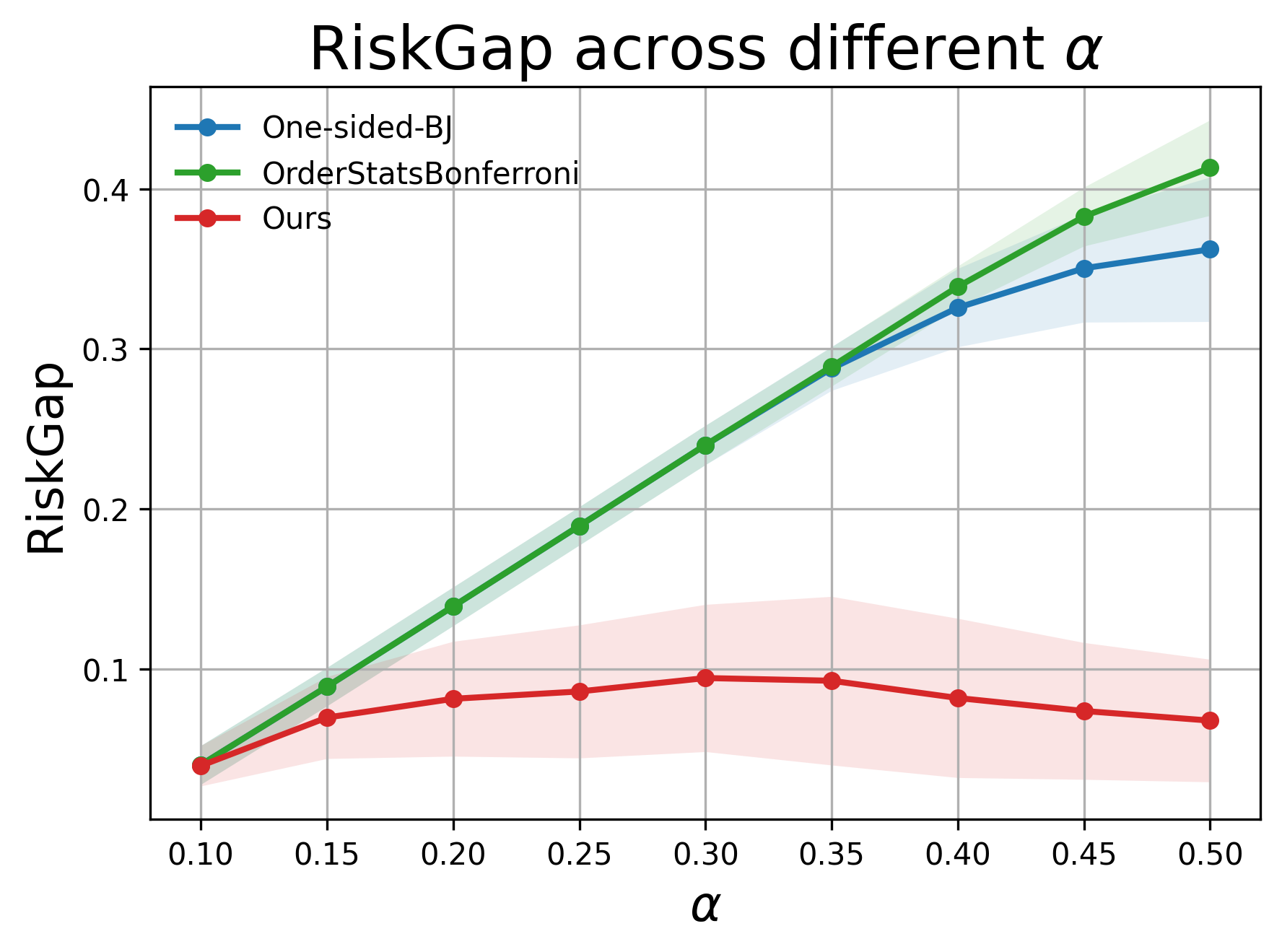}
    \end{minipage}\hfill
    \begin{minipage}{0.29\linewidth}
        \centering
        \includegraphics[width=\linewidth]{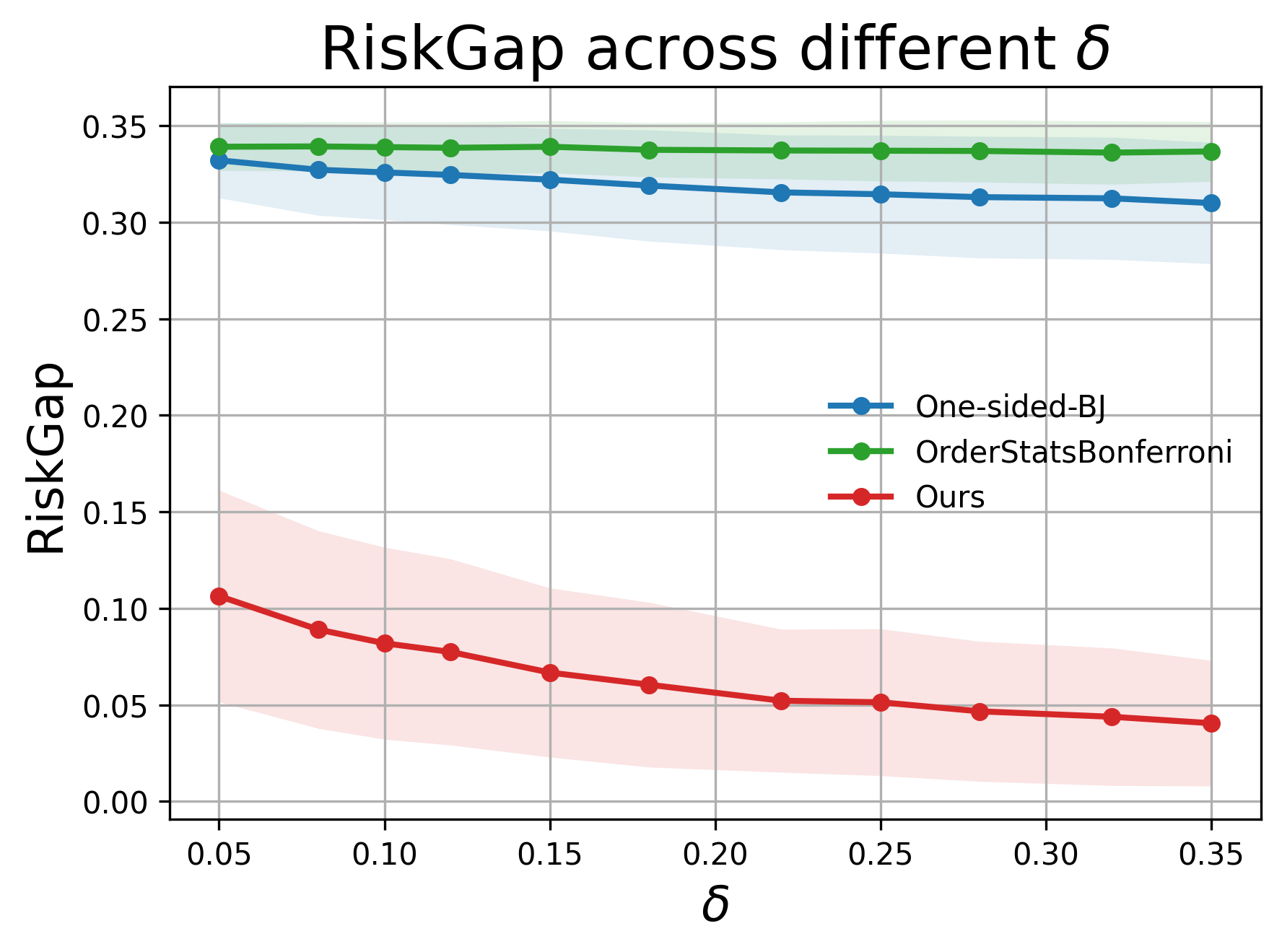}
    \end{minipage}\\
    \begin{minipage}{0.29\linewidth}
        \centering
        \includegraphics[width=\linewidth]{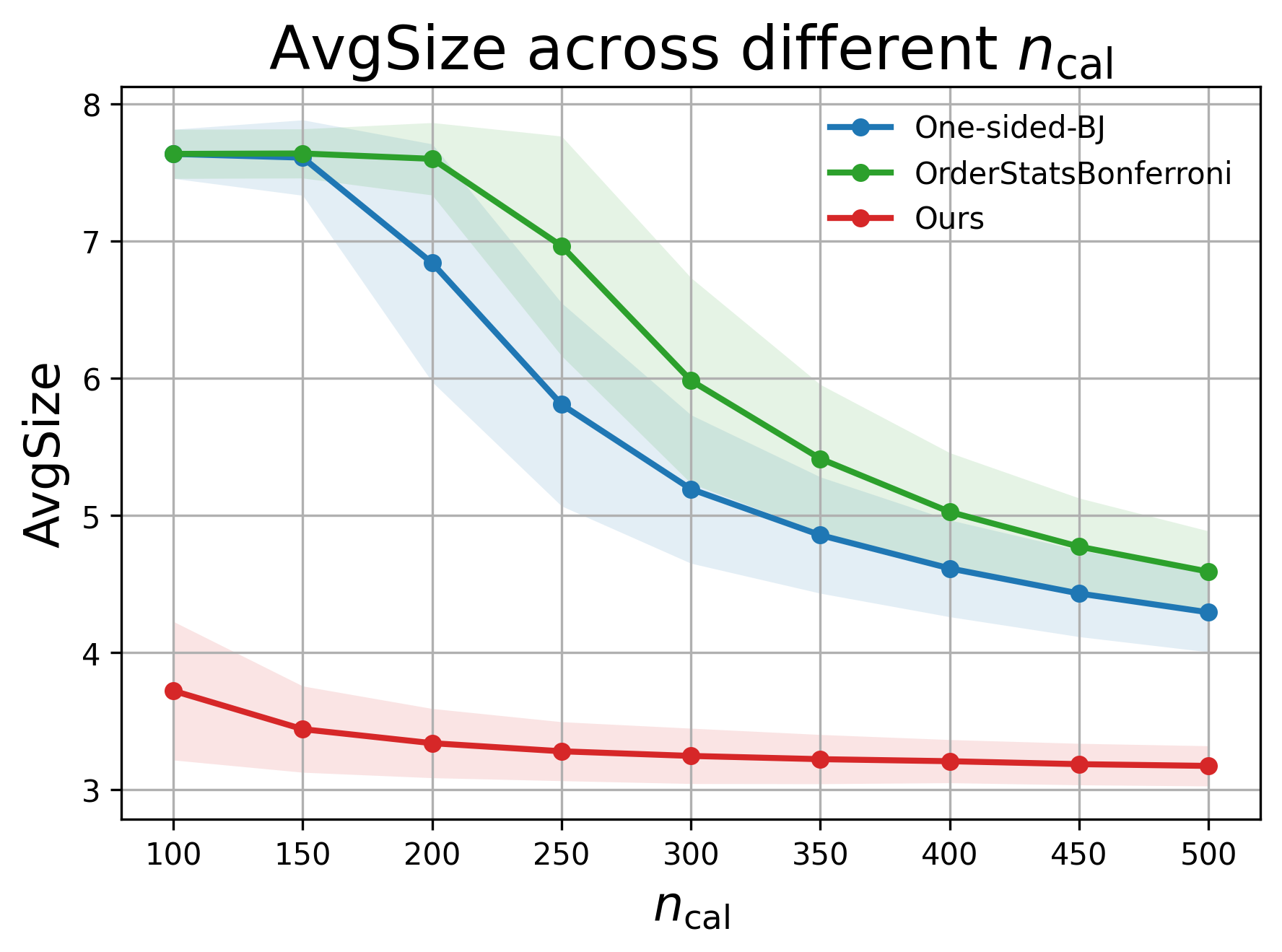}
    \end{minipage}\hfill
    \begin{minipage}{0.29\linewidth}
        \centering
        \includegraphics[width=\linewidth]{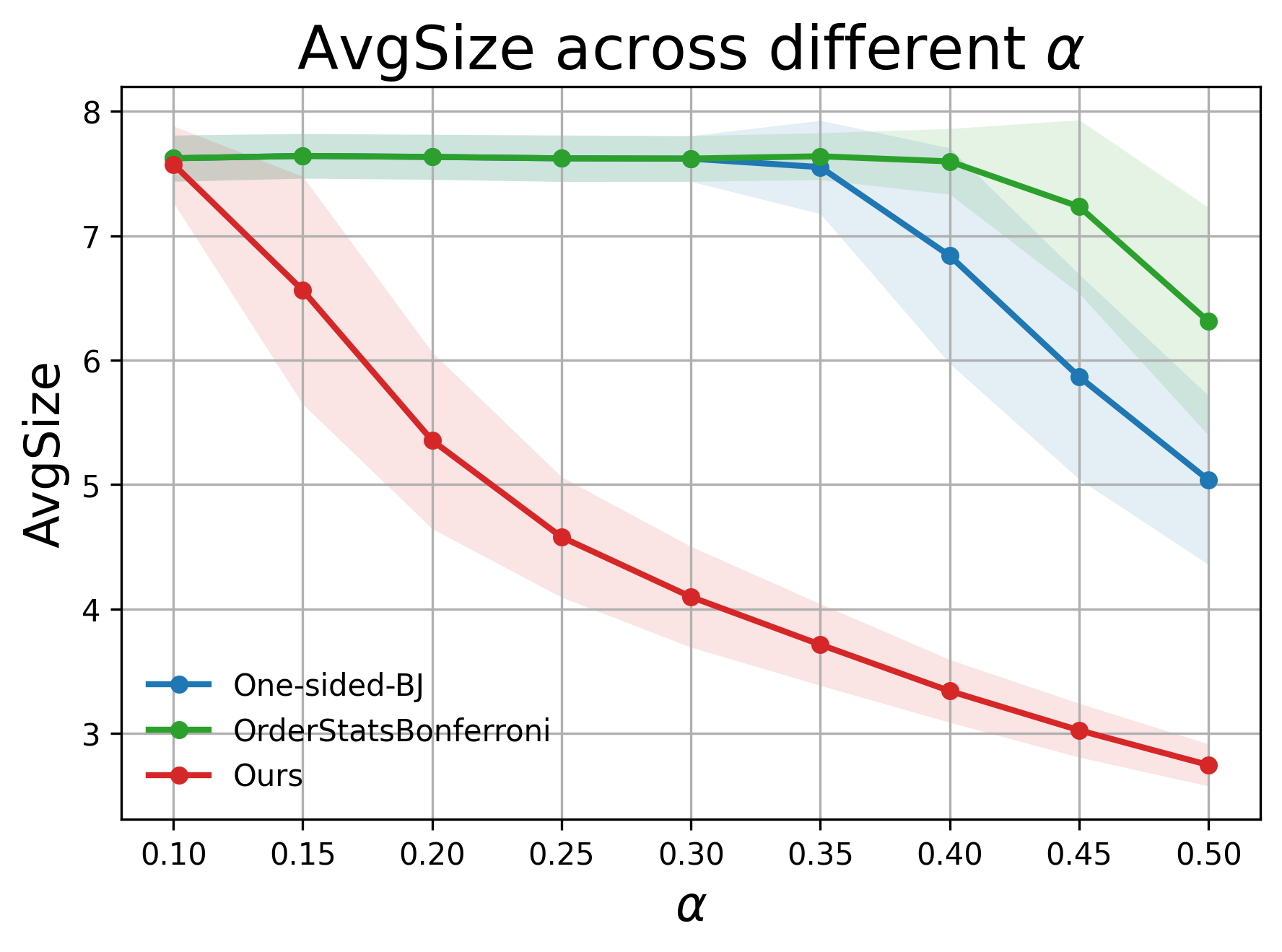}
    \end{minipage}\hfill
    \begin{minipage}{0.29\linewidth}
        \centering
        \includegraphics[width=\linewidth]{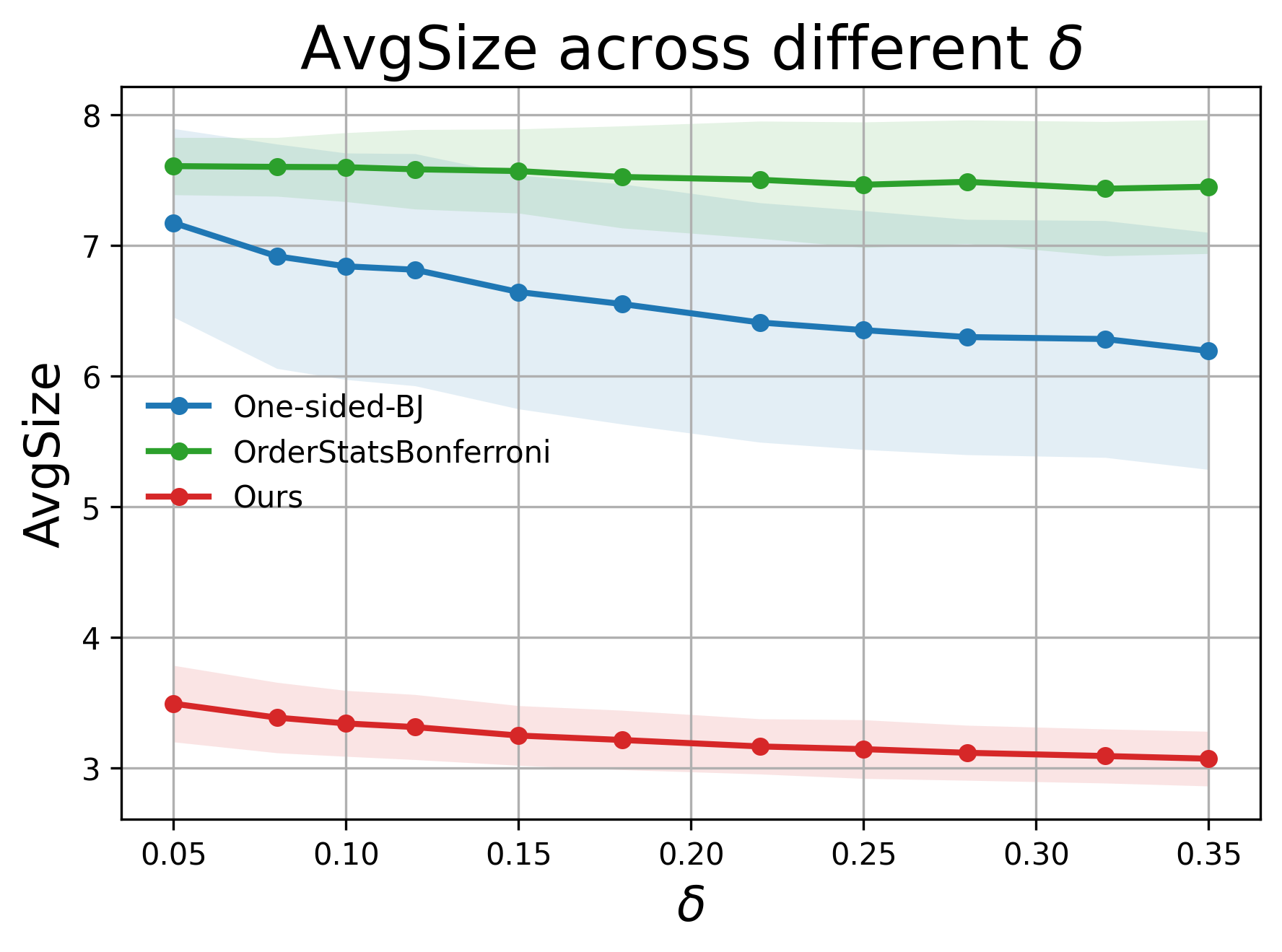}
    \end{minipage}
    \vspace{-3mm}
    \caption{\textbf{Performance comparison across various hyperparameter configurations.} The task is image multi-label classification, and the risk measure is $0.9$-CVaR.}
    \label{fig:parameter_analysis_coco_1}
    \vspace{-5mm}
\end{figure}

\begin{figure}[htbp]
    \centering
    \begin{minipage}{0.29\linewidth}
        \centering
        \includegraphics[width=\linewidth]{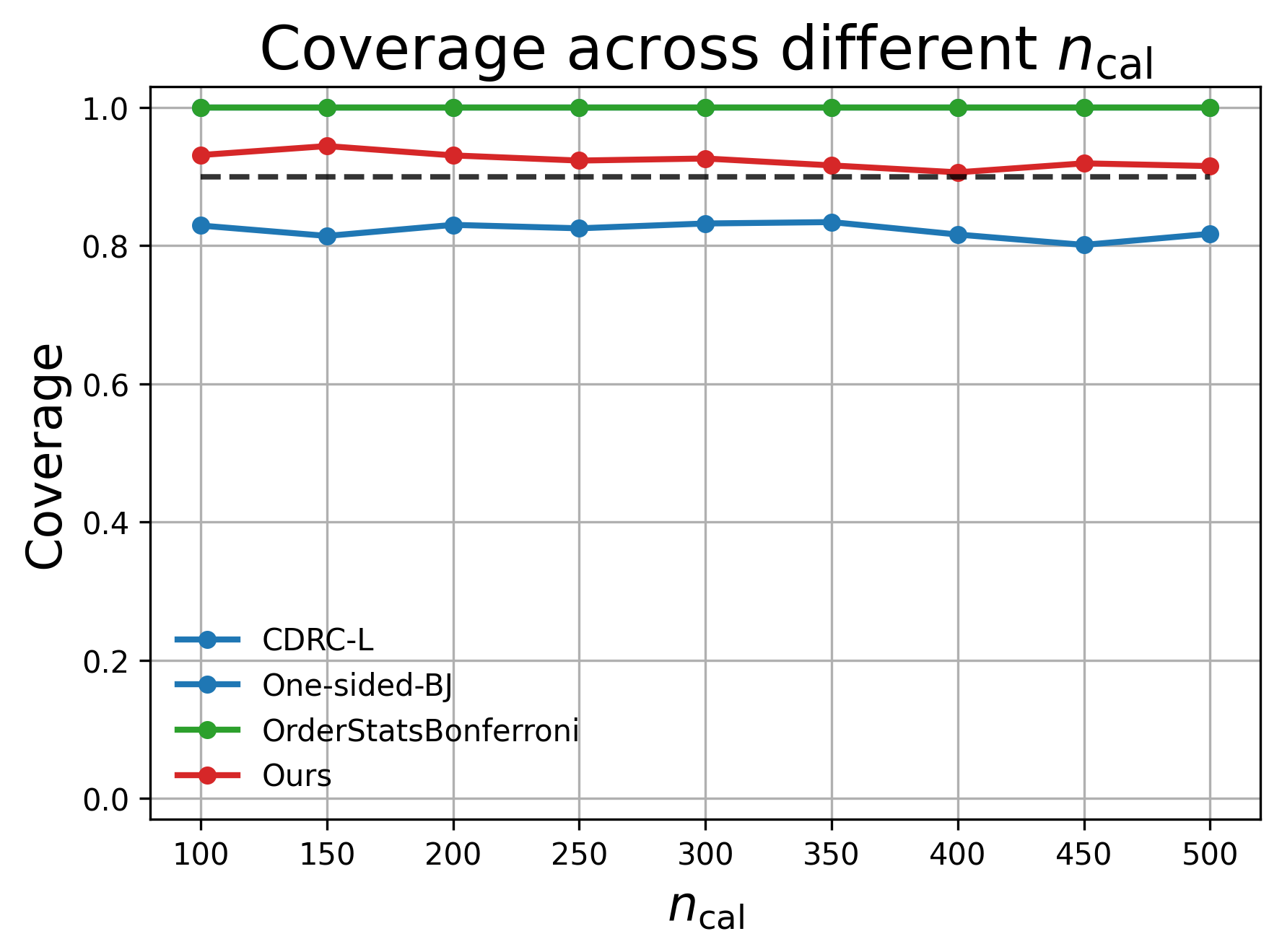}
    \end{minipage}\hfill
    \begin{minipage}{0.29\linewidth}
        \centering
        \includegraphics[width=\linewidth]{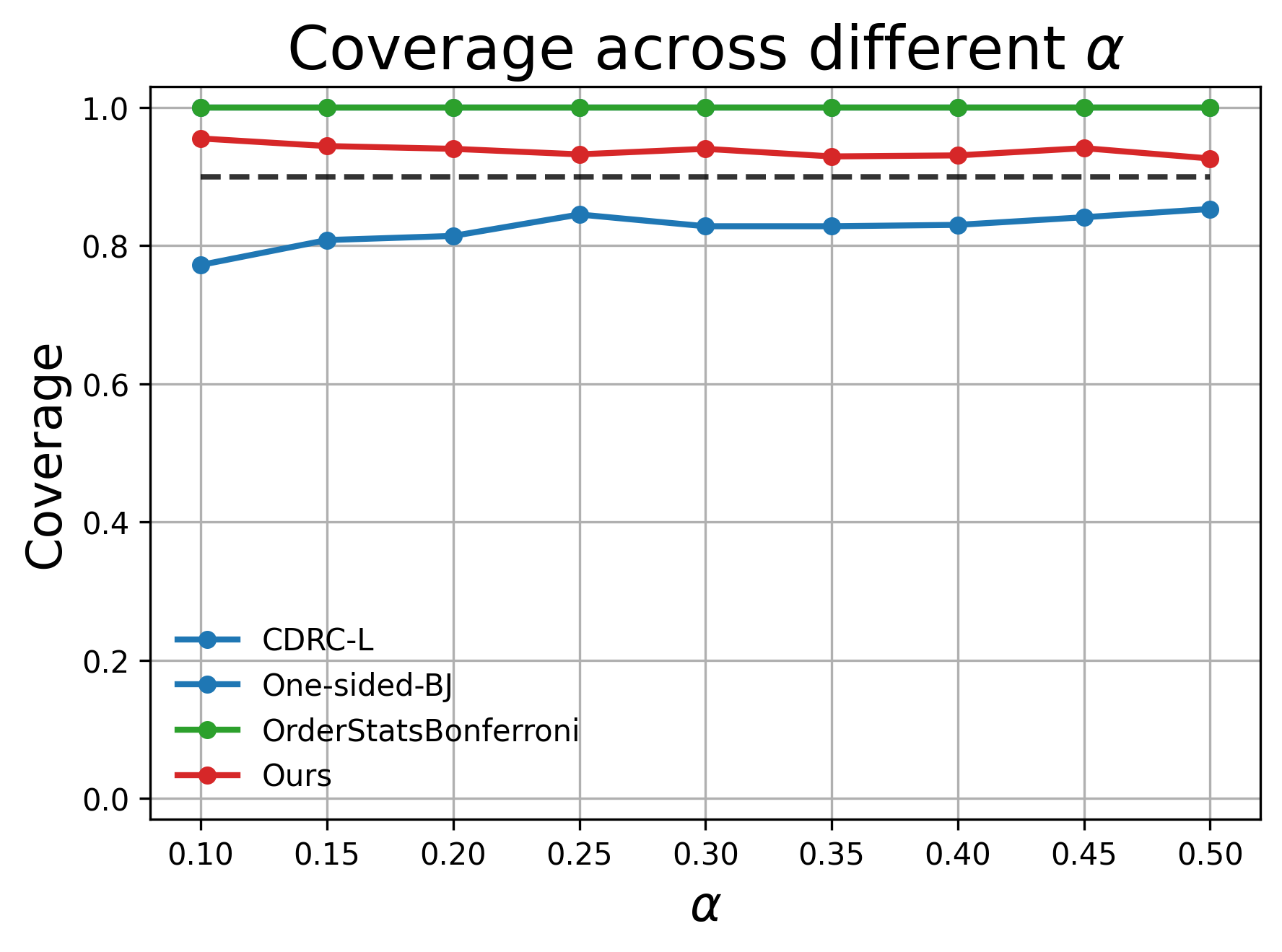}
    \end{minipage}\hfill
    \begin{minipage}{0.29\linewidth}
        \centering
        \includegraphics[width=\linewidth]{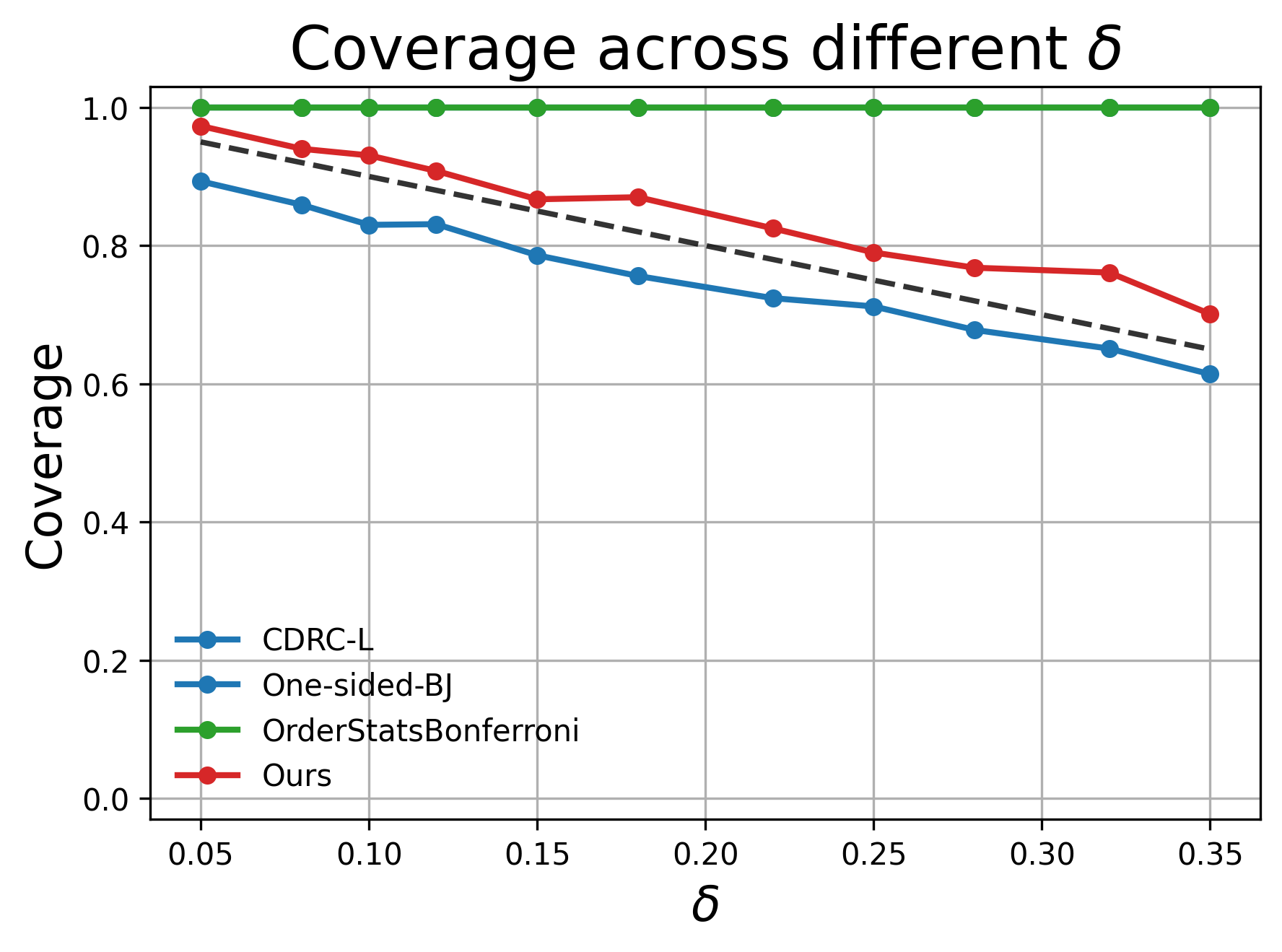}
    \end{minipage}\\
    \begin{minipage}{0.29\linewidth}
        \centering
        \includegraphics[width=\linewidth]{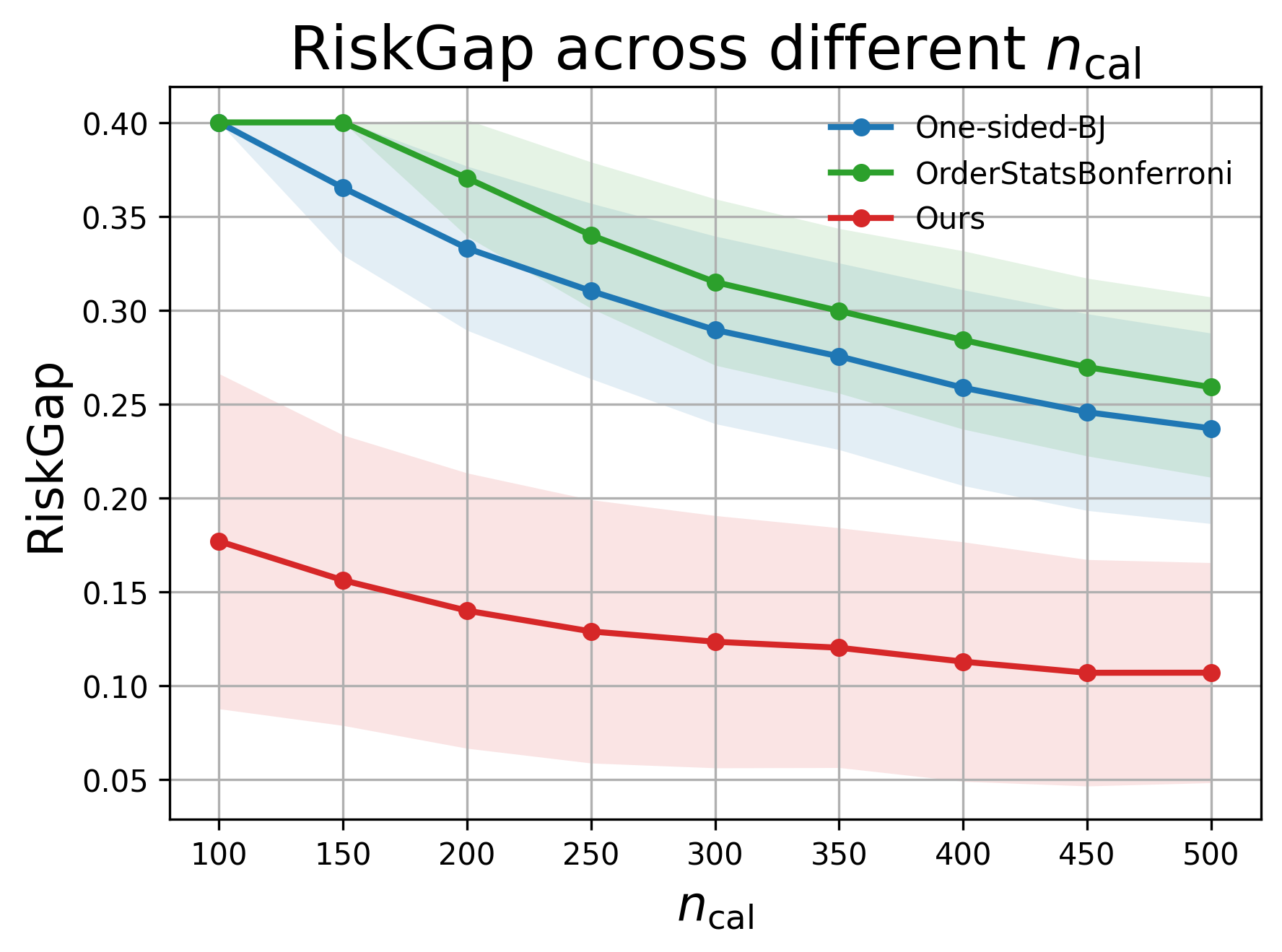}
    \end{minipage}\hfill
    \begin{minipage}{0.29\linewidth}
        \centering
        \includegraphics[width=\linewidth]{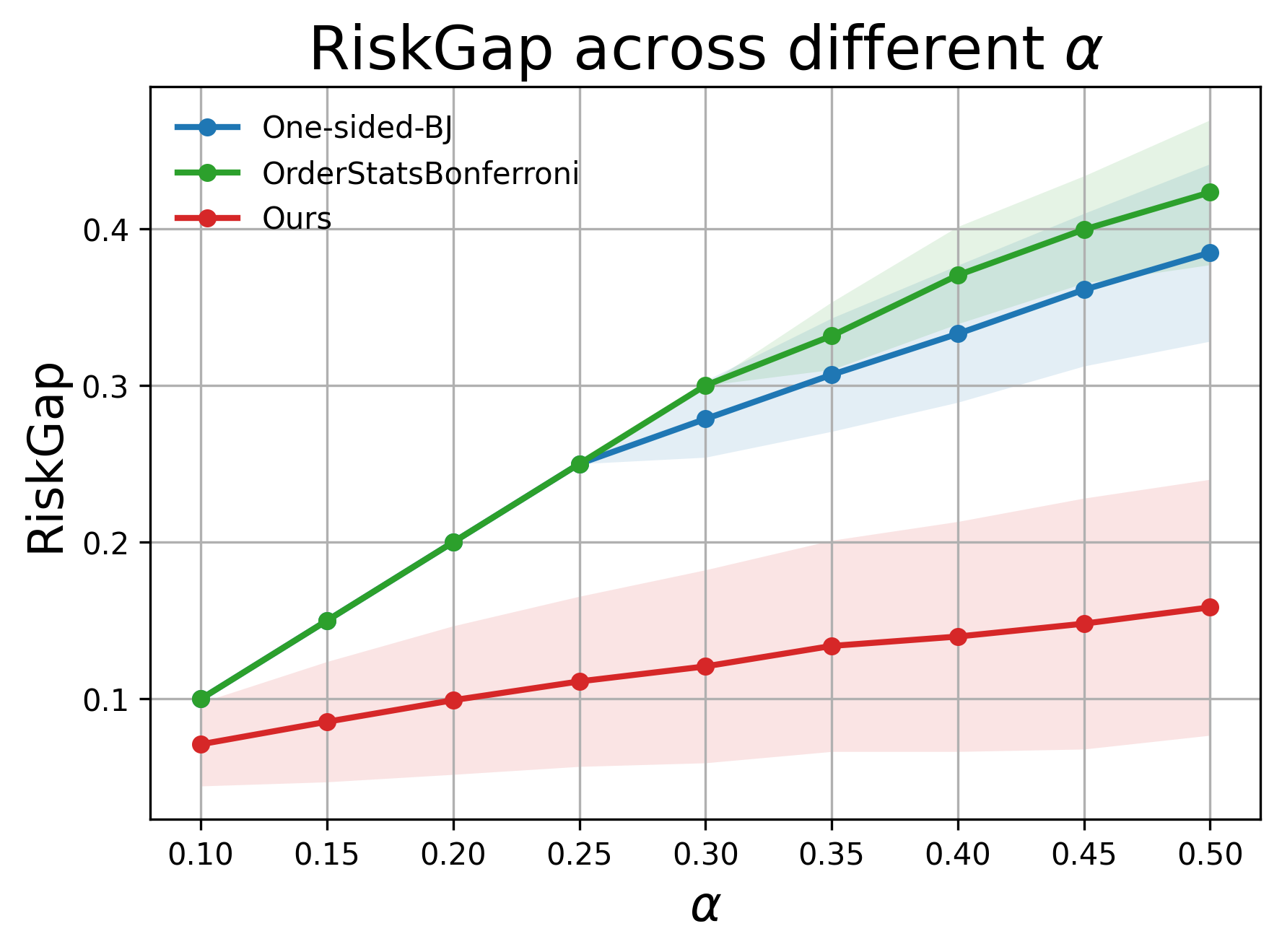}
    \end{minipage}\hfill
    \begin{minipage}{0.29\linewidth}
        \centering
        \includegraphics[width=\linewidth]{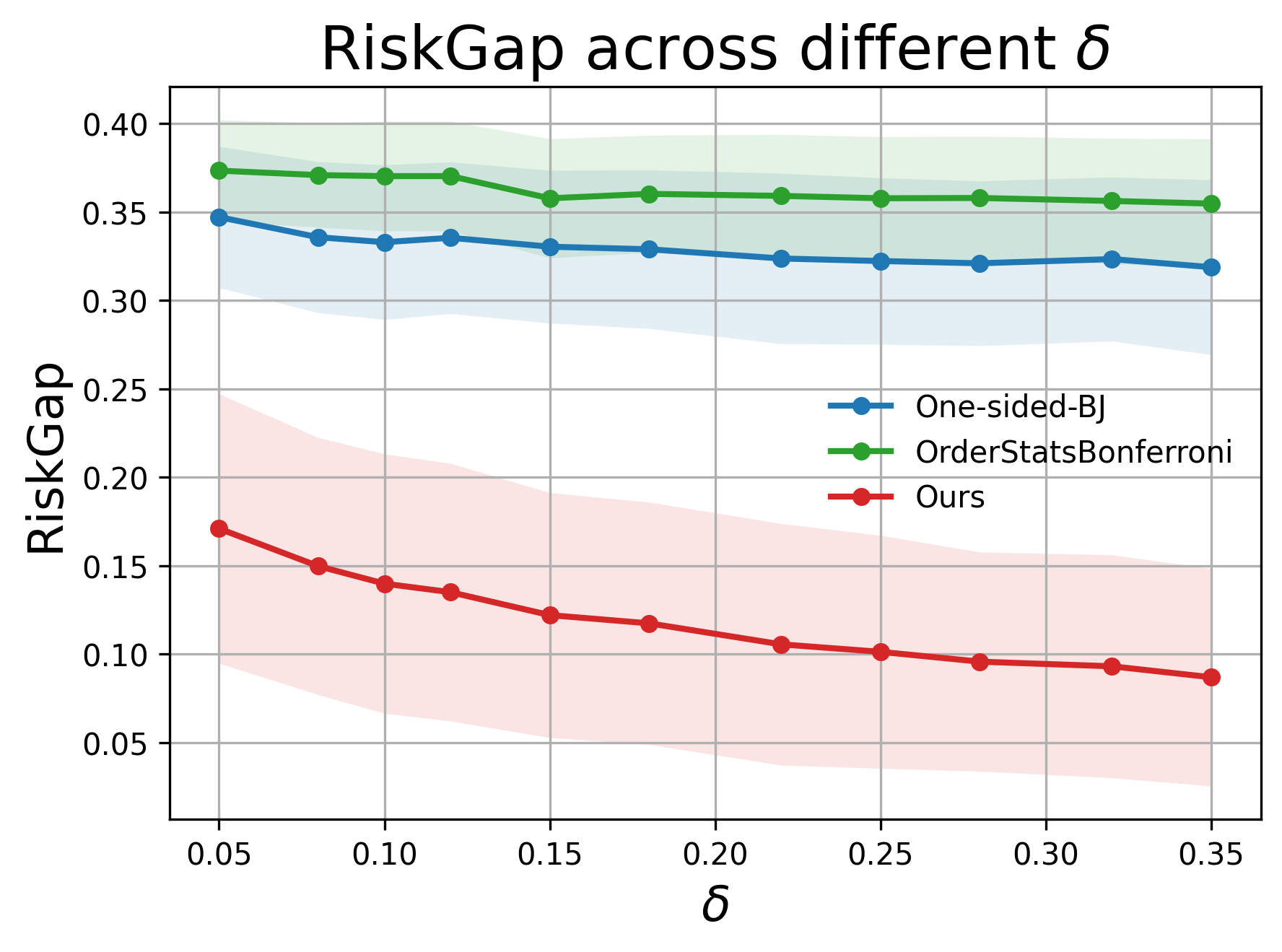}
    \end{minipage}\\
    \begin{minipage}{0.29\linewidth}
        \centering
        \includegraphics[width=\linewidth]{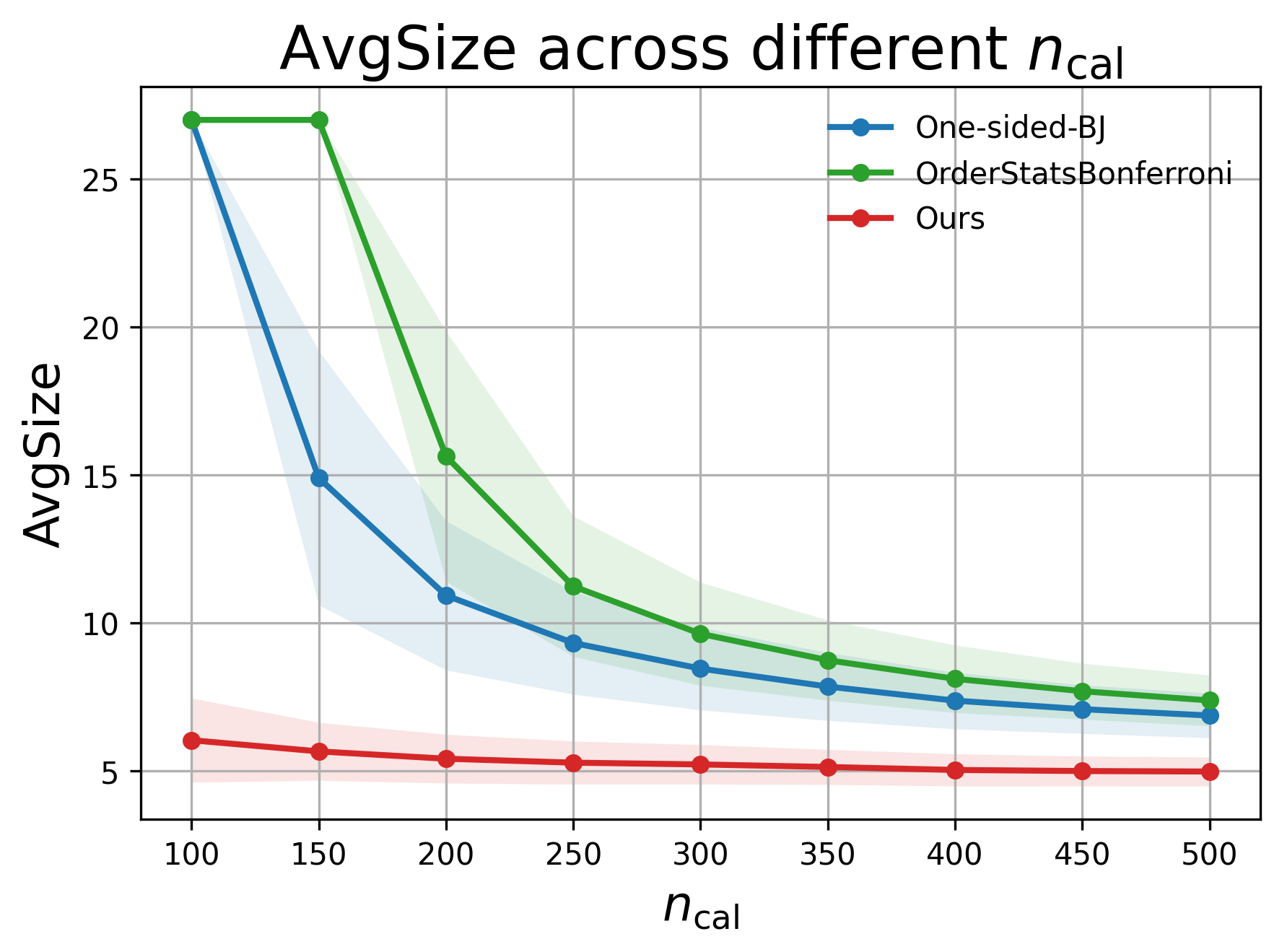}
    \end{minipage}\hfill
    \begin{minipage}{0.29\linewidth}
        \centering
        \includegraphics[width=\linewidth]{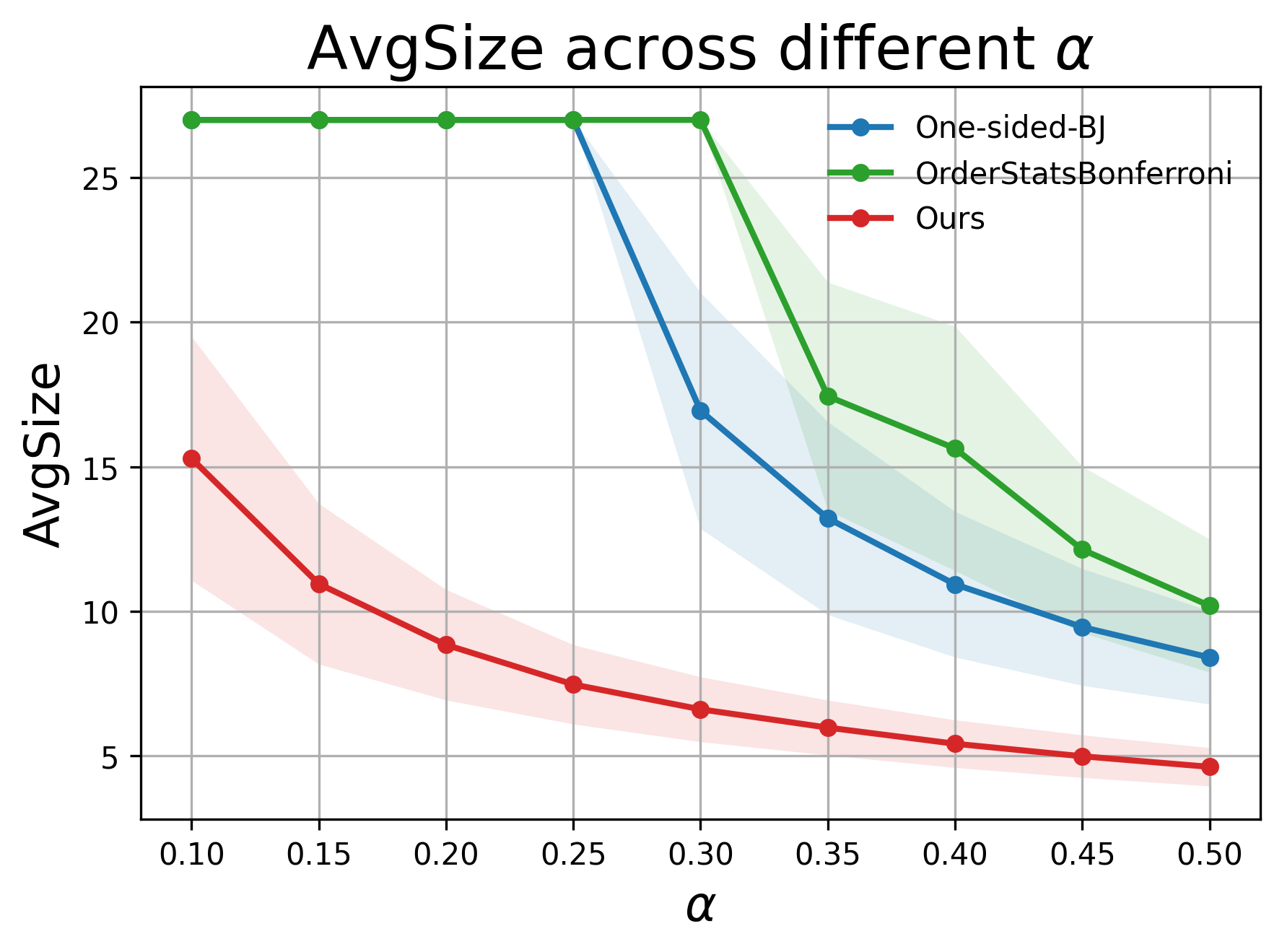}
    \end{minipage}\hfill
    \begin{minipage}{0.29\linewidth}
        \centering
        \includegraphics[width=\linewidth]{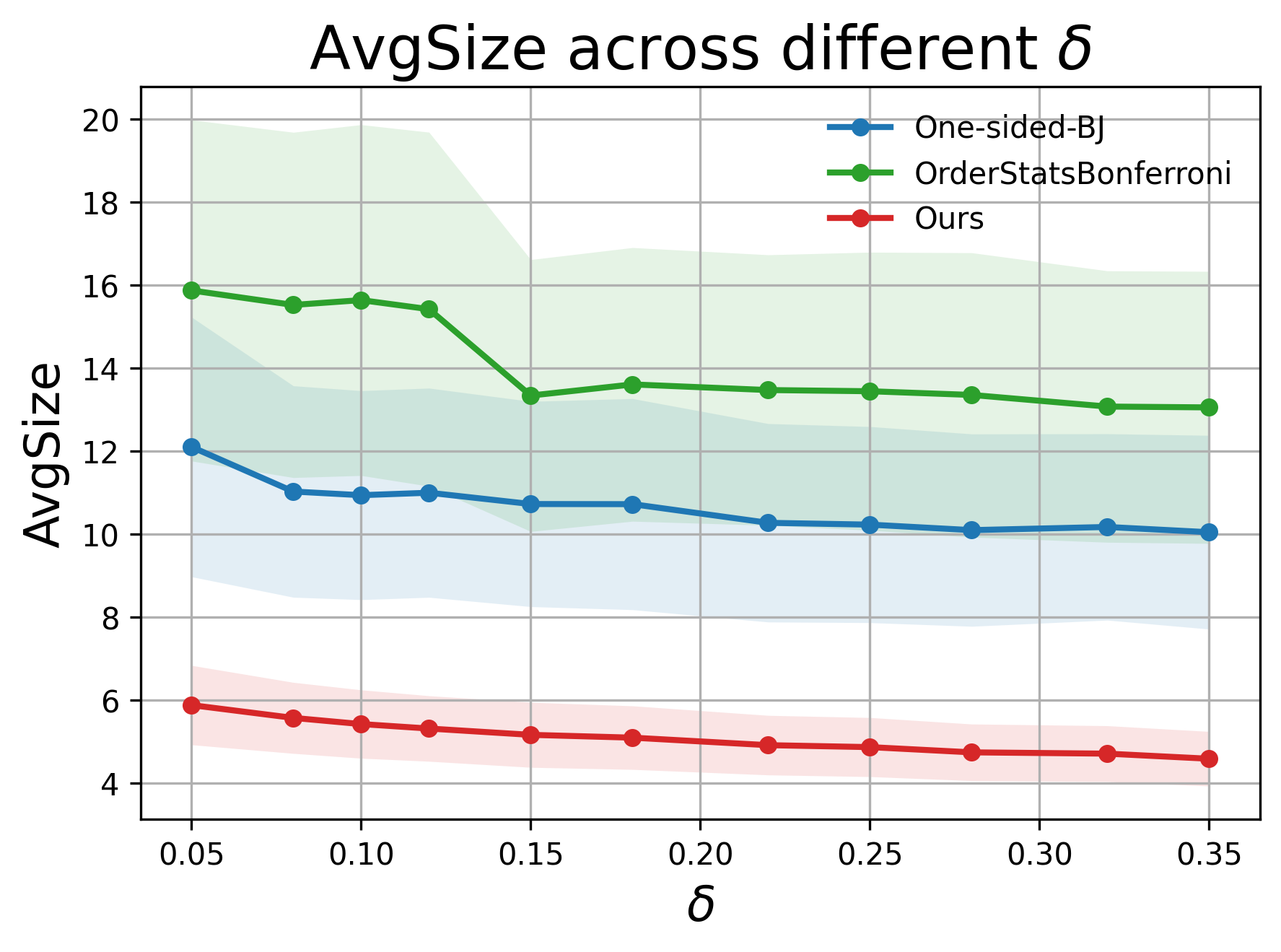}
    \end{minipage}
    \vspace{-3mm}
    \caption{\textbf{Performance comparison across various hyperparameter configurations.} The task is text emotion recognition, and the risk measure is $0.9$-CVaR.}
    \label{fig:parameter_analysis_go_1}
    \vspace{-5mm}
\end{figure}

\begin{figure}[htbp]
    \centering
    \begin{minipage}{0.29\linewidth}
        \centering
        \includegraphics[width=\linewidth]{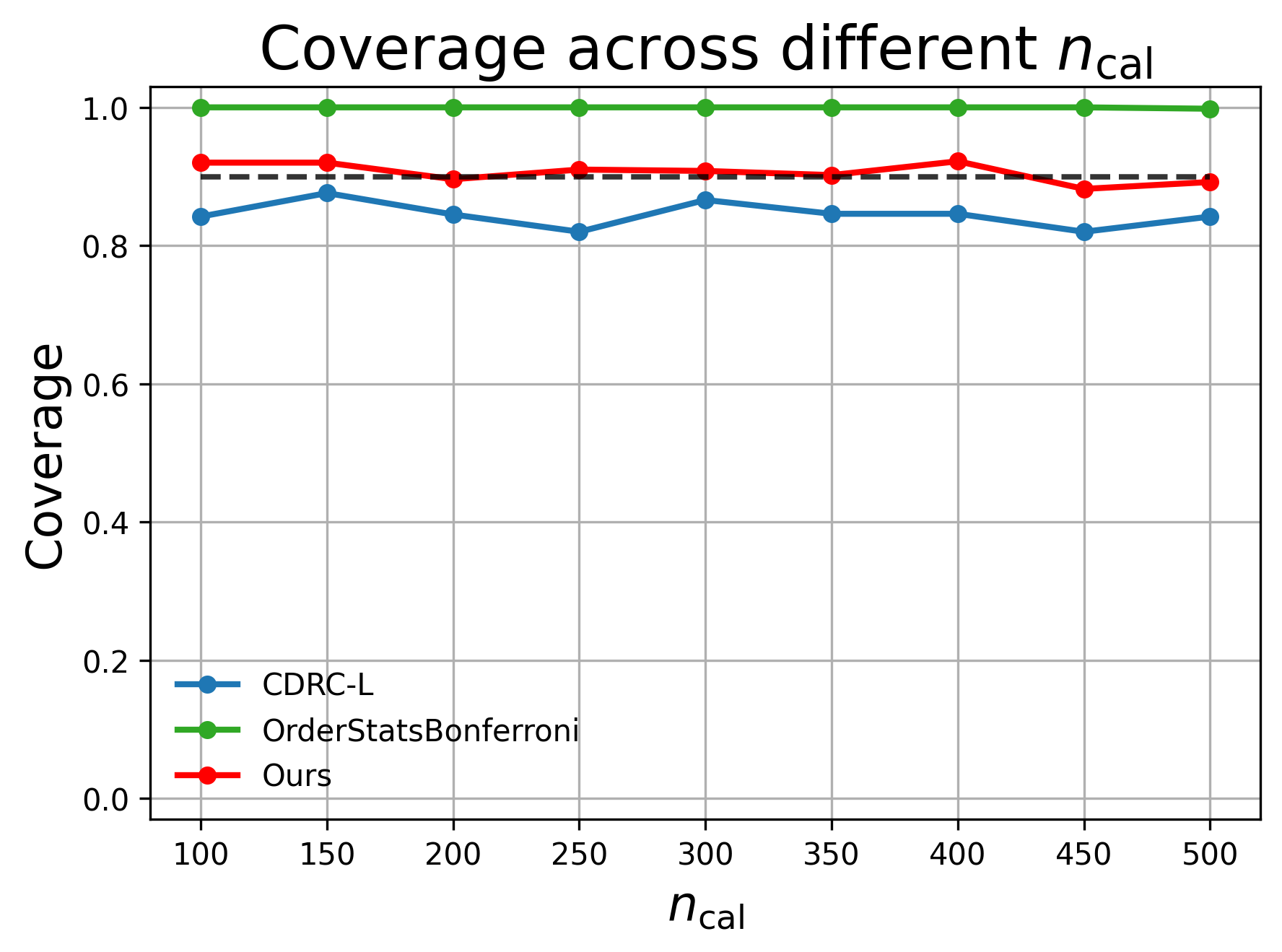}
    \end{minipage}\hfill
    \begin{minipage}{0.29\linewidth}
        \centering
        \includegraphics[width=\linewidth]{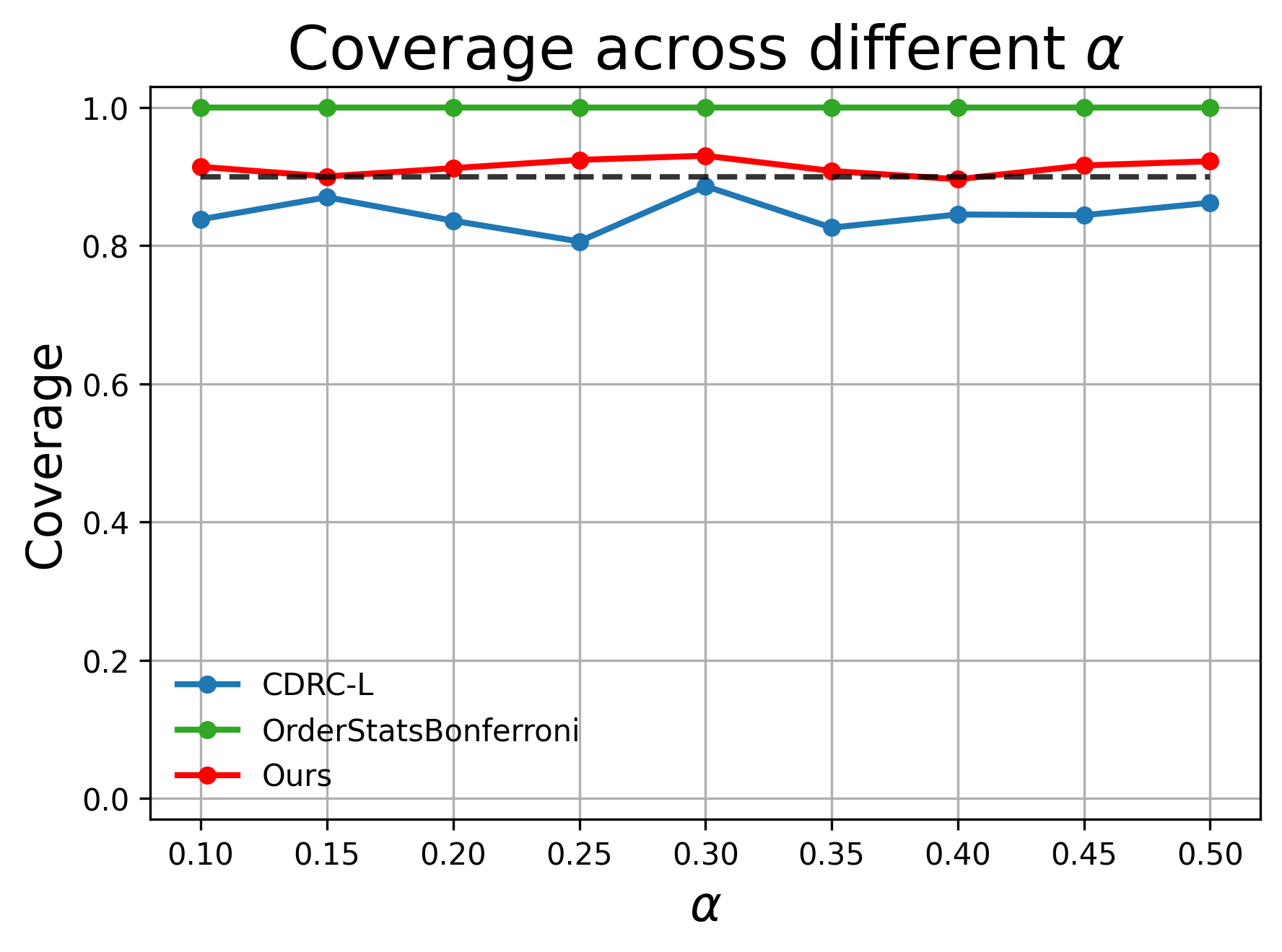}
    \end{minipage}\hfill
    \begin{minipage}{0.29\linewidth}
        \centering
        \includegraphics[width=\linewidth]{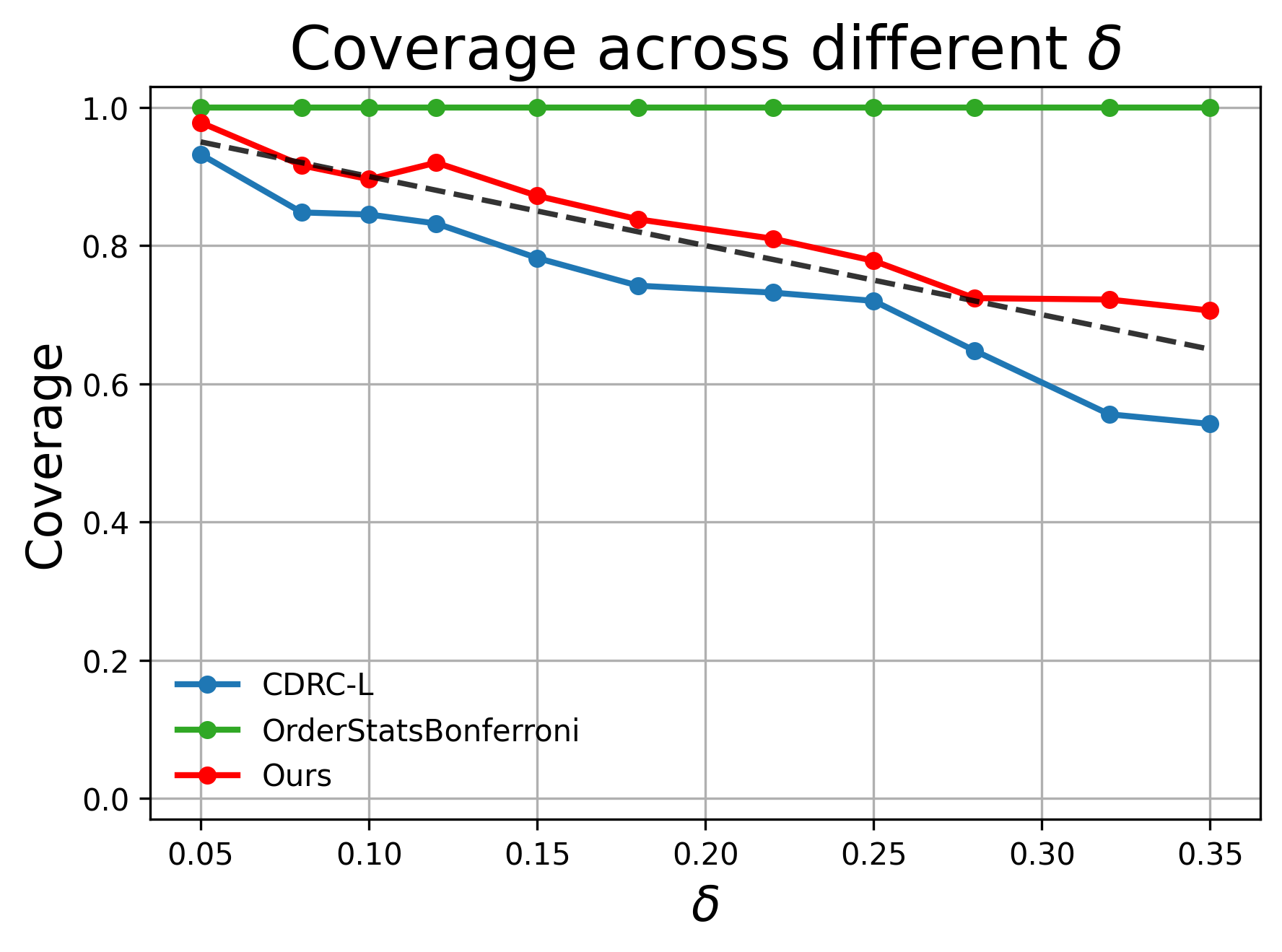}
    \end{minipage}\\
    \begin{minipage}{0.29\linewidth}
        \centering
        \includegraphics[width=\linewidth]{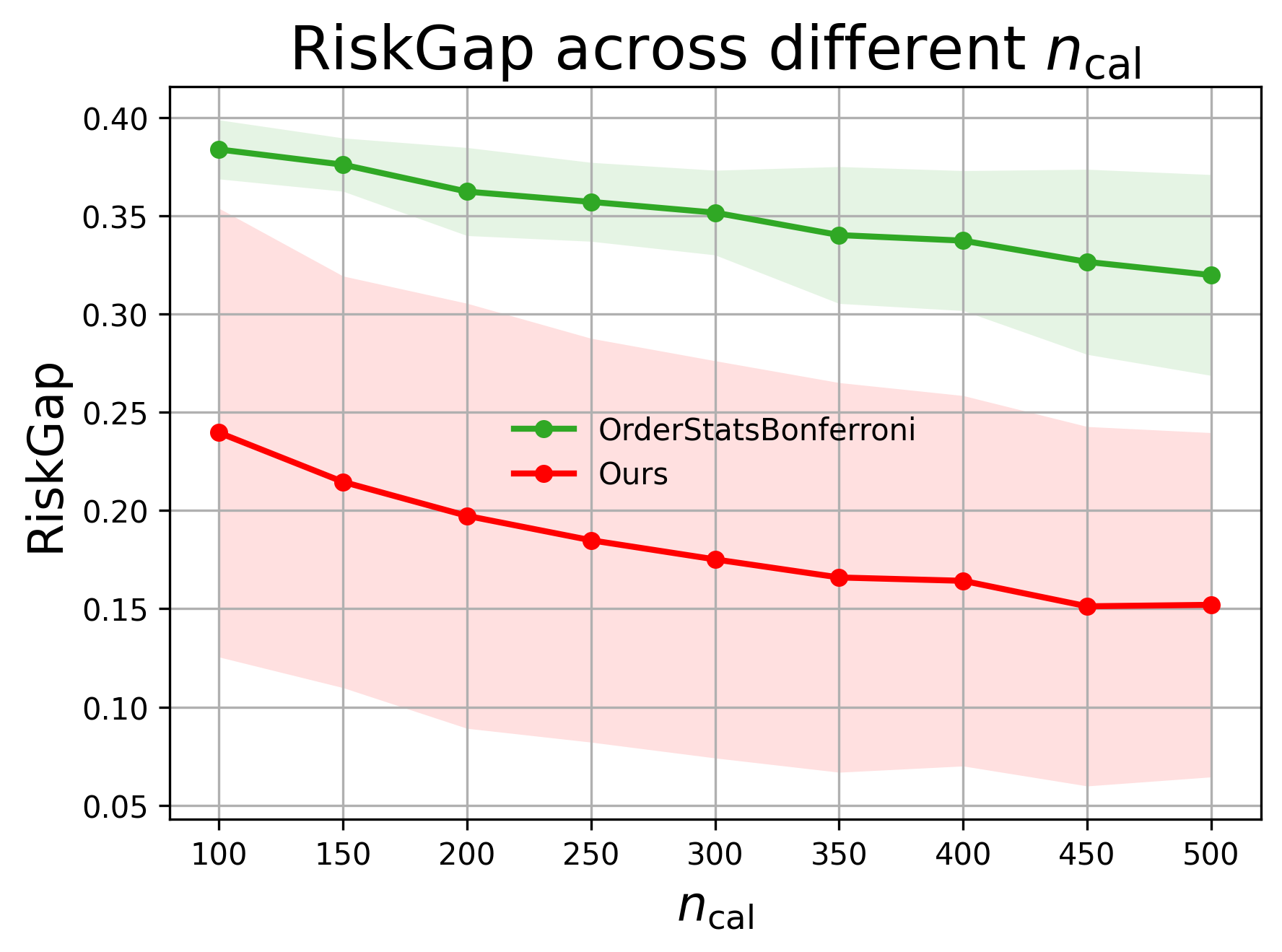}
    \end{minipage}\hfill
    \begin{minipage}{0.29\linewidth}
        \centering
        \includegraphics[width=\linewidth]{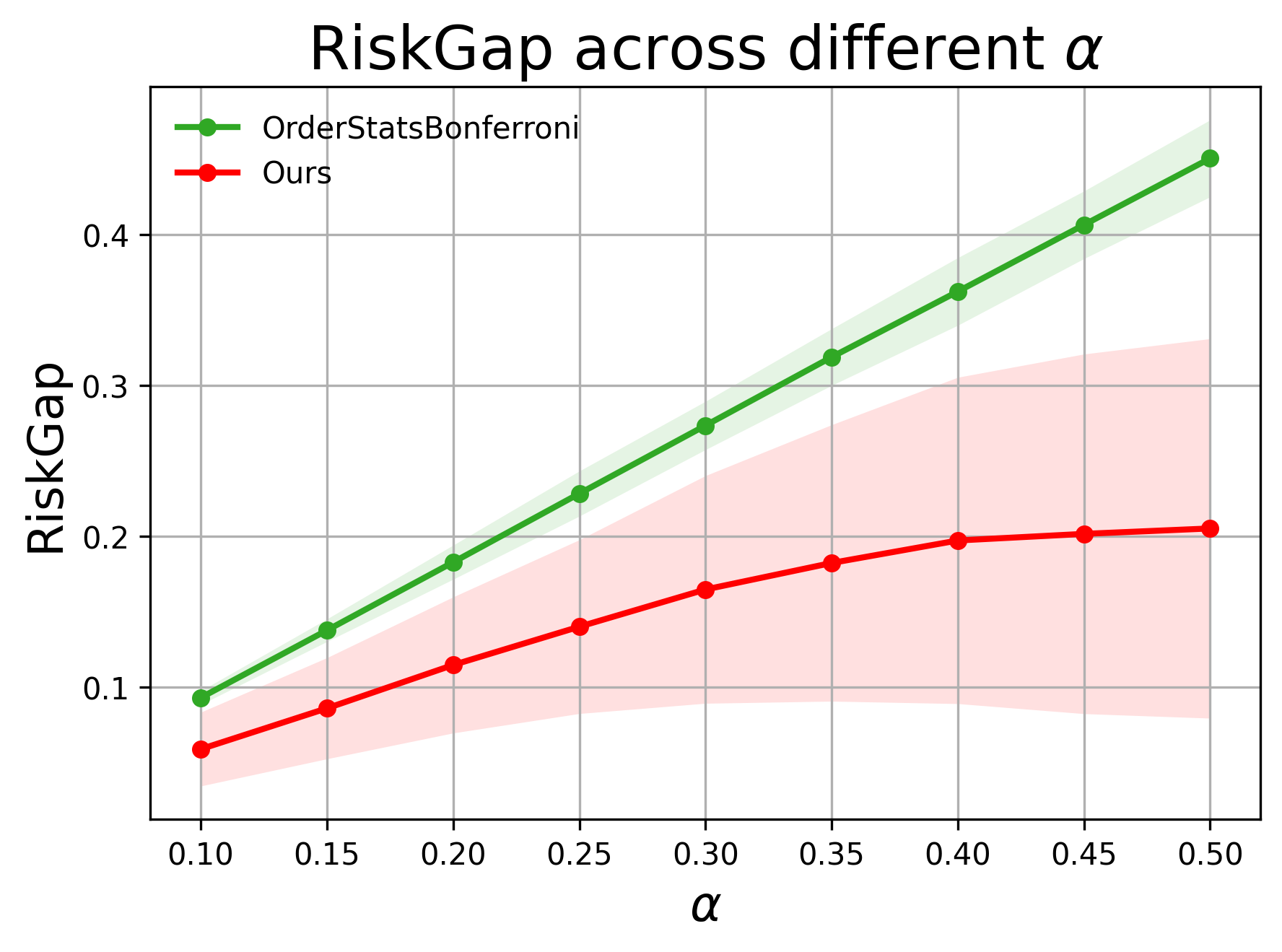}
    \end{minipage}\hfill
    \begin{minipage}{0.29\linewidth}
        \centering
        \includegraphics[width=\linewidth]{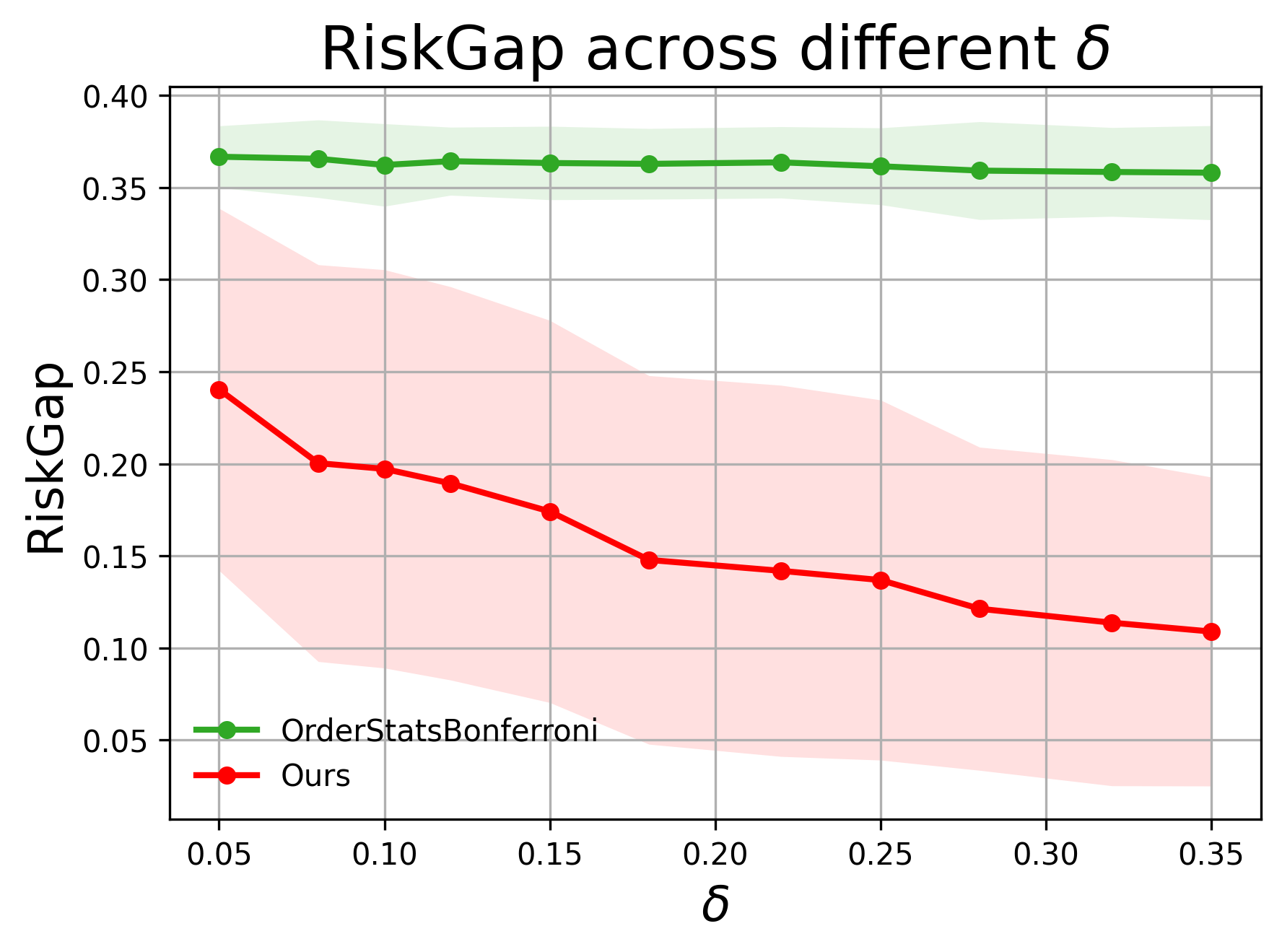}
    \end{minipage}\\
    \begin{minipage}{0.29\linewidth}
        \centering
        \includegraphics[width=\linewidth]{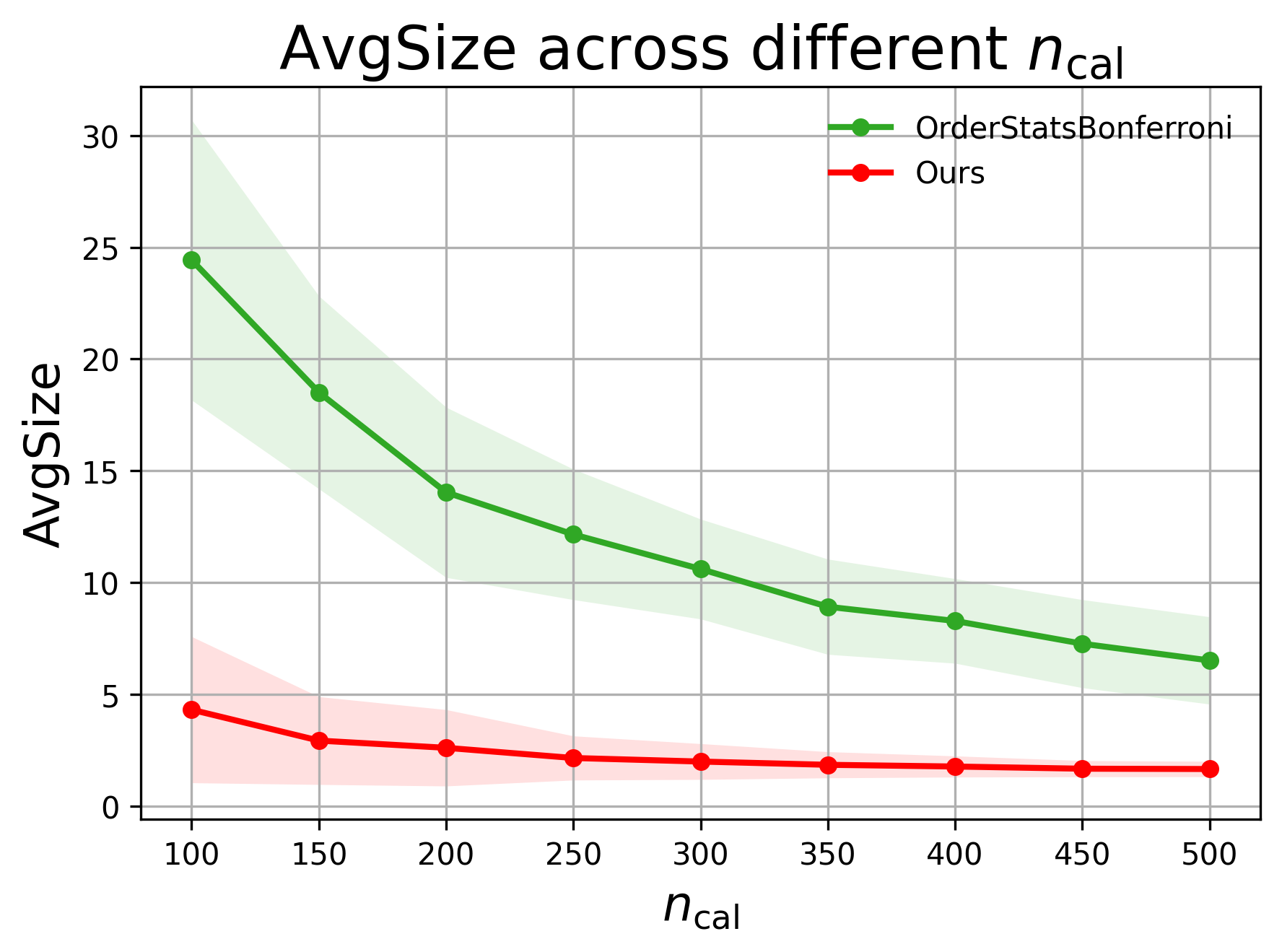}
    \end{minipage}\hfill
    \begin{minipage}{0.29\linewidth}
        \centering
        \includegraphics[width=\linewidth]{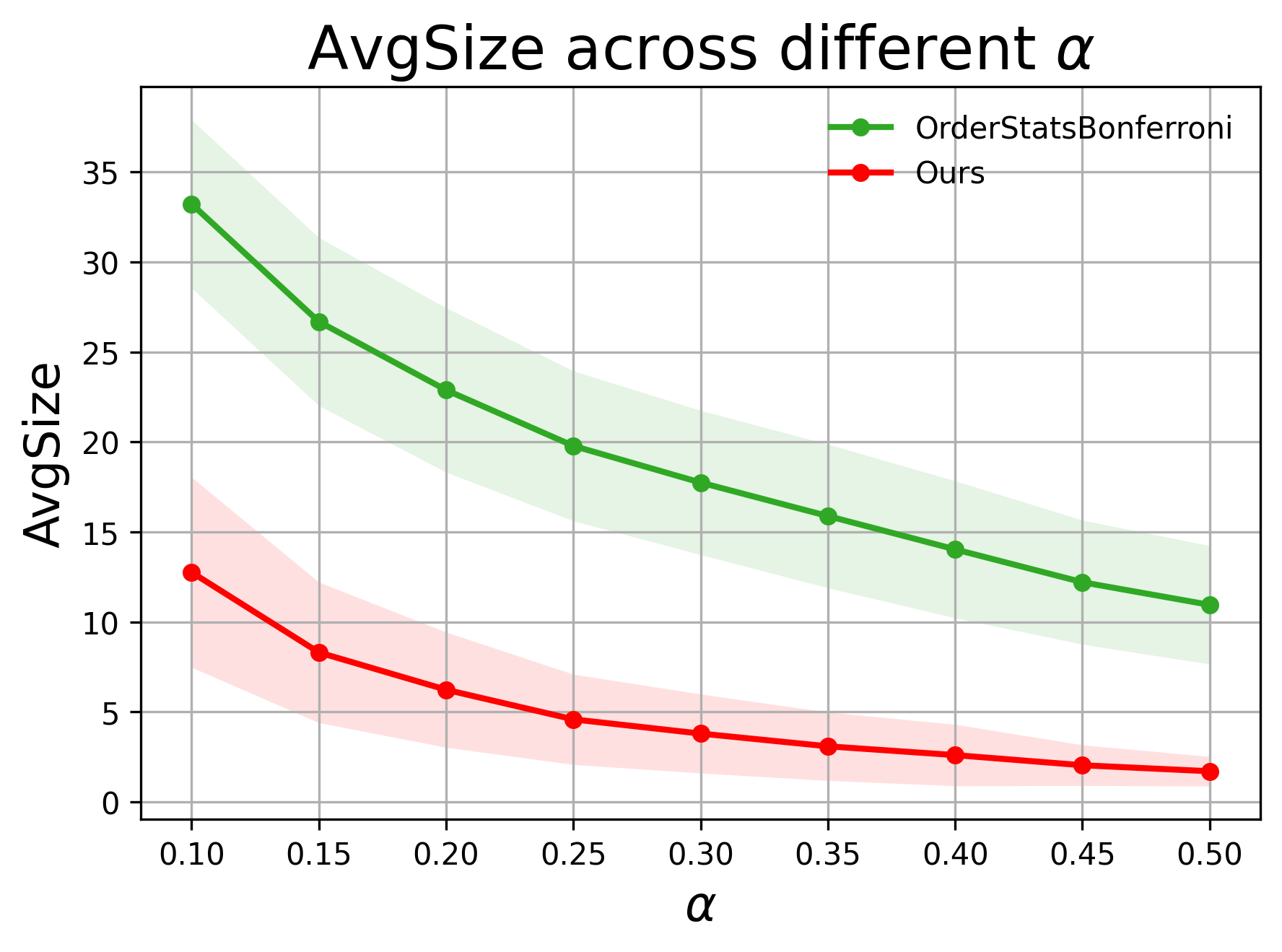}
    \end{minipage}\hfill
    \begin{minipage}{0.29\linewidth}
        \centering
        \includegraphics[width=\linewidth]{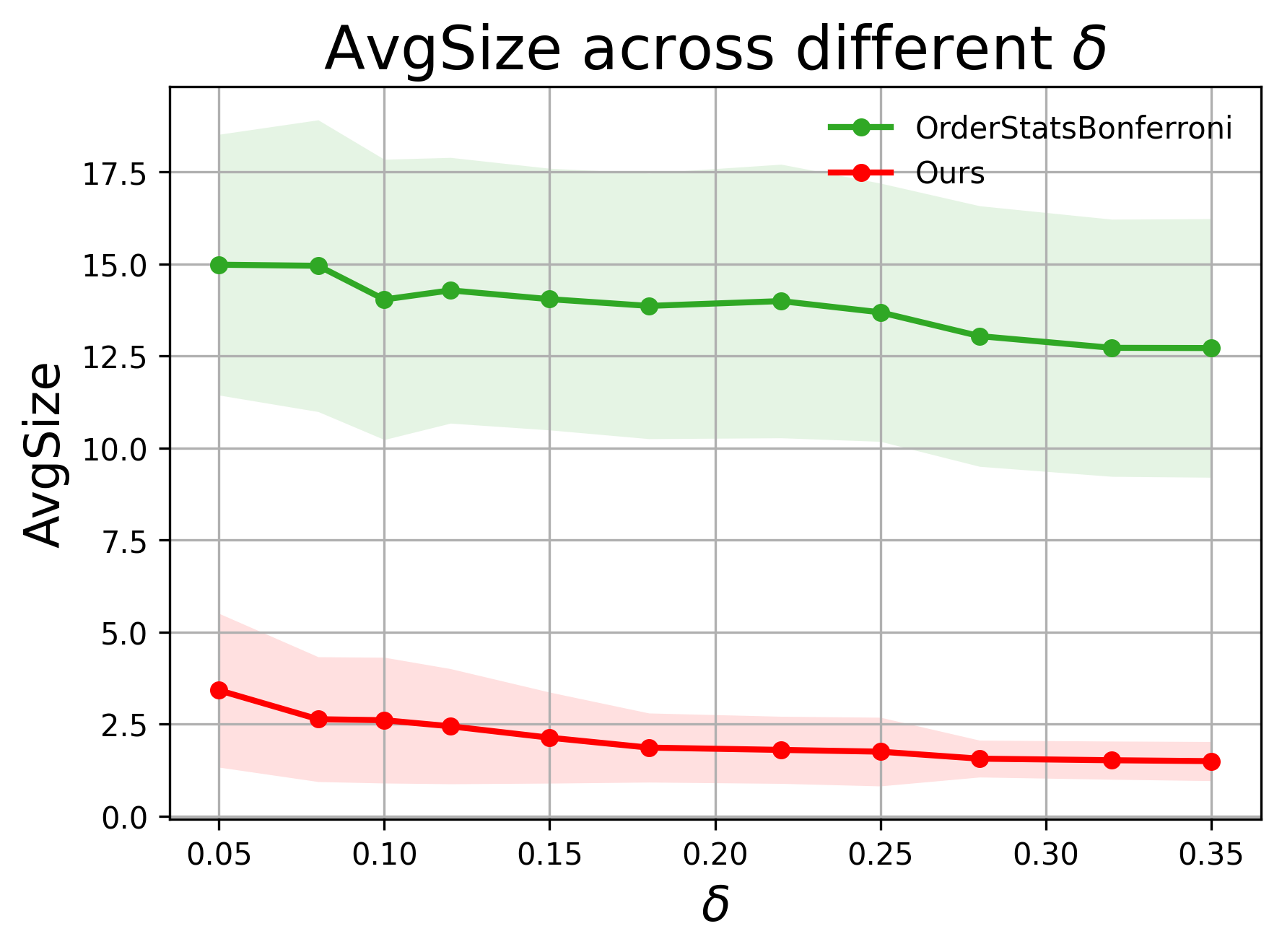}
    \end{minipage}
    \vspace{-3mm}
    \caption{\textbf{Performance comparison across various hyperparameter configurations.} The task is tumor segmentation, and the risk measure is $0.8$-VaR.}
    \label{fig:parameter_analysis_polyp_0.801}
    \vspace{-5mm}
\end{figure}

\begin{figure}[htbp]
    \centering
    \begin{minipage}{0.29\linewidth}
        \centering
        \includegraphics[width=\linewidth]{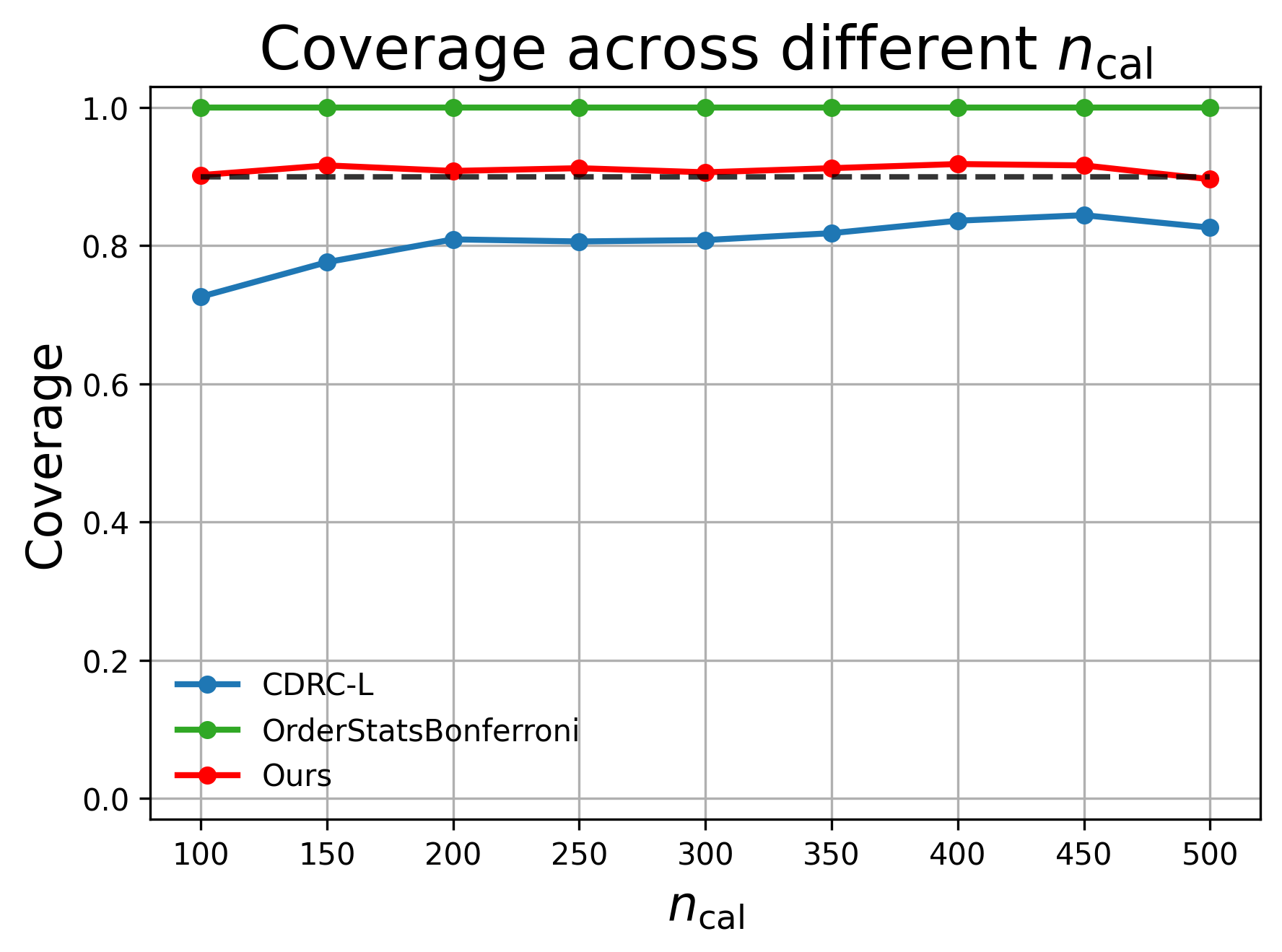}
    \end{minipage}\hfill
    \begin{minipage}{0.29\linewidth}
        \centering
        \includegraphics[width=\linewidth]{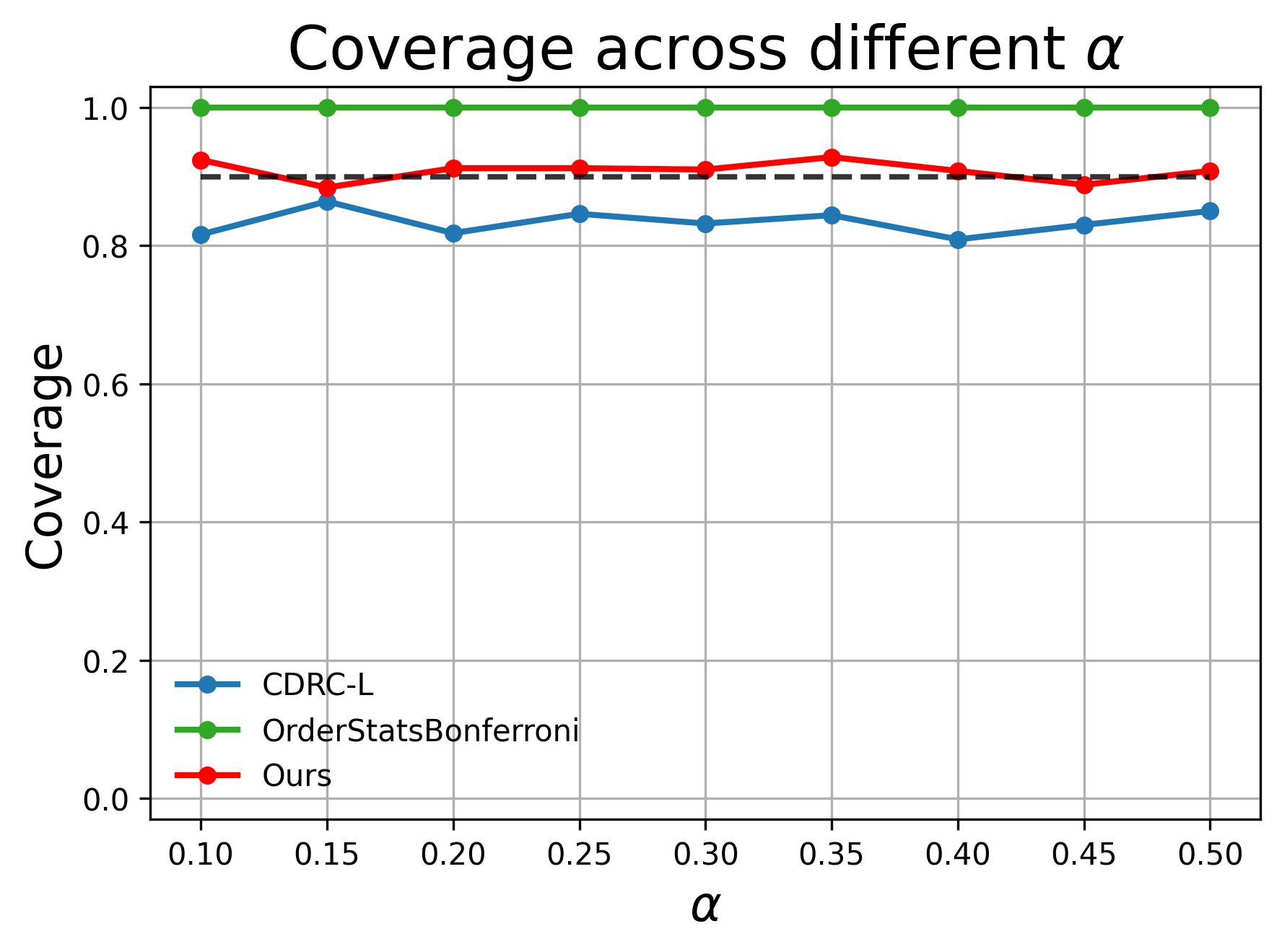}
    \end{minipage}\hfill
    \begin{minipage}{0.29\linewidth}
        \centering
        \includegraphics[width=\linewidth]{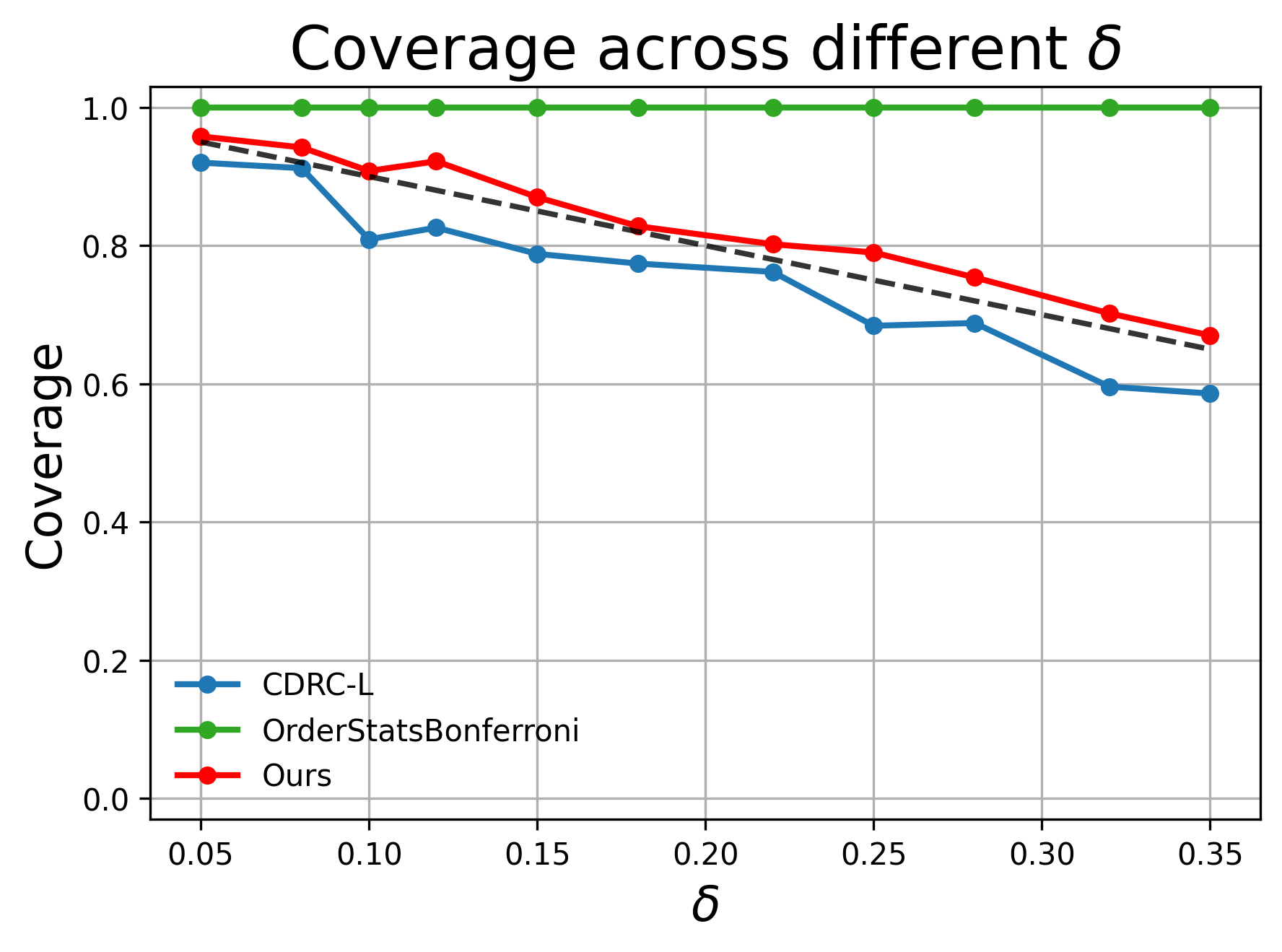}
    \end{minipage}\\
    \begin{minipage}{0.29\linewidth}
        \centering
        \includegraphics[width=\linewidth]{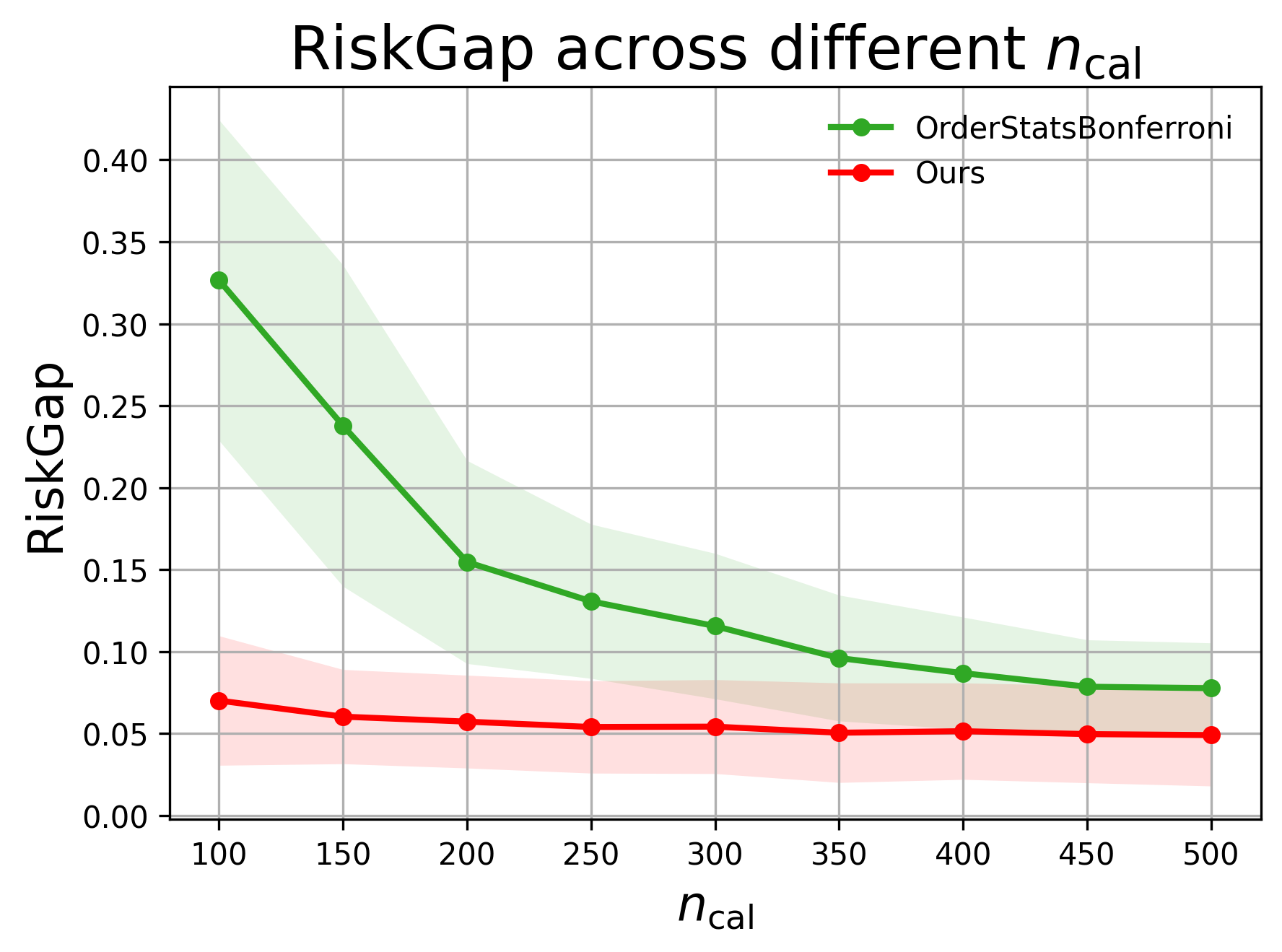}
    \end{minipage}\hfill
    \begin{minipage}{0.29\linewidth}
        \centering
        \includegraphics[width=\linewidth]{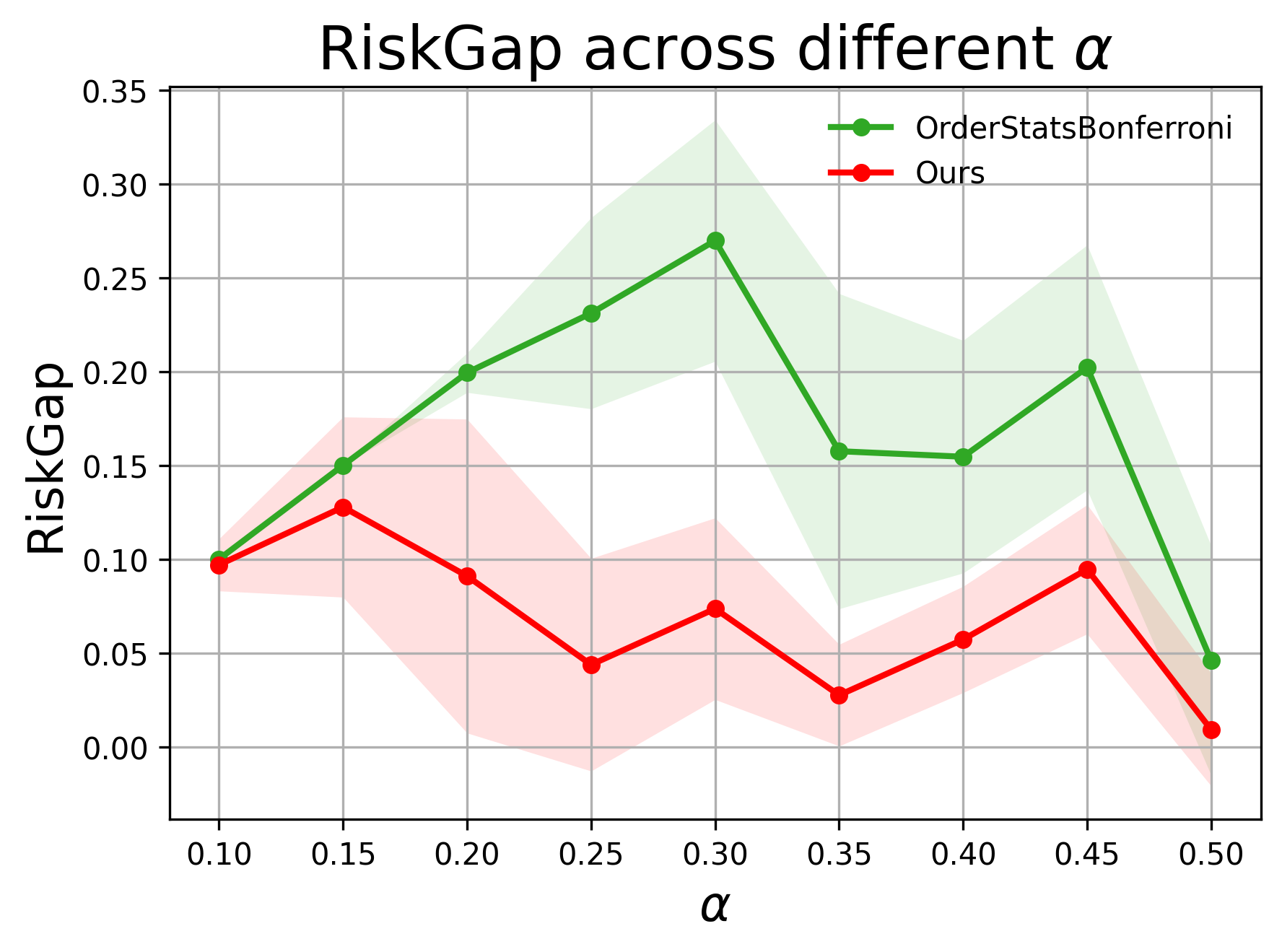}
    \end{minipage}\hfill
    \begin{minipage}{0.29\linewidth}
        \centering
        \includegraphics[width=\linewidth]{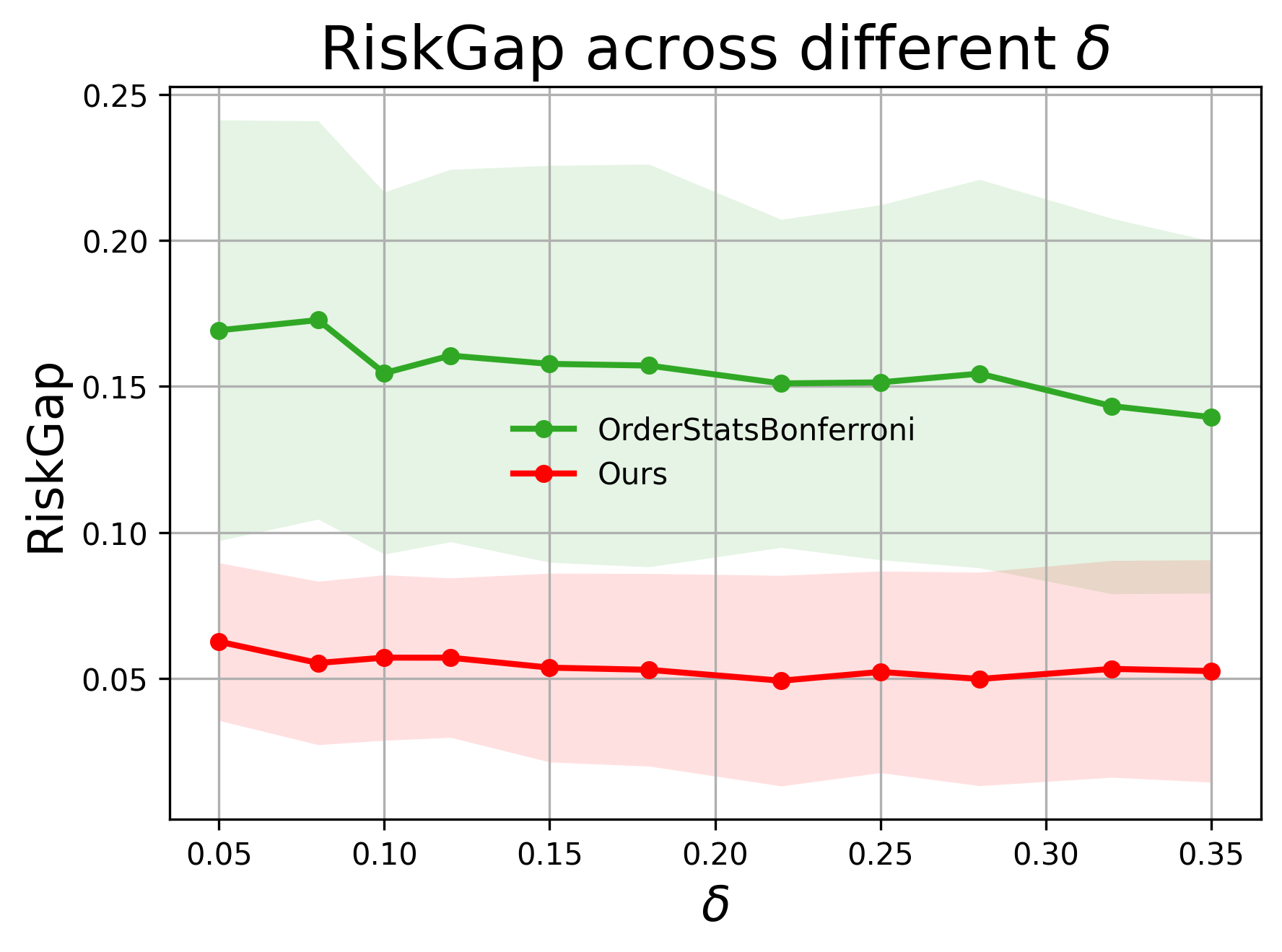}
    \end{minipage}\\
    \begin{minipage}{0.29\linewidth}
        \centering
        \includegraphics[width=\linewidth]{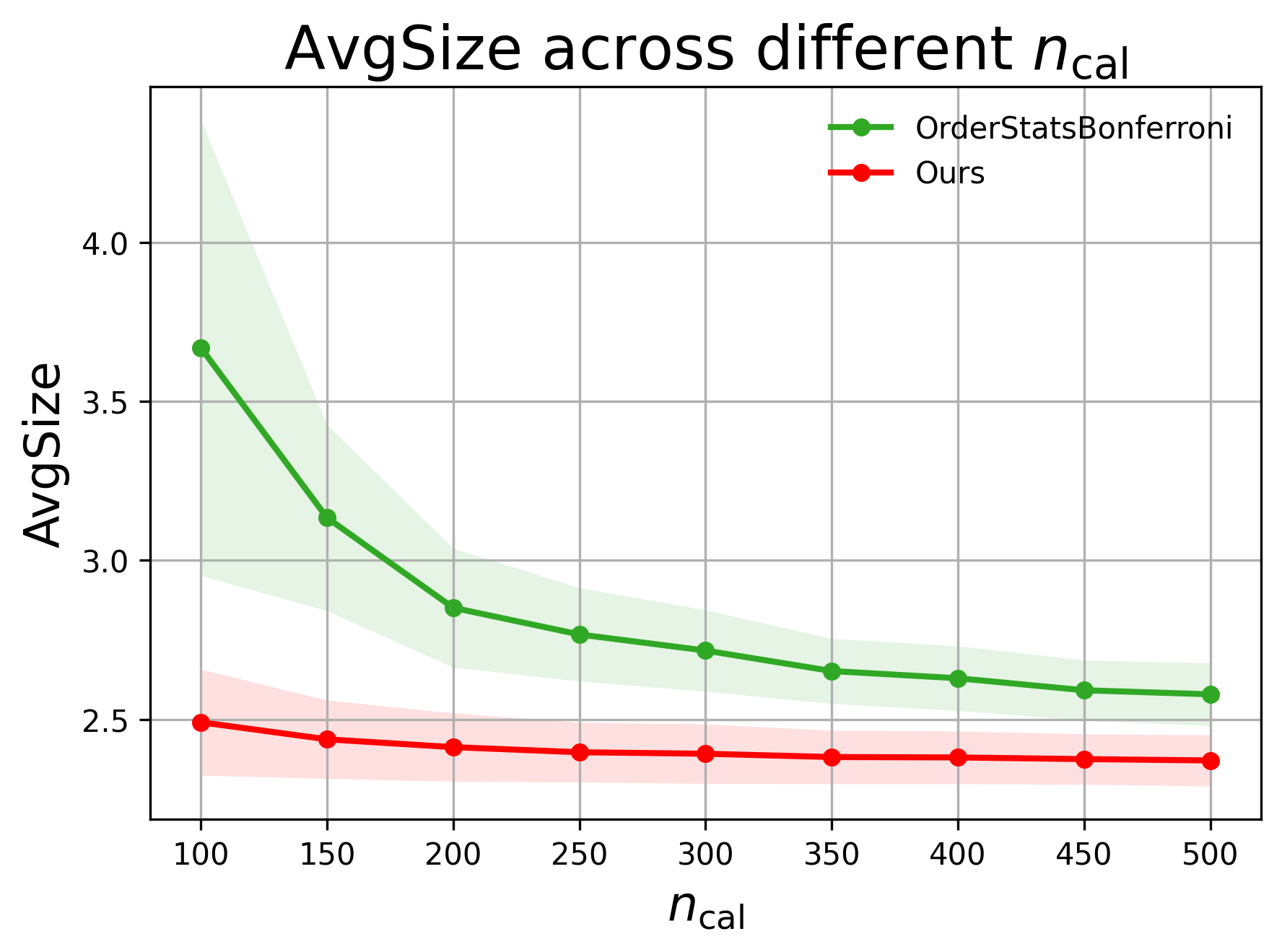}
    \end{minipage}\hfill
    \begin{minipage}{0.29\linewidth}
        \centering
        \includegraphics[width=\linewidth]{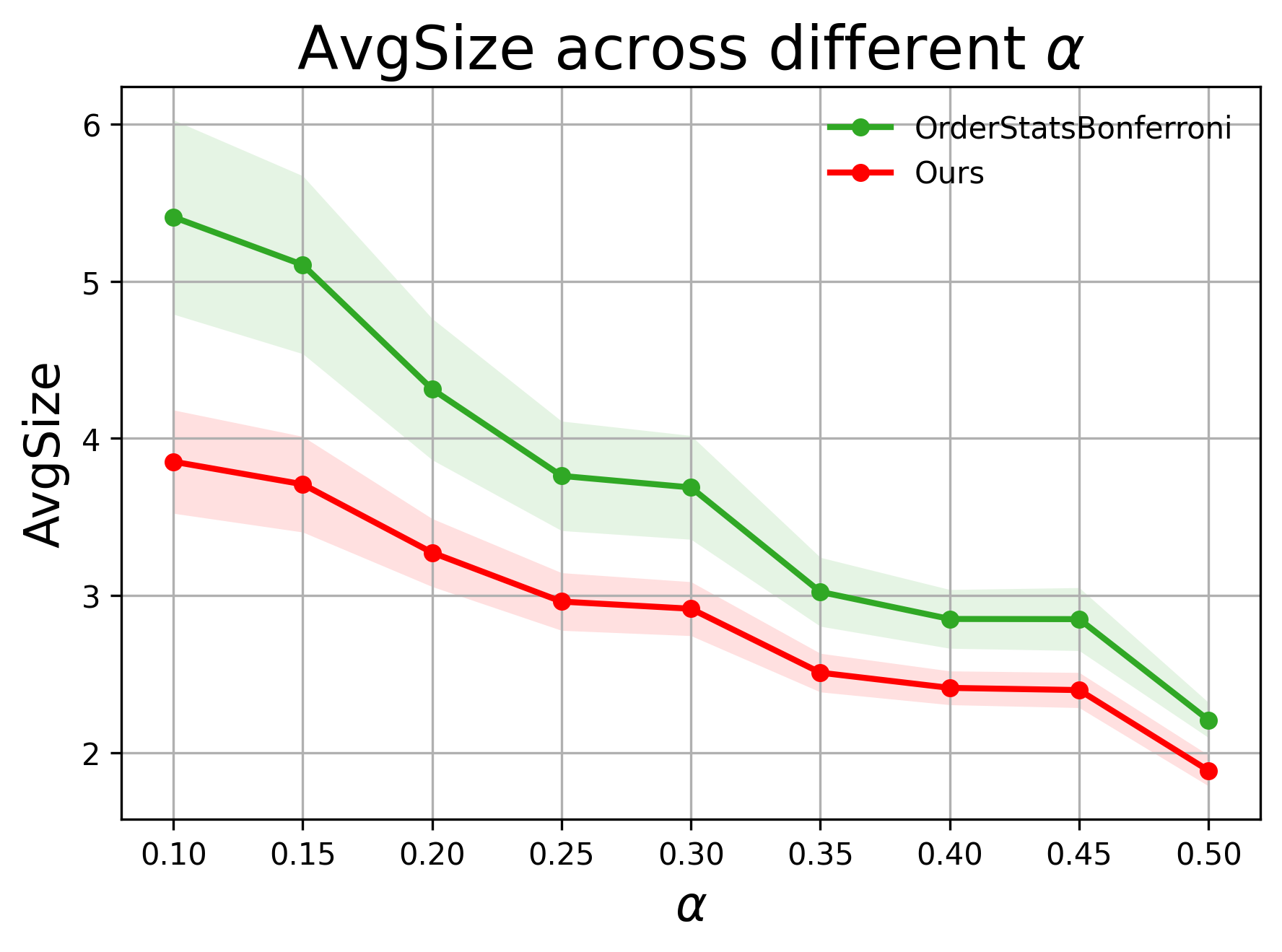}
    \end{minipage}\hfill
    \begin{minipage}{0.29\linewidth}
        \centering
        \includegraphics[width=\linewidth]{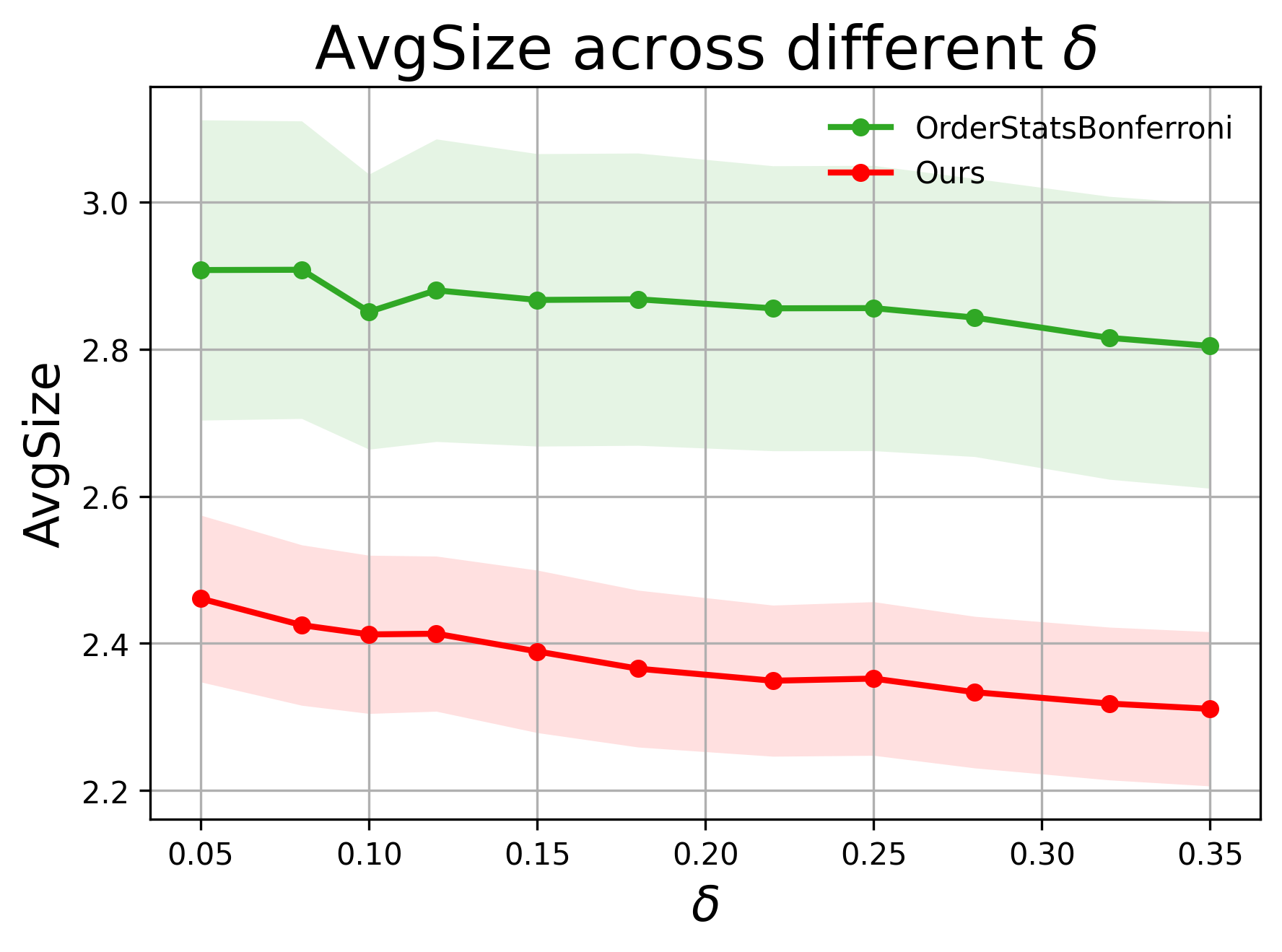}
    \end{minipage}
    \vspace{-3mm}
    \caption{\textbf{Performance comparison across various hyperparameter configurations.} The task is image multi-label classification, and the risk measure is $0.8$-VaR.}
    \label{fig:parameter_analysis_coco_0.801}
    \vspace{-5mm}
\end{figure}

\begin{figure}[htbp]
    \centering
    \begin{minipage}{0.29\linewidth}
        \centering
        \includegraphics[width=\linewidth]{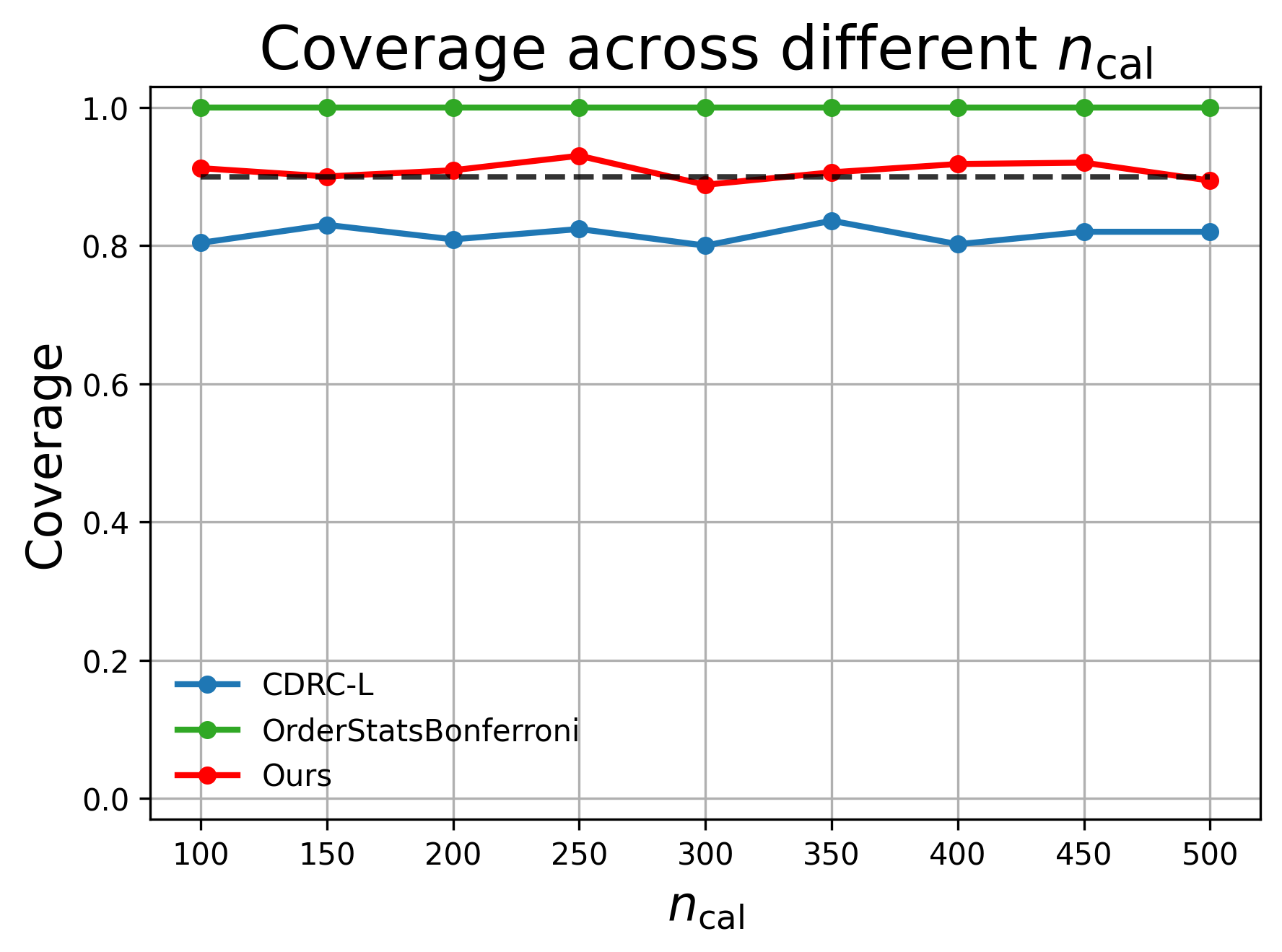}
    \end{minipage}\hfill
    \begin{minipage}{0.29\linewidth}
        \centering
        \includegraphics[width=\linewidth]{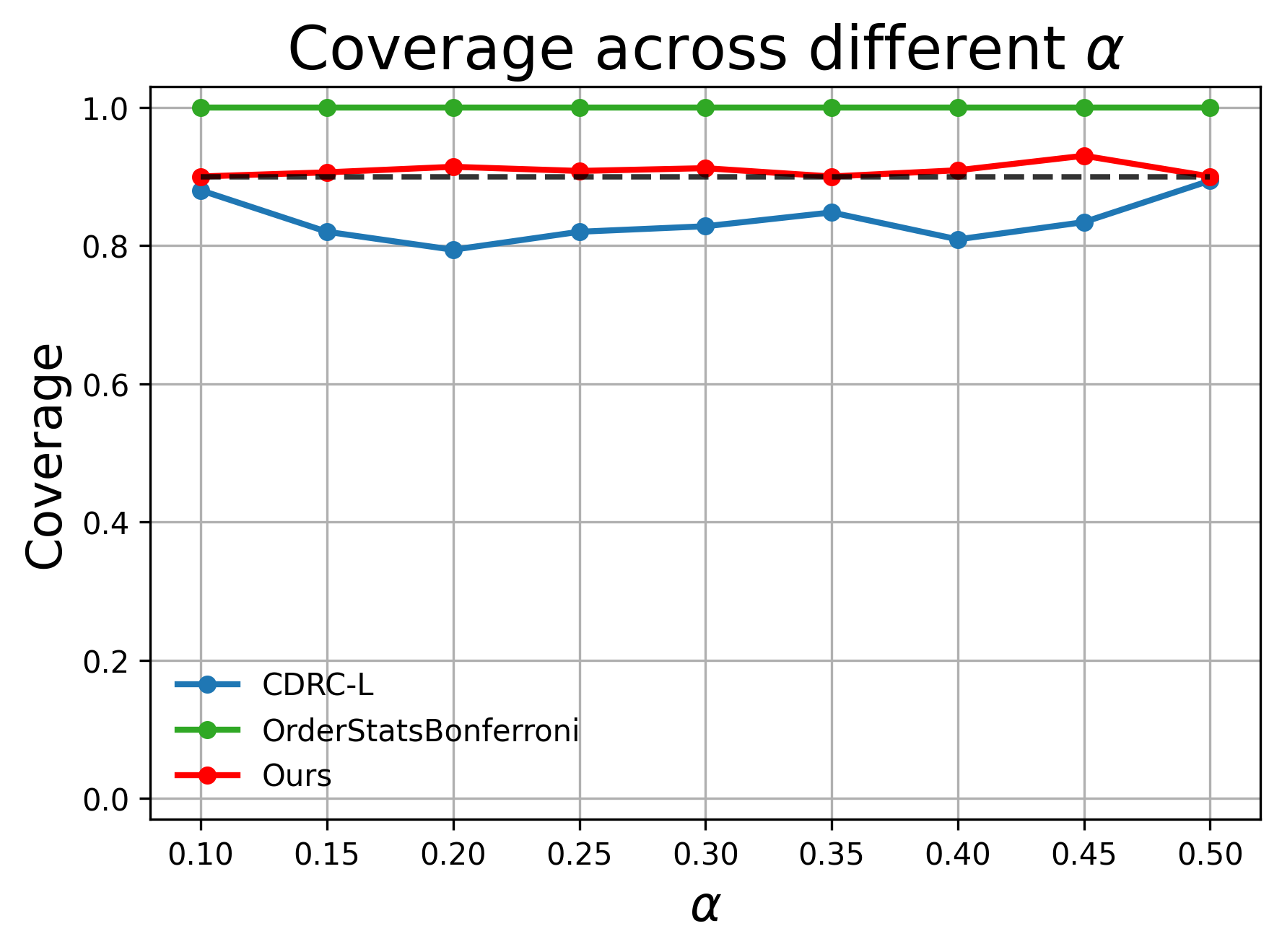}
    \end{minipage}\hfill
    \begin{minipage}{0.29\linewidth}
        \centering
        \includegraphics[width=\linewidth]{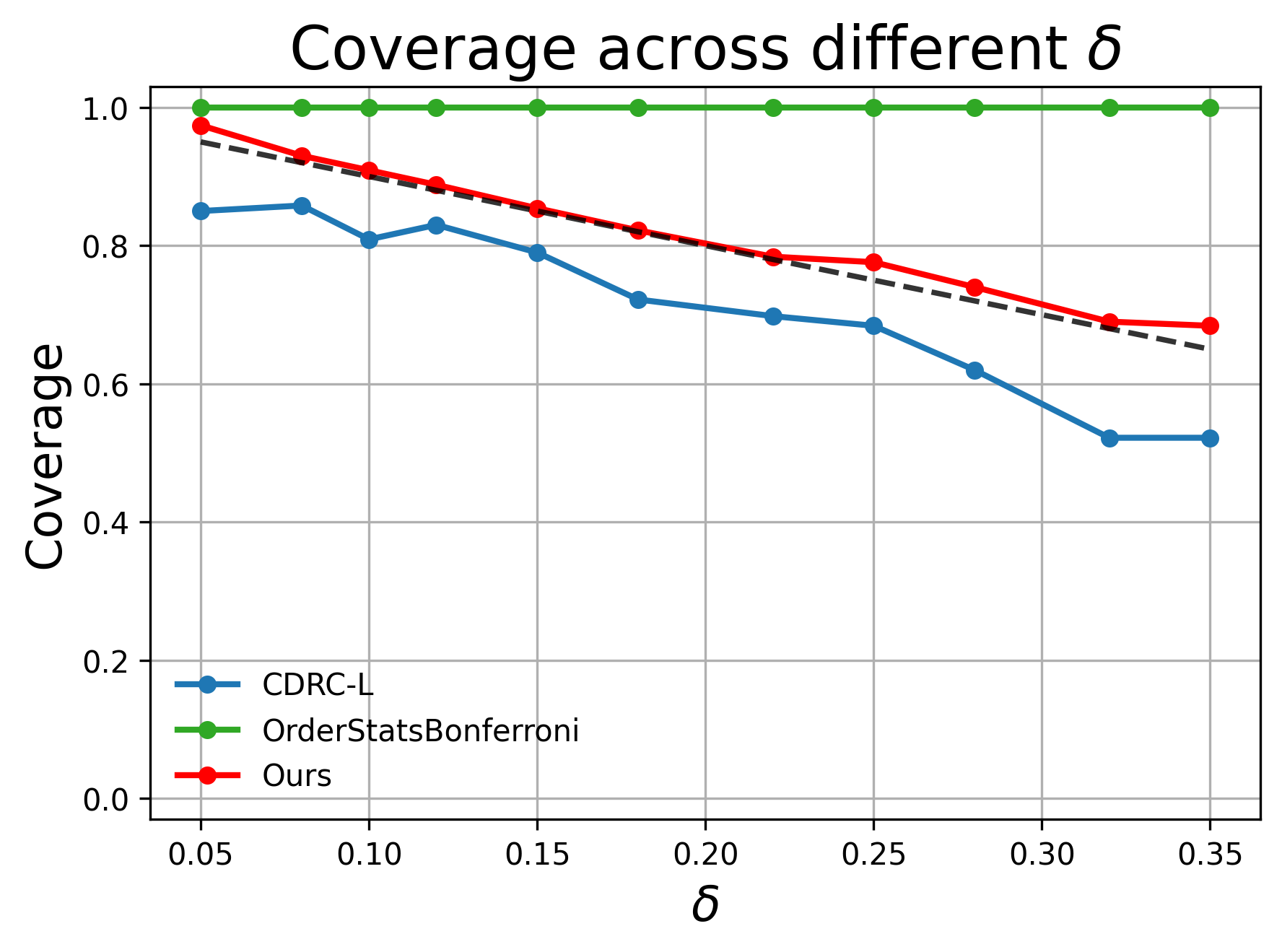}
    \end{minipage}\\
    \begin{minipage}{0.29\linewidth}
        \centering
        \includegraphics[width=\linewidth]{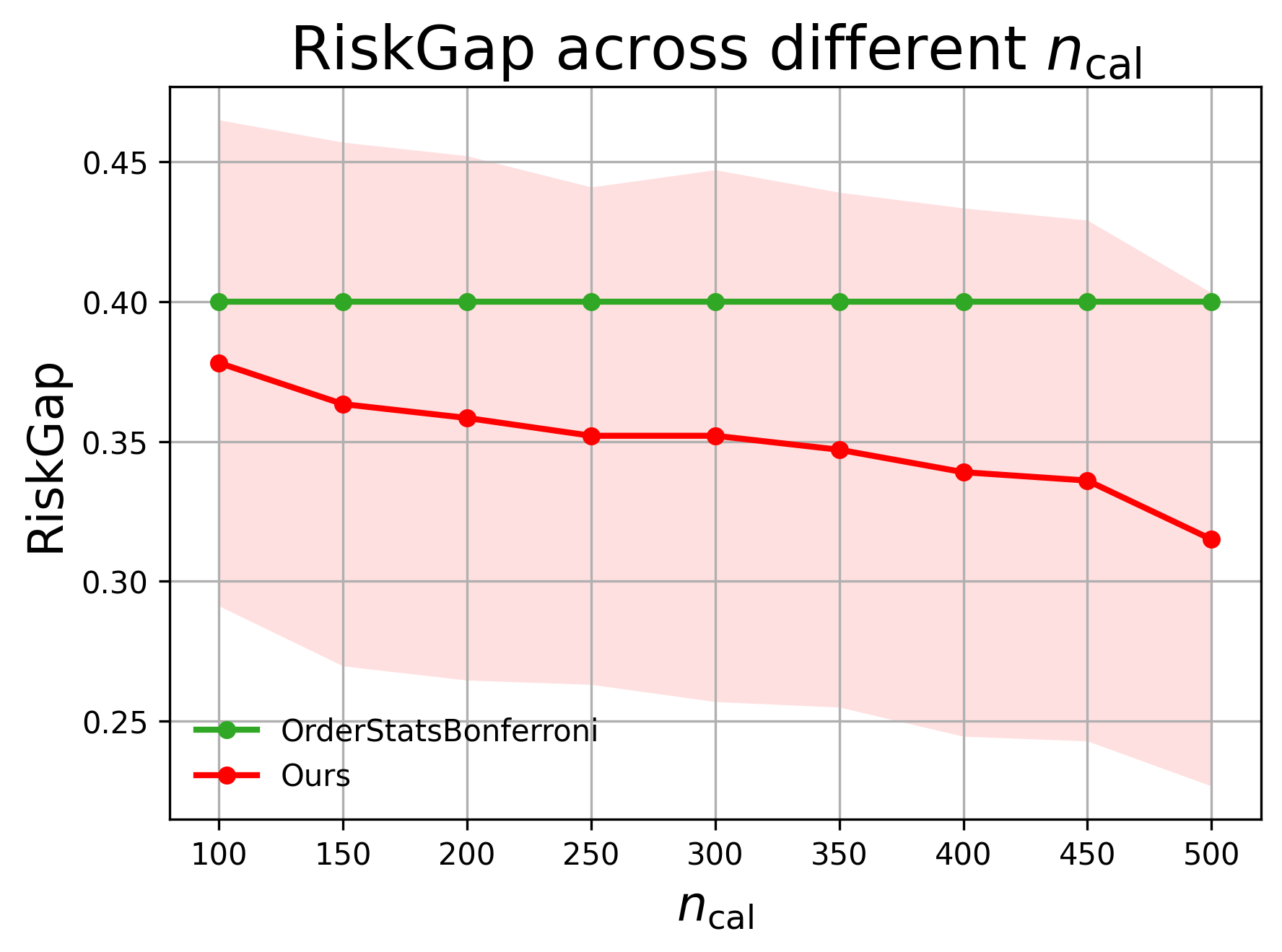}
    \end{minipage}\hfill
    \begin{minipage}{0.29\linewidth}
        \centering
        \includegraphics[width=\linewidth]{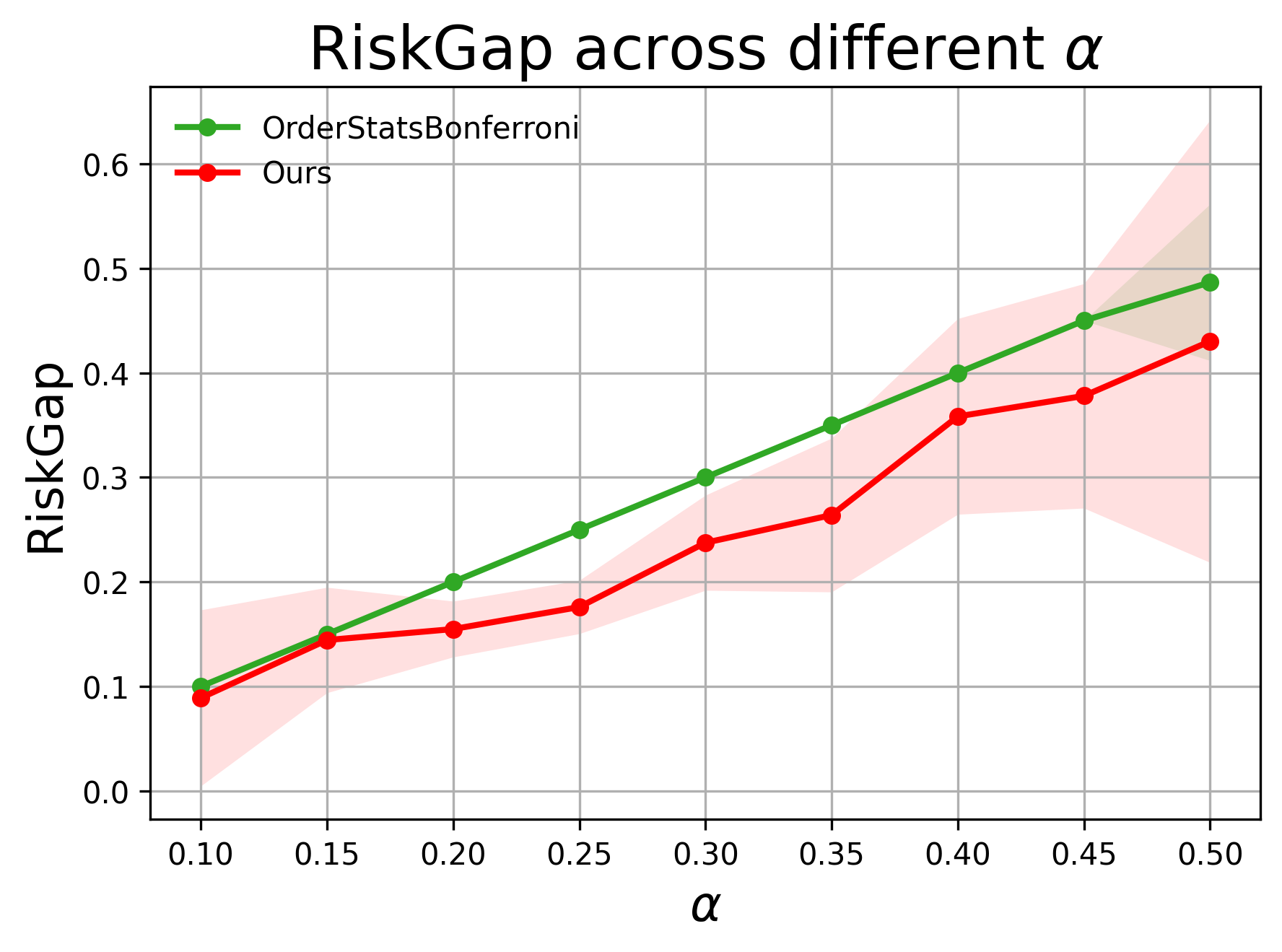}
    \end{minipage}\hfill
    \begin{minipage}{0.29\linewidth}
        \centering
        \includegraphics[width=\linewidth]{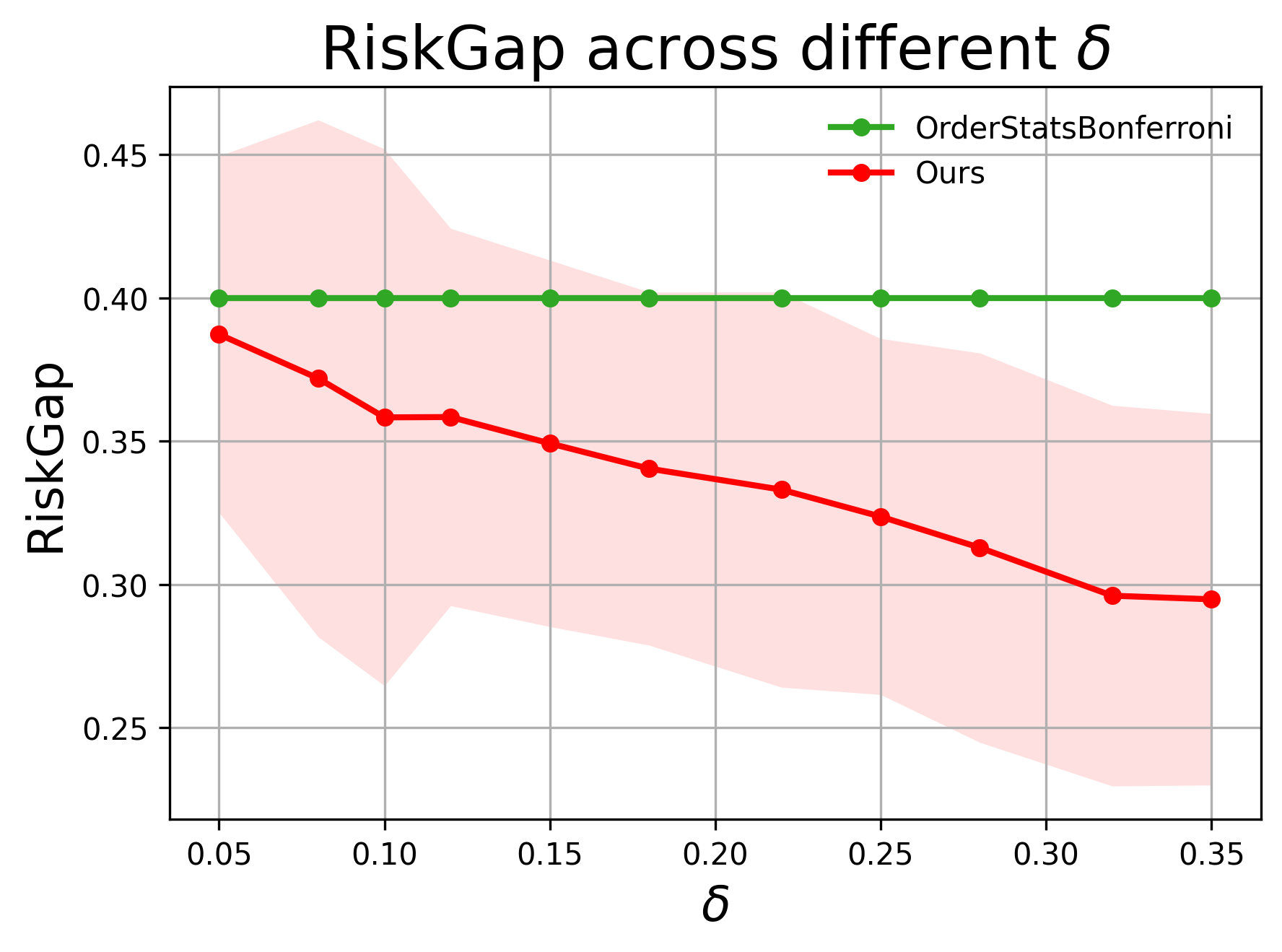}
    \end{minipage}\\
    \begin{minipage}{0.29\linewidth}
        \centering
        \includegraphics[width=\linewidth]{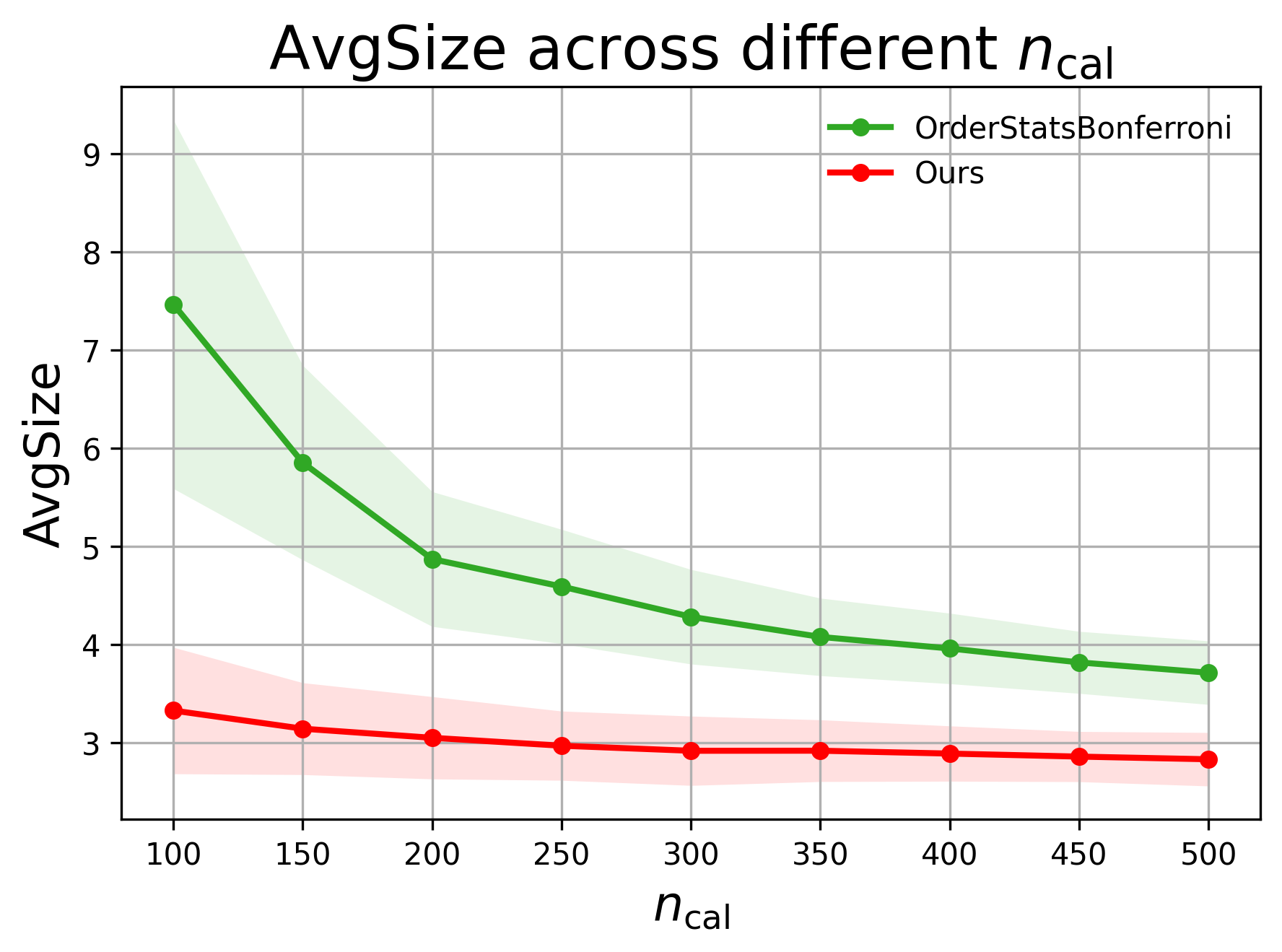}
    \end{minipage}\hfill
    \begin{minipage}{0.29\linewidth}
        \centering
        \includegraphics[width=\linewidth]{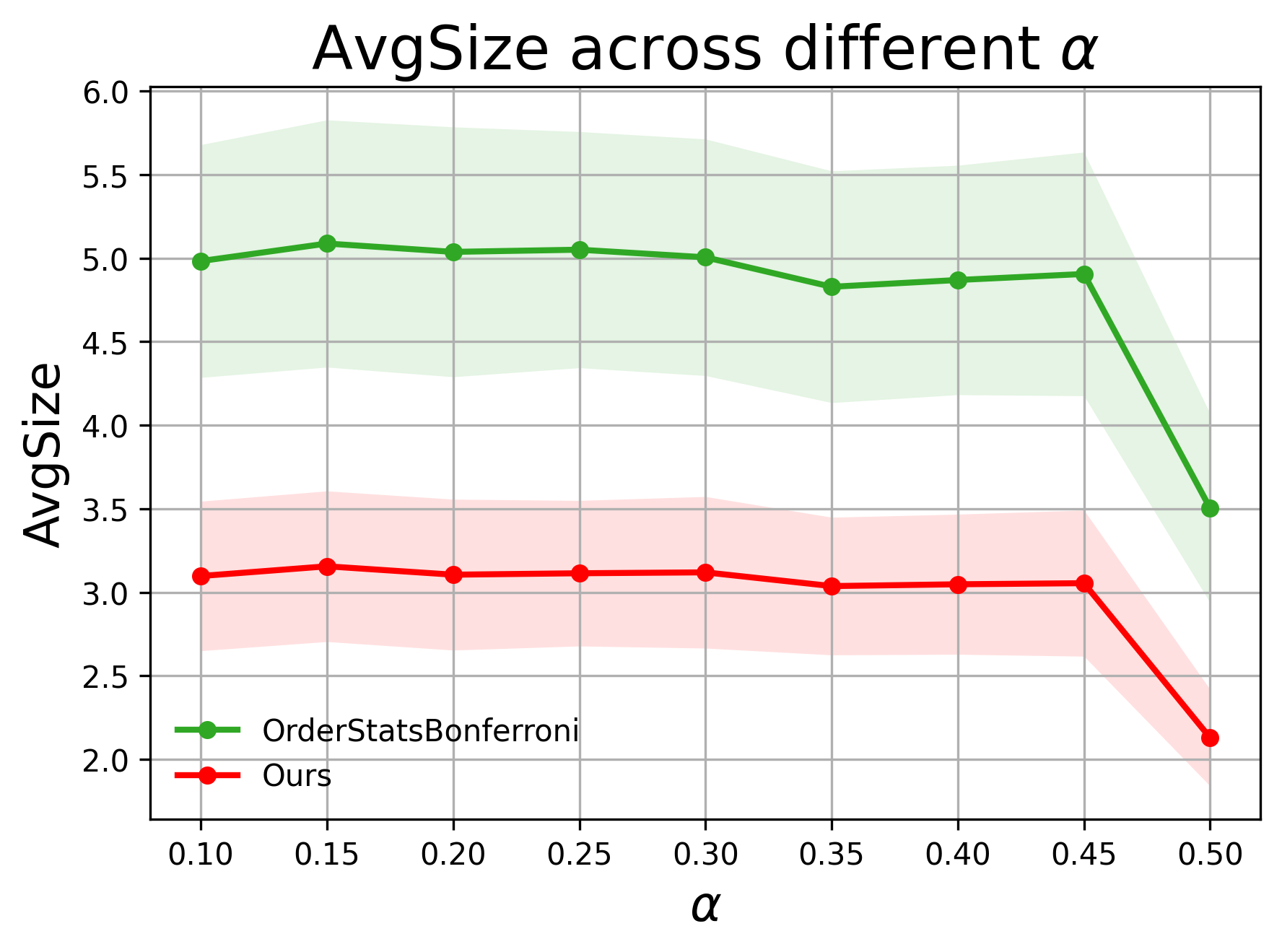}
    \end{minipage}\hfill
    \begin{minipage}{0.29\linewidth}
        \centering
        \includegraphics[width=\linewidth]{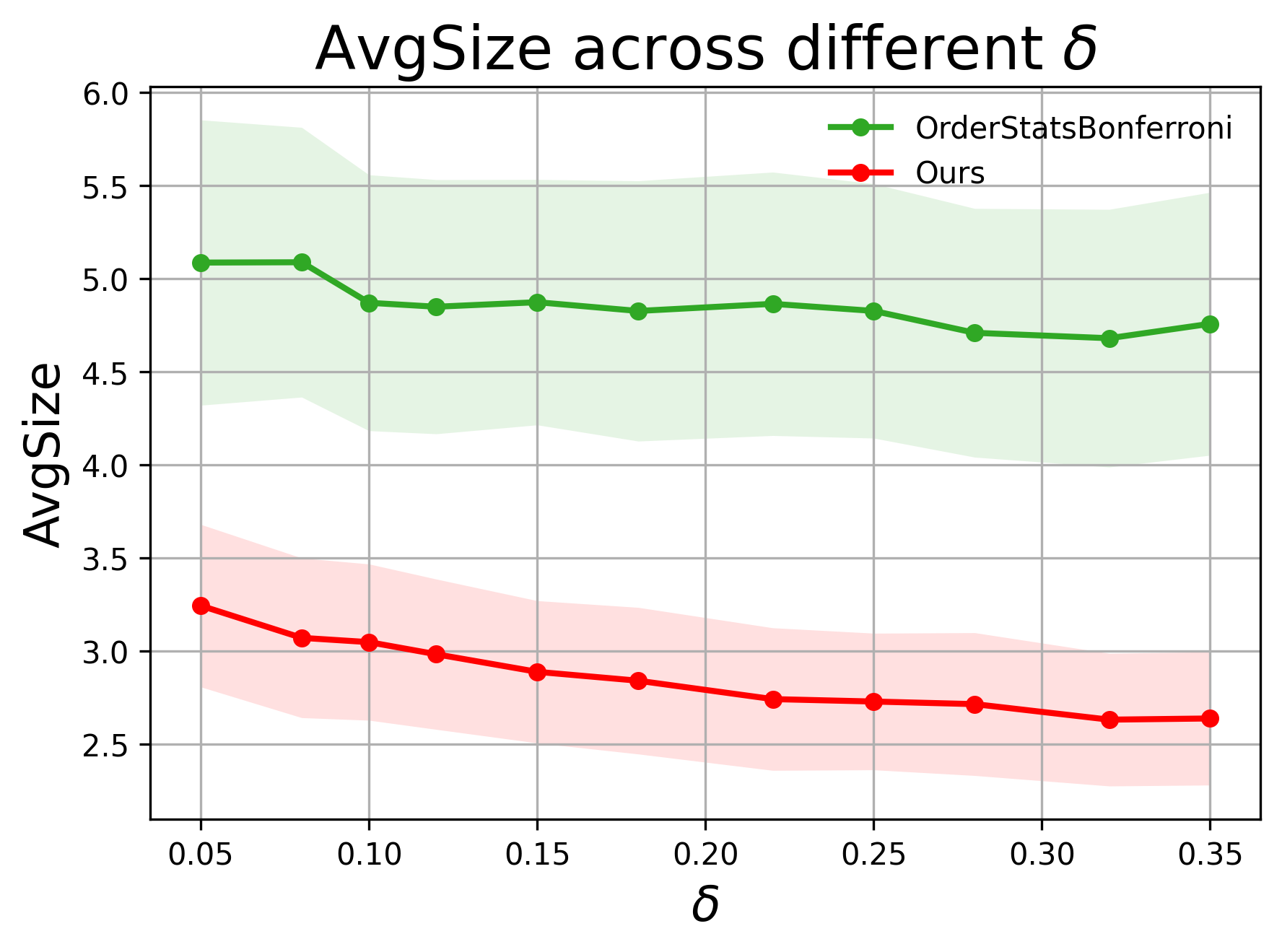}
    \end{minipage}
    \vspace{-3mm}
    \caption{\textbf{Performance comparison across various hyperparameter configurations.} The task is text emotion recognition, and the risk measure is $0.8$-VaR.}
    \label{fig:parameter_analysis_go_0.801}
    \vspace{-5mm}
\end{figure}

\begin{figure}[htbp]
    \centering
    \includegraphics[width=\linewidth]{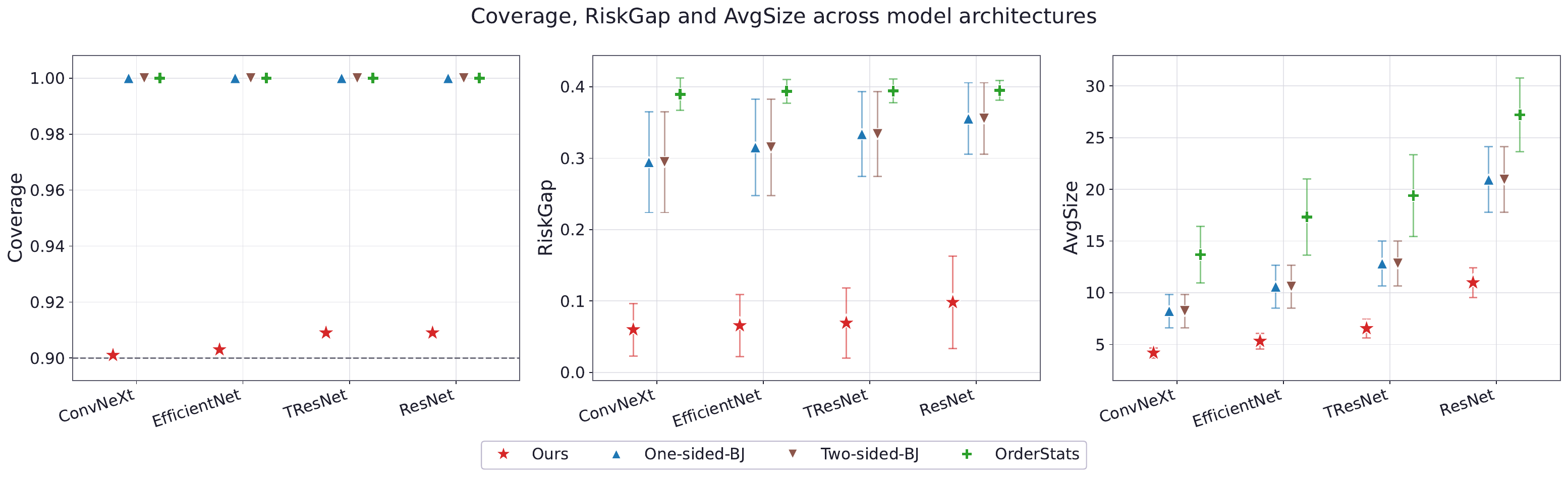}
    \caption{Performance comparison across different model architectures under VaR-Interval.}
    \label{fig:different_model_0.95}
    \vspace{-2mm}
\end{figure}

\begin{figure}[htbp]
    \centering
    \includegraphics[width=\linewidth]{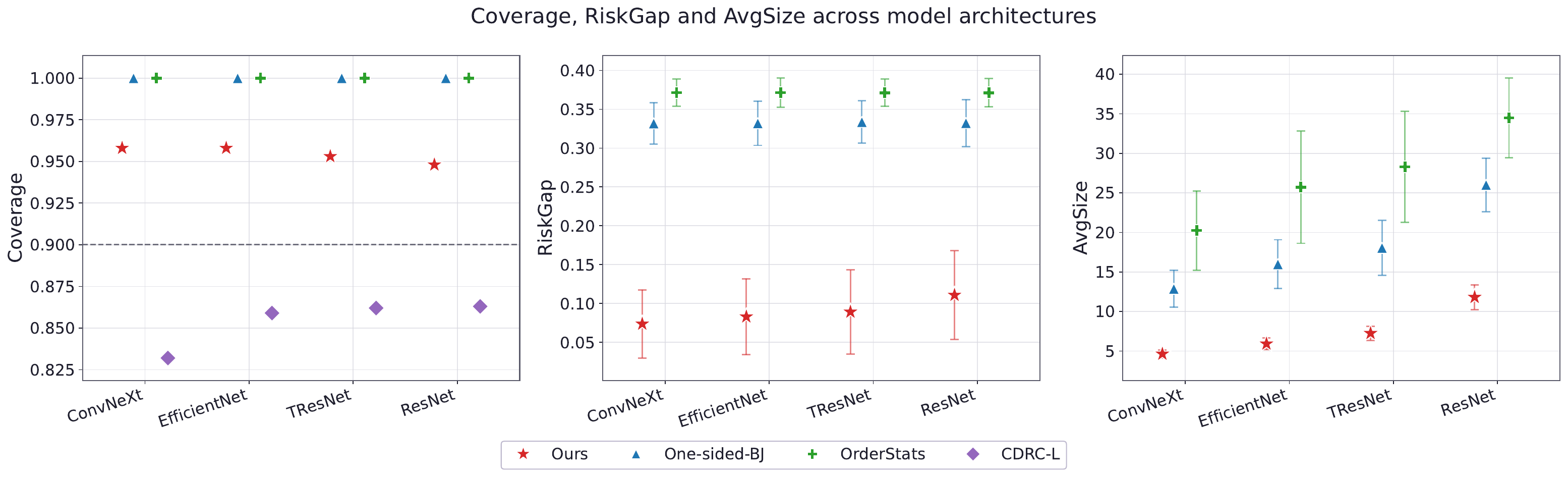}
    \caption{Performance comparison across different model architectures under $0.8$-CVaR.}
    \label{fig:different_model_1}
    \vspace{-2mm}
\end{figure}

\begin{figure}[htbp]
    \centering
    \includegraphics[width=\linewidth]{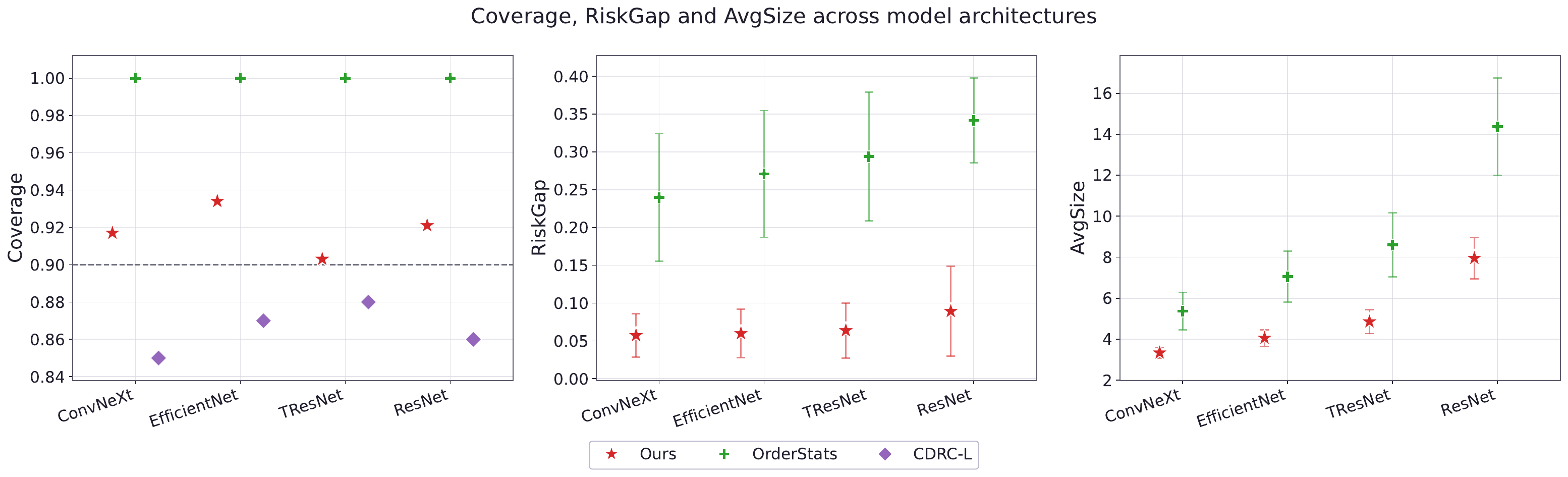}
    \caption{Performance comparison across different model architectures under $0.8$-VaR.}
    \label{fig:different_model_0.9}
    \vspace{-2mm}
\end{figure}

\newpage
\section*{NeurIPS Paper Checklist}

%%% BEGIN INSTRUCTIONS %%%
The checklist is designed to encourage best practices for responsible machine learning research, addressing issues of reproducibility, transparency, research ethics, and societal impact. Do not remove the checklist: {\bf The papers not including the checklist will be desk rejected.} The checklist should follow the references and follow the (optional) supplemental material.  The checklist does NOT count towards the page
limit. 

Please read the checklist guidelines carefully for information on how to answer these questions. For each question in the checklist:
\begin{itemize}
    \item You should answer \answerYes{}, \answerNo{}, or \answerNA{}.
    \item \answerNA{} means either that the question is Not Applicable for that particular paper or the relevant information is Not Available.
    \item Please provide a short (1--2 sentence) justification right after your answer (even for \answerNA). 
   % \item {\bf The papers not including the checklist will be desk rejected.}
\end{itemize}

{\bf The checklist answers are an integral part of your paper submission.} They are visible to the reviewers, area chairs, senior area chairs, and ethics reviewers. You will also be asked to include it (after eventual revisions) with the final version of your paper, and its final version will be published with the paper.

The reviewers of your paper will be asked to use the checklist as one of the factors in their evaluation. While \answerYes{} is generally preferable to \answerNo{}, it is perfectly acceptable to answer \answerNo{} provided a proper justification is given (e.g., error bars are not reported because it would be too computationally expensive'' or ``we were unable to find the license for the dataset we used''). In general, answering \answerNo{} or \answerNA{} is not grounds for rejection. While the questions are phrased in a binary way, we acknowledge that the true answer is often more nuanced, so please just use your best judgment and write a justification to elaborate. All supporting evidence can appear either in the main paper or the supplemental material, provided in appendix. If you answer \answerYes{} to a question, in the justification please point to the section(s) where related material for the question can be found.

IMPORTANT, please:
\begin{itemize}
    \item {\bf Delete this instruction block, but keep the section heading ``NeurIPS Paper Checklist"},
    \item  {\bf Keep the checklist subsection headings, questions/answers and guidelines below.}
    \item {\bf Do not modify the questions and only use the provided macros for your answers}.
\end{itemize}

%%% END INSTRUCTIONS %%%

\begin{enumerate}

\item {\bf Claims}
    \item[] Question: Do the main claims made in the abstract and introduction accurately reflect the paper's contributions and scope?
    \item[] Answer: \answerYes{} % Replace by \answerYes{}, \answerNo{}, or \answerNA{}.
    \item[] Justification: Our theoretical and experimental results consistently support the claims presented in the abstract and introduction. 
    \item[] Guidelines:
    \begin{itemize}
        \item The answer \answerNA{} means that the abstract and introduction do not include the claims made in the paper.
        \item The abstract and/or introduction should clearly state the claims made, including the contributions made in the paper and important assumptions and limitations. A \answerNo{} or \answerNA{} answer to this question will not be perceived well by the reviewers. 
        \item The claims made should match theoretical and experimental results, and reflect how much the results can be expected to generalize to other settings. 
        \item It is fine to include aspirational goals as motivation as long as it is clear that these goals are not attained by the paper. 
    \end{itemize}

\item {\bf Limitations}
    \item[] Question: Does the paper discuss the limitations of the work performed by the authors?
    \item[] Answer: \answerYes{} % Replace by \answerYes{}, \answerNo{}, or \answerNA{}.
    \item[] Justification: We discuss the limitation of this work in the conclusion section~\ref{sec:conclusion}. 
    \item[] Guidelines:
    \begin{itemize}
        \item The answer \answerNA{} means that the paper has no limitation while the answer \answerNo{} means that the paper has limitations, but those are not discussed in the paper. 
        \item The authors are encouraged to create a separate ``Limitations'' section in their paper.
        \item The paper should point out any strong assumptions and how robust the results are to violations of these assumptions (e.g., independence assumptions, noiseless settings, model well-specification, asymptotic approximations only holding locally). The authors should reflect on how these assumptions might be violated in practice and what the implications would be.
        \item The authors should reflect on the scope of the claims made, e.g., if the approach was only tested on a few datasets or with a few runs. In general, empirical results often depend on implicit assumptions, which should be articulated.
        \item The authors should reflect on the factors that influence the performance of the approach. For example, a facial recognition algorithm may perform poorly when image resolution is low or images are taken in low lighting. Or a speech-to-text system might not be used reliably to provide closed captions for online lectures because it fails to handle technical jargon.
        \item The authors should discuss the computational efficiency of the proposed algorithms and how they scale with dataset size.
        \item If applicable, the authors should discuss possible limitations of their approach to address problems of privacy and fairness.
        \item While the authors might fear that complete honesty about limitations might be used by reviewers as grounds for rejection, a worse outcome might be that reviewers discover limitations that aren't acknowledged in the paper. The authors should use their best judgment and recognize that individual actions in favor of transparency play an important role in developing norms that preserve the integrity of the community. Reviewers will be specifically instructed to not penalize honesty concerning limitations.
    \end{itemize}

\item {\bf Theory assumptions and proofs}
    \item[] Question: For each theoretical result, does the paper provide the full set of assumptions and a complete (and correct) proof?
    \item[] Answer: \answerYes{} % Replace by \answerYes{}, \answerNo{}, or \answerNA{}.
    \item[] Justification: We provide the proof of our theoretical result in the Appendix~\ref{sec:proofs}. 
    \item[] Guidelines:
    \begin{itemize}
        \item The answer \answerNA{} means that the paper does not include theoretical results. 
        \item All the theorems, formulas, and proofs in the paper should be numbered and cross-referenced.
        \item All assumptions should be clearly stated or referenced in the statement of any theorems.
        \item The proofs can either appear in the main paper or the supplemental material, but if they appear in the supplemental material, the authors are encouraged to provide a short proof sketch to provide intuition. 
        \item Inversely, any informal proof provided in the core of the paper should be complemented by formal proofs provided in appendix or supplemental material.
        \item Theorems and Lemmas that the proof relies upon should be properly referenced. 
    \end{itemize}

    \item {\bf Experimental result reproducibility}
    \item[] Question: Does the paper fully disclose all the information needed to reproduce the main experimental results of the paper to the extent that it affects the main claims and/or conclusions of the paper (regardless of whether the code and data are provided or not)?
    \item[] Answer: \answerYes{} % Replace by \answerYes{}, \answerNo{}, or \answerNA{}.
    \item[] Justification: We provide the detailed implementation details to reproduce our results.
    \item[] Guidelines:
    \begin{itemize}
        \item The answer \answerNA{} means that the paper does not include experiments.
        \item If the paper includes experiments, a \answerNo{} answer to this question will not be perceived well by the reviewers: Making the paper reproducible is important, regardless of whether the code and data are provided or not.
        \item If the contribution is a dataset and\slash or model, the authors should describe the steps taken to make their results reproducible or verifiable. 
        \item Depending on the contribution, reproducibility can be accomplished in various ways. For example, if the contribution is a novel architecture, describing the architecture fully might suffice, or if the contribution is a specific model and empirical evaluation, it may be necessary to either make it possible for others to replicate the model with the same dataset, or provide access to the model. In general. releasing code and data is often one good way to accomplish this, but reproducibility can also be provided via detailed instructions for how to replicate the results, access to a hosted model (e.g., in the case of a large language model), releasing of a model checkpoint, or other means that are appropriate to the research performed.
        \item While NeurIPS does not require releasing code, the conference does require all submissions to provide some reasonable avenue for reproducibility, which may depend on the nature of the contribution. For example
        \begin{enumerate}
            \item If the contribution is primarily a new algorithm, the paper should make it clear how to reproduce that algorithm.
            \item If the contribution is primarily a new model architecture, the paper should describe the architecture clearly and fully.
            \item If the contribution is a new model (e.g., a large language model), then there should either be a way to access this model for reproducing the results or a way to reproduce the model (e.g., with an open-source dataset or instructions for how to construct the dataset).
            \item We recognize that reproducibility may be tricky in some cases, in which case authors are welcome to describe the particular way they provide for reproducibility. In the case of closed-source models, it may be that access to the model is limited in some way (e.g., to registered users), but it should be possible for other researchers to have some path to reproducing or verifying the results.
        \end{enumerate}
    \end{itemize}

\item {\bf Open access to data and code}
    \item[] Question: Does the paper provide open access to the data and code, with sufficient instructions to faithfully reproduce the main experimental results, as described in supplemental material?
    \item[] Answer: \answerYes{} % Replace by \answerYes{}, \answerNo{}, or \answerNA{}.
    \item[] Justification: We will release the code of our work. 
    \item[] Guidelines:
    \begin{itemize}
        \item The answer \answerNA{} means that paper does not include experiments requiring code.
        \item Please see the NeurIPS code and data submission guidelines (\url{https://neurips.cc/public/guides/CodeSubmissionPolicy}) for more details.
        \item While we encourage the release of code and data, we understand that this might not be possible, so \answerNo{} is an acceptable answer. Papers cannot be rejected simply for not including code, unless this is central to the contribution (e.g., for a new open-source benchmark).
        \item The instructions should contain the exact command and environment needed to run to reproduce the results. See the NeurIPS code and data submission guidelines (\url{https://neurips.cc/public/guides/CodeSubmissionPolicy}) for more details.
        \item The authors should provide instructions on data access and preparation, including how to access the raw data, preprocessed data, intermediate data, and generated data, etc.
        \item The authors should provide scripts to reproduce all experimental results for the new proposed method and baselines. If only a subset of experiments are reproducible, they should state which ones are omitted from the script and why.
        \item At submission time, to preserve anonymity, the authors should release anonymized versions (if applicable).
        \item Providing as much information as possible in supplemental material (appended to the paper) is recommended, but including URLs to data and code is permitted.
    \end{itemize}

\item {\bf Experimental setting/details}
    \item[] Question: Does the paper specify all the training and test details (e.g., data splits, hyperparameters, how they were chosen, type of optimizer) necessary to understand the results?
    \item[] Answer: \answerYes{} % Replace by \answerYes{}, \answerNo{}, or \answerNA{}.
    \item[] Justification: We provide the information in Appendix~\ref{sec:details}.
    \item[] Guidelines:
    \begin{itemize}
        \item The answer \answerNA{} means that the paper does not include experiments.
        \item The experimental setting should be presented in the core of the paper to a level of detail that is necessary to appreciate the results and make sense of them.
        \item The full details can be provided either with the code, in appendix, or as supplemental material.
    \end{itemize}

\item {\bf Experiment statistical significance}
    \item[] Question: Does the paper report error bars suitably and correctly defined or other appropriate information about the statistical significance of the experiments?
    \item[] Answer: \answerYes{} % Replace by \answerYes{}, \answerNo{}, or \answerNA{}.
    \item[] Justification: Our results are provided in mean and standard error format.
    \item[] Guidelines:
    \begin{itemize}
        \item The answer \answerNA{} means that the paper does not include experiments.
        \item The authors should answer \answerYes{} if the results are accompanied by error bars, confidence intervals, or statistical significance tests, at least for the experiments that support the main claims of the paper.
        \item The factors of variability that the error bars are capturing should be clearly stated (for example, train/test split, initialization, random drawing of some parameter, or overall run with given experimental conditions).
        \item The method for calculating the error bars should be explained (closed form formula, call to a library function, bootstrap, etc.)
        \item The assumptions made should be given (e.g., Normally distributed errors).
        \item It should be clear whether the error bar is the standard deviation or the standard error of the mean.
        \item It is OK to report 1-sigma error bars, but one should state it. The authors should preferably report a 2-sigma error bar than state that they have a 96\% CI, if the hypothesis of Normality of errors is not verified.
        \item For asymmetric distributions, the authors should be careful not to show in tables or figures symmetric error bars that would yield results that are out of range (e.g., negative error rates).
        \item If error bars are reported in tables or plots, the authors should explain in the text how they were calculated and reference the corresponding figures or tables in the text.
    \end{itemize}

\item {\bf Experiments compute resources}
    \item[] Question: For each experiment, does the paper provide sufficient information on the computer resources (type of compute workers, memory, time of execution) needed to reproduce the experiments?
    \item[] Answer: \answerYes{} % Replace by \answerYes{}, \answerNo{}, or \answerNA{}.
    \item[] Justification: The compute resources are provided in Appendix~\ref{sec:details}.
    \item[] Guidelines:
    \begin{itemize}
        \item The answer \answerNA{} means that the paper does not include experiments.
        \item The paper should indicate the type of compute workers CPU or GPU, internal cluster, or cloud provider, including relevant memory and storage.
        \item The paper should provide the amount of compute required for each of the individual experimental runs as well as estimate the total compute. 
        \item The paper should disclose whether the full research project required more compute than the experiments reported in the paper (e.g., preliminary or failed experiments that didn't make it into the paper). 
    \end{itemize}
    
\item {\bf Code of ethics}
    \item[] Question: Does the research conducted in the paper conform, in every respect, with the NeurIPS Code of Ethics \url{https://neurips.cc/public/EthicsGuidelines}?
    \item[] Answer: \answerYes{} % Replace by \answerYes{}, \answerNo{}, or \answerNA{}.
    \item[] Justification: Yes, we affirm that our research adheres to the NeurIPS Code of Ethics in all aspects, including considerations related to data usage, transparency, and potential societal impact. 
    \item[] Guidelines:
    \begin{itemize}
        \item The answer \answerNA{} means that the authors have not reviewed the NeurIPS Code of Ethics.
        \item If the authors answer \answerNo, they should explain the special circumstances that require a deviation from the Code of Ethics.
        \item The authors should make sure to preserve anonymity (e.g., if there is a special consideration due to laws or regulations in their jurisdiction).
    \end{itemize}

\item {\bf Broader impacts}
    \item[] Question: Does the paper discuss both potential positive societal impacts and negative societal impacts of the work performed?
    \item[] Answer: \answerYes{} % Replace by \answerYes{}, \answerNo{}, or \answerNA{}.
    \item[] Justification: We discuss about the societal impacts of risk control in the introduction. 
    \item[] Guidelines:
    \begin{itemize}
        \item The answer \answerNA{} means that there is no societal impact of the work performed.
        \item If the authors answer \answerNA{} or \answerNo, they should explain why their work has no societal impact or why the paper does not address societal impact.
        \item Examples of negative societal impacts include potential malicious or unintended uses (e.g., disinformation, generating fake profiles, surveillance), fairness considerations (e.g., deployment of technologies that could make decisions that unfairly impact specific groups), privacy considerations, and security considerations.
        \item The conference expects that many papers will be foundational research and not tied to particular applications, let alone deployments. However, if there is a direct path to any negative applications, the authors should point it out. For example, it is legitimate to point out that an improvement in the quality of generative models could be used to generate Deepfakes for disinformation. On the other hand, it is not needed to point out that a generic algorithm for optimizing neural networks could enable people to train models that generate Deepfakes faster.
        \item The authors should consider possible harms that could arise when the technology is being used as intended and functioning correctly, harms that could arise when the technology is being used as intended but gives incorrect results, and harms following from (intentional or unintentional) misuse of the technology.
        \item If there are negative societal impacts, the authors could also discuss possible mitigation strategies (e.g., gated release of models, providing defenses in addition to attacks, mechanisms for monitoring misuse, mechanisms to monitor how a system learns from feedback over time, improving the efficiency and accessibility of ML).
    \end{itemize}
    
\item {\bf Safeguards}
    \item[] Question: Does the paper describe safeguards that have been put in place for responsible release of data or models that have a high risk for misuse (e.g., pre-trained language models, image generators, or scraped datasets)?
    \item[] Answer: \answerNA{} % Replace by \answerYes{}, \answerNo{}, or \answerNA{}.
    \item[] Justification: Our work does not involve the release of model and dataset.
    \item[] Guidelines:
    \begin{itemize}
        \item The answer \answerNA{} means that the paper poses no such risks.
        \item Released models that have a high risk for misuse or dual-use should be released with necessary safeguards to allow for controlled use of the model, for example by requiring that users adhere to usage guidelines or restrictions to access the model or implementing safety filters. 
        \item Datasets that have been scraped from the Internet could pose safety risks. The authors should describe how they avoided releasing unsafe images.
        \item We recognize that providing effective safeguards is challenging, and many papers do not require this, but we encourage authors to take this into account and make a best faith effort.
    \end{itemize}

\item {\bf Licenses for existing assets}
    \item[] Question: Are the creators or original owners of assets (e.g., code, data, models), used in the paper, properly credited and are the license and terms of use explicitly mentioned and properly respected?
    \item[] Answer: \answerYes{} % Replace by \answerYes{}, \answerNo{}, or \answerNA{}.
    \item[] Justification: The dataset and model we use are both open-sourced. Our code is based on PyTorch. 
    \item[] Guidelines:
    \begin{itemize}
        \item The answer \answerNA{} means that the paper does not use existing assets.
        \item The authors should cite the original paper that produced the code package or dataset.
        \item The authors should state which version of the asset is used and, if possible, include a URL.
        \item The name of the license (e.g., CC-BY 4.0) should be included for each asset.
        \item For scraped data from a particular source (e.g., website), the copyright and terms of service of that source should be provided.
        \item If assets are released, the license, copyright information, and terms of use in the package should be provided. For popular datasets, \url{paperswithcode.com/datasets} has curated licenses for some datasets. Their licensing guide can help determine the license of a dataset.
        \item For existing datasets that are re-packaged, both the original license and the license of the derived asset (if it has changed) should be provided.
        \item If this information is not available online, the authors are encouraged to reach out to the asset's creators.
    \end{itemize}

\item {\bf New assets}
    \item[] Question: Are new assets introduced in the paper well documented and is the documentation provided alongside the assets?
    \item[] Answer: \answerYes{} % Replace by \answerYes{}, \answerNo{}, or \answerNA{}.
    \item[] Justification: We will upload our code in the supplement material.
    \item[] Guidelines:
    \begin{itemize}
        \item The answer \answerNA{} means that the paper does not release new assets.
        \item Researchers should communicate the details of the dataset\slash code\slash model as part of their submissions via structured templates. This includes details about training, license, limitations, etc. 
        \item The paper should discuss whether and how consent was obtained from people whose asset is used.
        \item At submission time, remember to anonymize your assets (if applicable). You can either create an anonymized URL or include an anonymized zip file.
    \end{itemize}

\item {\bf Crowdsourcing and research with human subjects}
    \item[] Question: For crowdsourcing experiments and research with human subjects, does the paper include the full text of instructions given to participants and screenshots, if applicable, as well as details about compensation (if any)? 
    \item[] Answer: \answerNA{} % Replace by \answerYes{}, \answerNo{}, or \answerNA{}.
    \item[] Justification: Our work does not involve crowdsourcing nor research with human subjects. 
    \item[] Guidelines:
    \begin{itemize}
        \item The answer \answerNA{} means that the paper does not involve crowdsourcing nor research with human subjects.
        \item Including this information in the supplemental material is fine, but if the main contribution of the paper involves human subjects, then as much detail as possible should be included in the main paper. 
        \item According to the NeurIPS Code of Ethics, workers involved in data collection, curation, or other labor should be paid at least the minimum wage in the country of the data collector. 
    \end{itemize}

\item {\bf Institutional review board (IRB) approvals or equivalent for research with human subjects}
    \item[] Question: Does the paper describe potential risks incurred by study participants, whether such risks were disclosed to the subjects, and whether Institutional Review Board (IRB) approvals (or an equivalent approval/review based on the requirements of your country or institution) were obtained?
    \item[] Answer: \answerNA{} % Replace by \answerYes{}, \answerNo{}, or \answerNA{}.
    \item[] Justification: Our work does not involve crowdsourcing nor research with human subjects.
    \item[] Guidelines:
    \begin{itemize}
        \item The answer \answerNA{} means that the paper does not involve crowdsourcing nor research with human subjects.
        \item Depending on the country in which research is conducted, IRB approval (or equivalent) may be required for any human subjects research. If you obtained IRB approval, you should clearly state this in the paper. 
        \item We recognize that the procedures for this may vary significantly between institutions and locations, and we expect authors to adhere to the NeurIPS Code of Ethics and the guidelines for their institution. 
        \item For initial submissions, do not include any information that would break anonymity (if applicable), such as the institution conducting the review.
    \end{itemize}

\item {\bf Declaration of LLM usage}
    \item[] Question: Does the paper describe the usage of LLMs if it is an important, original, or non-standard component of the core methods in this research? Note that if the LLM is used only for writing, editing, or formatting purposes and does \emph{not} impact the core methodology, scientific rigor, or originality of the research, declaration is not required.
    %this research? 
    \item[] Answer: \answerNA{} % Replace by \answerYes{}, \answerNo{}, or \answerNA{}.
    \item[] Justification: The core method development in this research does not involve LLMs as any important, original, or non-standard components. 
    \item[] Guidelines:
    \begin{itemize}
        \item The answer \answerNA{} means that the core method development in this research does not involve LLMs as any important, original, or non-standard components.
        \item Please refer to our LLM policy in the NeurIPS handbook for what should or should not be described.
    \end{itemize}

\end{enumerate}

\end{document}